\documentclass[letterpaper, 10 pt, conference]{ieeeconf}  

\IEEEoverridecommandlockouts                              
\makeatletter
\let\NAT@parse\undefined
\makeatother
 
\usepackage{microtype}
\usepackage{graphicx}
\usepackage{subfigure}
\usepackage{booktabs} 
\usepackage{multirow} 
\usepackage{rotating} 
\usepackage{graphicx}
\usepackage{tikz}
\usepackage{hyperref}       
\usepackage{url}            
\usepackage{booktabs}       
\usepackage{nicefrac}       
\usepackage{xcolor}         
\usepackage{algorithm,algorithmic}
\usepackage{mathrsfs}

\usepackage{amsthm}
\usepackage{amsmath}
\usepackage{amssymb}
\usepackage{mathtools}
\usepackage{amsfonts}       

\usepackage{color, colortbl}
\usepackage{hhline} 

\usepackage{pifont}
\usepackage[most]{tcolorbox}

\newcommand{\cmark}{\ding{51}}%
\newcommand{\xmark}{\ding{55}}%

\DeclareMathOperator*{\argmin}{arg\,min}
\DeclareMathOperator*{\argmax}{arg\,max}
\DeclareMathOperator*{\prox}{prox}

\newcommand{\red}[1]{\textcolor{black}{#1}}

\def\bx{{\mathbf{x}}}
\def\by{{\mathbf{y}}}
\def\bz{{\mathbf{z}}}
\def\bu{{\mathbf{u}}}

\def\L{{\mathcal{L}}}

\def\ssx{{\textsf{x}}}
\def\ssy{{\textsf{y}}}
\def\ssw{{\textsf{w}}}
\def\ssD{{\textsf{D}}}
\def\ssH{{\textsf{H}}}

\newtheorem{theorem}{Theorem}
\newtheorem{proposition}[theorem]{Proposition}
\newtheorem{lemma}{Lemma}
\newtheorem{corollary}{Corollary}
\theoremstyle{definition}

\theoremstyle{remark}

\newtheoremstyle{assumptionstyle}
{0em}
{0em}
{\slshape}
{}
{\bfseries}
{.}
{0.5em}
{}

\theoremstyle{assumptionstyle}

\newtheorem{myassumption}{Assumption}

\allowdisplaybreaks
\usepackage{hyperref}
\hypersetup{
	colorlinks   = true, 
	urlcolor     = blue, 
	linkcolor    = blue, 
	citecolor   = red 
}

\title{\LARGE \bf Switch and Conquer: Efficient Algorithms By Switching Stochastic Gradient Oracles For Decentralized Saddle Point Problems 
}

\author{Chhavi Sharma, Vishnu Narayanan and P. Balamurugan 
	\thanks{Chhavi Sharma, Vishnu Narayanan and P. Balamurugan are with  Industrial Engineering and Operations Research (IEOR), IIT Bombay, Mumbai, India-400076.
		{\tt\small Email: \{chhavisharma, vishnu, balamurugan.palaniappan\}@iitb.ac.in}}%
}

\begin{document}
	\setlength{\abovedisplayskip}{1pt}
	\setlength{\belowdisplayskip}{1pt}
	\setlength{\abovecaptionskip}{1ex}
	\setlength{\belowcaptionskip}{1ex}
	\setlength{\floatsep}{1ex}
	\setlength{\textfloatsep}{1ex}

	\maketitle
	\thispagestyle{empty}
	\pagestyle{empty}

	\begin{abstract}
We consider a class of non-smooth strongly convex-strongly concave saddle point problems in a decentralized setting without a central server. To solve a consensus formulation of problems in this class,  we develop an inexact primal dual hybrid gradient (inexact PDHG) procedure that allows generic gradient computation oracles to update the primal and dual variables. We first investigate the performance of inexact PDHG with stochastic variance reduction gradient (SVRG) oracle. Our numerical study uncovers a significant phenomenon of initial conservative progress of iterates of IPDHG with SVRG oracle. To tackle this, we develop a simple  and effective switching idea, where a generalized stochastic gradient (GSG) computation oracle is employed to hasten the iterates' progress to a saddle point solution during the initial phase of updates, followed by a switch to the SVRG oracle at an appropriate juncture. The proposed algorithm is named Decentralized Proximal Switching Stochastic Gradient method with Compression (C-DPSSG), and is proven to converge to an $\epsilon$-accurate saddle point solution with linear rate. Apart from delivering highly accurate solutions, our study reveals that utilizing the best convergence phases of GSG and SVRG oracles makes C-DPSSG well suited for obtaining solutions of low/medium accuracy faster, useful for certain applications. Numerical experiments on two benchmark machine learning applications show C-DPSSG's competitive performance which validate our theoretical findings. The codes used in the experiments can be found \href{https://github.com/chhavisharma123/C-DPSSG-CDC2023}{here}.

	\end{abstract}

\section{Introduction}
\label{sec:introduction}

We focus on solving the following saddle point (or mini-max) problem in a fully decentralized setting \textbf{without a central server}:
\begin{align}
		\min_{x \in \mathbb{R}^{d_x}} \max_{y \in \mathbb{R}^{d_y}} \frac{1}{m}\sum_{i = 1}^m (f_i(x,y) +g(x) -r(y) ) \label{eq:main_opt_problem} \tag{SPP} ,
\end{align}
 where $f_i$ $:$ $\mathbb{R}^{d_x}$$\times$$\mathbb{R}^{d_y} \rightarrow \mathbb{R}$ private to every node $i \in \{1,2,\ldots,m\}\eqqcolon [m]$ is smooth, strongly convex in primal variable $x$ and strongly concave in dual variable $y$ and  $g:\mathbb{R}^{d_x} \rightarrow \mathbb{R}$ and $r:$ $\mathbb{R}^{d_y}$$ \rightarrow$$\mathbb{R}$ are proper, convex and potentially non-smooth functions. This class of saddle point problems finds its use in  distributionally robust optimization, robust classification and regression applications, AUC maximization problems \cite{balamuruganbach2016svrgsaddle,yanetal2019stocprimaldual,zecchin2022communicationefficient} and multi-agent reinforcement learning \cite{wai2018multi}. Additionally, saddle point problems appear in the Lagrangian formulations of constrained minimization problems  \cite{zhu2011distributed,nunezcortes2017distsaddlesubgrad}. Decentralized environments are useful for large-scale systems where privacy and other constraints on data sharing (\text{e.g.} legal, geographical) prevent the availability of entire data set in a single computing machine (or node). In this work, we consider a decentralized environment where the computing nodes possess similar processing and storage capabilities.
 The (static) topology of the decentralized environment is represented using an undirected, connected, simple graph  $\mathscr{G} =(\mathcal{V},\mathcal{E})$, where $\mathcal{V}=[m]$ denotes the set of $m$ computing nodes and an edge $e_{ij} \in \mathcal{E}$ denotes the fact that nodes $i, j \in \mathcal{V}$ are connected. Also, we assume that the communication is synchronous and at every synchronization step, node $i$ communicates only with its neighbors $\mathcal{N}(i) = \{j \in \mathcal{V}: e_{ij} \in \mathcal{E}\}$. 

General stochastic gradient oracle (GSGO) \cite{sgd_leonbottou},
popularly used to solve saddle point problems \cite{liu2020decentralized,xianetal2021fasterdecentnoncvxsp,zecchin2022communicationefficient,beznosikov2020distributed}, unfortunately suffers from inherent variance developed due to stochastic gradients used for updating primal and dual variables at every epoch. 
Despite the availability of stochastic variance reduction gradient oracle (SVRGO) \cite{johnsontong2013svrg, kovalev2022optimal}, which addresses GSGO's variance issue, GSGO is adopted by practitioners due to its simplicity and fast progress in the initial stage. 
SVRGO prepares itself from the start to keep the variance under control which affects the crucial initial phase convergence.  However, the variance in SVRGO vanishes asymptotically speeding up its progress at the later stages. The fast convergence behavior of GSGO at the initial stage and SVRGO at the later stage respectively, provide inspiration for developing a novel algorithm in this work, where a switch is performed between these stochastic gradient oracles. 

Apart from gradient computations, high dimensional parameters are communicated by each node in the decentralized environment with its neighbors, which becomes expensive. 
Thus in this paper, we aim to develop a primal dual decentralized algorithm which attains efficiency in gradient computations by harnessing the best convergence phases of GSGO and SVRGO, and communication efficiency by using compressed representations \cite{lin2017deep, mishchenko2019distributed, liu2020linear, zecchin2022communicationefficient} of iterates. 
We summarize below the \textbf{contributions} of this work:

\begin{enumerate}
\item Inspired by algorithms developed for decentralized minimization problems \cite{liu2020linear, li2021decentralized}, we design a decentralized inexact primal dual hybrid gradient method with compression (IPDHG) by exploiting the consensus constrained formulation of \eqref{eq:main_opt_problem}. 
\item We numerically study the initial behavior of IPDHG with SVRGO and GSGO. To improve the observed initial conservative progress of iterates of IPDHG with SVRGO towards an $\epsilon$-accurate saddle point solution, we propose a Decentralized Proximal Switching Stochastic Gradient method with Compression (C-DPSSG), where a generalized stochastic gradient oracle guides the initial progress of iterates, which switches to  SVRGO at an appropriate point during the iterative update process. C-DPSSG is useful to obtain solutions of low/medium accuracy (where $\epsilon \approx 10^{-4}$) faster, pertinent to certain applications. Using SVRGO at the later iterations of C-DPSSG reduces the variance and hence provides highly accurate solutions in the long run. We further prove that C-DPSSG converges to an $\epsilon$-accurate saddle point solution with linear rate.
\item  We conduct experiments on robust binary classification and AUC maximization problems to demonstrate the practical performance of proposed algorithms.
\end{enumerate}
To our knowledge, this is the first work which provides a closer look at the behavior of GSGO and SVRGO in a newly designed IPDHG scheme \textit{with compression} to solve saddle point problems of the form \eqref{eq:main_opt_problem}. 
Note that a practical improvement in SVRG is studied in \cite{babanezhad2015stopwasting} using a combination of GSGO and SVRGO for solving \textit{smooth convex minimization problems in a single machine setting}; however we leverage the best performance phases of GSGO and SVRGO to solve \textit{non-smooth saddle point problems in a decentralized environment}.  
We now present notations useful for subsequent discussion.

\textbf{Notations:} 
Let $z$$=$$(x, y)$$\in$$ {\mathbb{R}}^{d_x+d_y}$ denote the pair of primal variable $x$ and dual variable $y$, $z^\star$$=$$(x^\star, y^\star)$ denote a saddle point solution of problem \eqref{eq:main_opt_problem} and $\bz^\star$ $=$ $((z^\star)^\top,\ldots,(z^\star)^\top)$. 
Weights $W_{ij}$ associated with the communication link between a pair of nodes $(i,j) \in \mathcal{V} \times \mathcal{V}$ are collected into a matrix $W$ of size $m \times m$.  $I_d$ denotes a $d \times d$ identity matrix, $\mathbf{1}$ denotes a $m \times 1$ column vector of ones and $J = \frac{1}{m}\mathbf{1}\mathbf{1}^\top$ denotes a $m \times m$ matrix of uniform weights equal to $\frac{1}{m}$.
$A\otimes B$ denotes the Kronecker product of two matrices $A$ and $B$. 
Let $f(x,y) \coloneqq \sum_{i = 1}^m f_i(x,y)$.
Condition number $\kappa_f $ of $f$ is defined as $L/\mu$, where $L$ is the smoothness parameter of $f_i(x,y)$ (see Appendix \ref{smoothness_assumptions} in \cite{cdctechnicalreport}) and $\mu = \min\{\mu_x, \mu_y\}$ (see Assumptions \ref{s_convexity_assumption}-\ref{s_concavity_assumption}).
Condition number $\kappa_g$ of communication graph $\mathscr{G}$ is defined as the ratio of largest eigenvalue and second smallest eigenvalue of $I-W$. 
For a $d \times 1$ vector $u$ and for some $d \times d$ symmetric positive semi-definite (p.s.d) matrix $A$, we define $\| u \|^2_A = u^\top A u$. 

\textbf{Paper Organization:} We develop and interpret IPDHG algorithm in Section \ref{sec:algo_design}, followed by a discussion of assumptions (Section \ref{sec:assumptions}). Early stage behavior of IPDHG with SVRGO and GSGO is explained is Section \ref{sec:finitesum_main}.
The proposed C-DPSSG algorithm is presented in Section \ref{sec:genstoch_main}.
Related work is discussed in Section \ref{sec:related_work} and experimentation details are in Section \ref{sec:experiments}. Due to space constraints, all proofs and additional experiments are deferred to our technical report \cite{cdctechnicalreport}.
\section{Algorithm Development}
\label{sec:algo_design}
Before proceeding to the algorithm development, we first present few terminologies to be used in the remaining part of the paper. Assuming the local copy of $(x,y)$ in $i$-th node as $(x^i, y^i)$, we collect local primal and dual variables into $\bx = \begin{pmatrix}x^1,  x^2, \ldots  x^m \end{pmatrix} \in {\mathbb{R}}^{md_x}$ and $\by = \begin{pmatrix}y^1, y^2, \ldots  y^m \end{pmatrix} \in {\mathbb{R}}^{md_y}$.  Using this notation and following \cite{nunezcortes2017distsaddlesubgrad}, the problem \eqref{eq:main_opt_problem} can be formulated as:
\begin{align}
\min_{\bx \in \mathbb{R}^{md_x} } \max_{\by \in \mathbb{R}^{md_y}} \ F(\bx,\by) + G(\bx) -R(\by) \nonumber \\ \text{s.t.} \  (U \otimes I_{d_x})\bx = 0, \ \ (U \otimes I_{d_y})\by = 0, \label{eq:main_opt_consenus_constraint}
\end{align}
where $F(\bx, \by) = \sum_{i = 1}^m f_i(x^i,y^i)$, $G(\bx)= \sum_{i = 1}^m g(x^i)$, $R(\by)=  \sum_{i = 1}^m r(y^i)$, $U = \sqrt{I_m-W}$ and consensus constraints are present on $\bx$ and $\by$. 
The assumptions on $W$ (to be made later) would imply $I_m-W$ to be symmetric p.s.d. and hence leads to existence of $\sqrt{I_m-W}$. We consider the following Lagrangian function of problem \eqref{eq:main_opt_consenus_constraint}: 
\begin{align}
& \mathcal{L}(\bx, \by; S^{\bx}, S^{\by}) = F(\bx,\by) + G(\bx) - R(\by) 
\nonumber \\ & \hspace*{1.8cm} +  \langle S^{\bx}, (U \otimes I_{d_x})\bx \rangle  + \langle S^{\by}, (U \otimes I_{d_y})\by \rangle, \label{eq:lagrangian_form}
\end{align}
where $S^{\bx} \in \mathbb{R}^{md_x}$ and $S^{\by} \in \mathbb{R}^{md_y}$ denote the Lagrange multipliers associated with consensus constraints on variables $\bx$ and $\by$ respectively. 
We prove that solving constrained problem \eqref{eq:main_opt_consenus_constraint} is equivalent to solving the following problem (see Theorem \ref{thm:lag_equivalence} in \cite{cdctechnicalreport}):
\begin{align}
\min_{\bx \in \mathbb{R}^{md_x}, S^{\by} \in \mathbb{R}^{md_y}} \max_{\by \in \mathbb{R}^{md_y}, S^{\bx} \in \mathbb{R}^{md_x} } \mathcal{L}(\bx, \by; S^{\bx}, S^{\by}). \label{minmax_lagrange_problem}
\end{align}
A similar equivalence is provided in \cite{rogozin2021decentralized} under the assumption of convex compact constraint sets and bounded gradients of $F(\bx,\by)$. On the contrary, we formally show the equivalence using convexity-concavity of $f_i(x,y)$ and using properties of weight matrix $W$ (to be defined in next section). Further, our proof does not require compactness and bounded gradient assumptions.

To solve problem \eqref{minmax_lagrange_problem}, we propose gradient descent ascent \red{parallel updates} for the primal-dual variable pair $\bx, S^\bx$ and dual-primal pair $\by, S^\by$, illustrated in equations \eqref{ipdhg_primal_dual_updates_1} and  \eqref{ipdhg_dual_primal_updates_1}. Note that in eq. \eqref{ipdhg_primal_dual_updates_1}, $\nu^\bx_{t+1}$ is found using a prox-linear step involving linearization of $\red{F(\bx,\by)}$ with respect to $\bx$ and a penalized cost-to-move term  $\frac{1}{2s}\|\bx-\bx_t\|^2$, followed by an ascent step to update the Lagrange dual variable $S^\bx$. 
Then ${\hat{\bx}}_{t+1}$ is found using prox-linear step similar to the first step but using the recent $S_{t+1}^{\bx}$ to further correct the direction. Finally $\bx_{t+1}$ is found by a prox step where $\prox_{sG}(\bx) = \argmin_{\bu \in {\mathbb{R}}^{md_x}} G(\bu) + \frac{1}{2s} \|\bu-\bx\|^2$. Letting $D^\bx_t = (U \otimes I_{d_x})S^\bx_t$, and pre-multiplying by $U\otimes I_{d_x}$ in the update step of $S^\bx$, $S^\bx_{t+1}$ update reduces to $D^\bx_{t+1}  = D^\bx_t + \frac{\gamma}{2s}((I_m-W) \otimes I_{d_x})\nu^\bx_{t+1}$. Now using $\nu^\bx_{t+1}$ and $D^\bx_{t+1}$ updates, we can further reduce $\hat{\bx}_{t+1}$ update to $\hat{\bx}_{t+1} = \nu^\bx_{t+1} - \frac{\gamma}{2}((I_m-W) \otimes I_{d_x})\nu^\bx_{t+1}$.  Similarly, the updates to $\by, S^{\by}$ can be done using appropriate gradient ascent-descent steps which lead to corresponding equations \eqref{ipdhg_dual_primal_updates_1}. Further the update of $D^\by_{t+1}$ analogous to $D^\bx_{t+1}$ update can be obtained by letting $D^\by_t = -(U \otimes I_{d_y})S^\by_t$. 
 
A similar update process is explored in \cite{li2021decentralized}, however for solving \textit{convex minimization problems} only. The dual variable in \cite{li2021decentralized} is simpler since it arises from the Lagrangian formulation of consensus constrained minimization problem  and appears only as linear term in the Lagrangian function.  
However in our work, the Lagrangian function in eq. \eqref{eq:lagrangian_form} is not in general linear in the dual variable $\by$ despite the linear terms associated with Lagrange multipliers $S^\bx, S^\by$. Hence updates \eqref{ipdhg_primal_dual_updates_1} and \eqref{ipdhg_dual_primal_updates_1} in our work need to tackle the original primal dual pair $\bx, \by$ along with the Lagrange multipliers $S^\bx, S^\by$ related to consensus constraints. 


	\vspace*{0.1cm}
	\begin{figure}[!t]
		\begin{minipage}{0.95\linewidth}
			\textbf{Updates to primal dual pair $\bx, S^\bx$:} 
			\hrule
			\begin{align*}
				&\
				\left.\begin{aligned}
					\nu^\bx_{t+1} &= \bx_t - s\nabla_\bx F(\bx_t,\by_t) - s(U \otimes I_{d_x})S^\bx_t \\ 
					S^\bx_{t+1} &= S^\bx_t + \frac{\gamma}{2s}(U \otimes I_{d_x})\nu^\bx_{t+1}   \\
					\hat{\bx}_{t+1} &= \bx_t - s\nabla_\bx F(\bx_t,\by_t) - s(U \otimes I_{d_x})S^\bx_{t+1} \\
					\bx_{t+1} &= \prox_{sG}(\hat{\bx}_{t+1}).   
				\end{aligned} \hspace{-.2em} \right\} \label{ipdhg_primal_dual_updates_1} \tag{P1}   
			\end{align*}
		\end{minipage}
		\hrule
	\end{figure}

	\vspace*{0.1cm}
	\begin{figure}[t]
		\begin{minipage}{0.95\linewidth}
			\textbf{ Updates to dual primal pair $\by, S^\by$:} 
			\hrule
			\begin{align*}
				&\
				\left.\begin{aligned}
					\nu^\by_{t+1} &= \by_t + s\nabla_\by F(\bx_t,\by_t) - s(U \otimes I_{d_x})S^\by_t \nonumber \\
					S^\by_{t+1} &= S^\by_t - \frac{\gamma}{2s}(U \otimes I_{d_y})\nu^\by_{t+1}  \nonumber \\
					\hat{\by}_{t+1} &= \by_t + s\nabla_\by F(\bx_t,\by_t) - s(U \otimes I_{d_x})S^\by_{t+1} \nonumber \\
					\by_{t+1} &= \prox_{sR}(\hat{\by}_{t+1}). 
				\end{aligned}\right\} \hspace{-.8em} \label{ipdhg_dual_primal_updates_1} \tag{D1}   
			\end{align*}
		\end{minipage}
		\hrule
	\end{figure}
	\vspace*{-0.15cm}
Observe that the terms $((I_m-W) \otimes I_{d_x})\nu^\bx_{t+1}$ and $((I_m-W) \otimes I_{d_y})\nu^\by_{t+1}$ respectively in $D^\bx_{t+1}$ and $D^\by_{t+1}$ updates denote the communication of $\nu^\bx_{t+1}$ and $\nu^\by_{t+1}$ across the nodes. 
Further note that $\nu^\bx_{t+1}$ and $\nu^\by_{t+1}$ need to be communicated only once for updating $D^\bx_{t+1}, \hat{\bx}_{t+1}$ and $D^\by_{t+1}, \hat{\by}_{t+1}$
To improve the communication efficiency further, we propose to compress $\nu^\bx_{t+1}$ and $\nu^\by_{t+1}$ using a compression module (COMM
procedure \cite{liu2020linear}) as illustrated
in Algorithm \ref{comm_main}.
Algorithm \ref{alg:generic_procedure_sgda} illustrates the proposed Inexact Primal Dual Hybrid Gradient (IPDHG) method with compression. 
In the next section, we state assumptions useful for further discussions.
  


\begin{algorithm}[!h]\footnotesize
	\caption{Compressed Communication Procedure (COMM) \cite{liu2020linear}} 
	\label{comm_main}
	\begin{algorithmic}[1]
		\STATE{\textbf{INPUT:}} {$\nu_{t+1}, H_t, H^w_t, \alpha$}
		\STATE $Q^i_t = Q(\nu^i_{t+1} - H^i_t)$ \ \ (compression)
		\STATE $\hat{\nu}^i_{t+1} = H^i_t + Q^i_t$ , 
		\STATE $H^i_{t+1} = (1-\alpha)H^i_t + \alpha \hat{\nu}^i_{t+1}$ , 
		\STATE $\hat{\nu}^{i,w}_{t+1} = H^{i,w}_t + \sum_{j = 1}^m W_{ij}Q^j_t$, \ \ (communicating compressed vectors)
		\STATE $H^{i,w}_{t+1} = (1-\alpha)H^{i,w}_t + \alpha \hat{\nu}^{i,w}_{t+1}$ , 
		\STATE {\textbf{RETURN:}} $\hat{\nu}^{i}_{t+1}, \hat{\nu}^{i,w}_{t+1}, H^i_{t+1}, H^{i,w}_{t+1}$ for each node $i$ .
	\end{algorithmic}
\end{algorithm}

\begin{algorithm}[!htb] \footnotesize
	\caption{Inexact Primal Dual Hybrid Gradient method with compression using stochastic gradient oracle $\mathcal{G}$ (IPDHG)}
	\label{alg:generic_procedure_sgda}
	\begin{algorithmic}[1]
		\STATE{\textbf{INPUT:}} {\footnotesize{$\ssx$,
		$\ssy$, $\ssD^\ssx$, $\ssD^\ssy$, $\ssH^\ssx$, $\ssH^\ssy$, $\ssH^{\ssw,\ssx}$, $\ssH^{\ssw,\ssy}$, $s$, $\gamma_{\ssx}$, $\gamma_{\ssy}$, $\alpha_{\ssx}$, $\alpha_{\ssy}$,  $\mathcal{G}$}}
		\STATE Compute gradients $\mathcal{G}^{\ssx} $ and $\mathcal{G}^{\ssy}$ at $(\ssx,\ssy)$ via oracle $\mathcal{G} = (\mathcal{G}^\ssx, \mathcal{G}^\ssy)$
		\STATE $\nu^{\ssx} = \ssx - s \mathcal{G}^{\ssx} - s \ssD^{\ssx}$  
		\STATE $\hat{\nu}^{\ssx}, \hat{\nu}^{\ssw,\ssx}, \ssH^{\ssx}_{new}, \ssH^{\ssw,\ssx}_{new}= \text{COMM}\left(\nu^{\ssx}, \ssH^{\ssx}, \ssH^{\ssw,\ssx}, \alpha_{\ssx} \right)$
		\STATE $\ssD^{\ssx}_{new} = \ssD^{\ssx} + \frac{\gamma_{\ssx}}{2s}(\hat{\nu}^{\ssx} - \hat{\nu}^{\ssw,\ssx})$  
		\STATE $\hat{\ssx} = \nu^{\ssx} - \frac{\gamma_{\ssx}}{2} (\hat{\nu}^{\ssx} - \hat{\nu}^{\ssw,\ssx})$ 
		\STATE $\ssx_{new} = \prox_{sG} (\hat{\ssx})$ 
		\STATE $\nu^{\ssy} = \ssy + s \mathcal{G}^{\ssy} - s \ssD^{\ssy}$
		\STATE $\hat{\nu}^{\ssy}, \hat{\nu}^{\ssw,\ssy}, \ssH^{\ssy}_{new}, \ssH^{\ssw,\ssy}_{new}=\text{COMM}\left(\nu^{\ssy}, \ssH^{\ssy}, \ssH^{\ssw,\ssy}, \alpha_{\ssy} \right)$   
		\STATE $\ssD^{\ssy}_{new} = \ssD^{y} + \frac{\gamma_{\ssy}}{2s}(\hat{\nu}^{\ssy} - \hat{\nu}^{\ssw,\ssy})$  \label{Dy_update}
		\STATE $\hat{\ssy} = \nu^{\ssy} - \frac{\gamma_{\ssy}}{2} (\hat{\nu}^{\ssy} - \hat{\nu}^{\ssw,\ssy})$ 
		\STATE $\ssy_{new} = \prox_{sR} (\hat{\ssy})$ 
		\STATE {\textbf{RETURN:}} 
		 $\ssx_{new} , \ssy_{new}, \ssD^\ssx_{new}, \ssD^\ssy_{new},\ssH^{\ssx}_{new}, \ssH^{\ssy}_{new}, \ssH^{\ssw,\ssx}_{new}, \ssH^{\ssw,\ssy}_{new}$
	\end{algorithmic}
\end{algorithm}
\section{Assumptions}\label{sec:assumptions}

We make the following assumptions, which would be useful throughout this work.
\begin{myassumption} \label{s_convexity_assumption} Each $f_i(\cdot,y)$ is $\mu_x$-strongly convex for every $y \in \mathbb{R}^{d_y}$; hence for any $x_1, x_2 \in \mathbb{R}^{d_x}$ and fixed $y \in \mathbb{R}^{d_y}$, it holds: $f_i(x_1,y) \geq f_i(x_2,y) + \left\langle \nabla_x f_i(x_2,y), x_1 - x_2 \right\rangle + \frac{\mu_x}{2} \left\Vert x_1 - x_2 \right\Vert^2$.
\end{myassumption}
\begin{myassumption} \label{s_concavity_assumption} Each $f_i(x,\cdot)$ is $\mu_y$-strongly concave for every $x \in \mathbb{R}^{d_x}$; hence for any $y_1, y_2 \in \mathbb{R}^{d_y}$ and fixed $x \in \mathbb{R}^{d_x}$, it holds:
$f_i(x,y_1) \leq f_i(x,y_2) + \left\langle \nabla_y f_i(x,y_2), y_1 - y_2 \right\rangle -  \frac{\mu_y}{2} \left\Vert y_1 - y_2 \right\Vert^2$.
%
\end{myassumption}
\begin{myassumption}\label{nonsmooth_assumption} $g(x)$ and $r(y)$ are proper, convex and possibly non-smooth functions.
\end{myassumption}

\begin{myassumption} \label{compression_operator} {\red{The compression operator $Q$ (see Algorithm \ref{comm_main}) satisfies the following for every $u \in \mathbb{R}^d$: (i) $Q(u)$ is an unbiased estimate of $u$: $E\left[ Q(u) \right]=u$ (ii) $E[ \Vert Q(u) - u \Vert^2 ] \leq \delta \Vert u \Vert^2$}}, 
where the constant {\red{$ \delta \geq 0 $}} denotes the amount of compression induced by operator $Q$ and is called a compression factor.  When $\delta = 0$, $Q$ achieves no compression. 
\end{myassumption}
\begin{myassumption} \label{weight_matrix_assumption} Weight matrix $W$ is symmetric, row stochastic and $W_{ij} >0$ if and only if $(i,j) \in \mathcal{E}$ and $W_{ii} > 0$ for all $i \in [m]$. Eigenvalues of $W$ denoted by $\lambda_1, \ldots, \lambda_m$ satisfy: $-1 < \lambda_m \leq \lambda_{m-1} \leq \ldots \leq \lambda_2 < \lambda_1 = 1$.
\end{myassumption}
\begin{myassumption} \label{smoothness_x_svrg_main} Assume that each $f_{ij}(x,y)$ is $L_{xx}$ smooth in $x$, i.e. for every fixed $y$, $\Vert \nabla_x f_{ij}(x_1,y) - \nabla_x f_{ij}(x_2,y) \Vert$$\leq$$L_{xx}\Vert x_1 - x_2 \Vert, \forall x_1,x_2 \in \mathbb{R}^{d_x}$.
\end{myassumption}
\begin{myassumption} \label{smoothness_y_svrg_main} Assume that each $f_{ij}(x,y)$ is $L_{yy}$ smooth in $y$, i.e. for every fixed $x$, $\Vert \nabla_y f_{ij}(x,y_1) - \nabla_y f_{ij}(x,y_2) \Vert$$\leq$$L_{yy}\Vert y_1 - y_2 \Vert$, $\forall  y_1, y_2$  $\in \mathbb{R}^{d_y}$.
\end{myassumption}
\begin{myassumption} \label{lipschitz_xy_svrg_main} Assume that each $\nabla_x f_{ij}(x,y)$ is $L_{xy}$ Lipschitz in $y$, i.e. for every fixed $x$, $\Vert \nabla_x f_{ij}(x,y_1) - \nabla_x f_{ij}(x,y_2) \Vert$$\leq$$L_{xy}\Vert y_1 - y_2 \Vert,$  $\forall  y_1, y_2$  $\in \mathbb{R}^{d_y}$.
\end{myassumption}
\begin{myassumption} \label{lipschitz_yx_svrg_main} Assume that each $\nabla_y f_{ij}(x,y)$ is $L_{yx}$ Lipschitz in $x$, i.e. for every fixed $y$, $\Vert \nabla_y f_{ij}(x_1,y) - \nabla_y f_{ij}(x_2,y) \Vert$$\leq$$L_{yx}\Vert x_1 - x_2 \Vert$,  $\forall  x_1, x_2$  $\in \mathbb{R}^{d_x}$.
\end{myassumption}
Note that Assumptions \ref{s_convexity_assumption}-\ref{nonsmooth_assumption} and Assumptions \ref{weight_matrix_assumption}-\ref{lipschitz_yx_svrg_main} are standard in the study of saddle point problems (e.g. \cite{beznosikov2020distributed, beznosikovetal2020distsaddle,liu2020decentralized, mukherjeecharaborty2020decentralizedsaddle}). Assumption \ref{compression_operator} is also standard in existing works (e.g. \cite{alistarh2017qsgd, li2021decentralized, liu2020linear}).


\section{Understanding Early Stage Behavior of IPDHG with SVRGO and GSGO} \label{sec:finitesum_main}
In this section, we first recap SVRG oracle and then draw key observations on the behavior of IPDHG with SVRG oracle. We refer to IPDHG with SVRGO as Decentralized Proximal Stochastic Variance Reduced Gradient algorithm with Compression (C-DPSVRG).

We assume that each local function $f_i(x,y)$ is of the form $\frac{1}{n}\sum_{j = 1}^n f_{ij}(x,y)$ where $f_{ij}(x,y)$ represents the loss function at $j$-th batch of samples at node $i$. This type of structure can be seen for instance in empirical risk minimization problems \cite{sgd_leonbottou}. For simplicity, we assume that each node $i$ has same number of batches $n$. However, our analysis easily extends to different number of batches $n_i$. 
Let $N_\ell$ denote the number of $\ell$ocal samples at each node $i$. Then number of samples in the function component $f_{ij}$ is determined by the batch size $B = N_\ell/n$.
Let $\mathcal{P}_i = \{ p_{il}: l \in \{ 1, 2, \ldots, n \} \}$ denote a probability distribution where $p_{il}$ is the probability with which batch $l$ is sampled at node $i$. Let $p_{\min}\coloneqq\min_{i,l} p_{il}$. 
Without loss of generality we assume that $p_{\min} > 0$, hence each batch is chosen with a positive probability. 
Inspired from \cite{johnsontong2013svrg, kovalev2020don}, we consider stochastic variance reduced gradient  oracle comprising the following steps to compute stochastic gradients in IPDHG:

\begin{tcolorbox}[top=0pt,bottom=0pt]
	\textbf{Stochastic Variance Reduced Gradient Oracle (SVRGO):} \\	
\textbf{(1).} \textbf{Index sampling:} Sample $l \in \{ 1, 2, \ldots, n \} \sim \mathcal{P}_i$ for every node $i$. \\
\textbf{(2).} \textbf{Stochastic gradient computation with variance reduction:} For a reference point $\tilde{z}^i=(\tilde{x}^i,\tilde{y}^i)$, compute stochastic gradients at $z^i=(x^i,y^i)$ with respect to $x$ and $y$ as follows:
\begin{align}
&	\mathcal{G}^{i,x}  = \frac{1}{np_{il}} ( \nabla_x f_{il}(z^i) - \nabla_x f_{il}(\tilde{z}^i) ) + \nabla_x f_i(\tilde{z}^i) , \nonumber \\
&	\mathcal{G}^{i,y}  = \frac{1}{np_{il}} ( \nabla_y f_{il}(z^i) - \nabla_y f_{il}(\tilde{z}^i) ) + \nabla_y f_i(\tilde{z}^i). \nonumber 
\end{align} 
\textbf{(3).} \textbf{Reference point update:} Sample $\omega \sim \text{Bernoulli}(p)$ and update the reference point $\tilde{z}_i$ as follows:
\begin{align}
	\tilde{x}^i & \longleftarrow \omega x^i + (1-\omega)\tilde{x}^i, \ \ \tilde{y}^i \longleftarrow \omega y^i + (1-\omega)\tilde{y}^i . \nonumber
\end{align}

\end{tcolorbox}

As described above, SVRGO evaluates stochastic gradients at current iterate $z^i_t$ and reference point $\tilde{z}^i_t$ as follows:
\begin{align}
	\mathcal{G}^{i,x}_t & = \underbrace{\frac{1}{np_{il}} \left( \nabla_x f_{il}(z^i_t) - \nabla_x f_{il}(\tilde{z}^i_t) \right)} + \nabla_x f_i(\tilde{z}^i_t). \label{svrg_grad_x_with_error}
\end{align}
SVRGO contains an expensive but key component $\nabla_x f_i(\tilde{z}^i_t)$ obtained using full batch gradient evaluation to reduce the variance in stochastic gradients. Until the current iterate $z^i_t$ and reference point $\tilde{z}^i_t$ start converging to saddle point, there is a gradient approximation error captured by first term in  \eqref{svrg_grad_x_with_error}. Therefore, full batch gradients evaluated at the early iterations are not effective due to large distance between early iterates and saddle point solution. In addition, SVRGO evaluates on an average $2B +pN_{\ell}$ gradients per iterate. This phenomenon leads to slow convergence of C-DPSVRG in the initial stage with high computational cost. In this work, we propose a remedy to improve the early stage slow convergence of C-DPSVRG using a general stochastic gradient oracle described below:
\begin{tcolorbox}[top=0pt,bottom=0pt]
\textbf{General Stochastic Gradient Oracle (GSGO):} \\	
\textbf{(1).} \textbf{Index sampling:} Sample $l \in \{ 1, 2, \ldots, n \} \sim \mathcal{P}_i$ for every node $i$. \\
\textbf{(2).} \textbf{Stochastic gradient computation:} Compute stochastic gradients at $z^i=(x^i,y^i)$ with respect to $x$ and $y$ as follows:
\begin{align}
	\mathcal{G}^{i,x} & = \frac{1}{np_{il}} \nabla_x f_{il}(z^i), \ \ 
	\mathcal{G}^{i,y}  = \frac{1}{np_{il}}  \nabla_y f_{il}(z^i).  \nonumber
\end{align}
\end{tcolorbox}

Stochastic gradient descent ascent (SGDA) \cite{yan2020optimal,lin2020gradient,liu2020decentralized,xianetal2021fasterdecentnoncvxsp,zecchin2022communicationefficient} which uses GSGO or its variant is the workhorse of several algorithms due to its promising fast convergence in the initial phase of iterate updates along with low computational cost. Though SGDA exhibits slow convergence or saturation behavior asymptotically due to inherent variance in the stochastic gradients, its impressive behavior in initial phase 
motivates us to exploit GSGO in Algorithm \ref{alg:generic_procedure_sgda} to obtain fast convergence along with low computation cost during the initial iterate updates.
Incorporating GSGO in IPDHG scheme pushes the iterates to a region close to the saddle point solution which can potentially make the full batch gradients in SVRGO more effective. Leveraging fast early convergence using GSGO and fast asymptotic convergence using SVRGO, we propose a switching algorithm which uses GSGO in IPDHG for a fixed number of initial iterations and then switches to SVRGO to achieve highly accurate solution. Before discussing the algorithm details, we examine this behavior empirically on robust logistic regression problem \eqref{main_paper_robust_logistic_regression}. In Figure \ref{fig:a4a_grad_com_sgd_switch_svrg_motivation}, the behavior of C-DPSVRG on \eqref{main_paper_robust_logistic_regression} at every iterate update is compared with a switching scheme where GSGO is used for first $T_i$ iterate updates followed by SVRGO. We can clearly see that iterates of C-DPSVRG (blue line) make little progress in initial stage whereas the use of GSGO leads to faster progress of iterates towards saddle point solution with much lesser gradient computations. 
\begin{figure}[!h]
	\centering
	\includegraphics[width=1\linewidth]{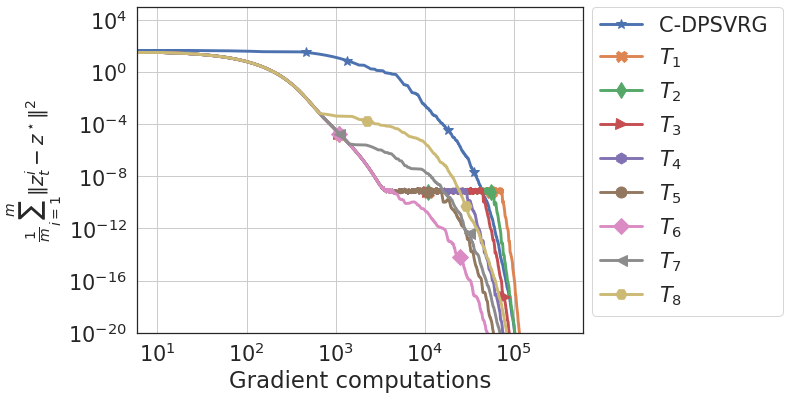} 
	\vspace{-0.8cm} 
	\caption{ Iterate convergence behavior of IPDHG using GSGO and SVRGO on robust logistic regression problem. $T_1 > T_2 >T_3 > T_4 > T_5 >T_6 > T_7 > T_8$ are the switching points from GSGO to SVRGO in IPDHG}
	\label{fig:a4a_grad_com_sgd_switch_svrg_motivation}
\end{figure}
 We further observe that larger values of $T_i$ lead to clear saturation of progress of iterates due to ineffective GSGO update steps after a while, whereas small values of $T_i$ cause early switching to SVRGO. Thus choosing a right switching point is crucial to harness the effectiveness of both GSGO and SVRGO, which we will discuss in the next section. 

\section{IPDHG with Switching between Stochastic Gradient Oracles}
\label{sec:genstoch_main}

In the previous section, we have seen the advantage of leveraging GSGO in the initial stage of IPDHG iterate updates. 
We are now ready to describe our novel switching algorithm to solve \eqref{eq:main_opt_problem}. Starting from initial points $x_0, y_0$, each node $i$ updates its primal and dual variables using IPDHG with GSGO as illustrated in Steps \ref{imp-GSGO}-\ref{IPDHG-GSGO} of Algorithm \ref{alg:IPDHG_with_sgd_svrg_oracle}. This process is repeated for the first $T_0$ iterations. After $T_0$ iterations, each node $i$ switches to SVRGO with reference point $\tilde{z}^i_{T_0}$ initialized to $(x^i_{T_0}, y^i_{T_0})$ and performs IPDHG updates using SVRGO for the remaining $T-T_0$ iterations. We see that Algorithm \ref{alg:IPDHG_with_sgd_svrg_oracle} requires the knowledge of switching point $T_0$. To address this, we first analyze the behavior of Algorithm \ref{alg:IPDHG_with_sgd_svrg_oracle} during the first $T_0$ iterations in the following lemma.
\begin{lemma} \label{lem:exp_phit+1_phi_0} Let  $\{\bx_{t} \}_t, \{\by_{t}\}_t$ be the sequences generated by Algorithm \ref{alg:IPDHG_with_sgd_svrg_oracle}. 
Suppose Assumptions \ref{s_convexity_assumption}-\ref{lipschitz_yx_svrg_main} hold.
Then for every $ 0 \leq t \leq T_0 -1 :$
	\begin{align}
		E_0[ \Phi_{t+1} ] & \leq ( \rho_0)^{t+1}  \Phi_{0}  + \frac{2s_0^2(C_x + C_y)}{(1-\rho_0)n^2p_{\min}}, \label{eq:exp_phit_phi_0_ub}
	\end{align}
where $E_0$ denotes the total expectation when $t \leq T_0-1$, $\Phi_t$ denotes the distance of the iterates $\bx_t, \by_t, D^\bx_t, D^\by_t, H^\bx_t, H^\by_t$ from their respective limit points (described in eq. \eqref{phi_kt_defn} in \cite{cdctechnicalreport}), $\rho_0 \in (0, 1) $ is a problem dependent parameter defined in eq. \eqref{rho_sgd} in \cite{cdctechnicalreport} and $C_x$ $=$ $\sum_{i = 1}^m \sum_{l = 1}^n \left\Vert \nabla_x f_{il}(z^\star) \right\Vert^2$, $C_y$ $=$ $\sum_{i = 1}^m \sum_{l = 1}^n \left\Vert \nabla_y f_{il}(z^\star) \right\Vert^2$.
\end{lemma}
Due to space considerations, the proof of Lemma \ref{lem:exp_phit+1_phi_0} is provided in Appendix \ref{appendix:exp_phit+1_phi_0} \cite{cdctechnicalreport}. Note that $C_x$ and $C_y$ in \eqref{eq:exp_phit_phi_0_ub} appear due to variance in the stochastic gradients. If $\mathcal{G}^i_{x,t}, \mathcal{G}^i_{y,t}$ are set to be $\nabla_x f_i(x^i_t,y^i_t)$ and  $ \nabla_y f_i(x^i_t,y^i_t)$ respectively in Algorithm \ref{alg:IPDHG_with_sgd_svrg_oracle} for $t\leq T_0 -1$, the second term in the r.h.s of \eqref{eq:exp_phit_phi_0_ub} will be absent. Consequently, using Lemma \ref{lem:exp_phit+1_phi_0}, IPDHG converges to $\epsilon$-accurate saddle point solution with linear rate. Without loss of generality, we assume that $C_x$ and $C_y$ are positive.  We can also see that $C_x$ and $C_y$ are bounded because $z^\star$ is unique since each $f_i$ is assumed to be strongly convex in $x$ and strongly concave in $y$.
 
Lemma \ref{lem:exp_phit+1_phi_0} indicates that Algorithm \ref{alg:IPDHG_with_sgd_svrg_oracle} has an error term (second term in the bound in RHS) due to the variance in stochastic gradients accumulated in the course of $t$ iterations. This shows that IPDHG with GSGO returns an approximate solution asymptotically. By choosing large $t = T_0$, the first term in the RHS of \eqref{eq:exp_phit_phi_0_ub} can be made sufficiently small. However, there might be wastage of iterations once the iterates converge in the neighborhood of saddle point solution as demonstrated in Figure \ref{fig:a4a_grad_com_sgd_switch_svrg_motivation}. Further, choosing small $T_0$ might not exploit the full potential of GSGO in the early stage. To address this situation, we propose to choose $T_0$ such that there is a sufficient drift from the initial value $\Phi_{0}$. We introduce a hyperparameter $\epsilon_0 \in (0,1)$ to achieve this and set $T_0$ such that $\rho_0^{T_0} = \epsilon_0$. This reduces the upper bound of $E_0[\Phi_{T_0}]$ to $\epsilon_0 \Phi_0 + \frac{2s_0^2(C_x + C_y)}{(1-\rho_0)n^2p_{\min}}$. We still have an unanswered question on the choice of $\epsilon_0$. We exploit the convergence behavior of Algorithm \ref{alg:IPDHG_with_sgd_svrg_oracle} for $t \geq T_0 $ to find a suitable $\epsilon_0$ and hence the switching point $T_0$.

\begin{algorithm}[ht!]\footnotesize
	\caption{\textbf{D}ecentralized \textbf{P}roximal \textbf{S}witching \textbf{S}tochastic \textbf{G}radient method with \textbf{C}ompression (C-DPSSG)}
	\label{alg:IPDHG_with_sgd_svrg_oracle}
	\begin{algorithmic}[1]
		\STATE {\textbf{INPUT:}} $\bx_{0} = (\mathbf{1} \otimes I_{d_x})x_0, \by_{0} = (\mathbf{1} \otimes I_{d_y})y_0, D^\bx_0 = D^\by_0 = \mathbf{0}, H^\bx_{0} = \bx_0$, $H^\by_{0} = \by_0$, $H^{\ssw,\bx}_{0} = (W \otimes I_{d_x})\bx_{0}$, $H^{\ssw,\by}_{0} = (W \otimes I_{d_y})\by_{0}, s_0 = \frac{np_{\min}}{4\sqrt{2}L\kappa_f}, s = \frac{\mu np_{\min}}{24L^2}$,  $\gamma_{x,0}, \gamma_x, \gamma_{y,0}, \gamma_y, \alpha_{x,0}, \alpha_x,  \alpha_{y,0}, \alpha_y$ defined in Appendices \ref{appendix_parameters_feasibility} and \ref{svrg_parameters_setup} \cite{cdctechnicalreport}, switching point $T_0 $. 
		\FOR{$t=0$ {\bfseries to} $T -1$ }
		\STATE Sample $l \in \{ 1, 2, \ldots, n \} \sim \mathcal{P}_i$ for every node $i$
		\IF{ $t \leq T_0 - 1$ }
		\STATE \label{imp-GSGO} $\mathcal{G}^{i,x}_t = \frac{1}{np_{il}}\nabla_x f_{il}(z^i_t)$ and $\mathcal{G}^{i,y}_t = \frac{1}{np_{il}}\nabla_y f_{il}(z^{i}_t)$ for every node $i$
		\STATE \label{IPDHG-GSGO}$\bx_{t+1}, \by_{t+1},D^{\bx}_{t+1},D^\by_{t+1},{{H^{\bx}_{t+1}, H^{\by}_{t+1}, H^{\ssw,\bx}_{t+1}, H^{\ssw,\by}_{t+1}}}$$ \newline 
		\hspace*{3.9em}  =$$\text{IPDHG}(\bx_{t}, \by_{t}, D^{\bx}_{t}, D^{\by}_{t}
		, H^{\bx}_{t}, H^{\by}_{t}, H^{\ssw,\bx}_{t},$ \newline 
		\hspace*{4.9em} $H^{\ssw,\by}_{t}, s_{0}, \gamma_{x,0}, \gamma_{y,0}, \alpha_{x,0}, \alpha_{y,0}, \mathcal{G}_t )$
		\ELSE
		\STATE $\tilde{\bx}_{T_0} = \bx_{T_0}$, $\tilde{\by}_{T_0} = \by_{T_0}$
		\STATE Compute stochastic gradients $\mathcal{G}^{i,x}_t$ and $\mathcal{G}^{i,y}_t$ using SVRGO for every node $i$
		\STATE $\bx_{t+1}, \by_{t+1},D^{\bx}_{t+1},D^\by_{t+1},{{H^{\bx}_{t+1}, H^{\by}_{t+1}, H^{\ssw,\bx}_{t+1}, H^{\ssw,\by}_{t+1}}}$$ \newline 
		\hspace*{3.9em}  =$$\text{IPDHG}(\bx_{t}, \by_{t}, D^{\bx}_{t}, D^{\by}_{t}
		 , H^{\bx}_{t}, H^{\by}_{t}, H^{\ssw,\bx}_{t},$ \newline 
		\hspace*{6.9em} $H^{\ssw,\by}_{t} , 
		s, \gamma_x, \gamma_y, \alpha_x, \alpha_{y}, \mathcal{G}_t )$
		\ENDIF
		\ENDFOR
		\STATE {\textbf{RETURN:}} $\bx_{T}, \by_{T}$.
	\end{algorithmic}
\end{algorithm}

\subsection{Determining Switching Point}

As discussed earlier, the switching point $T_0$ depends on a hyperparameter $\epsilon_0$. Using Lemma \ref{lem:exp_phit+1_phi_0} and the choice of $T_0$, it is clear that the upper bound on $E_0[ \Phi_{T_0} ]$ increases linearly with $\epsilon_0$. However, the effect of $\epsilon_0$ on the convergence behavior of Algorithm \ref{alg:IPDHG_with_sgd_svrg_oracle} is not clear after GSGO is switched to SVRGO (i.e. $t \geq T_0$). We investigate this effect in Lemma \ref{lem:tilde_phi_t_ub_T0_t}, which further paves the way for determining a suitable value of $\epsilon_0$.
\begin{lemma} \label{lem:tilde_phi_t_ub_T0_t} Let  $\{\bx_{t} \}_t, \{\by_{t}\}_t$ be the sequences generated by Algorithm \ref{alg:IPDHG_with_sgd_svrg_oracle}. Suppose Assumptions \ref{s_convexity_assumption}-\ref{lipschitz_yx_svrg_main} hold. Then for any $ T \geq T_0 + 1 $:
\begin{align}
E[\tilde{\Phi}_T] & \leq C_{\max} \Big( \frac{\epsilon_0\Phi_0}{\rho^{T_0}} \rho^{T}  + \frac{V_e \rho^{T}}{\epsilon_0} \Big)  + \frac{C_1\rho^{T}}{\epsilon_0} \label{eq:phi_tilde_rho_T0},
\end{align}
where $\tilde{\Phi}_T$ denotes the distance of the iterates $\bx_T, \tilde{\bx}_T, \by_T,  \tilde{\by}_T, D^\bx_T, D^\by_T, H^\bx_T, H^\by_T$ from their respective limit points (described in eq. \eqref{tilde_Phi_t} in \cite{cdctechnicalreport}), $V_e = \frac{2s_0^2(C_x + C_y)}{(1-\rho_0)n^2p_{\min}}$ and $C_{\max}$, $C_1$, $\rho \in (0,1)$ are problem dependent constant parameters defined in equations \eqref{cons_Cmax}, \eqref{cons_C1} and \eqref{rho_svrg} \cite{cdctechnicalreport}.
\end{lemma}
Lemma \ref{lem:tilde_phi_t_ub_T0_t} which is proven in Appendix \ref{switching_point_proofs} \cite{cdctechnicalreport}, leads to an upper bound on $E[\tilde{\Phi}_T]$ when GSGO switches to SVRGO at $T_0$ in Algorithm \ref{alg:IPDHG_with_sgd_svrg_oracle}. This upper bound is small for sufficiently large $T$ and hence iterates $x^i_T, y^i_T$ are also close to saddle point solution $x^\star, y^\star$ in expectation according to the definition of $\tilde{\Phi}_T$. 
We recall that $T_0$ is chosen such that a sufficient progress is obtained from initial value $\Phi_0$.  Now using the facts that $T \geq T_0 +1$, $\rho \in (0,1)$,  \eqref{eq:phi_tilde_rho_T0} reduces to 
\begin{align}
	E[\tilde{\Phi}_T] & \leq C_{\max} ( \epsilon_0 \Phi_{0}  + \frac{V_e \rho^{T}}{\epsilon_0} )  + \frac{C_1\rho^{T}}{\epsilon_0} \label{eq:phi_tilde_rho_epsilon0}.
\end{align}
 Interestingly, the expected value of $\tilde{\Phi}_T$ in \eqref{eq:phi_tilde_rho_epsilon0} is upper bounded by an $\epsilon_0$-dependent quantity which attains its minimum value at $\epsilon^\star_0 = \sqrt{\frac{(C_{\max} V_e + C_1)\rho^T}{C_{\max}\Phi_0}}$, where $T$ is the total number of iterations used in Algorithm \ref{alg:IPDHG_with_sgd_svrg_oracle}. One natural way is to set $T$ to be the total number of iterations  $T(\epsilon)$ required to achieve an $\epsilon$-accurate saddle point solution. 


\noindent \textbf{Computing} $T(\epsilon)$: By substituting $\epsilon_0 = \epsilon^\star_0$ in \eqref{eq:phi_tilde_rho_epsilon0}, we get $E[ \tilde{\Phi}_{T} ] \leq  2\sqrt{C_{\max}\Phi_0(C_{\max} V_e + C_1)\rho^T}$.
As a consequence, after $ T(\epsilon) = \frac{2}{-\log \rho}\log (\frac{2\sqrt{C_{\max}\Phi_0(C_{\max} V_e + C_1)}}{\epsilon})$ iterations,  Algorithm \ref{alg:IPDHG_with_sgd_svrg_oracle} returns an $\epsilon$-accurate saddle point solution in expectation. 
Therefore, we get $\epsilon^\star_0 = \frac{\epsilon}{2C_{\max}\Phi_0}$, and hence $T_0 = \lceil \frac{1}{\log \rho_0}\log(\frac{\epsilon}{2C_{\max}\Phi_0}) \rceil$. We see that the value of $T_0$ depends on 
\begin{align}
\Phi_{0} &= M_{x,0} \Vert \bx_{0} - (\mathbf{1} \otimes I_{d_x})x^\star  \Vert^2 + M_{y,0} \Vert \by_{0} - (\mathbf{1} \otimes I_{d_y})y^\star  \Vert^2   \nonumber \\ 
  & \ \ + \frac{2s_0^2}{\gamma_{x,0}}  \Vert ((I-J)\otimes I_{d_x})\nabla_x F(\bz^\star)  \Vert^2_{(I-W)^\dagger} \nonumber \\  
& \ \ + \frac{2s_0^2}{\gamma_{y,0}}  \Vert ((I-J)\otimes I_{d_y})\nabla_y F(\bz^\star)  \Vert^2_{(I-W)^\dagger} \nonumber \\ 
& \ \ + \sqrt{\delta}\Vert \bx_{0} - (\mathbf{1} \otimes I_{d_x})x^\star + \frac{s_0}{m}(\mathbf{1} \otimes I_{d_x} )\nabla_x f(z^\star) \Vert^2  \nonumber \\ 
 & + \sqrt{\delta}\Vert \by_{0} -  (\mathbf{1} \otimes I_{d_y})y^\star - \frac{s_0}{m}(\mathbf{1} \otimes I_{d_y})\nabla_y f(z^\star)\Vert^2.\label{phi_0}
\end{align}
It is clear that computing $\Phi_0$ requires knowledge of the saddle point solution $z^\star$ which is not available in practice. We also emphasize that $\Phi_{0}$ is a global quantity as it depends on the full gradient information of $f(x,y)$ which is inaccessible to the nodes. We address this issue by proposing a practical version of Algorithm \ref{alg:IPDHG_with_sgd_svrg_oracle} which approximates $\Phi_0$ using local information and without the knowledge of $z^\star$.

\subsection{Practical Approach for Determining Switching Point}

To determine the switching point in Algorithm \ref{alg:IPDHG_with_sgd_svrg_oracle}, we discuss a practical scheme whose broad idea is illustrated in Figure \ref{fig:heuristic_scheme_depiction}. This scheme consists of the following steps. We first allow IPDHG with GSGO to perform $T_0^{'} = \lceil \frac{\log 2}{-\log \rho_0} \rceil$ iterations to obtain primal iterate $\bx_{T^{'}_0}$ and dual iterate $\by_{T_0^{'}}$. These primal dual iterates might saturate to a point in $T_0^{'}$ iterations due to the nature of GSGO. To detect this behavior, each node $i$ computes the distance between the last two successive iterates and gets an approximation of the average quantity $m^{-1}\sum_{i = 1}^m\Vert z^i_{T_0^{'}} - z^i_{T_0^{'}-1}\Vert^2$ using the accelerated gossip scheme \cite{liu2011accelerated}. 
If the average distance is less than a suitable small threshold for atleast one node, then each node $i$ switches to SVRGO. If a fraction of nodes find the average distance to be within the threshold value, then this information can be spread to the entire network in maximum number of hops of order $\mathcal{O}(m)$. On the other hand, if the average distance is above threshold value, then we consider $\bz_{T^{'}_0}$ as a proxy of $\bz^\star$
to compute approximation of $\Phi_0$. Using this procedure, the last two terms of $\Phi_{0}$ depend on global gradients $\frac{1}{m}\sum_{i = 1}^m \nabla_x f_i(x^i_{T^{'}_0}, y^i_{T^{'}_0})$ and $\frac{1}{m}\sum_{i = 1}^m \nabla_y f_i(x^i_{T^{'}_0}, y^i_{T^{'}_0})$. Each node can now run accelerated gossip scheme on local gradients $\nabla_x f_i(x^i_{T^{'}_0}, y^i_{T^{'}_0})$ and $\nabla_y f_i(x^i_{T^{'}_0}, y^i_{T^{'}_0})$ to achieve approximations $\tilde{\mathcal{G}}^{i,x}_{T_0^{'}}$ and $\tilde{\mathcal{G}}^{i,y}_{T_0^{'}}$ of the global gradients. Finally, each node $i$ approximates $\Phi_0$ by average quantity $\bar{\Phi}^i_0(T'_0)$ where $\{\bar{\Phi}^i_0(T'_0)\}_{i=1}^m$ are obtained using accelerated gossip scheme on local scalar values $\Phi^i_0(T'_0)$ given by:
\begin{align}
 \Phi^i_0(T'_0) &= M_{x,0} \Vert x_{0} - x^i_{T_0^{'}} \Vert^2 + M_{y,0} \Vert y_{0}  - y^i_{T_0^{'}}  \Vert^2   \nonumber  \\
	& + \frac{2s_0^2}{\gamma_{x,0}}  \Vert \nabla_x f_i(z^i_{T_0^{'}}) -  \tilde{\mathcal{G}}^{i,x}_{T_0^{'}} \Vert^2 \lambda_{\max}(I-W)^\dagger \nonumber \\ & + \frac{2s_0^2}{\gamma_{y,0}}  \Vert \nabla_y f_i(z^i_{T_0^{'}}) -  \tilde{\mathcal{G}}^{i,y}_{T_0^{'}} \Vert^2 \lambda_{\max}(I-W)^\dagger  \nonumber \\
  & \hspace*{-1cm} + \sqrt{\delta}\Vert x_{0} - x^i_{T_0^{'}}  + s_0\tilde{\mathcal{G}}^{i,x}_{T_0^{'}} \Vert^2  + \sqrt{\delta}\Vert y_{0} -  y^i_{T_0^{'}} - s_0\tilde{\mathcal{G}}^{i,y}_{T_0^{'}} \Vert^2. \nonumber
\end{align}
Finally, the above process yields the approximated value of switching point as $T^i_0  = \lceil \frac{\log\bar{\epsilon}^i_0}{\log\rho_0} \rceil = \lceil \frac{1}{\log \rho_0}\log(\frac{\epsilon}{2C_{\max}\bar{\Phi}^i_0(T'_0)}) \rceil$. We note that values $\bar{\Phi}^i_0(T'_0)$ are close to each other because they are the outputs of gossip scheme and hence values $T_0^i$ are also similar for all nodes. A concise form of this entire procedure is formally demonstrated in Algorithm \ref{alg:detecting_switching_point}. It is worth noting that the above practical approach to detect switching point invokes gossip scheme thrice out of which two gossips are performed only on scalar values. 
\begin{figure}[!hbp]
	\begin{minipage}{.23\textwidth}
	\centering
	\includegraphics[width=1\linewidth]{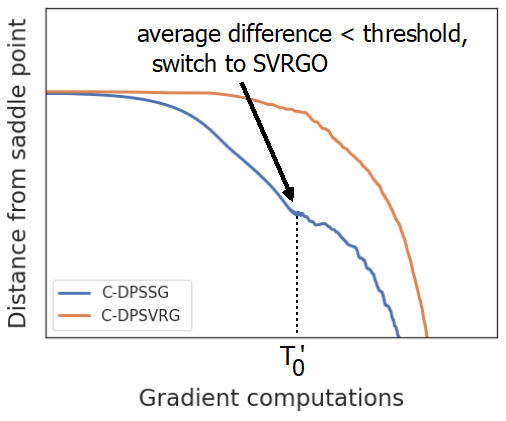}  
			\end{minipage}
		\begin{minipage}{.23\textwidth}
			\centering
			\includegraphics[width=1\linewidth]{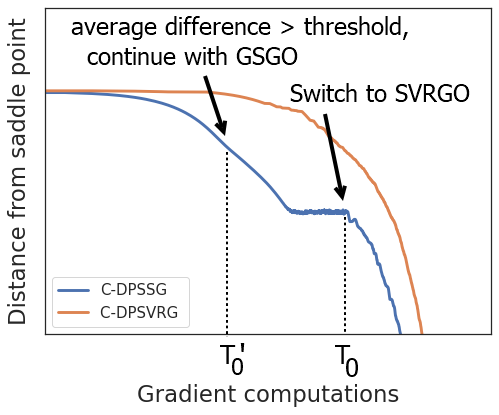}  
		\end{minipage}
\caption{ Broad illustration of proposed practical method to determine switching point. Average difference = $m^{-1}\Vert \bz_{T_0^{'}} - \bz_{T_0^{'}-1} \Vert^2$}	
\label{fig:heuristic_scheme_depiction}
\end{figure}
\begin{algorithm}[ht!]\footnotesize
	\caption{Practical way of determining switching point}
	\label{alg:detecting_switching_point}
	\begin{algorithmic}[1]
\STATE Implement $T'_0 = \lceil \frac{\log2}{-\log\rho_0} \rceil$ iterations of IPDHG with GSGO and obtain $\bz_{T_0^{'}}$, $\bz_{T_0^{'}-1} $
\STATE  Each node $i$ gets approximated value $\hat{z}^i_{T'_0}$ of $m^{-1}\Vert \bz_{T_0^{'}} - \bz_{T_0^{'}-1} \Vert^2$ by invoking accelerated gossip \cite{liu2011accelerated} on $\{ \Vert z^i_{T'_0} - z^i_{T'_0-1} \Vert^2 \}_{i=1}^m$
\IF { $\hat{z}^i_{T'_0} >$ threshold for every node $i$  }
\STATE  Each node $i$ obtains approximation $\bar{\Phi}^i_0(T'_0)$ of global quantity $\Phi_{0}$ by invoking accelerated gossip on $\{ \Phi^i_0(T'_0) \}_{i=1}^m$ and computes $T^i_0$
\STATE Each node $i$ continues with GSGO for remaining $T^i_0 - T'_0$ iterations
\ELSE
\STATE Switch to SVRGO and continue using SVRGO
\ENDIF
\end{algorithmic}
\end{algorithm}		

\textbf{Approximation quality of } $T_0$: From previous discussion, we have the approximated value $T^i_0 = \lceil \frac{\log\bar{\epsilon}^i_0}{\log \rho_0} \rceil$. 
Whenever $\bar{\epsilon}^i_0 \approx \epsilon^\star_0$, $T^i_0$ is a good approximation of $T_0$. 
We observe in our empirical study that the values of $\bar{\epsilon}^i_0$ and $\epsilon^\star_0$ are close to each other which in turn implies that $T^i_0$ and $T_0$ are also close (see Table \ref{empirical_epsilon0_mainpaper} in Section \ref{sec:experiments}).

\subsection{Complexity of Algorithm \ref{alg:IPDHG_with_sgd_svrg_oracle}}

We present the iteration complexity of Algorithm \ref{alg:IPDHG_with_sgd_svrg_oracle} in Theorem \ref{thm:switching_scheme_complexity} below.
\begin{theorem} \label{thm:switching_scheme_complexity}
	Let $\{\bx_{t} \}_t, \{\by_{t}\}_t$  be the sequences generated by Algorithm \ref{alg:IPDHG_with_sgd_svrg_oracle}. Suppose Assumptions \ref{s_convexity_assumption}-\ref{lipschitz_yx_svrg_main} hold. Then  iteration complexity of Algorithm \ref{alg:IPDHG_with_sgd_svrg_oracle} for achieving $\epsilon$-accurate saddle point solution in expectation is
	\begin{align}
		& \mathcal{O}  ( \max \{  \frac{\sqrt{\delta}(1+\delta)\kappa_g \kappa_f^2}{np_{\min}} ,  (1+\delta)\kappa_g, \frac{(1+\delta)\kappa_f^2}{np_{\min}} , \frac{2}{p}   \} \nonumber \\ & \hspace*{2em} \times \log (\frac{2\sqrt{C_{\max}\Phi_0(C_{\max} V_e + C_1)}}{\epsilon} ) ).
	\end{align}
\end{theorem}
The proof of Theorem \ref{thm:switching_scheme_complexity} requires significant technical background to be developed. Unfortunately, due to space constraints, we are unable to discuss relevant background details here and hence provide the proof in Appendix \ref{switching_point_proofs} of technical report \cite{cdctechnicalreport}. 
Theorem \ref{thm:switching_scheme_complexity} indicates that Algorithm \ref{alg:IPDHG_with_sgd_svrg_oracle} converges to $\epsilon$-accurate saddle point solution with linear rate. Further, the complexity in Theorem \ref{thm:switching_scheme_complexity} depends on compression factor as $\mathcal{O}(\max\{\sqrt{\delta}(1+\delta),1+\delta \})$. Without compression ($\delta = 0$), the iteration complexity reduces to $\mathcal{O}(\max \{\kappa_g, \frac{\kappa_f^2}{np_{\min}}, \frac{2}{p} \} \log (\frac{2\sqrt{C_{\max}\Phi_0(C_{\max} V_e + C_1)}}{\epsilon} ) )$.
The communication complexity of Algorithm \ref{alg:IPDHG_with_sgd_svrg_oracle} has a term similar to iteration complexity along with an additional number of communications required in gossip scheme at $T_0^{'}$-th iteration. In terms of gradient computations, Algorithm \ref{alg:IPDHG_with_sgd_svrg_oracle} requires $(2B+pN_\ell)T(\epsilon) - (B+pN_\ell)T_0$ gradient computations.

\section{Related Work}
\label{sec:related_work}

 A distributed saddle point algorithm with Laplacian averaging (DSPAwLA) is proposed in \cite{nunezcortes2017distsaddlesubgrad} to solve non-smooth convex-concave saddle point problems. 
 An extragradient method with gradient tracking (GT-EG) \cite{mukherjeecharaborty2020decentralizedsaddle} is shown to converge with linear rates for strongly convex-strongly concave problems, under a positive lower bound assumption on the gradient difference norm.
  A distributed Min-Max data similarity (MMDS) algorithm under a suitable data similarity assumption is proposed in \cite{beznosikovetal2020distsaddle} which requires solving an inner saddle point problem at every iteration. Another work \cite{qureshi2022distributed} has designed algorithms for smooth saddle point problems with bilinear structure. We emphasize that \cite{nunezcortes2017distsaddlesubgrad,mukherjeecharaborty2020decentralizedsaddle,beznosikovetal2020distsaddle,qureshi2022distributed} are based on non-compressed communications and full batch gradient computations which limit their applicability to large scale problems.
 
 Multiple works \cite{liu2020decentralized,beznosikov2020distributed,xianetal2021fasterdecentnoncvxsp,kovalev2022optimal,chen2022simple,gao2022decentralized} have developed algorithms using stochastic gradients, albeit \textit{without compression} for solving decentralized saddle point problems. Note that \cite{sharma2022stochastic} has designed two different compression based algorithms using GSGO's variant and SVRGO respectively for general stochastic setting and finite sum setting. 
  Figure \ref{fig:contribution_summary} below helps positioning our work in context of existing methods. 
 \begin{figure}[!htbp]
 	\centering
 	\includegraphics[width=1\linewidth]{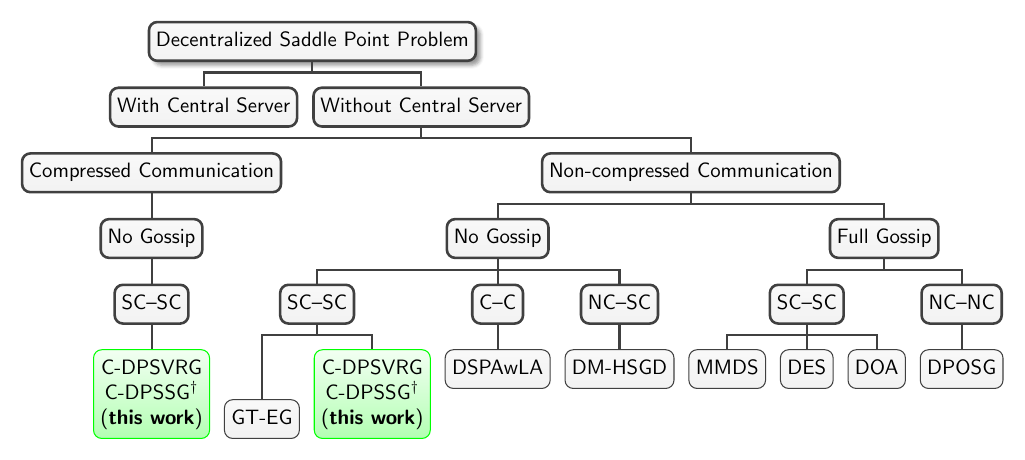}  
 	\vspace{-0.3in}
 	\caption{Our work in comparison to existing art. GT-EG: \cite{mukherjeecharaborty2020decentralizedsaddle}, DSPAwLA: \cite{nunezcortes2017distsaddlesubgrad}, DMHSGD: \cite{xianetal2021fasterdecentnoncvxsp}, MMDS: \cite{beznosikovetal2020distsaddle}, DES: \cite{beznosikov2020distributed}, DOA: \cite{kovalev2022optimal}, DPOSG: \cite{liu2020decentralized}. C-C, SC-SC, NC-NC, NC-SC denote respectively convex-concave, strongly convex-strongly concave, nonconvex-nonconcave, nonconvex-strongly concave. $^\dagger$: C-DPSSG uses gossip only for deciding switching point, not for iterate updates.}	
 	\label{fig:contribution_summary}
 \end{figure}
 \vspace*{-0.5cm}

\section{Numerical Experiments}
\label{sec:experiments}
We investigate\footnote{All codes are available at \url{https://github.com/chhavisharma123/C-DPSSG-CDC2023}} the performance of proposed algorithms on robust logistic regression and AUC maximization. We rely on binary classification datasets a4a and ijcnn1 from \url{https://www.csie.ntu.edu.tw/~cjlin/libsvmtools/datasets/}. 
We consider a 2d torus topology of $20$ nodes in all our experiments. Additional experiments for ijcnn1, phishing and sido data are presented in our technical report \cite{cdctechnicalreport} (see Appendices \ref{appendix_experiments}- \ref{appendix_auc_max}). The performance of proposed algorithms on ring topology, convergence behavior with number of nodes and bits used for compression are also presented in \cite{cdctechnicalreport}. We use an unbiased $b$-bits quantization operator $Q_{\infty}(\cdot)$ \cite{liu2020linear} in all the experiments. 
\subsection{Robust Logistic Regression }
We consider robust logistic regression problem
{\footnotesize
\begin{align}
&\min_{ x \in \mathcal{X}} \max_{y \in \mathcal{Y}} \frac{1}{N} \sum_{i = 1}^N \log( 1+ \exp(-b_ix^\top(a_i + y)) ) + \frac{\lambda}{2}\Vert x \Vert^2_2 -\frac{\beta}{2} \Vert y \Vert^2_2, \label{main_paper_robust_logistic_regression}
\end{align}
}
over a binary classification data set $\mathcal{D} = \{(a_i, b_i) \}_{i = 1}^N$.  The constraint sets $\mathcal{X}$ and $\mathcal{Y}$ are $\ell_2$ balls of radius $100$ and $1$ respectively.  We set number of bits $b = 4$ in quantization operator $Q_\infty(x)$ and $\lambda = \beta = 10$.

\textbf{Switching Point:} For C-DPSSG, we take threshold value to be $10^{-8}$ and implement 20 iterations of accelerated gossip to decide the switch to SVRGO. It turns out that C-DPSSG switches to SVRGO after performing $T_0^{'}$ iterations with GSGO because the gap between two consecutive iterates gets saturated in these many iterations. 

\textbf{Observations:} Switching method C-DPSSG converges faster than C-DPSVRG and other baseline methods as demonstrated in Figure \ref{fig:logistic_regression_AUC_val}. DPOSG and DM-HSGD converge only to a neighborhood of the saddle point solution and start oscillating after a number of iterations. Note that MMDS has poor performance because it is based on full batch gradient computations and multiple calls of gossip scheme at every iterate. 

\subsection{AUC maximization}

We evaluate the effectiveness of proposed algorithms on area under receiver operating characteristic curve (AUC) maximization \cite{ying2016stochastic} formulated as:
\begin{align}
\min_{x,u,v} \max_y \frac{1}{N}\sum_{i = 1}^N F(x,u,v,y;a_i,b_i) + \frac{\lambda}{2} \left\Vert x \right\Vert^2_2 , \label{auc_max}
\end{align}
where $F(x,u,v,y;a_i,b_i)$$=$$(1-q)$$(a_i^\top x - u)^2 \delta_{\left[b_i = 1 \right] }+ q(a_i^\top x - v)^2 \delta_{\left[b_i = -1 \right] }-q(1-q)y^2$$ + 2(1+y)\left( qa_i^\top x \delta_{\left[b_i = -1 \right] }- (1-q)a_i^\top x \delta_{\left[b_i = 1 \right] } \right)$, the fraction of positive samples is given by $q$.
We set $\lambda = 10^{-5}$ in \eqref{auc_max} and consider constraint sets as $\ell_2$ ball of radius 100 and 200 respectively on primal and dual variables.

\textbf{Observations:} We observe that C-DPSSG switches to SVRGO after $T^i_0$ iterations. The AUC plots on training set in Figure \ref{fig:logistic_regression_AUC_val} show that   
C-DPSSG achieves higher AUC value faster in terms of gradient computations, communications and bits transmitted. These observations suggest that switching scheme is beneficial over purely SVRGO based scheme for obtaining high AUC value as it saves time and gradient computations in the crucial early stage. 

\vspace*{-0.25cm}

{\scriptsize
\begin{table}[h]
	\begin{center}
		\setlength\extrarowheight{-1pt}
		\setlength{\tabcolsep}{0.4mm}
		\begin{tabular}{|c|c|c|c|c|c|c|c|}
			\hline
			Data & $\Phi_{0}$ & $ \bar{\Phi}^i_0(T'_0)$ & $\epsilon^\star_0 $ & $\bar{\epsilon}^i_0$ & $T_0$ & $T^i_0$ & $T_0^{'}$   \\ \hline
			a4a & $47.4$ & $8.6$ & $5.3 \times 10^{-11}$ & $2.9 \times 10^{-10}$ & $195315$ & $181248$ & $5720$  \\ \hline
			ijcnn1 & $23.8$ & $3.9$ & $ 10^{-10}$ & $6.3 \times 10^{-10}$ & $50947$ & $46960$ & $1536$ \\ \hline
		\end{tabular}
		\caption{Values of $\bar{\Phi}^i_0(T'_0)$, $\bar{\epsilon}^i_0$ $T^i_0$ (observed same for all nodes) obtained from Algorithm \eqref{alg:detecting_switching_point} and $ \epsilon^\star_0, T_0, T'_0$  on AUC maximization.} \label{empirical_epsilon0_mainpaper}
	\end{center}
\end{table}
}
\vspace*{-0.85cm}

Table \ref{empirical_epsilon0_mainpaper} reports the true values $\Phi_{0}, \epsilon_0^\star, T_0$ and the approximate values $ \bar{\Phi}^i_0(T'_0), \bar{\epsilon}^i_0, T_0^i$ in AUC maximization. We notice that $\epsilon_0^\star$ and $\bar{\epsilon}^i_0$ are close to each other and hence the difference between $T_0$ and $T^i_0$ is also small. 

{\footnotesize
\begin{figure}[!htbp]
	\begin{minipage}{0.25\textwidth}
		\centering
		\includegraphics[scale=0.25]{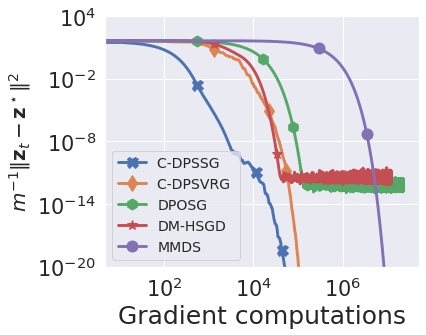}  
	\end{minipage}%
\begin{minipage}{0.25\textwidth}
\centering
\includegraphics[scale=0.2]{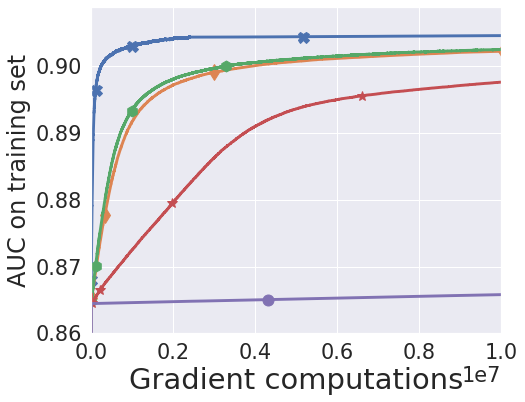}  
\end{minipage}
	\begin{minipage}{0.25\textwidth}
		\centering
		\includegraphics[scale=0.25]{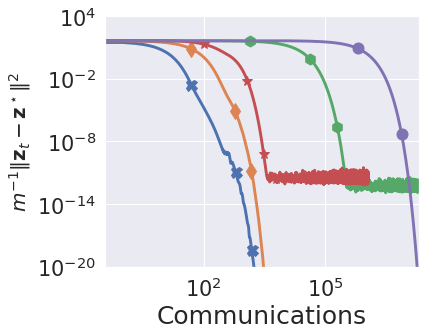}  
	\end{minipage}%
\begin{minipage}{0.25\textwidth}
	\centering
	\includegraphics[scale=0.2]{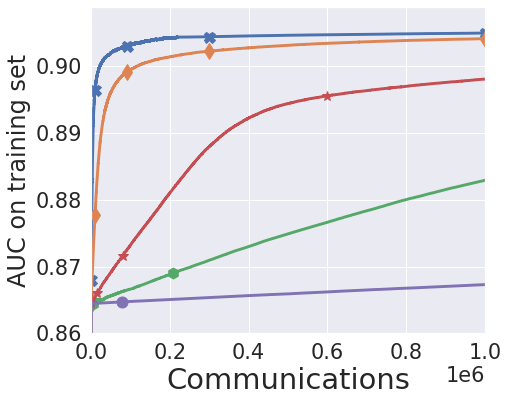}  
\end{minipage}
	\begin{minipage}{0.25\textwidth}
		\centering
		\includegraphics[scale=0.25]{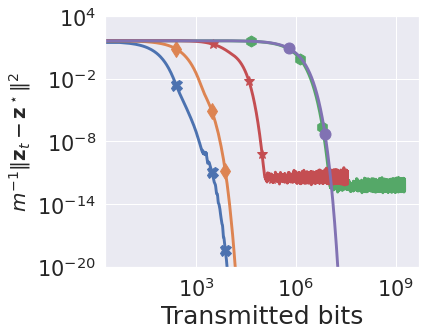}  
	\end{minipage}%
	\begin{minipage}{0.25\textwidth}
	\centering
	\includegraphics[scale=0.2]{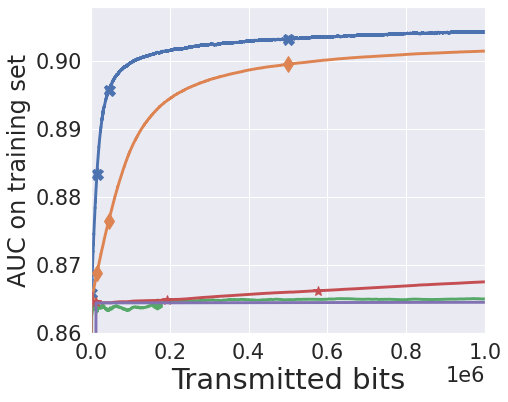}  
\end{minipage}
\vspace{-0.2cm}
	\caption{Left panel: Convergence behavior of iterates to saddle point for robust logistic regression. Right panel: AUC value on training set for AUC maximization problem.} \label{fig:logistic_regression_AUC_val}
\end{figure}
}

\vspace*{-0.25cm}
\section{Conclusion}
\label{sec:conclusion}
This work presents C-DPSSG, a technique that leverages the best phases of GSGO and SVRGO in a decentralized setting with compression, by performing a switch between them. The proposed algorithm offers practical advantages for efficiently obtaining low, medium and highly accurate solutions. 
%
Adapting the algorithm to cases where some constants are unknown in the problem setup would be an interesting direction to explore in future. 

\vspace*{-0.2cm}

%
%
	
	
%
%
%
%
	

\bibliographystyle{plain} 
\bibliography{references}
\onecolumn
\begin{center} 
	\huge{\textbf{Appendix}}
\end{center}
\tableofcontents
\newpage

\section{Related Work}
\label{app:related_work}

 A distributed saddle point algorithm with Laplacian averaging (DSPAwLA) in \cite{nunezcortes2017distsaddlesubgrad}, based on gradient descent ascent updates to solve non-smooth convex-concave saddle point problems converges with sublinear rate. DSPAwLA is designed by employing deterministic gradient descent ascent to a penalized form ($\ell_2$ norm of consensus constraints gets added to objective function) of \eqref{eq:main_opt_consenus_constraint} with a possibly different weight matrix. These types of penalties require diminishing step size to achieve a consensus point which slows down the speed of the algorithm. However, our work derives an equivalent Lagrangian formulation of consensus constrained saddle point problem \eqref{eq:main_opt_consenus_constraint} and updates primal-dual variables using IPDHG updates. An extragradient method with gradient tracking (GT-EG) proposed in \cite{mukherjeecharaborty2020decentralizedsaddle} is shown to converge with linear rates for strongly convex-strongly concave problems, under a positive lower bound assumption on the gradient difference norm. However, such assumptions might not hold for problems without bilinear structure. \red{Both \cite{mukherjeecharaborty2020decentralizedsaddle} and \cite{nunezcortes2017distsaddlesubgrad} are based on non-compressed communications and full batch gradient computations which limit their applicability to large scale problems.}

Recently, multiple works \cite{liu2020decentralized,xianetal2021fasterdecentnoncvxsp, beznosikov2020distributed} have proposed algorithms using minibatch gradients for solving decentralized saddle point problems. 
Decentralized extra step (DES) \cite{beznosikov2020distributed} shows linear  communication complexity with dependence on the graph condition number as $\sqrt{\kappa_g}$,   obtained at the cost of incorporating multiple rounds of communication of primal and dual updates. A near optimal distributed Min-Max data similarity (MMDS) algorithm is proposed in \cite{beznosikovetal2020distsaddle} under a suitable data similarity assumption. MMDS is based on full batch gradient computations and requires solving an inner saddle point problem at every iteration. MMDS allows communication efficiency by choosing only one node uniformly at random to update the iterates. However, every node computes the full batch gradient before heading to the next gradient based updates. Moreover, this scheme employs accelerated gossip \cite{liu2011accelerated} multiple times to propagate the gradients and model updates to the entire network. 
 
Decentralized parallel optimistic stochastic gradient method (DPOSG) was proposed in \cite{liu2020decentralized} for nonconvex-nonconcave saddle point problems. This method involves local model averaging step (multiple communication rounds) to reduce the effect of consensus error. A gradient tracking based algorithm called DM-HSGD for solving nonconvex-strongly concave saddle point problems proposed in \cite{xianetal2021fasterdecentnoncvxsp}, uses a large mini-batch at the first iteration and requires the nodes to communicate both model and gradient updates, to achieve better aggregates of quantities. Variance reduction based optimal methods without compression to solve strongly convex-strongly concave non-smooth finite sum variational inequalities are developed in \cite{kovalev2022optimal}. The improvement of complexity on graph condition number $\kappa_g$ in \cite{kovalev2022optimal} is achieved using an accelerated gossip scheme. However, C-DPSVRG does not involve any gossip scheme, and C-DPSSG involves gossip only to compute switching point; hence, both methods yield low communication cost per iterate. Figure \ref{fig:contribution_summary} and Table \ref{ratetable} help positioning our work in the context of existing methods. 


\setlength\arrayrulewidth{1pt}
\begin{table}[ht!]
	\setlength{\tabcolsep}{1pt}
		\centering
		\begin{tabular}{|cc|c|c|c|c|c|c|}
			\hline
			\multicolumn{2}{|l|}{}  & \textbf{Algorithm} &  \textbf{SG} &  \textbf{Non-} &  \textbf{Type of } & \textbf{Computation} & \textbf{Communication}  \\
			\multicolumn{2}{|l|}{} &  &  &   \textbf{smooth} &  \textbf{functions} & \textbf{Complexity} & \textbf{Complexity}  \\ \hhline{|-|-|-|-|-|-|-|-|}
			\multicolumn{1}{|c|}{\multirow{12}{*}{ \rotatebox{90}{\textbf{No compression}} }} & \multicolumn{1}{c|}{\multirow{4}{*}{\rotatebox{90}{\textbf{Gossip}}}} &  MMDS \cite{beznosikovetal2020distsaddle}  &  \xmark & \xmark & \cellcolor[HTML]{FFCB2F} SC-SC &  -- &  $\tilde{\mathcal{O}}( \log^2(\frac{1}{\epsilon}) )$  \\ \hhline{|~|~|-|-|-|-|-|-|}
			\multicolumn{1}{|l|}{} &  &  DES \cite{beznosikov2020distributed} & \cmark  &  \xmark & \cellcolor[HTML]{FFCB2F} SC-SC & $\tilde{\mathcal{O}}( \frac{\kappa_f^2}{L^2\epsilon} )$ & $\tilde{\mathcal{O}}( \kappa_f \sqrt{\kappa}_g \log(1/\epsilon ) )$  \\ \hhline{|~|~|-|-|-|-|-|-|}
			\multicolumn{1}{|l|}{} &  & DOA \cite{kovalev2022optimal}  & \cmark & \cmark & \cellcolor[HTML]{FFCB2F} SC-SC  & $\mathcal{O}\left(\kappa_f \log(1/\epsilon ) \right)$  & $\mathcal{O}\left(\kappa_f \sqrt{\kappa_g} \log(1/\epsilon ) \right)$ \\ \cline{3-8}
			\multicolumn{1}{|l|}{} &  &  DPOSG \cite{liu2020decentralized} & \cmark & \xmark & NC-NC & $\mathcal{O}(\frac{1}{\epsilon^{12}})$ & $\tilde{\mathcal{O}}( \frac{1}{\epsilon^{12}} )$ \\ \hhline{|~|-|-|-|-|-|-|-|} 
			\multicolumn{1}{|l|}{} & \multirow{8}{*}{ \rotatebox{90}{\textbf{No Gossip} } } &  GT-EG \cite{mukherjeecharaborty2020decentralizedsaddle} & \xmark & \xmark & \cellcolor[HTML]{FFCB2F} SC-SC & $\mathcal{O}\left(\kappa_f^{4/3} \kappa_g^{4/3} \log\left(\frac{1}{\epsilon} \right)\right)$ & $\mathcal{O}\left(\kappa_f^{4/3} \kappa_g^{4/3} \log\left(\frac{1}{\epsilon} \right)\right)$  \\ \hhline{|~|~|-|-|-|-|-|-|} 
			\multicolumn{1}{|l|}{} &  & DSPAwLA\cite{nunezcortes2017distsaddlesubgrad}  & \xmark & \cmark & C-C & $\mathcal{O}(1/\epsilon^2)$ & $\mathcal{O}(1/\epsilon^2)$  \\ \cline{3-8} 
			\multicolumn{1}{|l|}{} &  &   DMHSGD \cite{xianetal2021fasterdecentnoncvxsp} & \cmark & \xmark & NC-SC & $\mathcal{O}(\frac{\kappa^3}{(1-\lambda_2(W))^2\epsilon^{3}} )$ & $\mathcal{O}(\frac{\kappa^3}{(1-\lambda_2(W))^2\epsilon^{3}} )$  \\ \hhline{|~|~|-|-|-|-|-|-|}
			\multicolumn{1}{|l|}{} &  & $\textcolor{blue}{\text{C-DPSSG}}^\dagger$ \textbf{(Theorem \ref{thm:switching_scheme_complexity})} &  &  & \cellcolor[HTML]{FFCB2F} &  &   \\ 
			\multicolumn{1}{|l|}{} &  & {\color{blue}{C-DPSVRG}} \textbf{(Theorem \ref{thm:svrg_complexity})} & \cmark & \cmark & \cellcolor[HTML]{FFCB2F} SC-SC & $ \mathcal{O}(\max \{\kappa_f^2,\kappa_g \} \log(1/\epsilon)) $ & $ \mathcal{O}(\max \{\kappa_f^2,\kappa_g \} \log(1/\epsilon)) $  \\ 
			\multicolumn{1}{|l|}{} &  &   & & & \cellcolor[HTML]{FFCB2F}   &  &  \\ \hhline{|~|-|-|-|-|-|-|-|} 
			\multicolumn{1}{|c|}{\multirow{5}{*}{ \rotatebox{90}{\textbf{Compression}}}} & \multirow{5}{*}{\rotatebox{90}{\textbf{No Gossip}}} &  $\textcolor{blue}{\text{C-DPSSG}}^\dagger$ \textbf{(Theorem \ref{thm:switching_scheme_complexity})} &  &  & \cellcolor[HTML]{FFCB2F} & &  \\ 
			\multicolumn{1}{|l|}{} &  &   & \cmark & \cmark & \cellcolor[HTML]{FFCB2F} SC-SC & $ \mathcal{O}(\max \{\kappa_f^2, \sqrt{\delta}\kappa^2_f\kappa_g,\kappa_g \} $ &  $ \mathcal{O}(\max \{\kappa_f^2, \sqrt{\delta}\kappa^2_f\kappa_g,\kappa_g \}  $ \\
			\multicolumn{1}{|l|}{} &  &  {\color{blue}{C-DPSVRG}} \textbf{(Theorem \ref{thm:svrg_complexity})}  &  &   & \cellcolor[HTML]{FFCB2F} &  $\times (1+\delta)  \log(1/\epsilon))$  & $\times (1+\delta) \log(1/\epsilon))$  \\ 
			\multicolumn{1}{|l|}{} &  &   & & & \cellcolor[HTML]{FFCB2F} &  &  \\ 
			\multicolumn{1}{|l|}{} &  &   & & & \cellcolor[HTML]{FFCB2F} &  &  \\  \hline
		\end{tabular}
	\caption{Comparison of proposed optimization algorithms for decentralized saddle-point problems with state-of-the-art algorithms. SG denotes Stochastic Gradient. Abbreviations SC-SC, C-C, NC-NC, NC-SC respectively denote Strongly convex-Strongly concave, Convex-Concave, Nonconvex-Nonconcave, Nonconvex-Strongly Concave. $^\dagger$: C-DPSSG uses gossip only for deciding switching point, not for iterate updates. } 	\label{ratetable}
\end{table}
\setlength\arrayrulewidth{0.4pt}

\section{Compression Algorithm of \cite{liu2020linear} }
\label{appendix_compression}

We follow \cite{liu2020linear, mishchenko2019distributed} to compress a related difference vector instead of directly compressing $\nu^\bx_{t+1}$ and $\nu^\by_{t+1}$. We now describe the compression related updates for $\nu^\bx_{t+1}$. Each node $i$ is assumed to maintain a local vector $H^{i,\bx}$ and a stochastic compression operator $Q$ is applied on the difference vector $\nu^{i,\bx}_{t+1}-H^{i,\bx}_t$. Hence the compressed estimate $\hat{\nu}^{i,\bx}_{t+1}$ of $\nu^{i,\bx}_{t+1}$ is obtained by adding the local vector and the compressed difference vector using $\hat{\nu}^{i,\bx}_{t+1} = H^{i,\bx}_t + Q(\nu^{i,\bx}_{t+1}-H^{i,\bx}_t)$. The local vector $H^{i,\bx}_t$ is then updated using a convex combination of the previous local vector information and the new estimate $\hat{\nu}^{i,\bx}_{t+1}$ using $H^{i,\bx}_{t+1} = (1-\alpha) H^{i,\bx}_t + \alpha \hat{\nu}^{i,\bx}_{t+1}$ for a suitable $\alpha \in [0,1]$. Collecting the quantities in individual nodes into $H^\bx_{t+1} = (H^{1,\bx}_{t+1}, \ldots, H^{m,\bx}_{t+1})$ and $\nu^\bx_{t+1} = (\nu^{1,\bx}_{t+1}, \ldots, \nu^{m,\bx}_{t+1})$, the update step can be written as $H^{\bx}_{t+1} = (1-\alpha) H^{\bx}_t + \alpha \hat{\nu}^{\bx}_{t+1}$. Pre-multiplying both sides of $H^{\bx}_{t+1}$ update by $W \otimes I$, and denoting $(W \otimes I)H^{\bx}_{t}$ by $H^{w,\bx}_{t}$, and $(W \otimes I)\hat{\nu}^{\bx}_{t}$ by $\hat{\nu}^{w,\bx}_{t}$ we have: $H^{w,\bx}_{t+1} = (1-\alpha)H^{w,\bx}_{t} + \alpha \hat{\nu}^{w,\bx}_{t+1}$.  $\hat{\nu}^{w,\bx}_{t+1}$ can be further simplified as $\hat{\nu}^{w,\bx}_{t+1} = (W \otimes I)\hat{\nu}^{\bx}_{t+1} = (W \otimes I)(H^\bx_t + Q(\nu^{\bx}_{t+1} - H^{\bx}_t)) = H^{w,\bx}_t + (W \otimes I)Q(\nu^{\bx}_{t+1} - H^{\bx}_t)$. A similar update scheme is used for compressing $\nu^\by_{t+1}$. The entire  procedure is illustrated in Algorithm \ref{comm}, where we have used $\nu_{t+1}= (\nu^\bx_{t+1}, \nu^\by_{t+1})$, $H_{t}= (H^\bx_{t}, H^\by_{t})$, $H^w_{t}= (H^{w,\bx_{t}}, H^{w,\by_{t}})$. Further recall that $\nu^\bx_{t+1} = (\nu^{1,\bx}_{t+1}, \ldots, \nu^{m,\bx}_{t+1})$ denotes the collection of the local variables at $m$ nodes. Similar is the case for the other variables $\nu^\by_{t+1}$, $H^\bx_t$, $H^\by_t$, $H^{w,\bx}_t$, $H^{w,\by}_t$.   

\begin{algorithm}[!h]
	\caption{Compressed Communication Procedure (COMM) \cite{liu2020linear}} 
	\label{comm}
	\begin{algorithmic}[1]
		\STATE{\textbf{INPUT:}} {$\nu_{t+1}, H_t, H^w_t, \alpha$}
		\STATE $Q^i_t = Q(\nu^i_{t+1} - H^i_t)$ \ \ \COMMENT{(compression)}
		\STATE $\hat{\nu}^i_{t+1} = H^i_t + Q^i_t$ , 
		\STATE $H^i_{t+1} = (1-\alpha)H^i_t + \alpha \hat{\nu}^i_{t+1}$ , 
		\STATE $\hat{\nu}^{i,w}_{t+1} = H^{i,w}_t + \sum_{j = 1}^m W_{ij}Q^j_t$, \ \ \COMMENT{(communicating compressed vectors)}
		\STATE $H^{i,w}_{t+1} = (1-\alpha)H^{i,w}_t + \alpha \hat{\nu}^{i,w}_{t+1}$ , 
		\STATE {\textbf{RETURN:}} $\hat{\nu}^{i}_{t+1}, \hat{\nu}^{i,w}_{t+1}, H^i_{t+1}, H^{i,w}_{t+1}$ for each node $i$ .
	\end{algorithmic}
\end{algorithm}

\section{Smoothness Assumptions} \label{smoothness_assumptions}
We make the following smoothness assumptions on $f_{ij}$ to derive theoretical guarantees of proposed algorithms.
\begin{myassumption} \label{smoothness_x_svrg} Assume that each $f_{ij}(x,y)$ is $L_{xx}$ smooth in $x$; for every fixed $y$, \\ $\left\Vert \nabla_x f_{ij}(x_1,y) - \nabla_x f_{ij}(x_2,y) \right\Vert$$\leq$$L_{xx}\left\Vert x_1 - x_2 \right\Vert$, $\forall x_1,x_2$  $\in \mathbb{R}^{d_x}$
\end{myassumption}

\begin{myassumption} \label{smoothness_y_svrg} Assume that each $f_{ij}(x,y)$ is $L_{yy}$ smooth in $y$, i.e. for every fixed $x$, $\left\Vert \nabla_y f_{ij}(x,y_1) - \nabla_y f_{ij}(x,y_2) \right\Vert$$\leq$$L_{yy}\left\Vert y_1 - y_2 \right\Vert$, $\forall  y_1, y_2$  $\in \mathbb{R}^{d_y}$.
\end{myassumption}

\begin{myassumption} \label{lipschitz_xy_svrg} Assume that each $\nabla_x f_{ij}(x,y)$ is $L_{xy}$ Lipschitz in $y$, i.e. for every fixed $x$, $\left\Vert \nabla_x f_{ij}(x,y_1) - \nabla_x f_{ij}(x,y_2) \right\Vert$$\leq$$L_{xy}\left\Vert y_1 - y_2 \right\Vert,$  $\forall  y_1, y_2$  $\in \mathbb{R}^{d_y}$.
	
\end{myassumption}
\begin{myassumption} \label{lipschitz_yx_svrg} Assume that each $\nabla_y f_{ij}(x,y)$ is $L_{yx}$ Lipschitz in $x$, i.e. for every fixed $y$, $\left\Vert \nabla_y f_{ij}(x_1,y) - \nabla_y f_{ij}(x_2,y) \right\Vert$$\leq$$L_{yx}\left\Vert x_1 - x_2 \right\Vert$,  $\forall  x_1, x_2$  $\in \mathbb{R}^{d_x}$.
\end{myassumption}
For future discussion, we denote $L = \max\{L_{xx}, L_{yy}, L_{xy}, L_{yx}\}$.
\section{Basic Results and Inequalities}
\label{appendix_A}

\subsection{Equivalence between problems \eqref{eq:main_opt_consenus_constraint} and \eqref{minmax_lagrange_problem}.}

We first define few notations and prove a supporting result in Proposition \ref{null_sqrt_I_W_range_1}  to show the equivalence between problems \eqref{eq:main_opt_consenus_constraint} and \eqref{minmax_lagrange_problem}. Let $\text{Null}(\sqrt{I-W} \otimes I_{d}) := \{ \bz \in \mathbb{R}^{md} | (\sqrt{I-W} \otimes I_{d})\bz = 0 \}$, $\text{Range}(\mathbf{1} \otimes I_{d} ) := \{ (\mathbf{1} \otimes I_{d} )z | z \in \mathbb{R}^d  \}$, and analogously $\text{Range}(\sqrt{I-W} \otimes I_{d} )$. Further, $U_x := (\sqrt{I-W} \otimes I_{d_x})$ and $U_y := (\sqrt{I-W} \otimes I_{d_y})$, $(\text{Range}(\mathbf{1} \otimes I_{d}))^\bot := \{\bz \in \mathbb{R}^{md} | \langle \bz, \bz_0 \rangle = 0 \ \forall \ \bz_0 \in \text{Range}(\mathbf{1} \otimes I_{d} )  \}$.
\begin{proposition} \label{null_sqrt_I_W_range_1} Let $W$ be a weight matrix satisfying assumption \ref{weight_matrix_assumption}. Then $\text{Null}(\sqrt{I-W} \otimes I_{d}) = \text{Range}(\mathbf{1} \otimes I_{d} )$ and $\text{Range}(\sqrt{I-W} \otimes I_{d}) = (\text{Range}(\mathbf{1} \otimes I_{d}))^\bot$.
\end{proposition}

\begin{proof} Let $\bz \in \text{Range}(\mathbf{1} \otimes I_{d} )$, which implies that there exists a $z \in \mathbb{R}^d$ such that $\bz = (\mathbf{1} \otimes I_{d} )z$. Therefore, 
\begin{align}
& (\sqrt{I-W} \otimes I_{d})\bz = (\sqrt{I-W} \otimes I_{d})(\mathbf{1} \otimes I_{d} )z = (\sqrt{I-W}\mathbf{1} \otimes I_{d})z \overset{(i)}{=} 0,
\end{align}
where $\overset{(i)}{=}$ holds true because $\mathbf{1}$ is an eigenvector of $\sqrt{I-W}$ associated with eigenvalue $0$. It shows that $\bz \in \text{Null}(\sqrt{I-W} \otimes I_{d})$ and hence $\text{Range}(\mathbf{1} \otimes I_{d} ) \subseteq \text{Null}(\sqrt{I-W} \otimes I_{d})$. Next, we show that $ \text{Null}(\sqrt{I-W} \otimes I_{d}) \subseteq \text{Range}(\mathbf{1} \otimes I_{d} )$. To prove this inclusion, let $\bz \in \text{Null}(\sqrt{I-W} \otimes I_{d})$. Then 
\begin{align}
& (\sqrt{I-W} \otimes I_{d})\bz = 0 \Rightarrow ((I-W) \otimes I_{d})\bz = 0 \Rightarrow (W \otimes I_d)\bz = \bz. \label{eigenspace_W_Id}
\end{align}
\eqref{eigenspace_W_Id} shows that $\bz$ belongs to the eigenspace of $(W \otimes I_d)$ corresponding to eigenvalue $1$. Note that eigenvalue 1 of $(W \otimes I_d)$ has algebraic multiplicity $d$. We know that for every $z \in \mathbb{R}^d$, $(\mathbf{1} \otimes I_d)z$ is an eigenvector of $(W \otimes I_d)$ corresponding to eigenvalue $1$. Therefore, we can construct $d$ linearly independent eigenvectors associated with eigenvalue $1$ using basis vectors of $\mathbb{R}^d$. The number of linearly independent eigenvectors can not be greater than $d$ because the algebraic multiplicity of eigenvalue $1$ is $d$. So the eigenspace associated with eigenvalue $1$ is $\text{Range}(\mathbf{1} \otimes I_d)$ and hence $z$ must lies in $\text{Range}(\mathbf{1} \otimes I_d)$. Therefore, $ \text{Null}(\sqrt{I-W} \otimes I_{d}) \subseteq \text{Range}(\mathbf{1} \otimes I_{d} )$, which proves the first part.  

Now we show the second part of Proposition \ref{null_sqrt_I_W_range_1}. Let $\bz \in \text{Range}(\sqrt{I-W} \otimes I_{d})$, which means there exists $z \in \mathbb{R}^d$ such that $\bz = (\sqrt{I-W} \otimes I_{d})z$. Consider
\begin{align}
 \left\langle \bz, (\mathbf{1} \otimes I_{d})z_0 \right\rangle &= \left\langle (\sqrt{I-W} \otimes I_{d})z, (\mathbf{1} \otimes I_{d})z_0 \right\rangle \nonumber \\
& = z^\top (\sqrt{I-W} \otimes I_{d})^\top (\mathbf{1} \otimes I_{d})z_0 \nonumber \\
& = z^\top (\sqrt{I-W} \otimes I_{d}) (\mathbf{1} \otimes I_{d})z_0 \\
& = 0, \nonumber 
\end{align}
where the last step follows from first part of Proposition \ref{null_sqrt_I_W_range_1}. Therefore, $\bz \in (\text{Range}(\mathbf{1} \otimes I_{d}))^\bot$ and hence $\text{Range}(\sqrt{I-W} \otimes I_{d}) \subseteq (\text{Range}(\mathbf{1} \otimes I_{d}))^\bot$. We now show that $\text{dim}(\text{Range}(\sqrt{I-W} \otimes I_{d})) = \text{dim}((\text{Range}(\mathbf{1} \otimes I_{d}))^\bot)$. Since $\text{Null}(\sqrt{I-W} \otimes I_{d}) = \text{Range}(\mathbf{1} \otimes I_{d} )$ and $\text{dim}(\text{Range}(\mathbf{1} \otimes I_{d} )) = d$, therefore $\text{dim}(\text{Null}(\sqrt{I-W} \otimes I_{d})) = d$. Using Rank-Nullity Theorem \cite{friedberg2003linear}, $\text{dim}(\text{Range}(\sqrt{I-W} \otimes I_{d})) = md - d$. Furthermore, $\text{Range}(\mathbf{1} \otimes I_{d})$ is a $d$-dimensional subspace of $\mathbb{R}^{md}$, thus $\text{dim}(\text{Range}(\mathbf{1} \otimes I_{d}))^\bot = md - d$. In conclusion, we have $\text{Range}(\sqrt{I-W} \otimes I_{d}) \subseteq (\text{Range}(\mathbf{1} \otimes I_{d}))^\bot$ and $\text{dim}(\text{Range}(\sqrt{I-W} \otimes I_{d})) = \text{dim}(\text{Range}(\mathbf{1} \otimes I_{d}))^\bot$. Hence, using Theorem $1.11$ in \cite{friedberg2003linear}, we get $\text{Range}(\sqrt{I-W} \otimes I_{d}) = (\text{Range}(\mathbf{1} \otimes I_{d}))^\bot$.
\end{proof}

\begin{theorem} \label{thm:lag_equivalence} Let Assumptions \ref{s_convexity_assumption}-\ref{nonsmooth_assumption} and Assumption  \ref{weight_matrix_assumption} hold. Then $(\bx^\star,\by^\star)$ is saddle point of problem \eqref{eq:main_opt_consenus_constraint} if and only if there exists  $\tilde{S}^\bx,\tilde{S}^\by$ such that $(\bx^\star,\by^\star, \tilde{S}^\bx,\tilde{S}^\by)$ is a saddle point of \eqref{minmax_lagrange_problem}.
	
\end{theorem}

\begin{proof} 
	We first prove the forward direction. Let $(\bx^\star, \by^\star) = ((\mathbf{1} \otimes I_{d_x})x^\star_0,  (\mathbf{1} \otimes I_{d_y})y^\star_0)$ be a saddle point of \eqref{eq:main_opt_consenus_constraint}. Therefore, $ (x^\star_0, y^\star_0)$ is a saddle point of
\begin{align}
\min_{x \in \mathbb{R}^{dx}} \max_{y \in \mathbb{R}^{dy}} \sum_{i = 1}^m (f_i(x,y) +g(x) -r(y) ) .\label{eq:main_opt_consenus_constraint_un_avg}
\end{align}	
Then using the definition of saddle point, we get
\begin{align}
	0 & \in \nabla_x f(x^\star_0, y^\star_0) + m\partial g(x^\star_0)  \\
	0 & \in \nabla_y f(x^\star_0, y^\star_0) - m\partial r(y^\star_0). 
\end{align}
So there exists $\xi^\star_g \in \partial g(x^\star_0)$ and $\xi^\star_r \in \partial r(y^\star_0)$ such that
\begin{align}
& \nabla_x f(x^\star_0, y^\star_0) + m\xi^\star_g = 0 \Rightarrow ((\mathbf{1} \otimes I_{d_x}) x_0)^\top\left( \nabla_x F(\bx^\star,\by^\star) + \xi^\star_x \right) = 0 \label{eq:reverse_dir_grad_x} \\
& \nabla_y f(x^\star_0, y^\star_0) - m\xi^\star_r = 0 \Rightarrow ((\mathbf{1} \otimes I_{d_y}) y_0)^\top\left( \nabla_y F(\bx^\star,\by^\star) - \xi^\star_y \right) = 0, \label{eq:reverse_dir_grad_y}
\end{align}
where $\xi^\star_x = \left[ \xi^\star_g;\ldots;\xi^\star_g \right] \in \mathbb{R}^{md_x}$, $\xi^\star_y = \left[ \xi^\star_r;\ldots;\xi^\star_r \right] \in \mathbb{R}^{md_y}$ and  $x_0 \in \mathbb{R}^{d_x}, y_0 \in \mathbb{R}^{d_y}$ are arbitrary vectors. Using \eqref{eq:reverse_dir_grad_x} and \eqref{eq:reverse_dir_grad_y}, we notice that $\nabla_x F(\bx^\star,\by^\star) + \xi^\star_x \in \left( \text{Range}( \mathbf{1} \otimes I_{d_x}) \right)^\bot$ and $\nabla_y F(\bx^\star,\by^\star) - \xi^\star_y \in \left(\text{Range}( \mathbf{1}^\top \otimes I_{d_y}) \right)^\bot$. From Proposition \ref{null_sqrt_I_W_range_1}, we have $\left( \text{Range}( \mathbf{1} \otimes I_{d_x} ) \right)^\bot = \text{Range}(U_x)$. It implies that there exists $-S^{\star,x} \in \mathbb{R}^{md_x}, -S^{\star,y} \in \mathbb{R}^{md_y}$ such that $ \nabla_x F(\bx^\star,\by^\star) + \xi^\star_x = -U_xS^{\star,x}$ and $ \nabla_y F(\bx^\star,\by^\star) - \xi^\star_y = -U_yS^{\star,y} $. It can be observed that $\xi^\star_x \in \partial G(\bx^\star)$ and $\xi^\star_y \in \partial R(\by^\star)$. By combining the last two arguments, we get 
\begin{align}
0 & \in \nabla_x F(\bx^\star,\by^\star) + \partial G(\bx^\star) + U_x S^{\star,\bx} \nonumber \\
0 & \in \nabla_y F(\bx^\star,\by^\star) - \partial R(\by^\star) + U_y S^{\star,\by} \nonumber .
\end{align}
Since $\bx^\star = (\mathbf{1} \otimes I_{d_x})x^\star_0 \in \text{Range}( \mathbf{1} \otimes I_{d_x} )  $ and $\by^\star = (\mathbf{1} \otimes I_{d_y})y^\star_0 \in \text{Range}( \mathbf{1} \otimes I_{d_y} ) $. Therefore, $U_x \bx^\star = 0$ and $U_y \by^\star = 0$ using Proposition \ref{null_sqrt_I_W_range_1}. In short, we have
\begin{align}
0 & \in \nabla_x F(\bx^\star,\by^\star) + \partial G(\bx^\star) + U_x S^{\star,\bx}, \ \ U_y \by^\star = 0  \ \ (\Leftrightarrow  0 \in \partial_\bx \L(\bx^\star,\by^\star; S^{\star,\bx},S^{\star,\by}), \ \ \nabla_{S^\by} \L(\bx^\star,\by^\star; S^{\star,\bx},S^{\star,\by}) = 0 ),  \label{eq:primal_var_gradient} \\
0 & \in \nabla_y F(\bx^\star,\by^\star) - \partial R(\by^\star) + U_y S^{\star,\by}, \ \ U_x \bx^\star = 0 \ \ (\Leftrightarrow  0 \in \partial_\by \L(\bx^\star,\by^\star; S^{\star,\bx},S^{\star,\by}), \ \ \nabla_{S^\bx} \L(\bx^\star,\by^\star; S^{\star,\bx},S^{\star,\by}) = 0 ). \label{eq:dual_var_gradient} 
\end{align}
Therefore,
\begin{align}
(\bx^\star, S^{\star,\by}) & = \argmin_{\bx,S^{\by}} \L(\bx,\by^\star; S^{\star,\bx},S^{\by}), \ \ (\by^\star, S^{\star,\bx}) = \argmax_{\by,S^{\bx}} \L(\bx^\star,\by; S^{\bx},S^{\star,\by}) .
\end{align}
Next, we prove the reverse direction. Let $(\bx^\star, \by^\star, S^{\star,\bx}, S^{\star,\by})$ be a saddle point of \eqref{minmax_lagrange_problem}. This implies that
\begin{align}
0 & \in \nabla_x F(\bx^\star,\by^\star) + \partial G(\bx^\star) + U_x S^{\star,\bx} \label{eq:primal_var_gradient_rev} \\
0 & \in \nabla_y F(\bx^\star,\by^\star) - \partial R(\by^\star) + U_y S^{\star,\by} \label{eq:dual_var_gradient_rev} \\
& U_x \bx^\star = 0, \ \ U_y \by^\star = 0  \label{eq:consensus_cond}
\end{align}
	Using \eqref{eq:consensus_cond}, we have $U_x \bx^\star = 0$ which shows that $x^\star \in \text{Null}(U_x) = \text{Null}\left( \sqrt{I-W} \otimes I_{d_x} \right) = \text{Range}(\mathbf{1} \otimes I_{d_x})$. Therefore, $\bx^\star_i = \bx^\star_j =: x^\star_0 \in \mathbb{R}^{d_x}$. Similarly, we get $\by^\star_i = \by^\star_j =: y^\star_0 \in \mathbb{R}^{d_y}$. We also know that 
	\begin{align}
		(\mathbf{1}^\top \otimes I_{d_x})\nabla_x F(\bx^\star, \by^\star) & = \sum_{i = 1}^m \nabla_x f_i(x^\star_0, y^\star_0) = \nabla_x f(x^\star_0, y^\star_0) \\
		(\mathbf{1}^\top \otimes I_{d_x}) \partial G(\bx^\star) & = \sum_{i = 1}^m \partial g(x^\star_0) = m\partial g(x^\star_0) \\
		(\mathbf{1}^\top \otimes I_{d_x}) U_x & = (\mathbf{1} \otimes I_{d_x})^\top \left( \sqrt{I-W} \otimes I_{d_x} \right) =  0 ,
	\end{align} 
	where the last equality follows from Proposition \ref{null_sqrt_I_W_range_1}. By combining last three equalities, we obtain
	\begin{align}
		(\mathbf{1}^\top \otimes I_{d_x})\left( \nabla_x F(\bx^\star,\by^\star) + \partial G(\bx^\star) + U_x S^{\star,\bx} \right) & = \nabla_x f(x^\star_0, y^\star_0) + m\partial g(x^\star_0) . \label{eq:gradF_term_grad_f}
	\end{align}
	By performing above steps with $\nabla_y F(\bx^\star,\by^\star)$ and  $ \partial R(\by^\star)$, we obtain
	\begin{align}
		(\mathbf{1}^\top \otimes I_{d_y})\left( \nabla_y F(\bx^\star,\by^\star) - \partial R(\by^\star) + U_y S^{\star,\by} \right) & = \nabla_y f(x^\star_0, y^\star_0) - m\partial r(y^\star_0).
	\end{align}
	Using \eqref{eq:primal_var_gradient_rev} and \eqref{eq:dual_var_gradient_rev}, we know that there exists $\xi^\star_x \in \partial{G}(\bx^\star)$ and $\xi^\star_y \in \partial{G}(\by^\star)$  such that
	\begin{align} 
	&	\nabla_x F(\bx^\star,\by^\star) + \xi^\star_x + U_x S^{\star,\bx}  = 0 \Rightarrow (\mathbf{1}^\top \otimes I_{d_x}) \left( \nabla_x F(\bx^\star,\by^\star) + \xi^\star_x + U_x S^{\star,\bx} \right) = 0 \\
&	 \nabla_y F(\bx^\star,\by^\star) - \xi^\star_y + U_y S^{\star,\by} = 0 \Rightarrow	(\mathbf{1}^\top \otimes I_{d_y}) \left( \nabla_y F(\bx^\star,\by^\star) - \xi^\star_y + U_y S^{\star,\by} \right) = 0 .
	\end{align}
	Therefore, $0 \in (\mathbf{1}^\top \otimes I_{d_y}) \left( \nabla_x F(\bx^\star,\by^\star) +  \partial{G}(\bx^\star) + U_x S^{\star,\bx} \right) = \nabla_x f(x^\star_0, y^\star_0) + m\partial g(x^\star_0)$ where the equality follows from \eqref{eq:gradF_term_grad_f}. This shows that
	\begin{align}
		0 & \in \nabla_x f(x^\star_0, y^\star_0) + m\partial g(x^\star_0) \label{eq:spp_firstorder_stat_x} \\
		0 & \in \nabla_y f(x^\star_0, y^\star_0) - m\partial r(y^\star_0) . \label{eq:spp_firstorder_stat_y}
	\end{align}
	Using \eqref{eq:spp_firstorder_stat_x} and \eqref{eq:spp_firstorder_stat_y}, we see that
	\begin{align}
		x^\star_0 & = \argmin_{x \in \mathbb{R}^{d_x}} f(x,y^\star_0) + mg(x) - mr(y^\star_0)  \\
		y^\star_0 & = \argmin_{y \in \mathbb{R}^{d_y}} f(x^\star_0,y) + mg(x^\star_0) - mr(y) .
	\end{align}
	Thus, using the definition of saddle point,  $(x^\star_0, y^\star_0)$ is a saddle point of $\min_{x \in \mathbb{R}^{dx}} \max_{y \in \mathbb{R}^{dy}} \sum_{i = 1}^m (f_i(x,y) +g(x) -r(y) )$. Hence $((\mathbf{1} \otimes I_{d_x}) x^\star_0, ((\mathbf{1} \otimes I_{d_y})y^\star_0)$ is a saddle point of \eqref{eq:main_opt_consenus_constraint}.

\end{proof}

We now provide the optimality conditions for the optimization problem \eqref{eq:main_opt_problem}. 
 
\subsection{Optimality Conditions of problem \eqref{eq:main_opt_problem}.}

Let $f(x,y) := \sum_{i = 1}^m f_i(x,y)$. Since $(x^\star, y^\star)$ is the saddle point solution to \eqref{eq:main_opt_problem}, we have $x^\star = \arg \min_{x \in \mathbb{R}^{d_x}} \Psi(x,y^\star)$ and $y^\star = \arg \max_{y \in \mathbb{R}^{d_y}} \Psi(x^\star,y)$.
\begin{align}
0 & \in \frac{1}{m} \nabla_x f(x^\star, y^\star) + \partial g(x^\star)  \\
& =  \partial g(x^\star)  + \frac{1}{s}\left(x^\star - \left( x^\star - \frac{s}{m} \nabla_x f(x^\star,y^\star) \right) \right) \\
& = \partial \left( g(x)  + \frac{1}{2s}\left\Vert x - \left( x^\star - \frac{s}{m} \nabla_x f(x^\star,y^\star) \right) \right\Vert^2 \right)_{x = x^\star} .
\end{align}
This implies that 
\begin{align}
x^\star & = \arg \min_{x \in \mathbb{R}^{d_x}} g(x) + \frac{1}{2s}\left\Vert x - \left( x^\star - \frac{s}{m} \nabla_x f(x^\star,y^\star) \right) \right\Vert^2 \\
& = \prox_{sg}\left( x^\star - \frac{s}{m} \nabla_x f(x^\star,y^\star) \right) .
\end{align}

We also have $y^\star = \arg \min_{y \in \mathbb{R}^{d_y}} -\Psi(x^\star,y) = \arg \min_{y \in \mathbb{R}^{d_y}} (-\Psi(x^\star,y) )$. Therefore,
\begin{align}
0 & \in \partial_y (-\Psi(x^\star,y^\star))  \\
& = -\frac{1}{m} \nabla_y f(x^\star, y^\star) + \partial r(y^\star)  \\
& = \partial r(y^\star) + \frac{1}{s}\left(y^\star - \left( y^\star + \frac{s}{m} \nabla_y f(x^\star,y^\star) \right) \right) \\
& = \partial \left( r(y) + \frac{1}{2s}\left\Vert y - \left( y^\star + \frac{s}{m} \nabla_y f(x^\star,y^\star) \right) \right\Vert^2 \right)_{y = y^\star}.
\end{align}
Therefore, $y^\star = \prox_{sr}\left( y^\star + \frac{s}{m} \nabla_y f(x^\star,y^\star) \right) $ .\\[1mm]

\subsection{Notations useful for further analysis:}
 
In the subsequent analysis, we assume $d_x = d_y = 1$ for simplicity of representation. Our analysis still holds for $d_x > 1$ and $d_y > 1$ by incorporating Kronecker product. We define Bregman distance with respect to each function $f_i(\cdot,y)$ and $-f_i(x,\cdot)$ as
\begin{align}
	V_{f_i,y}(x_1,x_2) & = f_i(x_1,y) - f_i(x_2,y) - \left\langle \nabla_x f_i(x_2,y), x_1 - x_2 \right\rangle \label{bregman_dist_Vy} \\
	V_{-f_i,x}(y_1,y_2) & = -f_i(x,y_1) + f_i(x,y_2) - \left\langle -\nabla_y f_i(x,y_2), y_1 - y_2 \right\rangle , \label{bregman_dist_Vx}
\end{align}
respectively. Let $L = \max \{L_{xx}, L_{yy}, L_{xy}, L_{yx} \}$ and $\mu = \min \{\mu_x, \mu_y \}$. Suppose $\lambda_{\max}(I-W)$, $\lambda_{m-1}(I-W)$ and $(I-W)^\dagger$ denote the largest eigenvalue, second smallest eigenvalue and pseudo inverse of $I-W$ respectively. Let $\kappa_f = L/\mu$ and $\kappa_g = \lambda_{\max}(I-W)/\lambda_{m-1}(I-W)$ denote the condition number of function $f$ and graph $G$ respectively. We further define $D^\star_\bx  := -(I-J) \nabla_x F(\mathbf{1} z^\star), 
D^\star_\by  := (I-J)\nabla_y F(\mathbf{1}z^\star)$, $H^\star_\bx := \mathbf{1} (x^\star - \frac{s}{m}\nabla_x f(z^\star))$, $H^\star_\by := \mathbf{1} (y^\star + \frac{s}{m}\nabla_y f(z^\star))$, \red{$\text{Range}(I-W)  := \{(I-W)z: z \in \mathbb{R}^m \} $, $ \text{Range}(\mathbf{1}) := \{\eta \mathbf{1} : \eta \in \mathbb{R} \}$ and 
$\text{Null}(I-W) := \{z : (I-W)z = 0 \} .$} Notation $A^\dagger$ denotes the pseudo-inverse of a square matrix $A$. 
We now state a few preliminary results that will be used in later sections. These results may be of independent interest as well.

\begin{proposition} \label{null_I_W_range_1} Let $W$ be a weight matrix satisfying assumption \ref{weight_matrix_assumption}. Then $\text{Null}(I-W) = \text{Range}(\mathbf{1})$.
\end{proposition}
\begin{proof} We prove this result in two parts. We first show that $\text{Null}(I-W) \subseteq \text{Range}(\mathbf{1})$ and then show that $\text{Range}(\mathbf{1}) \subseteq \text{Null}(I-W)$ . In this regard, let $y \in \text{Null}(I-W)$. Then we have $(I-W)y = 0$ which implies that $Wy = y$. Hence $y$ is an eigen vector of $W$ with eigen value $1$. We know that algebraic multiplicity of eigenvalue $1$ is one using assumption \ref{weight_matrix_assumption}. Therefore, there is only one linearly independent eigenvector associated with eigenvalue $1$. We also know that $\mathbf{1}$ is an eigenvector associated with eigenvalue $1$ because $W\mathbf{1} = \mathbf{1}$. Therefore, $y$ must belong to $\text{Range}(\mathbf{1})$. This completes the first part of the proof. To prove the other part, let $y \in \text{Range}(\mathbf{1})$. Then $(I-W)y = (I-W)\eta \mathbf{1} = 0$. This shows that $y \in \text{Null}(I-W)$. By combining both the parts, we get the desired result.
\end{proof}









\begin{proposition}\label{rem:rangeI-W} Let $W$ satisfy Assumption~\ref{weight_matrix_assumption} and let $D^\star_\bx$ and $D^\star_\by$ be as defined in above paragraph. Then, 
$D^\star_\bx \in \text{Range}(I-W)$ and $D^\star_\by \in \text{Range}(I-W)$.
\end{proposition}

\begin{proof} To prove this result, we first show that $\text{Range}(I-W) = \left( \text{Range}(\mathbf{1}) \right)^\bot$ using Assumption \ref{weight_matrix_assumption}. Then we prove that both $D^\star_\bx $ and $D^\star_\by$ lie in $ \left( \text{Range}(\mathbf{1}) \right)^\bot$.
\red{
\begin{align}
\text{Range}(I-W) & = \{(I-W)z: z \in \mathbb{R}^m \} , \ \text{Range}(\mathbf{1}) = \{\eta \mathbf{1} : \eta \in \mathbb{R} \} , \\
\text{Null}(I-W) & = \{z : (I-W)z = 0 \}  , \\
\left( \text{Range}(\mathbf{1}) \right)^\bot & = \{ x : \left\langle x,y \right\rangle = 0 \ \text{for all} \ y \in \text{Range}(\mathbf{1}) \} .
\end{align}
We first show that $\text{Range}(I-W) \subseteq \left( \text{Range}(\mathbf{1}^\top) \right)^\bot$. Towards that end, let $y \in \text{Range}(I-W)$. This implies that there exists a $z \in \mathbb{R}^m$ such that $(I-W)z = y$. Therefore,
\begin{align}
\left\langle y, \eta \mathbf{1}\right\rangle = \eta\mathbf{1}^\top y = \eta (\mathbf{1}^\top (I-W)z) = 0 \ \text{for all} \ \eta \in \mathbb{R}. 
\end{align}
The last step follows from $W\mathbf{1} = \mathbf{1}$ . This implies that $y \in \left( \text{Range}(\mathbf{1}) \right)^\bot$.  Therefore, $\text{Range}(I-W) \subseteq \left( \text{Range}(\mathbf{1}) \right)^\bot$. Next we show that dim(Range$(I-W)$) = dim($\left( \text{Range}(\mathbf{1}) \right)^\bot$). Using Proposition \ref{null_I_W_range_1}, we have $\text{Null}(I-W) = \text{Range}(\mathbf{1})$. This implies that dim$(\text{Null}(I-W)) = 1 $. Using Rank-Nullity Theorem \cite{friedberg2003linear}, we get dim$(\text{Range}(I-W)) = m-1$. Further $\text{Range}(\mathbf{1})$ is a one-dimensional subspace of $\mathbb{R}^m$ and hence dim$(( \text{Range}(\mathbf{1}))^\bot ) = m-1$. Therefore, dim(Range$(I-W)$) = dim($\left( \text{Range}(\mathbf{1}) \right)^\bot$). Using Theorem 1.11 in \cite{friedberg2003linear}, we get $\text{Range}(I-W) = \left( \text{Range}(\mathbf{1}) \right)^\bot$. This completes the first part of the proof. Recall
\begin{align}
D^\star_\bx & = -(I-J) \nabla_x F(\mathbf{1} z^\star)\\
D^\star_\by & = (I-J)\nabla_y F(\mathbf{1}z^\star) .
\end{align}
Therefore, $\eta \mathbf{1}^\top D^\star_\bx = -\eta\mathbf{1}^\top (I-J) \nabla_x F(\mathbf{1} z^\star) = 0$ because $\mathbf{1}^\top J = \mathbf{1}^\top$. Similarly, $\eta\mathbf{1}^\top D^\star_\by = \mathbf{1}^\top (I-J) \nabla_y F(\mathbf{1} z^\star) = 0$ . Hence, $D^\star_\bx \in \left( \text{Range}(\mathbf{1}) \right)^\bot = \text{Range}(I-W) $ and $D^\star_\by \in \left( \text{Range}(\mathbf{1}) \right)^\bot = \text{Range}(I-W) $. }
\end{proof}

\begin{proposition}\label{prop:smoothness}
 (\textbf{Smoothness in} $x$) Assume that $f(x,y)$ is convex and $L_{xx}$-smooth in $x$ for any fixed $y$. Then
\begin{align}
\frac{1}{2L_{xx}} \left\Vert \nabla_x f(x_1,y) - \nabla_x f(x_2,y) \right\Vert^2 & \leq V_{f,y}(x_1,x_2) \leq \frac{L_{xx}}{2} \left\Vert x_1 - x_2 \right\Vert^2 \ \text{for all} \ x_1, x_2.
\end{align}
\end{proposition}

\begin{proof} Using the smoothness of $f(\cdot, y)$, we have
\begin{align}
& f(x_1,y) \leq f(x_2,y) + \left\langle \nabla_x f(x_2,y), x_1 - x_2 \right\rangle + \frac{L_{xx}}{2} \left\Vert x_1 - x_2 \right\Vert^2 \\
& f(x_1,y) -  f(x_2,y) - \left\langle \nabla_x f(x_2,y), x_1 - x_2 \right\rangle \leq \frac{L_{xx}}{2} \left\Vert x_1 - x_2 \right\Vert^2 \\
& V_{f,y}(x_1,x_2) \leq \frac{L_{xx}}{2} \left\Vert x_1 - x_2 \right\Vert^2 .
\end{align}
This completes the proof of second inequality. Let $h(x_1) := V_{f,y}(x_1,x_2)$ for a given $y$ and $x_2$. Notice that $h(x_1) = 0$ at $x_1 = x_2$. Using convexity of $f(x,y)$ in $x$, $h(x_1) \geq 0$. Therefore, $h(x_1)$ achieves its minimum value at $x_2$ and the minimum value is $0$.
\begin{align}
\left\Vert \nabla_x h(x_1) - \nabla_x h(x_1^{'}) \right\Vert & = \left\Vert \nabla_x f(x_1,y) - \nabla_x f(x_1^{'},y) \right\Vert \\
& \leq L_{xx} \left\Vert x_1 - x_1^{'} \right\Vert .
\end{align}
This implies that $h(x_1)$ is also $L_{xx}$-smooth.
\begin{align}
h(\bar{x}) & \leq h(x_1) + \left\langle \nabla_x h(x_1), \bar{x} - x_1 \right\rangle + \frac{L_{xx}}{2} \left\Vert \bar{x} - x_1 \right\Vert^2 
\end{align}
Take minimization over $\bar{x}$ on both sides.
\begin{align}
	\min h(\bar{x}) & \leq h(x_1) + \min_{\bar{x}} \left( \left\langle \nabla_x h(x_1), \bar{x} - x_1 \right\rangle + \frac{L_{xx}}{2} \left\Vert \bar{x} - x_1 \right\Vert^2 \right) . \label{eq:min_h}
\end{align}
Let $\delta(\bar{x}) = \left\langle \nabla_x h(x_1), \bar{x} - x_1 \right\rangle + \frac{L_{xx}}{2} \left\Vert \bar{x} - x_1 \right\Vert^2.$
\begin{align}
	\nabla \delta(\bar{x}) =  \nabla_x h(x_1) + L_{xx}(\bar{x} - x_1),\ \ \nabla^2 \delta(\bar{x}) = L_{xx}I \succ 0 .
\end{align}
Therefore,
\begin{align}
	\min_{\bar{x}} \delta(\bar{x}) & = \left\langle \nabla_x h(x_1), \frac{-\nabla_x h(x_1)}{L_{xx}} + x_1 - x_1 \right\rangle + \frac{L_{xx}}{2} \left\Vert \frac{-\nabla_x h(x_1)}{L_{xx}} + x_1 - x_1 \right\Vert^2 \\
	& = \frac{-1}{L_{xx}} \left\Vert \nabla_x h(x_1) \right\Vert^2 + \frac{\left\Vert \nabla_x h(x_1) \right\Vert^2}{2L_{xx}} \\
	& = - \frac{\left\Vert \nabla_x h(x_1) \right\Vert^2}{2L_{xx}} .
\end{align}
%
Plug in above minimum value into \eqref{eq:min_h}.
\begin{align}
\min h(\bar{x}) & \leq h(x_1)  - \frac{\left\Vert \nabla_x h(x_1) \right\Vert^2}{2L_{xx}} \\
0 & \leq h(x_1)  - \frac{\left\Vert \nabla_x h(x_1) \right\Vert^2}{2L_{xx}} \\
& = V_{f,y}(x_1,x_2) - \frac{\left\Vert \nabla_x f(x_1,y) - \nabla_x f(x_2,y) \right\Vert^2}{2L_{xx}} .
\end{align}
This gives
\begin{align}
\frac{\left\Vert \nabla_x f(x_1,y) - \nabla_x f(x_2,y) \right\Vert^2}{2L_{xx}} & \leq V_{f,y}(x_1,x_2) .
\end{align}

\end{proof}

\begin{proposition} \label{prop:smoothnessy}
(\textbf{Smoothness in} $y$) Assume that $-f(x,y)$ is convex and $L_{yy}$-smooth in $y$ for any fixed $x$. Then
\begin{align}
\frac{1}{2L_{yy}} \left\Vert -\nabla_y f(x,y_1) + \nabla_y f(x,y_2) \right\Vert^2 & \leq V_{-f,x}(y_1,y_2) \leq \frac{L_{yy}}{2} \left\Vert y_1 - y_2 \right\Vert^2 .
\end{align}
\end{proposition}

The proof of Proposition~\ref{prop:smoothnessy} is similar to that of Proposition~\ref{prop:smoothness}, and is omitted.

\section{A Recursion Relationship Useful for Further Analysis}
\label{appendix_B}

In this section, we derive two recursive relations between the iterate updates at $(t+1)$-th iterate and $t$-th iterate of Algorithm \ref{alg:generic_procedure_sgda} in Lemma \ref{lem:recursion_y} and Lemma \ref{lem:recursion_x}.

\begin{lemma} \label{lem:recursion_y} 
Let $\bx_{t+1}$, $\by_{t+1}$, $D^\bx_{t+1}$, $D^\by_{t+1}$, $H^\bx_{t+1}$, $H^\by_{t+1}$, $H^{\ssw,\bx}_{t+1}$, $H^{\ssw,\by}_{t+1}$ be obtained from Algorithm \ref{alg:generic_procedure_sgda} using \text{IPDHG}($\bx_{t}$, $\by_{t}$, $D^{\bx}_{t}$, $D^{\by}_{t}$, $H^{\bx}_{t}$, $H^{\by}_{t}$, $H^{\ssw,\bx}_{t}$, $H^{\ssw,\by}_{t}$, $s$, $\gamma_{x}$, $\gamma_{y}$, $\alpha_{x}$, $\alpha_{y}$, $\mathcal{G}$). Let Assumptions \ref{compression_operator}-\ref{weight_matrix_assumption} hold. Suppose $\alpha_y \in \left(0,(1+\delta)^{-1} \right)$ and   
	\begin{align}
		\gamma_y & \in \left( 0, \min \left\lbrace \frac{2-2\sqrt{\delta}\alpha_y}{\lambda_{\max}(I-W)}, \frac{\alpha_y - (1+\delta)\alpha_y^2}{\sqrt{\delta}\lambda_{\max}(I-W)} \right\rbrace  \right) \label{gamma_y_assumption}
	\end{align}
	Then the following holds for all $t \geq 0$:
	\begin{align}
		& M_y E\left\Vert \by_{t+1} - \mathbf{1}y^\star  \right\Vert^2 + \frac{2s^2}{\gamma_y} E \left\Vert  D^\by_{t+1} - D^\star_\by  \right\Vert^2_{(I-W)^\dagger} + \sqrt{\delta}E\left\Vert H^\by_{t+1} - H^\star_\by \right\Vert^2 \nonumber \\
		& \leq \left\Vert \by_{t} - \mathbf{1}y^\star + s\mathcal{G}^\by_{t} - s\nabla_y F(\mathbf{1}x^\star,\mathbf{1}y^\star) \right\Vert^2 + \frac{2s^2}{\gamma_y}\left( 1- \frac{\gamma_y}{2}\lambda_{m-1}(I-W) \right)\left\Vert D^\by_{t} - D^\star_\by \right\Vert^2_{(I-W)^\dagger} \nonumber \\ & \ \ + \sqrt{\delta}(1-\alpha_{y})\left\Vert H^\by_{t} - H^\star_\by  \right\Vert^2  , \label{eq:one_step_progress_y}
	\end{align}
	where $E$ denotes the conditional expectation on stochastic compression at $t$-th update step and $M_{y} = 1-\frac{\sqrt{\delta }\alpha_{y}}{1-\frac{\gamma_{y}}{2}\lambda_{\max}(I-W)}$ and  $H^\star_\bx = \mathbf{1} (x^\star - \frac{s}{m}\nabla_x f(z^\star))$ and $H^\star_\by = \mathbf{1} (y^\star + \frac{s}{m}\nabla_y f(z^\star))$.
	.
\end{lemma}

\textbf{Proof of Lemma \ref{lem:recursion_y}}:
We follow \cite{li2021decentralized} to prove Lemma~\ref{lem:recursion_y}. We have $H^\star_\bx = \mathbf{1} (x^\star - \frac{s}{m}\nabla_x f(x^\star,y^\star))$ and $H^\star_\by = \mathbf{1} (y^\star + \frac{s}{m}\nabla_y f(x^\star,y^\star))$. First, we bound the terms appearing on the l.h.s. of~\eqref{eq:one_step_progress_y} as
\begin{equation}
    M_y E\left\Vert \by_{t+1} - \mathbf{1}y^\star  \right\Vert^2   \leq \left\Vert \nu^\by_{t+1} - H^\star_\by  \right\Vert^2_{I - \frac{\gamma_y}{2}(I-W) - \alpha_{y}\sqrt{\delta}I} + \frac{\gamma^2_yM_y}{4} E \left\Vert \hat{\nu}^\by_{t+1} - \nu^\by_{t+1} \right\Vert^2_{(I-W)^2} ,  \label{eq:y_t+1_My}
\end{equation}
and
\begin{align}
& \left(\frac{2s^2}{\gamma_y} E \left\Vert  D^\by_{t+1} - D^\star_\by  \right\Vert^2_{(I-W)^\dagger} +  \sqrt{\delta}E\left\Vert H^\by_{t+1} - H^\star_\by \right\Vert^2\right) + \left\Vert \nu^\by_{t+1} - H^\star_\by  \right\Vert^2_{I - \frac{\gamma_y}{2}(I-W) - \alpha_{y}\sqrt{\delta}I} \nonumber \\
 &= \frac{s^2}{\gamma_y} \left\Vert D^\by_{t} - D^\star_\by \right\Vert^2_{{2(I-W)^\dagger -\gamma_y I}} + \frac{1}{2}  E \left\Vert \hat{\nu}^\by_{t+1}  - \nu^\by_{t+1} \right\Vert^2_{\gamma_y(I-W) + 2\sqrt{\delta}\alpha_{y}^2}  + \left\Vert \by_{t} - \mathbf{1}y^\star + s\mathcal{G}^\by_{t} - s\nabla_y F(\mathbf{1}z^\star)  \right\Vert^2 \nonumber \\ & \ + \sqrt{\delta}(1-\alpha_{y})\left\Vert H^\by_{t} - H^\star_\by  \right\Vert^2 - \sqrt{\delta}\alpha_{y}(1-\alpha_{y}) \left\Vert \nu^\by_{t+1} - H^\by_{t}  \right\Vert^2 . \label{eq:Dy_Hy_sum}
\end{align}
Proofs of~\eqref{eq:y_t+1_My} and~\eqref{eq:Dy_Hy_sum} are provided in Sections~\ref{sec:y_t+1_My} and~\ref{sec:Dy_Hy_sum}, respectively.

On adding \eqref{eq:y_t+1_My} and \eqref{eq:Dy_Hy_sum}, we obtain
\begin{align}
& M_y E\left\Vert \by_{t+1} - \mathbf{1}y^\star  \right\Vert^2 + \frac{2s^2}{\gamma_y} E \left\Vert  D^\by_{t+1} - D^\star_\by  \right\Vert^2_{(I-W)^\dagger} + \sqrt{\delta}E\left\Vert H^\by_{t+1} - H^\star_\by \right\Vert^2 \nonumber\\
& \leq \left\Vert \by_{t} - \mathbf{1}y^\star + s\mathcal{G}^\by_{t} - s\nabla_y F(\mathbf{1}z^\star)  \right\Vert^2 + \frac{s^2}{\gamma_y} \left\Vert D^\by_{t} - D^\star_\by \right\Vert^2_{{2(I-W)^\dagger -\gamma_y I}} \nonumber\\ & \ \ + \sqrt{\delta}(1-\alpha_{y})\left\Vert H^\by_{t} - H^\star_\by  \right\Vert^2 - \sqrt{\delta}\alpha_{y}(1-\alpha_{y}) \left\Vert \nu^\by_{t+1} - H^\by_{t}  \right\Vert^2 \nonumber\\ & \ \ + \frac{1}{2}  E \left\Vert \hat{\nu}^\by_{t+1}  - \nu^\by_{t+1} \right\Vert^2_{\gamma_y(I-W) + 2\sqrt{\delta}\alpha_{y}^2} + \frac{\gamma^2_yM_y}{4} E \left\Vert \hat{\nu}^\by_{t+1} - \nu^\by_{t+1} \right\Vert^2_{(I-W)^2} . \label{eq:yt+1_My_Dy_Hy}
\end{align}

We now bound the terms on the r.h.s. of~\eqref{eq:yt+1_My_Dy_Hy}. First, observe that  
\begin{align}
& \frac{1}{2}  E \left\Vert \hat{\nu}^\by_{t+1}  - \nu^\by_{t+1} \right\Vert^2_{\gamma_y(I-W) + 2\sqrt{\delta}\alpha_{y}^2} + \frac{\gamma^2_yM_y}{4} E \left\Vert \hat{\nu}^\by_{t+1} - \nu^\by_{t+1} \right\Vert^2_{(I-W)^2} \nonumber \\
& = \frac{1}{2}  E \left\Vert \sqrt{\gamma_y(I-W) + 2\sqrt{\delta}\alpha_{y}^2} (\hat{\nu}^\by_{t+1}  - \nu^\by_{t+1}) \right\Vert^2 + \frac{\gamma^2_yM_y}{4} E \left\Vert (I-W)(\hat{\nu}^\by_{t+1} - \nu^\by_{t+1}) \right\Vert^2 \nonumber \\
& \leq \frac{1}{2} \left\Vert \sqrt{\gamma_y(I-W) + 2\sqrt{\delta}\alpha_{y}^2}  \right\Vert^2 E \left\Vert \hat{\nu}^\by_{t+1}  - \nu^\by_{t+1} \right\Vert^2 + \frac{\gamma^2_yM_y}{4} \left\Vert I-W \right\Vert^2 E \left\Vert \hat{\nu}^\by_{t+1} - \nu^\by_{t+1} \right\Vert^2 \nonumber \\
& \leq \left( \frac{1}{2}\left( \gamma_y \lambda_{\max}(I-W) + 2\sqrt{\delta}\alpha_{y}^2 \right) + \frac{\gamma^2_yM_y \lambda^2_{\max}(I-W)}{4} \right) E \left\Vert \hat{\nu}^\by_{t+1} - \nu^\by_{t+1} \right\Vert^2 . \label{eq:exp_nu_y_t+1}
\end{align}
We also have
\begin{align}
\hat{\nu}^\by_{t+1} - \nu^\by_{t+1} & = Q(\nu^\by_{t+1} - H^\by_{t} ) - \left(\nu^\by_{t+1} - H^\by_{t} \right)  \\
E \left\Vert \hat{\nu}^\by_{t+1} - \nu^\by_{t+1} \right\Vert^2 & = E \left\Vert Q(\nu^\by_{t+1} - H^\by_{t} ) - \left(\nu^\by_{t+1} - H^\by_{t} \right) \right\Vert^2 \nonumber \\
& \leq \delta \left\Vert \nu^\by_{t+1} - H^\by_{t} \right\Vert^2 .
\end{align}
Substituting this inequality in \eqref{eq:exp_nu_y_t+1}, we bound the last two terms in the r.h.s of~\eqref{eq:yt+1_My_Dy_Hy} as
\begin{align}
& \frac{1}{2}  E \left\Vert \hat{\nu}^\by_{t+1}  - \nu^\by_{t+1} \right\Vert^2_{\gamma_y(I-W) + 2\sqrt{\delta}\alpha_{y}^2} + \frac{\gamma^2_yM_y}{4} E \left\Vert \hat{\nu}^\by_{t+1} - \nu^\by_{t+1} \right\Vert^2_{(I-W)^2} \nonumber \\
& \leq \left( \frac{1}{2}\left( \gamma_y \lambda_{\max}(I-W) + 2\sqrt{\delta}\alpha_{y}^2 \right) + \frac{\gamma^2_yM_y \lambda^2_{\max}(I-W)}{4} \right) \delta \left\Vert \nu^\by_{t+1} - H^\by_{t} \right\Vert^2 . \label{eq:nuhaty_t+1}
\end{align}
We will now bound the term $\frac{s^2}{\gamma_y} \left\Vert D^\by_{t} - D^\star_\by \right\Vert^2_{{2(I-W)^\dagger -\gamma_y I}}$ .
\begin{align}
\frac{s^2}{\gamma_y} \left\Vert D^\by_{t} - D^\star_\by \right\Vert^2_{{2(I-W)^\dagger -\gamma_y I}} & = \frac{s^2}{\gamma_y} \left\Vert D^\by_{t} - D^\star_\by \right\Vert^2_{2(I-W)^\dagger} - \frac{s^2}{\gamma_y} \left\Vert D^\by_{t} - D^\star_\by \right\Vert^2_{\gamma_y I} \nonumber \\
& = \frac{s^2}{\gamma_y} \left\Vert D^\by_{t} - D^\star_\by \right\Vert^2_{2(I-W)^\dagger} - \frac{s^2 \gamma_y}{\gamma_y} \left\Vert D^\by_{t} - D^\star_\by \right\Vert^2 \nonumber\\
& = \frac{2s^2}{\gamma_y} \left\Vert D^\by_{t} - D^\star_\by \right\Vert^2_{(I-W)^\dagger} - s^2 \left\Vert D^\by_{t} - D^\star_\by \right\Vert^2 \nonumber\\
& = \frac{2s^2}{\gamma_y} \left\Vert D^\by_{t} - D^\star_\by \right\Vert^2_{(I-W)^\dagger} - s^2 \left\Vert D^\by_{t} - D^\star_\by \right\Vert^2 + s^2 \lambda_{m-1}(I-W) \left\Vert D^\by_{t} - D^\star_\by \right\Vert^2_{(I-W)^\dagger} \nonumber\\ & \ \ - s^2 \lambda_{m-1}(I-W) \left\Vert D^\by_{t} - D^\star_\by \right\Vert^2_{(I-W)^\dagger} \nonumber\\
& = s^2(D^\by_{t} - D^\star_\by)^\top \left( - I +  \lambda_{m-1}(I-W)(I-W)^\dagger \right)(D^\by_{t} - D^\star_\by)\nonumber \\ & \ \ + \frac{2s^2}{\gamma_y} \left\Vert D^\by_{t} - D^\star_\by \right\Vert^2_{(I-W)^\dagger}  - s^2 \lambda_{m-1}(I-W) \left\Vert D^\by_{t} - D^\star_\by \right\Vert^2_{(I-W)^\dagger} \nonumber\\
& \leq \frac{2s^2}{\gamma_y} \left\Vert D^\by_{t} - D^\star_\by \right\Vert^2_{(I-W)^\dagger} - s^2 \lambda_{m-1}(I-W) \left\Vert D^\by_{t} - D^\star_\by \right\Vert^2_{(I-W)^\dagger} \nonumber\\
& = \frac{2s^2}{\gamma_y}\left( 1- \frac{\gamma_y}{2}\lambda_{m-1}(I-W) \right)\left\Vert D^\by_{t} - D^\star_\by \right\Vert^2_{(I-W)^\dagger} . \label{eq:Dyt+1_simplify}
\end{align}

By substituting \eqref{eq:nuhaty_t+1} and \eqref{eq:Dyt+1_simplify} in \eqref{eq:yt+1_My_Dy_Hy}, we get
\begin{align}
& M_y E\left\Vert \by_{t+1} - \mathbf{1}y^\star  \right\Vert^2 + \frac{2s^2}{\gamma_y} E \left\Vert  D^\by_{t+1} - D^\star_\by  \right\Vert^2_{(I-W)^\dagger} + \sqrt{\delta}E\left\Vert H^\by_{t+1} - H^\star_\by \right\Vert^2 \nonumber\\
& \leq \left\Vert \by_{t} - \mathbf{1}y^\star + s\mathcal{G}^\by_{t} - s\nabla_y F(\mathbf{1}z^\star)  \right\Vert^2 + \frac{2s^2}{\gamma_y}\left( 1- \frac{\gamma_y}{2}\lambda_{m-1}(I-W) \right)\left\Vert D^\by_{t} - D^\star_\by \right\Vert^2_{(I-W)^\dagger} \nonumber\\ & \ \ + \left(\frac{\delta\gamma^2_yM_y \lambda^2_{\max}(I-W)}{4} + \frac{\gamma_y\delta}{2}\lambda_{\max}(I-W) + \sqrt{\delta} \delta \alpha_{y}^2 - \sqrt{\delta}\alpha_{y}(1-\alpha_{y}) \right) \left\Vert \nu^\by_{t+1} - H^\by_{t} \right\Vert^2 \nonumber\\ & \ \ \ + \sqrt{\delta}(1-\alpha_{y})\left\Vert H^\by_{t} - H^\star_\by  \right\Vert^2 . \label{eq:simplifiedrhs}
\end{align}

The coefficient of $\left\Vert \nu^\by_{t+1} - H^\by_{t} \right\Vert^2$ in~\eqref{eq:simplifiedrhs} is
\begin{align}
& \frac{\delta\gamma^2_yM_y \lambda^2_{\max}(I-W)}{4} + \frac{\gamma_y\delta}{2}\lambda_{\max}(I-W) + \sqrt{\delta} \delta \alpha_{y}^2 - \sqrt{\delta}\alpha_{y}(1-\alpha_{y}) \nonumber\\
& < \frac{\delta\gamma^2_y\lambda^2_{\max}(I-W)}{4} + \frac{\gamma_y\delta}{2}\lambda_{\max}(I-W) + \sqrt{\delta} (1+\delta) \alpha_{y}^2 - \sqrt{\delta}\alpha_{y}\nonumber \\
& = \frac{\delta\gamma_y}{2} \frac{\gamma_y \lambda_{\max}(I-W)}{2} \lambda_{\max}(I-W) + \frac{\gamma_y\delta}{2}\lambda_{\max}(I-W) - \sqrt{\delta}(\alpha_{y} - (1+\delta) \alpha_{y}^2 )\nonumber \\
& < \frac{\delta\gamma_y}{2} \lambda_{\max}(I-W) + \frac{\gamma_y\delta}{2}\lambda_{\max}(I-W) - \sqrt{\delta}(\alpha_{y} - (1+\delta) \alpha_{y}^2 ) \nonumber\\
& = \delta\gamma_y \lambda_{\max}(I-W) - \sqrt{\delta}(\alpha_{y} - (1+\delta) \alpha_{y}^2 ) \nonumber\\
& < 0 , \text{ as } \gamma_y < \frac{\alpha_y - (1+\delta)\alpha_y^2}{\sqrt{\delta}\lambda_{\max}(I-W)}. \label{coeff_bound}
\end{align}
In deriving \eqref{coeff_bound}, the first inequality follows from $M_y < 1$ (to be proved later in Section \ref{appendix_parameters_feasibility}) and the second inequality follows from $\frac{\gamma_y\lambda_{\max}(I-W)}{2}<1$, since from assumption \eqref{gamma_y_assumption}, it holds that $0<\gamma_y<\frac{2-2\sqrt{\delta}\alpha_y}{\lambda_{\max}(I-W)}\leq \frac{2}{\lambda_{\max}(I-W)}$. 
As a result,~\eqref{eq:simplifiedrhs} can be simplified as 
\begin{align}
& M_y E\left\Vert \by_{t+1} - \mathbf{1}y^\star  \right\Vert^2 + \frac{2s^2}{\gamma_y} E \left\Vert  D^\by_{t+1} - D^\star_\by  \right\Vert^2_{(I-W)^\dagger} + \sqrt{\delta}E\left\Vert H^\by_{t+1} - H^\star_\by \right\Vert^2 \nonumber\\
& \leq \left\Vert \by_{t} - \mathbf{1}y^\star + s\mathcal{G}^\by_{t} - s\nabla_y F(\mathbf{1}z^\star)  \right\Vert^2 + \frac{2s^2}{\gamma_y}\left( 1- \frac{\gamma_y}{2}\lambda_{m-1}(I-W) \right)\left\Vert D^\by_{t} - D^\star_\by \right\Vert^2_{(I-W)^\dagger} \nonumber\\ & \ \ + \sqrt{\delta}(1-\alpha_{y})\left\Vert H^\by_{t} - H^\star_\by  \right\Vert^2 ,
\end{align}

proving Lemma~\ref{lem:recursion_y}. The remainder of this section is devoted to prove~\eqref{eq:y_t+1_My} and \eqref{eq:Dy_Hy_sum}.

\subsection{Proof of (\ref{eq:y_t+1_My}) }\label{sec:y_t+1_My}
First, observe that
\begin{align}
\left\Vert \by_{t+1} - \mathbf{1}y^\star  \right\Vert^2  
& = \sum_{i = 1}^m \left\Vert y^i_{t+1} - y^\star  \right\Vert^2\nonumber \\
& = \sum_{i = 1}^m \left\Vert \prox_{sr}(\hat{y}^i_{t+1}) - \prox_{sr}\left( y^\star + \frac{s}{m} \nabla_y f(x^\star,y^\star) \right)  \right\Vert^2 \nonumber\\
& = \sum_{i = 1}^m \left\Vert \prox_{sr}(\hat{y}^i_{t+1}) - \prox_{sr}\left(  H^\star_{i,y} \right)  \right\Vert^2 \nonumber\\
& \leq \sum_{i = 1}^m \left\Vert \hat{y}^i_{t+1} -H^\star_{i,y}  \right\Vert^2 \text{(from non-expansivity of prox)} \nonumber \\
& = \left\Vert \hat{y}_{t+1} -H^\star_{\by} \right\Vert^2 \nonumber\\
& = \left\Vert \nu^\by_{t+1} - \frac{\gamma_{y}}{2}(I-W)\hat{\nu}^\by_{t+1} -H^\star_{\by} \right\Vert^2\nonumber \\
& = \left\Vert \nu^\by_{t+1} - \frac{\gamma_{y}}{2}(I-W)(\hat{\nu}^\by_{t+1} - \nu^\by_{t+1} + \nu^\by_{t+1} ) -H^\star_{\by} \right\Vert^2 \nonumber\\
& = \left\Vert \left(I - \frac{\gamma_y}{2}(I-W) \right) (\nu^\by_{t+1} - H^\star_\by) - \frac{\gamma_{y}}{2}(I-W)(\hat{\nu}^\by_{t+1} - \nu^\by_{t+1} \right\Vert^2\nonumber \\
& = \left\Vert \left(I - \frac{\gamma_y}{2}(I-W) \right) (\nu^\by_{t+1} - H^\star_\by) \right\Vert^2 + \frac{\gamma^2_y}{4} \left\Vert (I-W)(\hat{\nu}^\by_{t+1} - \nu^\by_{t+1}) \right\Vert^2 \nonumber\\ & \ \ \ + 2 \left\langle \left(I - \frac{\gamma_y}{2}(I-W) \right)(\nu^\by_{t+1} - H^\star_\by), -\frac{\gamma_y}{2}(I-W)(\hat{\nu}^\by_{t+1} - \nu^\by_{t+1}) \right\rangle .
\end{align}
By taking conditional expectation over stochastic compression at $t$-th iterate, we obtain
\begin{align}
E\left\Vert \by_{t+1} - \mathbf{1}y^\star  \right\Vert^2  & \leq \left\Vert \left(I - \frac{\gamma_y}{2}(I-W) \right) (\nu^\by_{t+1} - H^\star_\by) \right\Vert^2 + \frac{\gamma^2_y}{4} E \left\Vert (I-W)(\hat{\nu}^\by_{t+1} - \nu^\by_{t+1}) \right\Vert^2 \nonumber\\ & \ \ \ - \gamma_y \left\langle \left(I - \frac{\gamma_y}{2}(I-W) \right)(\nu^\by_{t+1} - H^\star_\by), (I-W)E(\hat{\nu}^\by_{t+1} - \nu^\by_{t+1}) \right\rangle . \label{eq:ykt+1_ystar}
\end{align}
We have
\begin{align}
\hat{\nu}^\by_{t+1} - \nu^\by_{t+1} & = H^\by_{t} + Q(\nu^\by_{t+1} - H^\by_{t} ) - \nu^\by_{t+1} \nonumber\\
& = Q(\nu^\by_{t+1} - H^\by_{t} ) - \left(\nu^\by_{t+1} - H^\by_{t} \right) \nonumber \\
E\left( \hat{\nu}^\by_{t+1} - \nu^\by_{t+1} \right) & = E\left( Q(\nu^\by_{t+1} - H^\by_{t} ) \right) - \left(\nu^\by_{t+1} - H^\by_{t} \right) \nonumber\\
& = 0 .
\end{align}
The last equality follows from Assumption \ref{compression_operator}. By substituting the above  equation in \eqref{eq:ykt+1_ystar}, we obtain
\begin{align}
E\left\Vert \by_{t+1} - \mathbf{1}y^\star  \right\Vert^2  & \leq \left\Vert \left(I - \frac{\gamma_y}{2}(I-W) \right) (\nu^\by_{t+1} - H^\star_\by) \right\Vert^2 + \frac{\gamma^2_y}{4} E \left\Vert (I-W)(\hat{\nu}^\by_{t+1} - \nu^\by_{t+1}) \right\Vert^2 . \label{eq:yk_t+1_ystar_vec_norm}
\end{align}
We will now convert the square norm terms of~\eqref{eq:yk_t+1_ystar_vec_norm} into matrix-norm based terms. Towards that end, observe that
\begin{align}
\left( (I - \frac{\gamma_y}{2}(I-W) \right)^2 & = I + \frac{\gamma^2_y}{4}(I-W)^2 -\gamma_y(I-W) \nonumber\\
& = I - \frac{\gamma_y}{2}(I-W) + \frac{\gamma_y}{2}(I-W) + \frac{\gamma^2_y}{4}(I-W)^2 - \gamma_y(I-W)\nonumber \\
& = I - \frac{\gamma_y}{2}(I-W) - \frac{\gamma_y}{2}(I-W) + \frac{\gamma^2_y}{4}(I-W)^2 \nonumber \\
& =  I - \frac{\gamma_y}{2}(I-W) + \frac{\gamma_y}{2}(I-W) \left( -I + \frac{\gamma_y}{2}(I-W) \right)\nonumber \\
& = I - \frac{\gamma_y}{2}(I-W) + \frac{\gamma_y}{2}(I-W)^{1/2} \left( -I + \frac{\gamma_y}{2}(I-W) \right)(I-W)^{1/2} . \label{square_norm_matrix_terms}
\end{align}
Note that $\left( -I + \frac{\gamma_y}{2}(I-W) \right)$ is a negative semidefinite matrix because $0 < \frac{\gamma_y}{2} \lambda_{\max}(I-W)  < 1$ from the choice of $\gamma_y$. Using this fact in equation \eqref{square_norm_matrix_terms}, we get $\forall x \in {\mathbb{R}}^m$, 
\begin{align}
x^\top\left( I - \frac{\gamma_y}{2}(I-W) \right)^2 x & \leq x^\top\left( I - \frac{\gamma_y}{2}(I-W) \right) x . 
\end{align} 
Consider 
\begin{align}
\left\Vert \left(I - \frac{\gamma_y}{2}(I-W) \right) (\nu^\by_{t+1} - H^\star_\by) \right\Vert^2  & = (\nu^\by_{t+1} - H^\star_\by)^\top \left(I - \frac{\gamma_y}{2}(I-W) \right)^2 (\nu^\by_{t+1} - H^\star_\by) \nonumber\\
& \leq (\nu^\by_{t+1} - H^\star_\by)^\top \left(I - \frac{\gamma_y}{2}(I-W) \right) (\nu^\by_{t+1} - H^\star_\by) \nonumber\\
& = \left\Vert \nu^\by_{t+1} - H^\star_\by \right\Vert^2_{I - \frac{\gamma_y}{2}(I-W)} . \label{matrix_sq_norm_term1}
\end{align}
Moreover,
\begin{align}
\left\Vert (I-W)(\hat{\nu}^\by_{t+1} - \nu^\by_{t+1}) \right\Vert^2 & = (\hat{\nu}^\by_{t+1} - \nu^\by_{t+1})^\top(I-W)^2(\hat{\nu}^\by_{t+1} - \nu^\by_{t+1}) \nonumber \\
& = \left\Vert \hat{\nu}^\by_{t+1} - \nu^\by_{t+1} \right\Vert^2_{(I-W)^2} . \label{matrix_sq_norm_term2}
\end{align}
On substituting above two equalities \eqref{matrix_sq_norm_term1} and \eqref{matrix_sq_norm_term2} in \eqref{eq:yk_t+1_ystar_vec_norm}, we obtain
\begin{align}
E\left\Vert \by_{t+1} - \mathbf{1}y^\star  \right\Vert^2  & \leq \left\Vert \nu^\by_{t+1} - H^\star_\by \right\Vert^2_{I - \frac{\gamma_y}{2}(I-W)} + \frac{\gamma^2_y}{4} E \left\Vert \hat{\nu}^\by_{t+1} - \nu^\by_{t+1} \right\Vert^2_{(I-W)^2}  . \label{eq:ykt+1_ystar_compact}
\end{align} 

To complete the proof of~\eqref{eq:y_t+1_My}, we first show that the first term on the r.h.s. of~\eqref{eq:ykt+1_ystar_compact} is at most 
$M^{-1}_y \left\Vert \nu^\by_{t+1} - H^\star_\by  \right\Vert^2_{I - \frac{\gamma_y}{2}(I-W) - \alpha_{y}\sqrt{\delta}I}$, 
where $M_y = 1-\frac{\sqrt{\delta}\alpha_y}{1-\frac{\gamma_y}{2}\lambda_{\max}(I-W)}$ . To this end, consider
\begin{align}
& \sqrt{\delta}\alpha_y I - \frac{\sqrt{\delta}\alpha_y}{1-\frac{\gamma_y}{2}\lambda_{\max}(I-W)}\left( I - \frac{\gamma_y}{2}(I-W) \right) \nonumber\\
& = \frac{\sqrt{\delta}\alpha_y \left( 1-\frac{\gamma_y}{2}\lambda_{\max}(I-W) \right) I - \sqrt{\delta}\alpha_y\left( I - \frac{\gamma_y}{2}(I-W) \right)}{1-\frac{\gamma_y}{2}\lambda_{\max}(I-W)} \nonumber\\
& = \frac{1}{1-\frac{\gamma_y}{2}\lambda_{\max}(I-W)} \left( -\frac{\sqrt{\delta}\alpha_y \gamma_y\lambda_{\max}(I-W)}{2} I + \frac{\sqrt{\delta}\alpha_y \gamma_y}{2} (I-W) \right) .
\end{align}
The largest eigenvalue of $\sqrt{\delta}\alpha_y I - \frac{\sqrt{\delta}\alpha_y}{1-\frac{\gamma_y}{2}\lambda_{\max}(I-W)}\left( I - \frac{\gamma_y}{2}(I-W) \right)$ is 
\begin{align}
& \frac{1}{1-\frac{\gamma_y}{2}\lambda_{\max}(I-W)} \left( -\frac{\sqrt{\delta}\alpha_y \gamma_y\lambda_{\max}(I-W)}{2} I + \frac{\sqrt{\delta}\alpha_y \gamma_y}{2} \lambda_{\max}(I-W) \right)\nonumber \\
& = 0 .
\end{align}
Therefore,
$\sqrt{\delta}\alpha_y I - \frac{\sqrt{\delta}\alpha_y}{1-\frac{\gamma_y}{2}\lambda_{\max}(I-W)}\left( I - \frac{\gamma_y}{2}(I-W) \right)$ is negative semidefinite, and 

\begin{align}
x^\top \left( I - \frac{\gamma_y}{2}(I-W) \right)x & = M_y^{-1} x^\top M_y \left( I - \frac{\gamma_y}{2}(I-W) \right)x \nonumber \\
& = M_y^{-1} x^\top \left( 1-\frac{\sqrt{\delta}\alpha_y}{1-\frac{\gamma_y}{2}\lambda_{\max}(I-W)} \right) \left( I - \frac{\gamma_y}{2}(I-W) \right)x \nonumber\\
& = M_y^{-1} x^\top \left( I - \frac{\gamma_y}{2}(I-W) - \frac{\sqrt{\delta}\alpha_y}{1-\frac{\gamma_y}{2}\lambda_{\max}(I-W)} \left( I - \frac{\gamma_y}{2}(I-W) \right)  \right)x\nonumber \\
& = M_y^{-1} x^\top \left( I - \frac{\gamma_y}{2}(I-W) - \sqrt{\delta}\alpha_y I  \right)x\nonumber \\ & \ \ + M_y^{-1} x^\top \left( \sqrt{\delta}\alpha_y I - \frac{\sqrt{\delta}\alpha_y}{1-\frac{\gamma_y}{2}\lambda_{\max}(I-W)} \left( I - \frac{\gamma_y}{2}(I-W) \right)  \right)x \nonumber\\
& \leq M_y^{-1} x^\top \left( I - \frac{\gamma_y}{2}(I-W) - \sqrt{\delta}\alpha_y I  \right)x .
\end{align}

Substituting $x = \nu^\by_{t+1} - H^\star_\by $ into the above inequality and using the definition of $\left\Vert x  \right\Vert^2_A$, we obtain 
\begin{align}
\left\Vert \nu^\by_{t+1} - H^\star_\by  \right\Vert^2_{I - \frac{\gamma_y}{2}(I-W)} & \leq M^{-1}_y \left\Vert \nu^\by_{t+1} - H^\star_\by  \right\Vert^2_{I - \frac{\gamma_y}{2}(I-W) - \alpha_{y}\sqrt{\delta}I} . \label{eq:nuyt+1_Hstar_My}
\end{align}

From \eqref{eq:ykt+1_ystar_compact} and \eqref{eq:nuyt+1_Hstar_My}, we see that
\begin{equation*}
E\left\Vert \by_{t+1} - \mathbf{1}y^\star  \right\Vert^2   \leq M^{-1}_y \left\Vert \nu^\by_{t+1} - H^\star_\by  \right\Vert^2_{I - \frac{\gamma_y}{2}(I-W) - \alpha_{y}\sqrt{\delta}I} + \frac{\gamma^2_y}{4} E \left\Vert \hat{\nu}^\by_{t+1} - \nu^\by_{t+1} \right\Vert^2_{(I-W)^2} . 
\end{equation*}
Multiplying throughout by $M_y$, the proof of~\eqref{eq:y_t+1_My} is complete.

\subsection{Proof of (\ref{eq:Dy_Hy_sum})}\label{sec:Dy_Hy_sum}

\subsubsection*{Step 1: Computing $E\left\Vert D^\by_{t+1} - D^\star_\by \right\Vert^2_{(I - W)^\dagger}$}
Observe that 
\begin{align}
 D^\by_{t+1} - D^\star_\by  & =  D^\by_{t} + \frac{\gamma_y}{2s}(I-W)\hat{\nu}^\by_{t+1} - D^\star_\by\nonumber  \\
& =  D^\by_{t} + \frac{\gamma_y}{2s}(I-W)\hat{\nu}^\by_{t+1} - D^\star_\by - \frac{\gamma_y}{2s}(I-W)H^\star_\by \nonumber \\
& =  D^\by_{t} - D^\star_\by + \frac{\gamma_y}{2s}(I-W)(\hat{\nu}^\by_{t+1}  - H^\star_\by) \nonumber \\
& = D^\by_{t} - D^\star_\by + \frac{\gamma_y}{2s}(I-W)(\hat{\nu}^\by_{t+1}  - \nu^\by_{t+1}) + \frac{\gamma_y}{2s}(I-W)(\nu^\by_{t+1}  - H^\star_\by)
\end{align}
On pre-multiplying both sides by $\sqrt{(I-W)^\dagger}$ and taking square norm on the resulting equality, we obtain
\begin{align}
& \left\Vert \sqrt{(I-W)^\dagger}\left( D^\by_{t+1} - D^\star_\by  \right) \right\Vert^2 \nonumber\\
& = \left\Vert \sqrt{(I-W)^\dagger}(D^\by_{t} - D^\star_\by) + \frac{\gamma_y}{2s}\sqrt{(I-W)^\dagger}(I-W)(\nu^\by_{t+1}  - H^\star_\by)  \right\Vert^2 \nonumber \\ & \ \ + \frac{\gamma^2_y}{4s^2}\left\Vert \sqrt{(I-W)^\dagger}(I-W)(\hat{\nu}^\by_{t+1}  - \nu^\by_{t+1}) \right\Vert^2 \nonumber \\ &  + 2 \left\langle \sqrt{(I-W)^\dagger}(D^\by_{t} - D^\star_\by) + \frac{\gamma_y}{2s}\sqrt{(I-W)^\dagger}(I-W)(\nu^\by_{t+1}  - H^\star_\by), \frac{\gamma_y}{2s} \sqrt{(I-W)^\dagger}(I-W)(\hat{\nu}^\by_{t+1}  - \nu^\by_{t+1}) \right\rangle .
\end{align}
By taking conditional expectation over compression at $t$-th iterate and using the result $E\left( \hat{\nu}^\by_{t+1} - \nu^\by_{t+1} \right) = 0$, we obtain
\begin{align}
E \left\Vert \sqrt{(I-W)^\dagger}\left( D^\by_{t+1} - D^\star_\by  \right) \right\Vert^2 
& = \left\Vert \sqrt{(I-W)^\dagger}(D^\by_{t} - D^\star_\by) + \frac{\gamma_y}{2s}\sqrt{(I-W)^\dagger}(I-W)(\nu^\by_{t+1}  - H^\star_\by)  \right\Vert^2 \nonumber\\ & \ \ \ + \frac{\gamma^2_y}{4s^2}E \left\Vert \sqrt{(I-W)^\dagger}(I-W)(\hat{\nu}^\by_{t+1}  - \nu^\by_{t+1}) \right\Vert^2 \nonumber \\
& = \left\Vert \sqrt{(I-W)^\dagger}(D^\by_{t} - D^\star_\by) \right\Vert^2 + \frac{\gamma^2_y}{4s^2} \left\Vert \sqrt{(I-W)^\dagger}(I-W)(\nu^\by_{t+1}  - H^\star_\by)  \right\Vert^2 \nonumber\\ & \ \ \ + 2 \left\langle \sqrt{(I-W)^\dagger}(D^\by_{t} - D^\star_\by), \frac{\gamma_y}{2s} \sqrt{(I-W)^\dagger}(I-W)(\nu^\by_{t+1}  - H^\star_\by) \right\rangle \nonumber\\ & \ \ \ + \frac{\gamma^2_y}{4s^2}E \left\Vert \sqrt{(I-W)^\dagger}(I-W)(\hat{\nu}^\by_{t+1}  - \nu^\by_{t+1}) \right\Vert^2 . \label{eq:Dk_t+1_Dstar_}
\end{align}
We have
\begin{align}
\left( \sqrt{(I-W)^\dagger}(I-W) \right)^\top \left(\sqrt{(I-W)^\dagger}(I-W) \right) & = (I-W)\sqrt{(I-W)^\dagger}\sqrt{(I-W)^\dagger}(I-W) \nonumber\\
& = (I-W)(I-W)^{\dagger}(I-W) \nonumber\\
& = I-W ,
\end{align}
where the last equality follows from the definition of pseudoinverse. Consider
\begin{align}
& \left\Vert \sqrt{(I-W)^\dagger}(I-W)(\nu^\by_{t+1}  - H^\star_\by)  \right\Vert^2 \nonumber \\
& = (\nu^\by_{t+1}  - H^\star_\by)^\top\left( \sqrt{(I-W)^\dagger}(I-W) \right)^\top \left(\sqrt{(I-W)^\dagger}(I-W) \right)(\nu^\by_{t+1}  - H^\star_\by)\nonumber \\
& = (\nu^\by_{t+1}  - H^\star_\by)^\top(I-W)(\nu^\by_{t+1}  - H^\star_\by) \nonumber\\
& = \left\Vert (\nu^\by_{t+1}  - H^\star_\by) \right\Vert^2_{I-W} .
\end{align}
Similarly, we see that $\left\Vert \sqrt{(I-W)^\dagger}(I-W)(\hat{\nu}^\by_{t+1}  - \nu^\by_{t+1}) \right\Vert^2 = \left\Vert \hat{\nu}^\by_{t+1}  - \nu^\by_{t+1} \right\Vert^2_{I-W}$
Substituting in \eqref{eq:Dk_t+1_Dstar_}, we obtain 
\begin{align}
E \left\Vert  D^\by_{t+1} - D^\star_\by \right\Vert^2_{(I-W)^\dagger}
& = \left\Vert D^\by_{t} - D^\star_\by \right\Vert^2_{(I-W)^\dagger} + \frac{\gamma^2_y}{4s^2} \left\Vert \nu^\by_{t+1}  - H^\star_\by \right\Vert^2_{I-W} + \frac{\gamma^2_y}{4s^2}  E \left\Vert \hat{\nu}^\by_{t+1}  - \nu^\by_{t+1} \right\Vert^2_{I-W} \nonumber\\ & \ \  + \frac{\gamma_y}{s} \left\langle \sqrt{(I-W)^\dagger}(D^\by_{t} - D^\star_\by),  \sqrt{(I-W)^\dagger}(I-W)(\nu^\by_{t+1}  - H^\star_\by) \right\rangle . \label{eq:Dyk_t+1_inner_product}
\end{align}
Now, we will simplify the last term of \eqref{eq:Dyk_t+1_inner_product}. From the property of adjoints,
\begin{align}
&  \left\langle \sqrt{(I-W)^\dagger}(D^\by_{t} - D^\star_\by),  \sqrt{(I-W)^\dagger}(I-W)(\nu^\by_{t+1}  - H^\star_\by) \right\rangle \nonumber \\
& = \left\langle (I-W)\sqrt{(I-W)^\dagger}\sqrt{(I-W)^\dagger}(D^\by_{t} - D^\star_\by),  \nu^\by_{t+1}  - H^\star_\by \right\rangle\nonumber \\
& = \left\langle (I-W)(I-W)^\dagger (D^\by_{t} - D^\star_\by),  \nu^\by_{t+1}  - H^\star_\by \right\rangle .
\end{align}
Note that $D^\star_\by \in \text{Range}(I-W)$ using Proposition~\ref{rem:rangeI-W}. Further note that $D^\by_{t} \in \text{Range}(I-W)$ because of update process (Step 10) in Algorithm \ref{alg:generic_procedure_sgda}. Therefore, there exists $\tilde{D}^\by_{t}$ and $\tilde{D}_y$ such that $D^\by_{t} = (I-W)\tilde{D}^\by_{t}$ and $D^\star_\by = (I-W)\tilde{D}_y$. 
\begin{align}
(I-W)(I-W)^\dagger (D^\by_{t} - D^\star_\by) & = (I-W)(I-W)^\dagger \left( (I-W)\tilde{D}^\by_{t} - (I-W)\tilde{D}_y) \right) \nonumber\\
& = (I-W)(I-W)^\dagger (I-W) \left( \tilde{D}^\by_{t} - \tilde{D}_y \right) \nonumber\\
& = (I-W) \left( \tilde{D}^\by_{t} - \tilde{D}_y \right) \nonumber\\
& = (I-W)\tilde{D}^\by_{t} - (I-W)\tilde{D}_y\nonumber \\
& = D^\by_{t} - D^\star_\by .
\end{align}
This gives
\begin{align}
\left\langle \sqrt{(I-W)^\dagger}(D^\by_{t} - D^\star_\by),  \sqrt{(I-W)^\dagger}(I-W)(\nu^\by_{t+1}  - H^\star_\by) \right\rangle & = \left\langle D^\by_{t} - D^\star_\by,  \nu^\by_{t+1}  - H^\star_\by \right\rangle .
\end{align}
On substituting above equality in \eqref{eq:Dyk_t+1_inner_product}, we obtain
\begin{align}
E \left\Vert  D^\by_{t+1} - D^\star_\by  \right\Vert^2_{(I-W)^\dagger} & = \left\Vert D^\by_{t} - D^\star_\by \right\Vert^2_{(I-W)^\dagger} + \frac{\gamma^2_y}{4s^2}  \left\Vert \nu^\by_{t+1}  - H^\star_\by \right\Vert^2_{I-W} + \frac{\gamma^2_y}{4s^2}  E \left\Vert \hat{\nu}^\by_{t+1}  - \nu^\by_{t+1} \right\Vert^2_{I-W} \nonumber\\ & \ \ + \frac{\gamma_y}{s} \left\langle D^\by_{t} - D^\star_\by,  \nu^\by_{t+1}  - H^\star_\by \right\rangle . \label{eq:Dy_k_t+1_Dstar_2}
\end{align}
Note that $\frac{1}{m}\mathbf{1}\nabla_y f(z^\star) = \frac{1}{m} \mathbf{1}\sum_{i = 1}^m \nabla_y f_i(z^\star) = J\nabla_y F(\mathbf{1}z^\star)$. Then we can write $H^\star_\by = \mathbf{1}y^\star + sJ\nabla_y F(\mathbf{1}z^\star)$. We have
\begin{align}
	\nu^\by_{t+1} - H^\star_\by & =  \nu^\by_{t+1} - \left( \mathbf{1}y^\star + sJ\nabla_y F(\mathbf{1}z^\star) \right) \nonumber \\
	& = \nu^\by_{t+1} - \left(-sD^\star_\by + \mathbf{1}y^\star + s\nabla_y F(\mathbf{1}z^\star) \right) \nonumber\\
	& = \by_{t} + s \mathcal{G}^\by_{t} - sD^\by_{t} + sD^\star - \mathbf{1}y^\star - s\nabla_y F(\mathbf{1}z^\star) \nonumber\\
	& = \by_{t} - \mathbf{1}y^\star + s\left( \mathcal{G}^\by_{t} - \nabla_y F(\mathbf{1}z^\star) \right) -s \left( D^\by_{t} - D^\star_\by \right) . \label{eq:nu_y_t+1_Hstar}
\end{align}
In deriving equation \eqref{eq:nu_y_t+1_Hstar}, the second equality follows from the definition of $D^\star_\by = (I-J)\nabla_y F(\mathbf{1}z^\star)$, and the third equality follows from the update step of $\nu^{y}_{t+1}$ (Step 9 in Algorithm \ref{alg:generic_procedure_sgda}). 

Substituting~\eqref{eq:nu_y_t+1_Hstar} in \eqref{eq:Dy_k_t+1_Dstar_2},  
\begin{align}
E \left\Vert  D^\by_{t+1} - D^\star_\by  \right\Vert^2_{(I-W)^\dagger} & = \left\Vert D^\by_{t} - D^\star_\by \right\Vert^2_{(I-W)^\dagger} + \frac{\gamma^2_y}{4s^2}  \left\Vert \nu^\by_{t+1}  - H^\star_\by \right\Vert^2_{I-W} + \frac{\gamma^2_y}{4s^2}  E \left\Vert \hat{\nu}^\by_{t+1}  - \nu^\by_{t+1} \right\Vert^2_{I-W} \nonumber\\ & \ \ + \frac{\gamma_y}{s} \left\langle D^\by_{t} - D^\star_\by,   \by_{t} - \mathbf{1}y^\star + s\left( \mathcal{G}^\by_{t} - \nabla_y F(\mathbf{1}z^\star) \right) \right\rangle - \gamma_y \left\Vert  D^\by_{t} - D^\star_\by  \right\Vert^2 . \label{eq:Dy_k_t+1_Dstar_compact}
\end{align}

\textbf{Step 2: Computing $\frac{2s^2}{\gamma_y} E \left\Vert  D^\by_{t+1} - D^\star_\by  \right\Vert^2_{(I-W)^\dagger} + \left\Vert \nu^\by_{t+1} - H^\star_\by \right\Vert^2_{I - \frac{\gamma_y}{2}(I-W)}$}

Taking square norm on both sides of \eqref{eq:nu_y_t+1_Hstar}, we obtain .
\begin{align}
\left\Vert \nu^\by_{t+1} - H^\star_\by \right\Vert^2 & = \left\Vert \by_{t} - \mathbf{1}y^\star + s\left( \mathcal{G}^\by_{t} - \nabla_y F(\mathbf{1}z^\star) \right)  \right\Vert^2 + s^2 \left\Vert D^\by_{t} - D^\star_\by \right\Vert^2 \nonumber\\ & \ \ - 2s \left\langle D^\by_{t} - D^\star_\by, \by_{t} - \mathbf{1}y^\star + s \mathcal{G}^\by_{t} - s\nabla_y F(\mathbf{1}z^\star) \right\rangle . \label{eq:nu_y_t+1_Hstar_norm}
\end{align}

Multiply \eqref{eq:Dy_k_t+1_Dstar_compact} by $\frac{2s^2}{\gamma_y}$ and adding the resulting inequality with \eqref{eq:nu_y_t+1_Hstar_norm}, 
\begin{align}
& \frac{2s^2}{\gamma_y} E \left\Vert  D^\by_{t+1} - D^\star_\by  \right\Vert^2_{(I-W)^\dagger} + \left\Vert \nu^\by_{t+1} - H^\star_\by \right\Vert^2 \nonumber\\
& = \frac{2s^2}{\gamma_y} \left\Vert D^\by_{t} - D^\star_\by \right\Vert^2_{(I-W)^\dagger} + \frac{\gamma_y}{2}  \left\Vert \nu^\by_{t+1}  - H^\star_\by \right\Vert^2_{I-W} + \frac{\gamma_y}{2}  E \left\Vert \hat{\nu}^\by_{t+1}  - \nu^\by_{t+1} \right\Vert^2_{I-W}  - 2s^2 \left\Vert  D^\by_{t} - D^\star_\by  \right\Vert^2 \nonumber\\ & \ \ + s^2 \left\Vert D^\by_{t} - D^\star_\by \right\Vert^2 + \left\Vert \by_{t} - \mathbf{1}y^\star + s\mathcal{G}^\by_{t} - s\nabla_y F(\mathbf{1}z^\star)  \right\Vert^2 \nonumber \\
& = \frac{2s^2}{\gamma_y} \left\Vert D^\by_{t} - D^\star_\by \right\Vert^2_{(I-W)^\dagger} + \frac{\gamma_y}{2} \left\Vert \nu^\by_{t+1}  - H^\star_\by \right\Vert^2_{I-W} + \frac{\gamma_y}{2}   E\left\Vert \hat{\nu}^\by_{t+1}  - \nu^\by_{t+1} \right\Vert^2_{I-W}  - s^2 \left\Vert  D^\by_{t} - D^\star_\by  \right\Vert^2 \nonumber\\ & \ \ + \left\Vert \by_{t} - \mathbf{1}y^\star + s\mathcal{G}^\by_{t} - s\nabla_y F(\mathbf{1}z^\star)  \right\Vert^2 . \label{eq:Dy_t+1_nuy_t+1}
\end{align}

We have
\begin{align}
	& \left\Vert \nu^\by_{t+1} - H^\star_\by \right\Vert^2 - \frac{\gamma_y}{2}\left\Vert \nu^\by_{t+1}  - H^\star_\by \right\Vert^2_{I-W} \nonumber \\
	& = (\nu^\by_{t+1} - H^\star_\by)^\top I (\nu^\by_{t+1} - H^\star_\by) -(\nu^\by_{t+1} - H^\star_\by)^\top  \frac{\gamma_y}{2}(I-W) (\nu^\by_{t+1} - H^\star_\by) \nonumber \\
	& = (\nu^\by_{t+1} - H^\star_\by)^\top \left(I - \frac{\gamma_y}{2}(I-W) \right)(\nu^\by_{t+1} - H^\star_\by) \nonumber\\
	& = \left\Vert \nu^\by_{t+1} - H^\star_\by \right\Vert^2_{I - \frac{\gamma_y}{2}(I-W)} \nonumber.
\end{align}
Therefore,
\begin{align}
	& \frac{\gamma_y}{2}\left\Vert \nu^\by_{t+1}  - H^\star_\by \right\Vert^2_{I-W} = \left\Vert \nu^\by_{t+1} - H^\star_\by \right\Vert^2 - \left\Vert \nu^\by_{t+1} - H^\star_\by \right\Vert^2_{I - \frac{\gamma_y}{2}(I-W)} .
\end{align}
By substituting above equality in \eqref{eq:Dy_t+1_nuy_t+1}, we obtain
\begin{align}
	\frac{2s^2}{\gamma_y} E \left\Vert  D^\by_{t+1} - D^\star_\by  \right\Vert^2_{(I-W)^\dagger} + \left\Vert \nu^\by_{t+1} - H^\star_\by \right\Vert^2 & = \frac{2s^2}{\gamma_y} \left\Vert D^\by_{t} - D^\star_\by \right\Vert^2_{(I-W)^\dagger} + \left\Vert \nu^\by_{t+1} - H^\star_\by \right\Vert^2 \nonumber \\ & \ \ - \left\Vert \nu^\by_{t+1} - H^\star_\by \right\Vert^2_{I - \frac{\gamma_y}{2}(I-W)} + \frac{\gamma_y}{2}  E \left\Vert \hat{\nu}^\by_{t+1}  - \nu^\by_{t+1} \right\Vert^2_{I-W} \nonumber \\ & \ \  - s^2 \left\Vert  D^\by_{t} - D^\star_\by  \right\Vert^2  + \left\Vert \by_{t} - \mathbf{1}y^\star + s\mathcal{G}^\by_{t} - s\nabla_y F(\mathbf{1}z^\star)  \right\Vert^2 \nonumber 
\end{align}
Hence we get 
\begin{align}
& \frac{2s^2}{\gamma_y} E \left\Vert  D^\by_{t+1} - D^\star_\by  \right\Vert^2_{(I-W)^\dagger}  + \left\Vert \nu^\by_{t+1} - H^\star_\by \right\Vert^2_{I - \frac{\gamma_y}{2}(I-W)}  \nonumber \\ 
& = \frac{2s^2}{\gamma_y} \left\Vert D^\by_{t} - D^\star_\by \right\Vert^2_{(I-W)^\dagger} + \frac{\gamma_y}{2}  E \left\Vert \hat{\nu}^\by_{t+1}  - \nu^\by_{t+1} \right\Vert^2_{I-W}  + \left\Vert \by_{t} - \mathbf{1}y^\star + s\mathcal{G}^\by_{t} - s\nabla_y F(\mathbf{1}z^\star)  \right\Vert^2   - s^2 \left\Vert  D^\by_{t} - D^\star_\by  \right\Vert^2 . \label{eq:Dy_t+1nu_y_mid}
\end{align}
%
Now we write $\frac{2s^2}{\gamma_y} \left\Vert D^\by_{t} - D^\star_\by \right\Vert^2_{(I-W)^\dagger} -  s^2 \left\Vert  D^\by_{t} - D^\star_\by  \right\Vert^2$ in terms of  $\frac{s^2}{\gamma_y} \left\Vert  D^\by_{t} - D^\star_\by  \right\Vert^2_{2(I-W)^\dagger -\gamma_y I}$ .
\begin{align}
& \frac{2s^2}{\gamma_y} \left\Vert D^\by_{t} - D^\star_\by \right\Vert^2_{(I-W)^\dagger} -  s^2 \left\Vert  D^\by_{t} - D^\star_\by  \right\Vert^2 \nonumber\\
& = \frac{2s^2}{\gamma_y}(D^\by_{t} - D^\star_\by)^\top (I-W)^\dagger (D^\by_{t} - D^\star_\by) - s^2(D^\by_{t} - D^\star_\by)^\top (D^\by_{t} - D^\star_\by) \nonumber\\
& = \frac{s^2}{\gamma_y} \left((D^\by_{t} - D^\star_\by)^\top (2(I-W)^\dagger -\gamma_y I)(D^\by_{t} - D^\star_\by) \right) \nonumber\\
& = \frac{s^2}{\gamma_y} \left\Vert  D^\by_{t} - D^\star_\by  \right\Vert^2_{2(I-W)^\dagger -\gamma_y I} .
\end{align}
By substituting this equality in \eqref{eq:Dy_t+1nu_y_mid}, we obtain
\begin{align}
& \frac{2s^2}{\gamma_y} E \left\Vert  D^\by_{t+1} - D^\star_\by  \right\Vert^2_{(I-W)^\dagger} + \left\Vert \nu^\by_{t+1} - H^\star_\by \right\Vert^2_{I - \frac{\gamma_y}{2}(I-W)} \nonumber\\ 
& = \frac{s^2}{\gamma_y} \left\Vert D^\by_{t} - D^\star_\by \right\Vert^2_{{2(I-W)^\dagger -\gamma_y I}} + \frac{\gamma_y}{2}  E \left\Vert \hat{\nu}^\by_{t+1}  - \nu^\by_{t+1} \right\Vert^2_{I-W}   + \left\Vert \by_{t} - \mathbf{1}y^\star + s\mathcal{G}^\by_{t} - s\nabla_y F(\mathbf{1}z^\star)  \right\Vert^2 . \label{eq:Dyt+1_Dstar_final}
\end{align}

\textbf{Step 3: Computing   $\sqrt{\delta}E\left\Vert H^\by_{t+1} - H^\star_\by \right\Vert^2$ and finishing the proof}

Observe from Step 4 in Algorithm \ref{comm} that $H^\by_{t+1}  = (1-\alpha_{y})H^\by_{t} + \alpha_{y}\hat{\nu}^\by_{t+1}$, and as a result, 
\begin{align}
H^\by_{t+1} - H^\star_\by & = (1-\alpha_{y})(H^\by_{t} - H^\star_\by ) + \alpha_{y}(\hat{\nu}^\by_{t+1} - \nu^\by_{t+1} ) + \alpha_{y}(\nu^\by_{t+1} - H^\star_\by ), \text{ and} \nonumber\\
\left\Vert H^\by_{t+1} - H^\star_\by \right\Vert^2 & = \left\Vert (1-\alpha_{y})(H^\by_{t} - H^\star_\by ) + \alpha_{y}(\nu^\by_{t+1} - H^\star_\by ) \right\Vert^2 + \alpha_{y}^2 \left\Vert \hat{\nu}^\by_{t+1} - \nu^\by_{t+1} \right\Vert^2 \nonumber\\ & \ \ + 2 \left\langle (1-\alpha_{y})(H^\by_{t} - H^\star_\by ) + \alpha_{y}(\nu^\by_{t+1} - H^\star_\by ) , \alpha_{y}(\hat{\nu}^\by_{t+1} - \nu^\by_{t+1}) \right\rangle .
\end{align}

Taking conditional expectation over compression at $t$-th iterate on both sides and substituting $E\left(\hat{\nu}^\by_{t+1} - \nu^\by_{t+1} \right) = 0$, we see that
\begin{align}
E\left\Vert H^\by_{t+1} - H^\star_\by \right\Vert^2 & = \left\Vert (1-\alpha_{y})(H^\by_{t} - H^\star_\by ) + \alpha_{y}(\nu^\by_{t+1} - H^\star_\by ) \right\Vert^2 + \alpha_{y}^2 E \left\Vert \hat{\nu}^\by_{t+1} - \nu^\by_{t+1} \right\Vert^2 \nonumber\\
& = (1-\alpha_{y})\left\Vert H^\by_{t} - H^\star_\by  \right\Vert^2 + \alpha_{y} \left\Vert \nu^\by_{t+1} - H^\star_\by  \right\Vert^2 - \alpha_{y}(1-\alpha_{y,k}) \left\Vert \nu^\by_{t+1} - H^\by_{t}  \right\Vert^2 \nonumber\\ & \ \ + \alpha_{y}^2 E \left\Vert \hat{\nu}^\by_{t+1} - \nu^\by_{t+1} \right\Vert^2 . \label{eq:Hykt_Hstar}
\end{align}
The last equality follows from the identity $\left\Vert (1-\alpha)x +\alpha y \right\Vert^2 = (1-\alpha)\left\Vert x \right\Vert^2 + \alpha \left\Vert y \right\Vert^2 -\alpha(1-\alpha)\left\Vert x- y \right\Vert^2$ .
On multiplying both sides of \eqref{eq:Hykt_Hstar} by $\sqrt{\delta}$, we obtain
\begin{align}
\sqrt{\delta}E\left\Vert H^\by_{t+1} - H^\star_\by \right\Vert^2 & = \sqrt{\delta}(1-\alpha_{y})\left\Vert H^\by_{t} - H^\star_\by  \right\Vert^2 + \sqrt{\delta} \alpha_{y} \left\Vert \nu^\by_{t+1} - H^\star_\by  \right\Vert^2 + \sqrt{\delta}\alpha_{y}^2 E \left\Vert \hat{\nu}^\by_{t+1} - \nu^\by_{t+1} \right\Vert^2 \nonumber\\ & \ \ - \sqrt{\delta}\alpha_{y}(1-\alpha_{y}) \left\Vert \nu^\by_{t+1} - H^\by_{t}  \right\Vert^2 . \label{eq:Hyt+1Hstar_norm}
\end{align}

We know that 
\begin{align}
\left\Vert \nu^\by_{t+1} - H^\star_\by  \right\Vert^2_{I - \frac{\gamma_y}{2}(I-W) - \alpha_{y}\sqrt{\delta}I} & = 	\left\Vert \nu^\by_{t+1} - H^\star_\by  \right\Vert^2_{I - \frac{\gamma_y}{2}(I-W)} - \left\Vert \nu^\by_{t+1} - H^\star_\by \right\Vert^2_{\alpha_{y}\sqrt{\delta}I} .
\end{align}
Therefore,
\begin{align}
& \frac{2s^2}{\gamma_y} E \left\Vert  D^\by_{t+1} - D^\star_\by  \right\Vert^2_{(I-W)^\dagger} + \left\Vert \nu^\by_{t+1} - H^\star_\by  \right\Vert^2_{I - \frac{\gamma_y}{2}(I-W) - \alpha_{y}\sqrt{\delta}I} \nonumber\\
& = \frac{2s^2}{\gamma_y} E \left\Vert  D^\by_{t+1} - D^\star_\by  \right\Vert^2_{(I-W)^\dagger} + \left\Vert \nu^\by_{t+1} - H^\star_\by  \right\Vert^2_{I - \frac{\gamma_y}{2}(I-W)} - \left\Vert \nu^\by_{t+1} - H^\star_\by \right\Vert^2_{\alpha_{y}\sqrt{\delta}I}\nonumber \\
& = \frac{2s^2}{\gamma_y} E \left\Vert  D^\by_{t+1} - D^\star_\by  \right\Vert^2_{(I-W)^\dagger} + \left\Vert \nu^\by_{t+1} - H^\star_\by  \right\Vert^2_{I - \frac{\gamma_y}{2}(I-W)} - \alpha_{y}\sqrt{\delta} \left\Vert \nu^\by_{t+1} - H^\star_\by \right\Vert^2\nonumber \\
& = \frac{s^2}{\gamma_y} \left\Vert D^\by_{t} - D^\star_\by \right\Vert^2_{{2(I-W)^\dagger -\gamma_y I}} + \frac{\gamma_y}{2}  E \left\Vert \hat{\nu}^\by_{t+1}  - \nu^\by_{t+1} \right\Vert^2_{I-W}   + \left\Vert \by_{t} - \mathbf{1}y^\star + s\mathcal{G}^\by_{t} - s\nabla_y F(\mathbf{1}z^\star)  \right\Vert^2 \nonumber\\ & \ \ \ - \alpha_{y}\sqrt{\delta} \left\Vert \nu^\by_{t+1} - H^\star_\by \right\Vert^2 ,
\end{align}
where the last equality follows from \eqref{eq:Dyt+1_Dstar_final} . Now we add above equality with \eqref{eq:Hyt+1Hstar_norm} and obtain the following expression:
\begin{align}
& \frac{2s^2}{\gamma_y} E \left\Vert  D^\by_{t+1} - D^\star_\by  \right\Vert^2_{(I-W)^\dagger} + \left\Vert \nu^\by_{t+1} - H^\star_\by  \right\Vert^2_{I - \frac{\gamma_y}{2}(I-W) - \alpha_{y}\sqrt{\delta}I} + \sqrt{\delta}E\left\Vert H^\by_{t+1} - H^\star_\by \right\Vert^2 \nonumber\\
& = \frac{s^2}{\gamma_y} \left\Vert D^\by_{t} - D^\star_\by \right\Vert^2_{{2(I-W)^\dagger -\gamma_y I}} + \frac{\gamma_y}{2}  E \left\Vert \hat{\nu}^\by_{t+1}  - \nu^\by_{t+1} \right\Vert^2_{I-W}   + \left\Vert \by_{t} - \mathbf{1}y^\star + s\mathcal{G}^\by_{t} - s\nabla_y F(\mathbf{1}z^\star)  \right\Vert^2 \nonumber \\ & \ + \sqrt{\delta}(1-\alpha_{y})\left\Vert H^\by_{t} - H^\star_\by  \right\Vert^2 + \sqrt{\delta}\alpha_{y}^2 E \left\Vert \hat{\nu}^\by_{t+1} - \nu^\by_{t+1} \right\Vert^2 - \sqrt{\delta}\alpha_{y}(1-\alpha_{y}) \left\Vert \nu^\by_{t+1} - H^\by_{t}  \right\Vert^2 .
\end{align}
Rearranging the r.h.s. of the above equation, we obtain~\eqref{eq:Dy_Hy_sum}. 

We have a similar recursion result in terms of $\bx$. 

\begin{lemma} \label{lem:recursion_x} 
	Let $\bx_{t+1}$, $\by_{t+1}$, $D^\bx_{t+1}$, $D^\by_{t+1}$, $H^\bx_{t+1}$, $H^\by_{t+1}$, $H^{\ssw,\bx}_{t+1}$, $H^{\ssw,\by}_{t+1}$ be obtained from Algorithm \ref{alg:generic_procedure_sgda} using \text{IPDHG}($\bx_{t}$, $\by_{t}$, $D^{\bx}_{t}$, $D^{\by}_{t}$, $H^{\bx}_{t}$, $H^{\by}_{t}$, $H^{\ssw,\bx}_{t}$, $H^{\ssw,\by}_{t}$, $s$, $\gamma_{x}$, $\gamma_{y}$, $\alpha_{x}$, $\alpha_{y}$, $\mathcal{G}$). 
	Let Assumptions \ref{compression_operator}- \ref{weight_matrix_assumption} hold. Suppose $\alpha_x \in \left(0,(1+\delta)^{-1} \right)$ and 
	\begin{align}
		\gamma_x & \in \left( 0, \min \left\lbrace \frac{2-2\sqrt{\delta}\alpha_x}{\lambda_{\max}(I-W)}, \frac{\alpha_x - (1+\delta)\alpha_x^2}{\sqrt{\delta}\lambda_{\max}(I-W)} \right\rbrace  \right) .
	\end{align}
	Then the following holds for all $t \geq 0$:
	\begin{align}
		& M_x E\left\Vert \bx_{t+1} - \mathbf{1}x^\star  \right\Vert^2 + \frac{2s^2}{\gamma_\bx} E \left\Vert  D^\bx_{t+1} - D^\star_\bx  \right\Vert^2_{(I-W)^\dagger} + \sqrt{\delta}E\left\Vert H^\bx_{t+1} - H^\star_\bx \right\Vert^2 \nonumber\\
		& \leq \left\Vert \bx_{t} - \mathbf{1}x^\star - s\mathcal{G}^\bx_{t} + s\nabla_x F(\mathbf{1}x^\star,\mathbf{1}y^\star)  \right\Vert^2 + \frac{2s^2}{\gamma_\bx}\left( 1- \frac{\gamma_\bx}{2}\lambda_{m-1}(I-W) \right)\left\Vert D^\bx_{t} - D^\star_\by \right\Vert^2_{(I-W)^\dagger} \nonumber\\ & \ \ + \sqrt{\delta}(1-\alpha_{x})\left\Vert H^\bx_{t} - H^\star_\bx  \right\Vert^2  , \label{eq:one_step_progress_x}
	\end{align}
	where $E$ denotes the conditional expectation on stochastic compression at $t$-th update step and $M_{x} = 1-\frac{\sqrt{\delta }\alpha_{x}}{1-\frac{\gamma_{x}}{2}\lambda_{\max}(I-W)}$.
\end{lemma}
We omit the proof of Lemma \ref{lem:recursion_x} as it is similar to the proof of Lemma \ref{lem:recursion_y}

\section{Convergence Behavior of Algorithm \ref{alg:IPDHG_with_sgd_svrg_oracle} with GSGO}
\label{appendix_C}
We define a quantity $\Phi_{t}$ consisting of primal and dual updates which is instrumental in deriving the convergence behavior of Algorithm \ref{alg:IPDHG_with_sgd_svrg_oracle}.
\begin{align}
	& \Phi_{t} =  M_{x,0} \left\Vert \bx_{t} - \mathbf{1}x^\star  \right\Vert^2 + \frac{2s_0^2}{\gamma_{x,0}} \left\Vert  D^\bx_{t} - D^\star_\bx  \right\Vert^2_{(I-W)^\dagger} + \sqrt{\delta}\left\Vert H^\bx_{t} - H^\star_{\bx,0} \right\Vert^2 \nonumber\\ & \hspace*{1cm} + M_{y,0} \left\Vert \by_{t} - \mathbf{1}y^\star  \right\Vert^2 + \frac{2s_0^2}{\gamma_{y,0}} \left\Vert  D^\by_{t} - D^\star_\by  \right\Vert^2_{(I-W)^\dagger} + \sqrt{\delta}\left\Vert H^\by_{t} - H^\star_{\by,0} \right\Vert^2, \label{phi_kt_defn}  \\
	& H^\star_{\bx,0} = \mathbf{1}x^\star  - \frac{s_0}{m}\mathbf{1} \nabla_x f(x^\star,y^\star), \ \ H^\star_{\by,0} = \mathbf{1}y^\star  + \frac{s_0}{m}\mathbf{1} \nabla_y f(x^\star,y^\star). \nonumber 
\end{align}

\begin{lemma} \label{lem:exp_ykt_ystar_xkt_xstar} Suppose $\{\bx_{t}\}_t$ and $\{\by_{t} \}_t$ are the sequences generated by Algorithm \ref{alg:IPDHG_with_sgd_svrg_oracle}. Then, under Assumptions \ref{s_convexity_assumption}- \ref{s_concavity_assumption}, Assumptions \ref{smoothness_x_svrg}-\ref{lipschitz_yx_svrg}, the followings hold for all $ 0 \leq t \leq T_0$:
	\begin{align}
		& E \left\Vert \bx_{t} - \mathbf{1}x^\star - s_0\mathcal{G}^\bx_{t} + s_0\nabla_x F(\mathbf{1}x^\star,\mathbf{1}y^\star)  \right\Vert^2 \nonumber \\
		& \leq  (1-\mu_x s_0) \left\Vert \bx_{t} - \mathbf{1}x^\star  \right\Vert^2 + \frac{4s_0^2L^2_{xy}}{np_{\min}} \left\Vert \by_{t} - \mathbf{1}y^\star \right\Vert^2 - (2s_0 - \frac{8s_0^2L_{xx}}{np_{\min}})\sum_{i = 1}^m V_{f_i,y^i_{t}}(x^\star,x^i_{t}) \nonumber \\ & \ + 2s_0\left( F(\bx_{t},\mathbf{1}y^\star) - F(\mathbf{1}x^\star,\mathbf{1}y^\star) + F(\mathbf{1}x^\star,\by_{t}) - F(z_{t}) \right) + \frac{2s_0^2}{n^2p_{\min}} \sum_{i = 1}^m \sum_{l = 1}^n \left\Vert \nabla_x f_{il}(z^\star) \right\Vert^2, \label{eq:xkt_xstar_grad_final}
	\end{align}
	\begin{align}
		& E \left\Vert \by_{t} - \mathbf{1}y^\star + s_0\mathcal{G}^\by_{t} - s_0\nabla_y F(\mathbf{1}x^\star,\mathbf{1}y^\star)  \right\Vert^2 \nonumber \\
		& \leq (1-\mu_y s_0) \left\Vert \by_{t} - \mathbf{1}y^\star  \right\Vert^2 + \frac{4s_0^2L^2_{yx}}{np_{\min}}  \left\Vert \bx_{t} - \mathbf{1}x^\star \right\Vert^2 - (2s_0 - \frac{8s_0^2L_{yy}}{np_{\min}})\sum_{i = 1}^m V_{-f_i,x^i_{t}}(y^\star,y^i_{t}) \nonumber \\ & \ + 2s_0\left( -F(\bx_{t},\mathbf{1}y^\star) + F(\mathbf{1}x^\star,\mathbf{1}y^\star) - F(\mathbf{1}x^\star,\by_{t}) + F(z_{t}) \right) + \frac{2s_0^2}{n^2p_{\min}}\sum_{i = 1}^m \sum_{l = 1}^n  \left\Vert \nabla_y f_{il}(z^\star) \right\Vert^2, \label{eq:ykt_ystar_grad_final}
	\end{align}
	where $E$ denotes the conditional expectation on stochastic gradient at $t$-th update step.
\end{lemma}

\begin{proof}
We will derive inequality \eqref{eq:xkt_xstar_grad_final} here. The proof of inequality \eqref{eq:ykt_ystar_grad_final} is similar and is omitted.

For $t \leq T_0$, we have $\mathcal{G}^{i,x}_t = \frac{1}{np_{il}}\nabla_x f_{il}(z^i_t)$ and $\mathcal{G}^{i,y}_t = \frac{1}{np_{il}}\nabla_y f_{il}(z^{i}_t)$ and step size is $s_0$ . We have
\begin{align}
	& E \left\Vert \bx_{t} - \mathbf{1}x^\star - s\mathcal{G}^\bx_{t} + s\nabla_x F(\mathbf{1}x^\star,\mathbf{1}y^\star)  \right\Vert^2 \nonumber\\
	& = \sum_{i = 1}^m E \left\Vert x^i_{t} - x^\star - s\mathcal{G}^{i,x}_{t} + s\nabla_x f_i(x^\star,y^\star)  \right\Vert^2 \nonumber\\
	& = \sum_{i = 1}^m \left\Vert x^i_{t} - x^\star  \right\Vert^2 + s^2 \sum_{i = 1}^mE \left\Vert \mathcal{G}^{i,x}_{t} - \nabla_x f_i(x^\star,y^\star) \right\Vert^2 \nonumber\\ & \ \ - 2s \sum_{i = 1}^m E \left\langle x^i_{t} - x^\star, \mathcal{G}^{i,x}_{t} - \nabla_x f_i(x^\star,y^\star) \right\rangle\nonumber \\
	& = \sum_{i = 1}^m \left\Vert x^i_{t} - x^\star  \right\Vert^2 + s^2 \sum_{i = 1}^mE \left\Vert \frac{1}{np_{il}}  \nabla_x f_{il}(z^i_{t}) - \nabla_x f_i(x^\star,y^\star) \right\Vert^2 \nonumber\\ & \ \ - 2s \sum_{i = 1}^m  \left\langle x^i_{t} - x^\star,\nabla_x f_{i}(z^i_{t}) - \nabla_x f_i(x^\star,y^\star) \right\rangle \label{eq:xt_xstar_sgd_oracle}
\end{align}
We first simplify second term in the r.h.s of above expression.
\begin{align}
	& \sum_{i = 1}^mE \left\Vert \frac{1}{np_{il}}  \nabla_x f_{il}(z^i_{t}) - \nabla_x f_i(x^\star,y^\star) \right\Vert^2 \nonumber \\ 
	& = \sum_{i = 1}^mE \left\Vert \frac{1}{np_{il}}  \nabla_x f_{il}(z^i_{t}) - \frac{1}{np_{il}}  \nabla_x f_{il}(z^\star) + \frac{1}{np_{il}}  \nabla_x f_{il}(z^\star)  - \nabla_x f_i(x^\star,y^\star) \right\Vert^2 \nonumber \\
	& \leq 2 \sum_{i = 1}^mE \left\Vert \frac{1}{np_{il}}  \nabla_x f_{il}(z^i_{t}) - \frac{1}{np_{il}}  \nabla_x f_{il}(z^\star) \right\Vert^2 + \sum_{i = 1}^m 2E\left\Vert \frac{1}{np_{il}}  \nabla_x f_{il}(z^\star)  - \nabla_x f_i(x^\star,y^\star) \right\Vert^2 \nonumber \\
	& = 2\sum_{i = 1}^m \sum_{l = 1}^n p_{il} \left\Vert \frac{1}{np_{il}}  \nabla_x f_{il}(z^i_{t}) - \frac{1}{np_{il}}  \nabla_x f_{il}(z^\star) \right\Vert^2 + 2\sum_{i = 1}^m E\left\Vert \frac{1}{np_{il}}  \nabla_x f_{il}(z^\star)  - \nabla_x f_i(x^\star,y^\star) \right\Vert^2 \nonumber \\
	& \leq \frac{2}{n^2p_{\min}}\sum_{i = 1}^m \sum_{l = 1}^n \left\Vert  \nabla_x f_{il}(z^i_{t}) - \nabla_x f_{il}(z^\star) \right\Vert^2 + 2\sum_{i = 1}^m E\left\Vert \frac{1}{np_{il}}  \nabla_x f_{il}(z^\star)  - \nabla_x f_i(x^\star,y^\star) \right\Vert^2 \nonumber \\
	& = \frac{2}{n^2p_{\min}}\sum_{i = 1}^m \sum_{l = 1}^n \left\Vert  \nabla_x f_{il}(z^i_{t}) - \nabla_x f_{il}(z^\star) \right\Vert^2 + 2\sum_{i = 1}^m E\left\Vert \frac{1}{np_{il}}  \nabla_x f_{il}(z^\star)  - E \left( \frac{1}{np_{il}}  \nabla_x f_{il}(z^\star) \right) \right\Vert^2 \nonumber \\
	& \leq \frac{2}{n^2p_{\min}}\sum_{i = 1}^m \sum_{l = 1}^n \left\Vert  \nabla_x f_{il}(z^i_{t}) - \nabla_x f_{il}(z^\star) \right\Vert^2 + 2\sum_{i = 1}^m E\left\Vert \frac{1}{np_{il}}  \nabla_x f_{il}(z^\star) \right\Vert^2 . \nonumber
\end{align}
The last inequality follows from $E\left\Vert u_i - Eu_i \right\Vert^2 \leq E\left\Vert u_i \right\Vert^2$. Therefore,
\begin{align}
	& \sum_{i = 1}^mE \left\Vert \frac{1}{np_{il}}  \nabla_x f_{il}(z^i_{t}) - \nabla_x f_i(x^\star,y^\star) \right\Vert^2 \nonumber \\ 
	& \leq \frac{2}{n^2p_{\min}}\sum_{i = 1}^m \sum_{l = 1}^n \left\Vert  \nabla_x f_{il}(z^i_{t}) - \nabla_x f_{il}(z^\star) \right\Vert^2 + 2\sum_{i = 1}^m E\left\Vert \frac{1}{np_{il}}  \nabla_x f_{il}(z^\star) \right\Vert^2 \nonumber \\
	& = \frac{2}{n^2p_{\min}}\sum_{i = 1}^m \sum_{l = 1}^n \left\Vert  \nabla_x f_{il}(z^i_{t}) - \nabla_x f_{il}(z^\star) \right\Vert^2 + 2\sum_{i = 1}^m \sum_{l = 1}^n p_{il}\left\Vert \frac{1}{np_{il}}  \nabla_x f_{il}(z^\star) \right\Vert^2 \nonumber \\
	& \leq \frac{2}{n^2p_{\min}}\sum_{i = 1}^m \sum_{l = 1}^n \left\Vert  \nabla_x f_{il}(z^i_{t}) - \nabla_x f_{il}(z^\star) \right\Vert^2 + \frac{2}{n^2p_{\min}} \sum_{i = 1}^m \sum_{l = 1}^n \left\Vert \nabla_x f_{il}(z^\star) \right\Vert^2 \nonumber \\
	& \leq  \frac{4}{n^2p_{\min}} \sum_{i = 1}^m \sum_{l = 1}^n  \left\Vert -\nabla_x f_{il}(z^i_{t}) + \nabla_x f_{il}(x^\star,y^i_{t}) \right\Vert^2 + \frac{4}{n^2p_{\min}} \sum_{i = 1}^m \sum_{l = 1}^n  \left\Vert \nabla_x f_{il}(x^\star,y^i_{t}) -  \nabla_x f_{il}(z^\star) \right\Vert^2 \nonumber \\ \ \ \ & + \frac{2}{n^2p_{\min}} \sum_{i = 1}^m \sum_{l = 1}^n \left\Vert \nabla_x f_{il}(z^\star) \right\Vert^2 \nonumber \\
	& \leq \frac{8L_{xx}}{n^2p_{\min}} \sum_{i = 1}^m \sum_{l = 1}^n V_{f_{il},y^i_{t}}(x^\star,x^i_{t}) + \frac{4L^2_{xy}}{np_{\min}} \sum_{i = 1}^m  \left\Vert y^i_{t} - y^\star  \right\Vert^2 + + \frac{2}{n^2p_{\min}} \sum_{i = 1}^m \sum_{l = 1}^n \left\Vert \nabla_x f_{il}(z^\star) \right\Vert^2 ,\nonumber 
\end{align}
where the last inequality follows from Proposition \ref{prop:smoothness} and Assumptions \ref{lipschitz_xy_svrg}-\ref{lipschitz_yx_svrg}. By substituting above inequality in \eqref{eq:xt_xstar_sgd_oracle}
\begin{align}
	& E \left\Vert \bx_{t} - \mathbf{1}x^\star - s\mathcal{G}^\bx_{t} + s\nabla_x F(\mathbf{1}x^\star,\mathbf{1}y^\star)  \right\Vert^2 \nonumber\\ 
	& \leq \left\Vert \bx_{t} - \mathbf{1}x^\star  \right\Vert^2 + \frac{8s_0^2L_{xx}}{n^2p_{\min}} \sum_{i = 1}^m \sum_{l = 1}^n V_{f_{il},y^i_{t}}(x^\star,x^i_{t}) + \frac{4s_0^2L^2_{xy}}{np_{\min}} \sum_{i = 1}^m  \left\Vert y^i_{t} - y^\star  \right\Vert^2 +  \frac{2s_0^2}{n^2p_{\min}} \sum_{i = 1}^m \sum_{l = 1}^n \left\Vert \nabla_x f_{il}(z^\star) \right\Vert^2  \nonumber\\ & \ \ - 2s \sum_{i = 1}^m  \left\langle x^i_{t} - x^\star,\nabla_x f_{i}(z^i_{t}) - \nabla_x f_i(x^\star,y^\star) \right\rangle  \label{eq:xkt_xstar_grad_ip}
\end{align}
We now simplify the inner product $\left\langle \bx_{t} - \mathbf{1}x^\star, \nabla_x F(\mathbf{1}x^\star,\mathbf{1}y^\star) - \nabla_x F(z_{t}) \right\rangle$  term present in the r.h.s of \eqref{eq:xkt_xstar_grad_ip}. Recall the definition of Bregman distance $V_{f_i,y}(x_1,x_2)$:
\begin{align}
	V_{f_i,y^i_{t}}(x^\star,x^i_{t}) & = f_i(x^\star,y^i_{t}) - f_i(x^i_{t},y^i_{t}) - \left\langle \nabla_x f_i(x^i_{t},y^i_{t}), x^\star - x^i_{t} \right\rangle \\
	\left\langle \nabla_x f_i(z^i_{t}), -x^\star + x^i_{t} \right\rangle & = -f_i(x^\star,y^i_{t}) + f_i(z^i_{t}) + V_{f_i,y^i_{t}}(x^\star,x^i_{t}) . \label{eq:grad_zkt_inn_Vfi}
\end{align}
Using $\mu_x$ strong convexity of $f_i(\cdot,y)$, we have
\begin{align}
	f_i(x^i_{t},y^\star) & \geq f_i(z^\star) + \left\langle \nabla_x f_i(z^\star), x^i_{t} -x^\star \right\rangle + \frac{\mu_x}{2} \left\Vert x^i_{t} - x^\star \right\Vert^2 \\
	\left\langle \nabla_x f_i(z^\star), x^i_{t} -x^\star \right\rangle & \leq f_i(x^i_{t},y^\star) - f_i(z^\star) - \frac{\mu_x}{2} \left\Vert x^i_{t} - x^\star \right\Vert^2 . \label{eq:grad_zstar_inn_mu_x}
\end{align}

We now compute
\begin{align}
	& \left\langle \bx_{t} - \mathbf{1}x^\star, \nabla_x F(\mathbf{1}x^\star,\mathbf{1}y^\star) - \nabla_x F(z_{t}) \right\rangle  = \sum_{i = 1}^m \left\langle x^i_{t} - x^\star, \nabla_x f_i(z^\star) - \nabla_x f_i(z^i_{t}) \right\rangle \nonumber \\
	& \leq \sum_{i = 1}^m \left( f_i(x^i_{t},y^\star) - f_i(z^\star) - \frac{\mu_x}{2} \left\Vert x^i_{t} - x^\star \right\Vert^2 + f_i(x^\star,y^i_{t}) - f_i(z^i_{t}) - V_{f_i,y^i_{t}}(x^\star,x^i_{t}) \right)\nonumber \\
	& = F(\bx_{t},\mathbf{1}y^\star) - F(\mathbf{1}x^\star,\mathbf{1}y^\star) + F(\mathbf{1}x^\star,\by_{t}) - F(z_{t}) - \frac{\mu_x}{2}  \left\Vert \bx_{t} - \mathbf{1}x^\star \right\Vert^2 \nonumber \\ & \ \ \ \ - \sum_{i = 1}^m V_{f_i,y^i_{t}}(x^\star,x^i_{t}) , \label{eq:inner_prod_grad_zkt_zstar}
\end{align}
where the second last step follows from \eqref{eq:grad_zkt_inn_Vfi} and \eqref{eq:grad_zstar_inn_mu_x}. On substituting \eqref{eq:inner_prod_grad_zkt_zstar} in \eqref{eq:xkt_xstar_grad_ip}, we obtain
\begin{align}
	& E \left\Vert \bx_{t} - \mathbf{1}x^\star - s_0\mathcal{G}^\bx_{t} + s_0\nabla_x F(\mathbf{1}x^\star,\mathbf{1}y^\star)  \right\Vert^2 \nonumber \\
	& \leq \left\Vert \bx_{t} - \mathbf{1}x^\star  \right\Vert^2 + \frac{8s_0^2L_{xx}}{n^2p_{\min}} \sum_{i = 1}^m \sum_{l = 1}^n V_{f_{il},y^i_{t}}(x^\star,x^i_{t}) + \frac{4s_0^2L^2_{xy}}{np_{\min}} \left\Vert \by_{t} - \mathbf{1}y^\star \right\Vert^2 +  \frac{2s_0^2}{n^2p_{\min}} \sum_{i = 1}^m \sum_{l = 1}^n \left\Vert \nabla_x f_{il}(z^\star) \right\Vert^2  \nonumber\\ & \ \ + 2s_0(F(\bx_{t},\mathbf{1}y^\star) - F(\mathbf{1}x^\star,\mathbf{1}y^\star) + F(\mathbf{1}x^\star,\by_{t}) - F(z_{t})) - \mu_xs_0  \left\Vert \bx_{t} - \mathbf{1}x^\star \right\Vert^2 - 2s_0\sum_{i = 1}^m V_{f_i,y^i_{t}}(x^\star,x^i_{t})  \nonumber \\
	& = (1-\mu_x s_0) \left\Vert \bx_{t} - \mathbf{1}x^\star  \right\Vert^2 + \frac{4s_0^2L^2_{xy}}{np_{\min}} \left\Vert \by_{t} - \mathbf{1}y^\star \right\Vert^2 - (2s_0 - \frac{8s_0^2L_{xx}}{np_{\min}})\sum_{i = 1}^m V_{f_i,y^i_{t}}(x^\star,x^i_{t}) \nonumber \\ & \ + 2s_0\left( F(\bx_{t},\mathbf{1}y^\star) - F(\mathbf{1}x^\star,\mathbf{1}y^\star) + F(\mathbf{1}x^\star,\by_{t}) - F(z_{t}) \right) + \frac{2s_0^2}{n^2p_{\min}} \sum_{i = 1}^m \sum_{l = 1}^n \left\Vert \nabla_x f_{il}(z^\star) \right\Vert^2, \tag{\ref{eq:xkt_xstar_grad_final}}
\end{align}
completing the proof. 
\end{proof}
We now have the following corollary obtained by setting step size $s_0$ in Lemma \ref{lem:exp_ykt_ystar_xkt_xstar}.		
\begin{corollary} \label{cor:exp_xkt_ykt_xstar_ystar_compact} Let $s_0 = \frac{np_{\min}}{4\sqrt{2}\kappa_fL}$. Then, under the setting of Lemma \ref{lem:exp_ykt_ystar_xkt_xstar} ,
	\begin{align}
		& E \left\Vert \bx_{t} - \mathbf{1}x^\star - s_0\mathcal{G}^\bx_{t} + s_0\nabla_x F(\mathbf{1}x^\star,\mathbf{1}y^\star)  \right\Vert^2 +  E \left\Vert \by_{t} - \mathbf{1}y^\star + s_0\mathcal{G}^\by_{t} - s_0\nabla_y F(\mathbf{1}x^\star,\mathbf{1}y^\star)  \right\Vert^2 \nonumber  \\
		& \leq (1-b_{x,0})\left\Vert \bx_{t} - \mathbf{1}x^\star  \right\Vert^2 + (1- b_{y,0}) \left\Vert \by_{t} - \mathbf{1}y^\star  \right\Vert^2 + \frac{2s_0^2(C_x + C_y)}{n^2p_{\min}} \nonumber 
	\end{align}
	for all $ t \geq 0$, where $b_{x,0} = \mu_x s_0 - \frac{4s_0^2L^2_{yx}}{np_{\min}}$, $b_{y,0} = \mu_y s_0 - \frac{4s_0^2L^2_{xy}}{np_{\min}}$.
\end{corollary}
\begin{proof}
As the step size $s_0 = \frac{np_{\min}}{4\sqrt{2}L\kappa_f}$, we have $s_0 \leq \frac{1}{4L}$. We now show that the terms $(2s_0 - \frac{8s_0^2L_{xx}}{np_{\min}})$ and $(2s_0 - \frac{8s_0^2L_{yy}}{np_{\min}})$ appearing in~\eqref{eq:xkt_xstar_grad_final} and~\eqref{eq:ykt_ystar_grad_final} are non-negative:

\begin{align}
	2s_0 - \frac{8s_0^2L_{xx}}{np_{\min}} & = \frac{2np_{\min}}{4\sqrt{2}L\kappa_f } - \frac{8L_{xx}}{np_{\min}} \frac{n^2p^2_{\min}}{32L^2\kappa^2_f} \nonumber \\
	& = \frac{np_{\min}}{2\sqrt{2}L\kappa_f } - \frac{L_{xx}np_{\min}}{4L^2\kappa^2_f} \nonumber \\
	& \geq \frac{np_{\min}}{2\sqrt{2}L\kappa_f } - \frac{Lnp_{\min}}{4L^2\kappa^2_f} = \frac{np_{\min}}{2\sqrt{2}L\kappa_f } - \frac{np_{\min}}{4L\kappa^2_f } \nonumber \\
	& = \frac{np_{\min}}{2\sqrt{2}L\kappa_f }\left(1 - \frac{1}{\sqrt{2}\kappa_f} \right) \geq 0 .
\end{align}
Similarly, we get $2s_0 - 8s^2_0L_{yy} \geq 0$. Recall 
\begin{align}
	b_{x,0} = \mu_x s_0 - \frac{4s_0^2L^2_{xy}}{np_{\min}} & = \frac{\mu_x np_{\min}}{4\sqrt{2}L\kappa_f} - \frac{4L^2_{yx}n^2p^2_{\min}}{32L^2\kappa^2_f np_{\min}} \nonumber \\
	& \geq \frac{\mu  np_{\min}}{4\sqrt{2}L\kappa_f } - \frac{4L^2np_{\min}}{32L^2\kappa^2_f } \nonumber \\
	& = \frac{np_{\min}}{4\sqrt{2}\kappa^2_f } - \frac{np_{\min}}{8\kappa^2_f} \nonumber \\
	& = \left(1-\frac{1}{\sqrt{2}} \right)\frac{np_{\min}}{4\sqrt{2}\kappa^2_f } . \label{eq:axk_lb}
\end{align}
We now show that $b_{x,0} <  1$. 
\begin{align}
	b_{x,0} < \mu_x s_0 = \frac{\mu_xnp_{\min}}{4\sqrt{2}L\kappa_f } \leq \frac{\mu_x np_{\min}}{4\sqrt{2}L_{xx}\kappa_f} = \frac{np_{\min}}{4\sqrt{2}\kappa_x \kappa_f } < 1 . \label{eq:axk_ub}
\end{align}
Therefore, $b_{x,0} \in (0,1)$. In a similar fashion, we obtain $b_{y,0} \in (0,1)$. On adding \eqref{eq:xkt_xstar_grad_final} and \eqref{eq:ykt_ystar_grad_final}, we obtain
\begin{align}
	& E \left\Vert \bx_{t} - \mathbf{1}x^\star - s_0\mathcal{G}^\bx_{t} + s_0\nabla_x F(\mathbf{1}z^\star)  \right\Vert^2 +  E \left\Vert \by_{t} - \mathbf{1}y^\star + s_0\mathcal{G}^\by_{t} - s_0\nabla_y F(\mathbf{1}z^\star)  \right\Vert^2 \nonumber\\
	& \leq (1-\mu_x s_0 + \frac{4s_0^2L^2_{yx}}{np_{\min}}) \left\Vert \bx_{t} - \mathbf{1}x^\star  \right\Vert^2 + (1-\mu_x s_0 + \frac{4s_0^2L^2_{xy}}{np_{\min}}) \left\Vert \by_{t} - \mathbf{1}y^\star  \right\Vert^2 \nonumber\\ & \ - (2s_0 - \frac{8s_0^2L_{xx}}{np_{\min}})\sum_{i = 1}^m V_{f_i,y^i_{t}}(x^\star,x^i_{t}) -(2s_0 - \frac{8s_0^2L_{yy}}{np_{\min}})\sum_{i = 1}^m V_{-f_i,x^i_{t}}(y^\star,y^i_{t}) + \frac{2s_0^2}{n^2p_{\min}}(C_x + C_y) \nonumber\\
	& \leq (1-b_{x,0})\left\Vert \bx_{t} - \mathbf{1}x^\star  \right\Vert^2 + (1- b_{y,0}) \left\Vert \by_{t} - \mathbf{1}y^\star  \right\Vert^2 + \frac{2s_0^2}{n^2p_{\min}}(C_x + C_y). \label{bx_by_bound}
\end{align}
The last inequality follows from non-negativity of $V_{f_i,y^i_{t}}(x^\star,x^i_{t}), V_{-f_i,x^i_{t}}(y^\star,y^i_{t}), 2s_0 - \frac{8s_0^2L_{xx}}{np_{\min}}$ and $2s_0 - \frac{8s_0^2L_{yy}}{np_{\min}}$.	
\end{proof}

\subsection{Parameters Setting and their Feasibility}  
\label{appendix_parameters_feasibility}

\textbf{Parameters setting:} From Corollary \ref{cor:exp_xkt_ykt_xstar_ystar_compact}, the step size used in Algorithm \ref{alg:IPDHG_with_sgd_svrg_oracle} is $s_0 = \frac{np_{\min}}{4\sqrt{2}\kappa_f L }$. We choose the parameters involved in \textbf{COMM} procedure and other parameters $\gamma_{x,0}, \gamma_{y,0}$ as follows:
\begin{align}
	& \alpha_{x,0} = \frac{b_{x,0}}{1+\delta} , \ \alpha_{y,0} = \frac{b_{y,0}}{1+\delta} \\
	&  \gamma_{x,0} = \min \left\lbrace \frac{b_{x,0}}{4\sqrt{\delta}(1+\delta)\lambda_{\max}(I-W)} , \frac{1}{4(1+\delta)\lambda_{\max}(I-W)} \right\rbrace   \label{eq:param1}  \\
	&  \gamma_{y,0} := \min \left\lbrace \frac{b_{y,0}}{4\sqrt{\delta}(1+\delta)\lambda_{\max}(I-W)} , \frac{1}{4(1+\delta)\lambda_{\max}(I-W)} \right\rbrace \label{eq:gaamy_sgd}  \\
	& M_{x,0} = 1-\frac{\sqrt{\delta }\alpha_{x,0}}{1-\frac{\gamma_{x,0}}{2}\lambda_{\max}(I-W)} , \ M_{y,0} = 1-\frac{\sqrt{\delta}\alpha_{y,0}}{1-\frac{\gamma_{y,0}}{2}\lambda_{\max}(I-W)} \\
	&	\rho_0 = \max \left\lbrace 1-\frac{3b_{x,0}}{7} , 1-\frac{3b_{y,0}}{7} ,  1- \frac{\gamma_{x,0}}{2} \lambda_{m-1}(I-W) , 1- \frac{\gamma_{y,0}}{2} \lambda_{m-1}(I-W) , 1-\alpha_{x,0}, 1-\alpha_{y,0}  \right\rbrace \label{rho_sgd} \\
	& \tilde{\rho}_0  = \min \Big \{ \left(1 - \frac{1}{\sqrt{2}} \right)\frac{3np_{\min}}{28\sqrt{2}\kappa_f^2}  , \frac{1}{8(1+\delta)\kappa_g}, \left(1 - \frac{1}{\sqrt{2}} \right) \frac{np_{\min}}{32\sqrt{2\delta}(1+\delta)} \frac{1}{\kappa_f^2 \kappa_g} \nonumber \\ & \hspace*{2cm}, \frac{1}{1+\delta} \left(1 - \frac{1}{\sqrt{2}} \right) \frac{np_{\min}}{4\sqrt{2}\kappa_f^2} \Big \} \label{tilde_rho_0} 
\end{align}

\textbf{Parameters Feasibility:} Above choice of parameters should satisfy the following conditions:
\begin{align}
	& \alpha^x_{0} < \min \left\lbrace \frac{b_{x,0}}{\sqrt{\delta}} , \frac{1}{1+\delta} \right\rbrace , \  \alpha_{y,0} < \min \left\lbrace \frac{b_{y,0}}{\sqrt{\delta}} , \frac{1}{1+\delta} \right\rbrace  \label{eq:param2_start} \\
	& \gamma_{x,0} \in \left( 0, \min \left\lbrace \frac{2-2\sqrt{\delta}\alpha_{x,0}}{\lambda_{\max}(I-W)}, \frac{\alpha_{x,0} - (1+\delta)\alpha_{x,0}^2}{\sqrt{\delta}\lambda_{\max}(I-W)} \right\rbrace  \right) , \\
	& \gamma_{y,0} \in \left( 0, \min \left\lbrace \frac{2-2\sqrt{\delta}\alpha_{y,0}}{\lambda_{\max}(I-W)}, \frac{\alpha_{y,0} - (1+\delta)\alpha_{y,0}^2}{\sqrt{\delta}\lambda_{\max}(I-W)} \right\rbrace  \right) , \\
	& \frac{\gamma_{x,0}}{2}\lambda_{m-1}(I-W) \ \in \ (0,1) , \ \frac{\gamma_{y,0}}{2}\lambda_{m-1}(I-W) \ \in \  (0,1) , \\
	& M_{x,0} \in (0,1), \ M_{y,0} \in (0,1) , \\
	& \frac{1-b_{x,0}}{M_{x,0}} \in (0,1), \ \ \frac{1-b_{y,0}}{M_{y,0}} \in (0,1) .\label{eq:param2_end}
\end{align}

In this section, we show that all parameters specified in~\eqref{eq:param1} satisfy all requirements of~\eqref{eq:param2_start}-\eqref{eq:param2_end}.

\paragraph{Feasibility of ${\alpha_{x,0}}$ {and} $\alpha_{y,0}$.}

From~\eqref{eq:axk_lb} and \eqref{eq:axk_ub}, we have $0 < b_{x,0} < 1$. Therefore, $\alpha_{x,0} < \frac{1}{1+\delta}$. Moreover, $\frac{\sqrt{\delta}}{1+\delta} \leq 1/2$ as $\delta \in [0,1]$. Therefore, $\alpha_{x,0} \leq \frac{b_{x,0}}{2\sqrt{\delta}} < b_{x,0}/\sqrt{\delta}$. Hence, $\alpha_{x,0} < \min \left\lbrace \frac{b_{x,0}}{\sqrt{\delta}} , \frac{1}{1+\delta} \right\rbrace$ . Similarly, $\alpha_{y,0} < \min \left\lbrace \frac{b_{y,0}}{\sqrt{\delta}} , \frac{1}{1+\delta} \right\rbrace$ because $b_{y,0} \in (0,1)$.

\paragraph{Feasibility of $\gamma_{x,0}$ and $\gamma_{y,0}$.} If $\delta = 0$, $\gamma_{x,0} = \gamma_{y,0} = \frac{1}{4\lambda_{\max}(I-W)} < \frac{2}{\lambda_{\max}(I-W)}$. Therefore, without loss of generality we assume that $\delta > 0$. We consider two cases to verify the feasibility of $\gamma_{x,0}$ and $\gamma_{y,0}$. 

\textbf{Case I:} $b_{x,0} \leq \sqrt{\delta}$.

This gives $\gamma_{x,0} = \frac{b_{x,0}}{4\sqrt{\delta}(1+\delta)\lambda_{\max}(I-W)}$. Consider
\begin{align}
	\frac{\alpha_{x,0} - (1+\delta)\alpha_{x,0}^2}{\sqrt{\delta}\lambda_{\max}(I-W)} & = \frac{b_{x,0} - b_{x,0}^2}{\sqrt{\delta}(1+\delta)\lambda_{\max}(I-W)} .
\end{align}
Using \eqref{eq:axk_ub}, we have $b_{x,0}  \leq \frac{1}{4\kappa_x \kappa_f 2^{k/2}} < 0.25$. This allows us to use the inequality $2x-2x^2 \geq x/2$ for all $0 \leq x \leq 0.75$. Therefore,
\begin{align}
	\frac{\alpha_{x,0} - (1+\delta)\alpha_{x,0}^2}{\sqrt{\delta}\lambda_{\max}(I-W)} & > \frac{b_{x,0}}{4\sqrt{\delta}(1+\delta) \lambda_{\max}(I-W)} \nonumber\\
	& = \gamma_{x,0} .
\end{align}
We also have 
\begin{align}
	\frac{2-2\sqrt{\delta}\alpha_{x,0}}{\lambda_{\max}(I-W)} & = \left( 2 - \frac{2\sqrt{\delta}b_{x,0}}{1+\delta} \right) \frac{1}{\lambda_{\max}(I-W)} \geq \left( 2 - \frac{2\sqrt{\delta}}{1+\delta} \right) \frac{1}{\lambda_{\max}(I-W)} \nonumber\\
	& \geq \frac{1}{\lambda_{\max}(I-W)} > \frac{1}{4(1+\delta) \lambda_{\max}(I-W)} \nonumber\\
	& > \frac{b_{x,0}}{4\sqrt{\delta}(1+\delta) \lambda_{\max}(I-W)} \nonumber\\
	& = \gamma_{x,0} ,
\end{align}
where the second inequality uses the relation $\frac{\sqrt{\delta}}{1+\delta} \leq \frac{1}{2}$ and the last inequality uses $b_x \leq \sqrt{\delta}$. We know that $b_{y,0} \in (0,1)$. Therefore, by following similar steps, the chosen $\gamma_{y,0}$ is also feasible. 

\textbf{Case II:} $ b_{x,0} > \sqrt{\delta}$

This give $\gamma_{x,0} = \frac{1}{4(1+\delta)\lambda_{\max}(I-W)}$. 

\begin{align}
	\frac{\alpha_{x,0} - (1+\delta)\alpha_{x,0}^2}{\sqrt{\delta}\lambda_{\max}(I-W)} & = \frac{b_{x,0} - b_{x,0}^2}{\sqrt{\delta}(1+\delta)\lambda_{\max}(I-W)} \nonumber\\
	& \geq \frac{b_{x,0}}{4\sqrt{\delta}(1+\delta) \lambda_{\max}(I-W)} \nonumber\\
	& > \frac{1}{4(1+\delta) \lambda_{\max}(I-W)} \nonumber\\
	& = \gamma_{x,0} .
\end{align}
Consider
\begin{align}
	\frac{2-2\sqrt{\delta}\alpha_{x,0}}{\lambda_{\max}(I-W)} & = \left( 2 - \frac{2\sqrt{\delta}b_{x,0}}{1+\delta} \right) \frac{1}{\lambda_{\max}(I-W)} \nonumber\\
	& \geq \left( 2 - \frac{2\sqrt{\delta}}{1+\delta} \right) \frac{1}{\lambda_{\max}(I-W)} \nonumber\\
	& \geq \frac{1}{\lambda_{\max}(I-W)} \nonumber\\
	& > \frac{1}{4(1+\delta) \lambda_{\max}(I-W)} \nonumber\\
	& = \gamma_{x,0} .
\end{align}
Therefore, $\gamma_{x,0} < \min \left\lbrace \frac{\alpha_{x,0} - (1+\delta)\alpha_{x,0}^2}{\sqrt{\delta}\lambda_{\max}(I-W)} , \frac{2-2\sqrt{\delta}\alpha_{x,0}}{\lambda_{\max}(I-W)} \right\rbrace $.

As $\gamma_{x,0} < \frac{2-2\sqrt{\delta}\alpha_{x,0}}{\lambda_{\max}(I-W)} < \frac{2}{\lambda_{\max}(I-W)}$. Notice that $\lambda_{m-1}(I-W) < \lambda_{\max}(I-W)$ Therefore,
\begin{align}
	\frac{\gamma_{x,0}}{2}\lambda_{m-1}(I-W) < \frac{\gamma_{x,0}}{2}\lambda_{\max}(I-W) < 1 .
\end{align}
Similarly, $\frac{\gamma_{y,0}}{2} \lambda_{m-1}(I-W) < 1$.

\paragraph{Feasibility of $M_{x,0}$ and $M_{y,0}$.}

Recall $M_{x,0} = 1-\frac{\sqrt{\delta }\alpha_{x,0}}{1-\frac{\gamma_{x,0}}{2}\lambda_{\max}(I-W)}$ and $M_{y,0} = 1-\frac{\sqrt{\delta }\alpha_{y,0}}{1-\frac{\gamma_{y,0}}{2}\lambda_{\max}(I-W)}$.
We have 
\begin{align}
	\gamma_{x,0} &< \frac{2-2\sqrt{\delta}\alpha_{x,0}}{\lambda_{\max}(I-W)} \nonumber\\
	\Rightarrow \frac{\gamma_{x,0}\lambda_{\max}(I-W)}{2} &< 1-\sqrt{\delta}\alpha_{x,0} \nonumber\\
	\Rightarrow 1 - \frac{\gamma_{x,0}\lambda_{\max}(I-W)}{2} &> \sqrt{\delta}\alpha_{x,0} \nonumber\\
	\Rightarrow \frac{\sqrt{\delta}\alpha_{x,0}}{1 - \frac{\gamma_{x,0}\lambda_{\max}(I-W)}{2}} &< 1 .
\end{align}
Moreover, $\frac{\sqrt{\delta}\alpha_{x,0}}{1 - \frac{\gamma_{x,0}\lambda_{\max}(I-W)}{2}} > 0$. Therefore, $M_{x,0} \in (0,1) $. The feasibility of $M_{y,0}$ can be proved similarly.

\paragraph{Feasibility of $\frac{1-b_{x,0}}{M_{x,0}}$ and $\frac{1-b_{y,0}}{M_{y,0}}$.}
We derive upper bounds on $\frac{1-b_{x,0}}{M_{x,0}}$ and $\frac{1-b_{y,0}}{M_{y,0}} $ to verify the feasibility.
We divide the derivation into two cases.

\textbf{Case I:} $b_{x,0} \leq \sqrt{\delta}$ 

This implies that
\begin{align}
	\gamma_{x,0} = \frac{b_{x,0}}{4\sqrt{\delta}(1+\delta)\lambda_{\max}(I-W)} \\
	\frac{\gamma_{x,0}}{2} \lambda_{\max}(I-W) = \frac{b_{x}}{8\sqrt{\delta}(1+\delta)} .
\end{align}
Recall $M_{x,0}$:
\begin{align}
	M_{x,0} & = 1 - \frac{\sqrt{\delta }\alpha_{x,0}}{1-\frac{\gamma_{x,0}}{2}\lambda_{\max}(I-W)} \nonumber\\
	& = 1 - \frac{\frac{\sqrt{\delta }b_{x,0}}{1+\delta}}{1-\frac{b_{x,0}}{8\sqrt{\delta}(1+\delta)}} \nonumber\\
	& = 1 - \frac{\sqrt{\delta}b_{x,0} \times 8 \sqrt{\delta}(1+\delta)}{(1+\delta)\left(8\sqrt{\delta}(1+\delta) - b_{x,0} \right)} \nonumber\\
	& = 1 - \frac{8\delta b_{x,0}}{\left(8\sqrt{\delta}(1+\delta) - b_{x,0} \right)} \nonumber\\
	& = 1 - \frac{8\delta }{\frac{8\sqrt{\delta}(1+\delta)}{b_{x,0}} - 1} . \label{eq:Mx_sgd_lb_mid}
\end{align}
We know that $\frac{\sqrt{\delta}}{b_{x,0}} \geq 1$. Therefore, $\frac{\sqrt{\delta}(1+\delta)}{b_{x,0}} > 1$ which in turn implies that 
\begin{align}
	\frac{8\sqrt{\delta}(1+\delta)}{b_{x,0}} - 1 & > \frac{8\sqrt{\delta}(1+\delta)}{b_{x,0}} -  \frac{\sqrt{\delta}(1+\delta)}{b_{x,0}} \nonumber\\
	& = \frac{7\sqrt{\delta}(1+\delta)}{b_{x,0}} \nonumber\\
	\frac{1}{\frac{8\sqrt{\delta}(1+\delta)}{b_{x,0}} - 1} & < \frac{b_{x,0}}{7\sqrt{\delta}(1+\delta)} .
\end{align}
By using above relation in \eqref{eq:Mx_sgd_lb_mid}, we obtain
\begin{align}
	M_x & \geq 1 - \frac{8\delta b_{x,0} }{7\sqrt{\delta}(1+\delta)} = 1 - \frac{8b_{x,0}\sqrt{\delta  }}{7(1+\delta)} \label{eq:Mx_lb_delta_sgd} \\
	& \geq 1 - \frac{8b_{x,0}}{7} \frac{1}{2}  = 1 - \frac{4b_{x,0}}{7} , \label{eq:Mx_lb_sgd} 
\end{align}
where the last inequality uses $\frac{\sqrt{\delta}}{1+\delta} \leq \frac{1}{2}$.
\begin{align}
	\frac{1-b_{x,0}}{M_{x,0}} & = 1 + \frac{1-b_{x,0}}{M_{x,0}} - 1  \leq 1 + \frac{1-b_{x,0}}{1 - \frac{8b_{x,0}\sqrt{\delta  }}{7(1+\delta)}} - 1 \nonumber\\
	& =  1 + \frac{1-b_{x,0} - 1+ \frac{8b_{x,0}\sqrt{\delta  }}{7(1+\delta)} }{1 - \frac{8b_{x,0}\sqrt{\delta  }}{7(1+\delta)}}  = 1 - \frac{b_{x,0} - \frac{8b_{x,0}\sqrt{\delta  }}{7(1+\delta)} }{1 - \frac{8b_{x,0}\sqrt{\delta  }}{7(1+\delta)}} \nonumber\\
	& = 1 - \frac{7b_{x,0}(1+\delta) - 8b_{x,0}\sqrt{\delta} }{7(1+\delta) - 8b_{x,0}\sqrt{\delta} } = 1 - \frac{7(1+\delta) - 8\sqrt{\delta} }{\frac{7(1+\delta)}{b_{x,0}} - 8\sqrt{\delta} } \nonumber\\
	& \leq 1 - \frac{7(1+\delta) - \frac{8(1+\delta)}{2} }{\frac{7(1+\delta)}{b_{x,0}} - 8\sqrt{\delta} } = 1 - \frac{3(1+\delta)}{\frac{7(1+\delta)}{b_{x,0}} - 8\sqrt{\delta} } \nonumber\\
	& < 1 - \frac{3(1+\delta)}{\frac{7(1+\delta)}{b_{x,0}}} \nonumber\\
	& = 1 - \frac{3b_{x,0}}{7} \label{eq:one_bx_Mx_ub_sgd} .
\end{align}
Similarly, we obtain 
\begin{align}
	& M_y \geq 1 - \frac{8b_{y,0}\sqrt{\delta  }}{7(1+\delta)} \geq  1 - \frac{4b_{y,0}}{7} \ \text{and, } \label{eq:My_lb_sgd} \\
	& \frac{1-b_{y,0}}{M_{y,0}} < 1 - \frac{3b_{y,0}}{7} . \label{eq:one_by_My_ub_sgd}
\end{align}
\textbf{Case II:}  $b_{x,0} > \sqrt{\delta}$ .
\begin{align}
	\gamma_{x,0} = \frac{1}{4(1+\delta)\lambda_{\max}(I-W)} .
\end{align}
We have
\begin{align}
	M_{x,0} & = 1 - \frac{\sqrt{\delta }\alpha_{x,0}}{1-\frac{\gamma_{x,0}}{2}\lambda_{\max}(I-W)} \nonumber\\
	& = 1 - \frac{\sqrt{\delta }\alpha_{x,0}}{1-\frac{1}{8(1+\delta)}} \nonumber\\
	& =  1 - \frac{\frac{\sqrt{\delta }b_{x,0}}{1+\delta}}{1-\frac{1}{8(1+\delta)}} \nonumber\\
	& = 1 - \frac{\sqrt{\delta }b_{x,0} \times 8(1+\delta) }{(1+\delta)(8(1+\delta) - 1)} \nonumber\\
	& = 1 - \frac{8\sqrt{\delta}b_{x,0}}{8(1+\delta) - 1} .
\end{align}
As $8(1+\delta) - 1 > 8(1+\delta) - 1  -\delta = 7(1+\delta)$. Therefore,
\begin{align}
	M_{x,0} & \geq 1 - \frac{8\sqrt{\delta}b_{x,0}}{7(1+\delta)} .
\end{align}
Notice that above lower bound matches with lower bound in \eqref{eq:Mx_lb_delta_sgd}. Therefore, by following steps similar to Case I, we obtain
\begin{align}
	\frac{1-b_{x,0}}{M_{x,0}} & < 1 - \frac{3b_{x,0}}{7}  \leq 1- \left(1 - \frac{1}{\sqrt{2}} \right)\frac{3np_{\min}}{28\sqrt{2}\kappa_f^2} , \ \text{and} \label{eq:bxk_Mxk_ub} \\
	\frac{1-b_{y,0}}{M_{y,0}} & < 1 - \frac{3b_{y,0}}{7} \leq 1- \left(1 - \frac{1}{\sqrt{2}} \right)\frac{3np_{\min}}{28\sqrt{2}\kappa_f^2} \label{eq:byk_Myk_ub} .
\end{align}

We now establish a recursion for $E_0\left[\Phi_{t}\right]$.
\begin{lemma} \label{lem:Phi_kt_recursion} Suppose $\{\bx_{t}\}_t$ and $\{\by_{t} \}_t$ are the sequences generated by Algorithm \ref{alg:IPDHG_with_sgd_svrg_oracle}. Suppose Assumptions \ref{s_convexity_assumption}-\ref{weight_matrix_assumption} and Assumptions \ref{smoothness_x_svrg}-\ref{lipschitz_yx_svrg} hold. Let step size $s_0$ is chosen according to Corollary \ref{cor:exp_xkt_ykt_xstar_ystar_compact}. Then, for every $0 \leq t \leq T_0-1$, the following holds:
	\begin{align}
		E_0\left[ \Phi_{t+1} \right] & \leq \rho_0 E_0\left[ \Phi_{t} \right] + \frac{2s_0^2(C_x + C_y)}{n^2p_{\min}},
	\end{align}
	where $E_0$ is the expectation over randomness in Algorithm \ref{alg:IPDHG_with_sgd_svrg_oracle} for $t \leq T_0-1$,  $\rho_0$ is defined in equation \eqref{rho_sgd}, $C_x$ $=$ $\sum_{i = 1}^m \sum_{l = 1}^n \left\Vert \nabla_x f_{il}(z^\star) \right\Vert^2$, $C_y$ $=$ $\sum_{i = 1}^m \sum_{l = 1}^n \left\Vert \nabla_y f_{il}(z^\star) \right\Vert^2$.
\end{lemma}

\begin{proof}  Iterates $\bx_{t+1},\by_{t+1} $ of Algorithm  \ref{alg:IPDHG_with_sgd_svrg_oracle} are obtained by invoking Algorithm \ref{alg:generic_procedure_sgda}. Therefore, Lemma \ref{lem:recursion_y} and Lemma \ref{lem:recursion_x} also holds for Algorithm \ref{alg:IPDHG_with_sgd_svrg_oracle}. Adding inequalities~\eqref{eq:one_step_progress_y} and~\eqref{eq:one_step_progress_x}  (Lemma \ref{lem:recursion_y} and Lemma \ref{lem:recursion_x}), we have
	\begin{align}
		& M_{x,0} E\left\Vert \bx_{t+1} - \mathbf{1}x^\star  \right\Vert^2 + \frac{2s_0^2}{\gamma_{x,0}} E \left\Vert  D^\bx_{t+1} - D^\star_\bx  \right\Vert^2_{(I-W)^\dagger} + \sqrt{\delta}E\left\Vert H^\bx_{t+1} - H^\star_{\bx,0} \right\Vert^2 \nonumber\\ & \ + M_{y,0} E\left\Vert \by_{t+1} - \mathbf{1}y^\star  \right\Vert^2 + \frac{2s_0^2}{\gamma_{y,0}} E \left\Vert  D^\by_{t+1} - D^\star_\by  \right\Vert^2_{(I-W)^\dagger} + \sqrt{\delta}E\left\Vert H^\by_{t+1} - H^\star_{\by,0} \right\Vert^2 \nonumber\\
		& \leq \left\Vert \bx_{t} - \mathbf{1}x^\star - s_0\mathcal{G}^\bx_{t} + s_0\nabla_x F(\mathbf{1}z^\star)  \right\Vert^2 + \frac{2s_0^2}{\gamma_{x,0}}\left( 1- \frac{\gamma_{x,0}}{2}\lambda_{m-1}(I-W) \right)\left\Vert D^\bx_{t} - D^\star_\by \right\Vert^2_{(I-W)^\dagger} \nonumber\\ & \ \ + \sqrt{\delta}(1-\alpha_{x,0})\left\Vert H^\bx_{t} - H^\star_{\bx,0}  \right\Vert^2 + \left\Vert \by_{t} - \mathbf{1}y^\star + s_0\mathcal{G}^\by_{t} - s_0\nabla_y F(\mathbf{1}z^\star)  \right\Vert^2 \nonumber \\ & \ \ + \frac{2s_0^2}{\gamma_{y,0}}\left( 1- \frac{\gamma_{y,0}}{2}\lambda_{m-1}(I-W) \right)\left\Vert D^\by_{t} - D^\star_\by \right\Vert^2_{(I-W)^\dagger}  + \sqrt{\delta}(1-\alpha_{y,0})\left\Vert H^\by_{t} - H^\star_{\by,0}  \right\Vert^2  .
	\end{align}
	
	By taking conditional expectation on stochastic gradient at $t$-th step on both sides of above inequality and applying Tower property, we obtain
	\begin{align}
		& M_{x,0} E\left\Vert \bx_{t+1} - \mathbf{1}x^\star  \right\Vert^2 + \frac{2s_0^2}{\gamma_{x,0}} E \left\Vert  D^\bx_{t+1} - D^\star_\bx  \right\Vert^2_{(I-W)^\dagger} + \sqrt{\delta}E\left\Vert H^\bx_{t+1} - H^\star_{\bx,0} \right\Vert^2 \nonumber\\ & \ + M_{y,0} E\left\Vert \by_{t+1} - \mathbf{1}y^\star  \right\Vert^2 + \frac{2s_0^2}{\gamma_{y,0}} E \left\Vert  D^\by_{t+1} - D^\star_\by  \right\Vert^2_{(I-W)^\dagger} + \sqrt{\delta}E\left\Vert H^by_{t+1} - H^\star_{\by,0} \right\Vert^2 \nonumber\\
		& \leq E\left\Vert \bx_{t} - \mathbf{1}x^\star - s_0\mathcal{G}^\bx_{t} + s_0\nabla_x F(\mathbf{1}x^\star) ,\mathbf{1}y^\star) \right\Vert^2 + \frac{2s_0^2}{\gamma_{x,0}}\left( 1- \frac{\gamma_{x,0}}{2}\lambda_{m-1}(I-W) \right)\left\Vert D^\bx_{t} - D^\star_\bx \right\Vert^2_{(I-W)^\dagger} \nonumber\\ & \ \ + \sqrt{\delta}(1-\alpha_{x,0})\left\Vert H^\bx_{t} - H^\star_{\bx,0}  \right\Vert^2 + E\left\Vert \by_{t} - \mathbf{1}y^\star + s_0\mathcal{G}^\by_{t} - s_0\nabla_y F(\mathbf{1}x^\star,\mathbf{1}y^\star))  \right\Vert^2 \nonumber \\ & \ \ + \frac{2s_0^2}{\gamma_{y,0}}\left( 1- \frac{\gamma_{y,0}}{2}\lambda_{m-1}(I-W) \right)\left\Vert D^\by_{t} - D^\star_\by \right\Vert^2_{(I-W)^\dagger} + \sqrt{\delta}(1-\alpha_{y,0})\left\Vert H^\by_{t} - H^\star_{\by,0}  \right\Vert^2 \nonumber \\
		& \leq (1-b_{x,0})\left\Vert \bx_{t} - \mathbf{1}x^\star  \right\Vert^2 + (1- b_{y,0}) \left\Vert \by_{t} - \mathbf{1}y^\star  \right\Vert^2 \nonumber\\ & + \frac{2s_0^2}{\gamma_{x,0}}\left( 1- \frac{\gamma_{x,0}}{2}\lambda_{m-1}(I-W) \right)\left\Vert D^\bx_{t} - D^\star_\bx \right\Vert^2_{(I-W)^\dagger}  +  \frac{2s_0^2}{\gamma_{y,0}}\left( 1- \frac{\gamma_{y,0}}{2}\lambda_{m-1}(I-W) \right)\left\Vert D^\by_{t} - D^\star_\by \right\Vert^2_{(I-W)^\dagger}  \nonumber\\ & \ + \sqrt{\delta}(1-\alpha_{x,0})\left\Vert H^\bx_{t} - H^\star_{\bx,0}  \right\Vert^2 +  \sqrt{\delta}(1-\alpha_{y,0})\left\Vert H^\by_{t} - H^\star_{\by,0}  \right\Vert^2  +  \frac{2s_0^2(C_x + C_y)}{n^2p_{\min}}
	\end{align}
	where the last inequality follows from inequality \eqref{bx_by_bound}.
	
	By taking total expectation on both sides of above inequality, using tower property and using the definition of $\Phi_{t}$, we obtain
	\begin{align}
		& E_0\left[ \Phi_{t+1} \right] \nonumber\\
		& \leq (1-b_{x,0})E_0\left\Vert \bx_{t} - \mathbf{1}x^\star  \right\Vert^2 + (1- b_{y,0}) E_0\left\Vert \by_{t} - \mathbf{1}y^\star  \right\Vert^2 \nonumber\\ & + \frac{2s_0^2}{\gamma_{x,0}}\left( 1- \frac{\gamma_{x,0}}{2}\lambda_{m-1}(I-W) \right)E_0\left\Vert D^\bx_{t} - D^\star_\bx \right\Vert^2_{(I-W)^\dagger}  +  \frac{2s_0^2}{\gamma_{y,0}}\left( 1- \frac{\gamma_{y,0}}{2}\lambda_{m-1}(I-W) \right)E_0\left\Vert D^\by_{t} - D^\star_\by \right\Vert^2_{(I-W)^\dagger} \nonumber \\ & \ + \sqrt{\delta}(1-\alpha_{x,0})E_0\left\Vert H^\bx_{t} - H^\star_{\bx,0}  \right\Vert^2 +  \sqrt{\delta}(1-\alpha_{y,0})E_0\left\Vert H^\by_{t} - H^\star_{\by,0}  \right\Vert^2  +  \frac{2s_0^2(C_x + C_y)}{n^2p_{\min}} \nonumber\\
		& = \frac{(1-b_{x,0})}{M_{x,0}}M_{x,0}E_0\left\Vert \bx_{t} - \mathbf{1}x^\star  \right\Vert^2 + \frac{(1- b_{y,0})}{M_{y,0}}M_{y,0} E_0\left\Vert \by_{t} - \mathbf{1}y^\star  \right\Vert^2 \nonumber\\ & + \frac{2s_0^2}{\gamma_{x,0}}\left( 1- \frac{\gamma_{x,0}}{2}\lambda_{m-1}(I-W) \right)E_0\left\Vert D^\bx_{t} - D^\star_\bx \right\Vert^2_{(I-W)^\dagger}  +  \frac{2s_0^2}{\gamma_{y,0}}\left( 1- \frac{\gamma_{y,0}}{2}\lambda_{m-1}(I-W) \right)E_0\left\Vert D^\by_{t} - D^\star_\by \right\Vert^2_{(I-W)^\dagger}  \nonumber \\ & \ + \sqrt{\delta}(1-\alpha_{x,0})E_0\left\Vert H^\bx_{t} - H^\star_{\bx,0}  \right\Vert^2 +  \sqrt{\delta}(1-\alpha_{y,0})E_0\left\Vert H^\by_{t} - H^\star_{\by,0}  \right\Vert^2  +  \frac{2s_0^2(C_x + C_y)}{n^2p_{\min}} \nonumber\\
		& \leq \max \left\lbrace 1-\frac{3b_{x,0}}{7} , 1-\frac{3b_{y,0}}{7} ,  1- \frac{\gamma_{x,0}}{2} \lambda_{m-1}(I-W) , 1- \frac{\gamma_{y,0}}{2} \lambda_{m-1}(I-W) , 1-\alpha_{x,0}, 1-\alpha_{y,0}  \right\rbrace  \times E_0\left[ \Phi_{t} \right] \nonumber\\ & \ \ + \frac{2s_0^2(C_x + C_y)}{n^2p_{\min}} \nonumber \\
		& = \rho_0 E_0\left[ \Phi_{t} \right] + \frac{2s_0^2(C_x + C_y)}{n^2p_{\min}} ,
	\end{align}
	where second last step uses \eqref{eq:one_bx_Mx_ub_sgd} and \eqref{eq:one_by_My_ub_sgd}. The last equality follows from $\rho_0$ defined in \eqref{rho_sgd}.
\end{proof}

\subsection{Proof of Lemma \ref{lem:exp_phit+1_phi_0}} \label{appendix:exp_phit+1_phi_0}
Using Lemma \ref{lem:Phi_kt_recursion}, we have $	E_0\left[ \Phi_{t+1} \right]  \leq \rho_0 E_0\left[ \Phi_{t} \right] + \frac{2s_0^2(C_x + C_y)}{n^2p_{\min}}$
By letting  $A_1 := \frac{2(C_x + C_y)}{n^2p_{\min}}$, we unroll the recursion to obtain
\begin{align}
	E_0\left[ \Phi_{t+1} \right] & \leq \rho^{t+1}_0  \Phi_{0}  + \sum_{l = 0}^t \rho_0^{t-l} A_1 s^2_0 = \rho^{t+1}_0  \Phi_{0} + A_1 s^2_0 \rho_0^{t} \sum_{l = 0}^t \rho_0^{-l} \nonumber\\
	& =  \rho^{t+1}_0  \Phi_{0}  + A_1 s^2_0 \rho_0^{t} \frac{\rho_0^{-(t+1)} - 1}{\rho_0^{-1} - 1} \nonumber\\
	& \leq  \rho^{t+1}_0  \Phi_{0}  + A_1 s^2_0 \rho_0^{t} \frac{\rho_0^{-(t+1)}}{\rho_0^{-1} - 1} = \rho^{t+1}_0  \Phi_{0}  + A_1 s^2_0 \rho_0^{-1} \frac{\rho_0}{1 - \rho_0} \nonumber\\
	& = \rho^{t+1}_0 \Phi_{0} + A_1 s^2_0 \frac{1}{1 - \rho_0} = \left( \rho_0 \right)^{t+1} \Phi_{0} + \frac{A_1 s^2_0}{1-\rho_0}. \nonumber
\end{align}
Substituting $A_1 = \frac{2(C_x + C_y)}{n^2p_{\min}}$, we get
\begin{align}
	E_0\left[ \Phi_{t+1} \right] & \leq \left( \rho_0 \right)^{t+1} \Phi_{0} + \frac{2s_0^2(C_x + C_y)}{(1-\rho_0)n^2p_{\min}},
\end{align}
which completes the proof of Lemma \ref{lem:exp_phit+1_phi_0}.

\section{Convergence Behavior of Algorithm \ref{alg:IPDHG_with_sgd_svrg_oracle} with SVRGO}
\label{appendix_finite_sum}

We first prove all results related to the convergence behavior of IPDHG with SVRGO.
We begin with few intermediate results which will help us in getting the final convergence result of Algorithm \ref{alg:IPDHG_with_sgd_svrg_oracle}. 

\begin{lemma} \label{lem:exp_xkt_ykt_svrg} Let $\{\bx_{t} \}_t, \{\by_{t}\}_t$ be the sequences generated by Algorithm \ref{alg:generic_procedure_sgda} with $\mathcal{G}^\bx_{t}$ and $ \mathcal{G}^\by_{t}$ obtained from SVRGO. Then, under Assumptions \ref{s_convexity_assumption}-\ref{s_concavity_assumption} and Assumptions \ref{smoothness_x_svrg}-\ref{lipschitz_yx_svrg}, the following holds for all $t \geq 1$:
\begin{align}
& E \left\Vert \bx_{t} - \mathbf{1}x^\star - s\mathcal{G}^\bx_{t} + s\nabla_x F(\mathbf{1}x^\star,\mathbf{1}y^\star)  \right\Vert^2 + E \left\Vert \by_{t} - \mathbf{1}y^\star + s\mathcal{G}^\by_{t} - s\nabla_y F(\mathbf{1}x^\star,\mathbf{1}y^\star)  \right\Vert^2 \nonumber \\
& \leq  \left( 1-\mu_x s + \frac{4s^2L^2_{yx}}{np_{\min}} \right)\left\Vert \bx_{t} - \mathbf{1}x^\star \right\Vert^2 + \left(1-s\mu_y + \frac{4s^2L^2_{xy}}{np_{\min}} \right) \left\Vert \by_{t} - \mathbf{1}y^\star  \right\Vert^2 \nonumber \\ & \ - \left( 2s - \frac{8s^2L_{xx}}{np_{\min}} \right) \sum_{i = 1}^m  V_{f_i,y^i_{t}}(x^{\star},x^i_{t}) - \left( 2s - \frac{8s^2L_{yy}}{np_{\min}} \right) \sum_{i = 1}^m V_{-f_i,x^i_{t}}(y^\star,y^i_{t}) \nonumber \\ & \ + \frac{4s^2(L^2_{xx} + L^2_{yx})}{np_{\min}} \left\Vert \tilde{\bx}_{t} - \mathbf{1}x^\star \right\Vert^2 + \frac{4s^2(L^2_{yy} + L^2_{xy})}{np_{\min}} \left\Vert \tilde{\by}_{t} - \mathbf{1}y^\star \right\Vert^2 , \label{eq:exp_xkt_ykt_vfi_svrg}
\end{align}
where $p_{\min} := \min_{i,j} \{p_{ij} \}$.
\end{lemma}

\begin{proof}
We begin the proof by bounding the primal ($\bx$) and dual ($\by$) updates on the l.h.s. of~\eqref{eq:exp_xkt_ykt_vfi_svrg} separately. In particular, we show that
\begin{align}
& E \left\Vert \bx_{t} - \mathbf{1}x^\star - s\mathcal{G}^\bx_{t} + s\nabla_x F(\mathbf{1}x^\star,\mathbf{1}y^\star)  \right\Vert^2 \nonumber\\
& \leq \left( 1-\mu_x s \right)\left\Vert \bx_{t} - \mathbf{1}x^\star \right\Vert^2 \nonumber\\ & \ + \frac{2s^2}{n^2p_{\min}} \sum_{i = 1}^m \sum_{j = 1}^n  \left\Vert \nabla_x f_{ij}(z^i_{t}) - \nabla_x f_{ij}(z^\star) \right\Vert^2 + \frac{2s^2}{n^2p_{\min}} \sum_{i = 1}^m  \sum_{j = 1}^n  \left\Vert \nabla_x f_{ij}(\tilde{z}^i_{t}) - \nabla_x f_{ij}(z^\star) \right\Vert^2   \nonumber\\ & \ - 2s \sum_{i = 1}^m  V_{f_i,y^i_{t}}(x^{\star},x^i_{t}) + 2s \left( F(\mathbf{1}x^\star,\by_{t}) - F(z_{t}) + F(\bx_{t},\mathbf{1}y^\star) - F(\mathbf{1}z^\star) \right)  \label{eq:exp_xkt_xstar_svrg_final}
\end{align}
and 
\begin{align}
& E \left\Vert \by_{t} - \mathbf{1}y^\star + s\mathcal{G}^\by_{t} - s\nabla_y F(\mathbf{1}x^\star,\mathbf{1}y^\star)  \right\Vert^2 \nonumber\\
& \leq (1-s\mu_y)\left\Vert \by_{t} - \mathbf{1}y^\star  \right\Vert^2  +2s \left( -F(\bx_{t},\mathbf{1}y^\star) + F(\mathbf{1}z^\star) - F(\mathbf{1}x^\star,\by_{t}) + F(z_{t}) \right) - 2s\sum_{i = 1}^m V_{-f_i,x^i_{t}}(y^\star,y^i_{t}) \nonumber\\ & \ + \frac{2s^2}{n^2p_{\min}} \sum_{i = 1}^m \sum_{j = 1}^n  \left\Vert \nabla_y f_{ij}(z^i_{t}) - \nabla_y f_{ij}(z^\star) \right\Vert^2 + \frac{2s^2}{n^2p_{\min}} \sum_{i = 1}^m  \sum_{j = 1}^n  \left\Vert \nabla_y f_{ij}(\tilde{z}^i_{t}) - \nabla_y f_{ij}(z^\star) \right\Vert^2 . \label{eq:exp_ykt_ystar_svrg_final}
\end{align}

Observe that~\eqref{eq:exp_xkt_xstar_svrg_final} and~\eqref{eq:exp_ykt_ystar_svrg_final} are similar, and we only prove~\eqref{eq:exp_xkt_xstar_svrg_final} in Section~\ref{sec:exp_xkt} below. Adding \eqref{eq:exp_xkt_xstar_svrg_final} and \eqref{eq:exp_ykt_ystar_svrg_final}, we obtain
\begin{align}
& E \left\Vert \bx_{t} - \mathbf{1}x^\star - s\mathcal{G}^\bx_{t} + s\nabla_x F(\mathbf{1}x^\star,\mathbf{1}y^\star)  \right\Vert^2 + E \left\Vert \by_{t} - \mathbf{1}y^\star + s\mathcal{G}^\by_{t} - s\nabla_y F(\mathbf{1}x^\star,\mathbf{1}y^\star)  \right\Vert^2 \nonumber\\
& \leq \left( 1-\mu_x s \right)\left\Vert \bx_{t} - \mathbf{1}x^\star \right\Vert^2 + (1-s\mu_y)\left\Vert \by_{t} - \mathbf{1}y^\star  \right\Vert^2  \\ & \ - 2s \sum_{i = 1}^m  V_{f_i,y^i_{t}}(x^{\star},x^i_{t}) - 2s\sum_{i = 1}^m V_{-f_i,x^i_{t}}(y^\star,y^i_{t}) \nonumber\\ & \ + \frac{2s^2}{n^2p_{\min}} \sum_{i = 1}^m \sum_{j = 1}^n \left( \left\Vert \nabla_x f_{ij}(z^i_{t}) - \nabla_x f_{ij}(z^\star) \right\Vert^2 + \left\Vert \nabla_y f_{ij}(z^i_{t}) - \nabla_y f_{ij}(z^\star) \right\Vert^2 \right) \nonumber\\ & \ + \frac{2s^2}{n^2p_{\min}} \sum_{i = 1}^m  \sum_{j = 1}^n \left( \left\Vert \nabla_x f_{ij}(\tilde{z}^i_{t}) - \nabla_x f_{ij}(z^\star) \right\Vert^2 + \left\Vert \nabla_y f_{ij}(\tilde{z}^i_{t}) - \nabla_y f_{ij}(z^\star) \right\Vert^2 \right) . \label{eq:exp_xkt_ykt_grad_diff_svrg}
\end{align}
To finish the proof of Lemma~\ref{lem:exp_xkt_ykt_svrg}, we bound the last two terms of~\eqref{eq:exp_xkt_ykt_grad_diff_svrg} as shown in Section~\ref{sec:finish_proof_xtk_ykt_svrg}.
\paragraph{Proof of~(\ref{eq:exp_xkt_xstar_svrg_final})}\label{sec:exp_xkt}
First, consider the primal update term

\begin{align}
& E \left\Vert \bx_{t} - \mathbf{1}x^\star - s\mathcal{G}^\bx_{t} + s\nabla_x F(\mathbf{1}x^\star,\mathbf{1}y^\star)  \right\Vert^2 \nonumber\\
& = \sum_{i = 1}^m E \left\Vert x^i_{t} - x^\star - s\mathcal{G}^{i,x}_{t} + s\nabla_x f_i(x^\star,y^\star)  \right\Vert^2 \nonumber\\
& = \sum_{i = 1}^m \left\Vert x^i_{t} - x^\star  \right\Vert^2 + s^2 \sum_{i = 1}^mE \left\Vert \mathcal{G}^{i,x}_{t} - \nabla_x f_i(x^\star,y^\star) \right\Vert^2 \nonumber\\ & \ \ - 2s \sum_{i = 1}^m E \left\langle x^i_{t} - x^\star, \mathcal{G}^{i,x}_{t} - \nabla_x f_i(x^\star,y^\star) \right\rangle\nonumber \\
& = \sum_{i = 1}^m \left\Vert x^i_{t} - x^\star  \right\Vert^2 + s^2 \sum_{i = 1}^mE \left\Vert \frac{1}{np_{il}} \left( \nabla_x f_{il}(z^i_{t}) - \nabla_x f_{il}(\tilde{z}^i_{t}) \right) + \nabla_x f_i(\tilde{z}^i_{t}) - \nabla_x f_i(x^\star,y^\star) \right\Vert^2 \nonumber\\ & \ \ - 2s \sum_{i = 1}^m E \left\langle x^i_{t} - x^\star, \frac{1}{np_{il}} \left( \nabla_x f_{il}(z^i_{t}) - \nabla_x f_{il}(\tilde{z}^i_{t}) \right) + \nabla_x f_i(\tilde{z}^i_{t}) - \nabla_x f_i(x^\star,y^\star) \right\rangle . \label{eq:exp_xkt_xstar_svrg}
\end{align}
Observe that
\begin{align}
& E \left[  \frac{1}{np_{il}} \left( \nabla_x f_{il}(z^i_{t}) - \nabla_x f_{il}(\tilde{z}^i_{t}) \right) + \nabla_x f_i(\tilde{z}^i_{t}) - \nabla_x f_i(x^\star,y^\star) \right] \nonumber\\
& = \sum_{l = 1}^n \frac{ \nabla_x f_{il}(z^i_{t}) - \nabla_x f_{il}(\tilde{z}^i_{t}) }{np_{il}} \times p_{il} + \nabla_x f_i(\tilde{z}^i_{t}) - \nabla_x f_i(x^\star,y^\star) \nonumber\\
& = \frac{1}{n} \sum_{l = 1}^n \nabla_x f_{il}(z^i_{t}) - \frac{1}{n} \sum_{l = 1}^n \nabla_x f_{il}(\tilde{z}^i_{t}) + \nabla_x f_i(\tilde{z}^i_{t}) - \nabla_x f_i(z^\star) \nonumber\\
& = \nabla_x f_{i}(z^i_{t}) - \nabla_x f_{i}(\tilde{z}^i_{t}) + \nabla_x f_i(\tilde{z}^i_{t}) - \nabla_x f_i(z^\star) \nonumber\\
& = \nabla_x f_{i}(z^i_{t}) - \nabla_x f_i(z^\star) ,
\end{align}
where the first equality and second last equality follows respectively from step $(1)$ of SVRGO and definition of $f_i(x,y)$. Substituting the above in the last term of~\eqref{eq:exp_xkt_xstar_svrg}, we see that 
\begin{align}
& E \left\Vert \bx_{t} - \mathbf{1}x^\star - s\mathcal{G}^\bx_{t} + s\nabla_x F(\mathbf{1}x^\star,\mathbf{1}y^\star)  \right\Vert^2 \nonumber\\
& \leq \sum_{i = 1}^m \left\Vert x^i_{t} - x^\star  \right\Vert^2 + s^2 \sum_{i = 1}^mE \left\Vert \frac{1}{np_{il}} \left( \nabla_x f_{il}(z^i_{t}) - \nabla_x f_{il}(\tilde{z}^i_{t}) \right) + \nabla_x f_i(\tilde{z}^i_{t}) - \nabla_x f_i(x^\star,y^\star) \right\Vert^2 \nonumber\\ & \ \ - 2s \sum_{i = 1}^m  \left\langle x^i_{t} - x^\star, \nabla_x f_{i}(z^i_{t}) - \nabla_x f_i(z^\star) \right\rangle . \label{eq:exp_xkt_xstar_svrg_mid}
\end{align}
Substituting \eqref{eq:grad_zkt_inn_Vfi} (i.e. Bregman distance) and \eqref{eq:grad_zstar_inn_mu_x} (i.e., strong convexity of $f$) 
in \eqref{eq:exp_xkt_xstar_svrg_mid}, we obtain
\begin{align}
& E \left\Vert \bx_{t} - \mathbf{1}x^\star - s\mathcal{G}^\bx_{t} + s\nabla_x F(\mathbf{1}x^\star,\mathbf{1}y^\star)  \right\Vert^2 \nonumber\\
& \leq \sum_{i = 1}^m \left\Vert x^i_{t} - x^\star  \right\Vert^2 + s^2 \sum_{i = 1}^mE \left\Vert \frac{1}{np_{il}} \left( \nabla_x f_{il}(z^i_{t}) - \nabla_x f_{il}(\tilde{z}^i_{t}) \right) + \nabla_x f_i(\tilde{z}^i_{t}) - \nabla_x f_i(x^\star,y^\star) \right\Vert^2 \nonumber \\ & \ \ - 2s \sum_{i = 1}^m  \left( -f_i(x^\star,y^i_{t}) + f_i(z^i_{t}) + V_{f_i,y^i_{t}}(x^\star,x^i_{t}) \right) + 2s \sum_{i = 1}^m \left( f_i(x^i_{t},y^\star) - f_i(z^\star) - \frac{\mu_x}{2} \left\Vert x^i_{t} - x^\star \right\Vert^2 \right) \nonumber\\
& = \sum_{i = 1}^m \left\Vert x^i_{t} - x^\star  \right\Vert^2 + s^2 \sum_{i = 1}^mE \left\Vert \frac{1}{np_{il}} \left( \nabla_x f_{il}(z^i_{t}) - \nabla_x f_{il}(\tilde{z}^i_{t}) \right) + \nabla_x f_i(\tilde{z}^i_{t}) - \nabla_x f_i(x^\star,y^\star) \right\Vert^2 \nonumber \\ & \ \ + 2s( F(\mathbf{1}x^\star,\by_{t}) - F(z_{t})) - 2s \sum_{i = 1}^m  V_{f_i,y^i_{t}}(x^{\star},x^i_{t}) + 2s \left( F(\bx_{t},\mathbf{1}y^\star) - F(\mathbf{1}z^\star) \right) - s\mu_x \left\Vert \bx_{t} - \mathbf{1}x^\star \right\Vert^2 \nonumber\\
& = \left( 1-\mu_x s \right)\left\Vert \bx_{t} - \mathbf{1}x^\star \right\Vert^2 + s^2 \sum_{i = 1}^mE \left\Vert \frac{1}{np_{il}} \left( \nabla_x f_{il}(z^i_{t}) - \nabla_x f_{il}(\tilde{z}^i_{t}) \right) + \nabla_x f_i(\tilde{z}^i_{t}) - \nabla_x f_i(x^\star,y^\star) \right\Vert^2 \nonumber \\ & \ - 2s \sum_{i = 1}^m  V_{f_i,y^i_{t}}(x^{\star},x^i_{t}) + 2s \left( F(\mathbf{1}x^\star,\by_{t}) - F(z_{t}) + F(\bx_{t},\mathbf{1}y^\star) - F(\mathbf{1}z^\star) \right) . \label{eq:exp_xkt_xstar_svrg_F}
\end{align}
Now we bound the second term on the r.h.s. of~\eqref{eq:exp_xkt_xstar_svrg_F} in terms of $\left\Vert \bx_{t} - \mathbf{1}x^\star \right\Vert^2$ and $\left\Vert \by_{t} - \mathbf{1}y^\star \right\Vert^2$ as follows:
\begin{align}
& s^2 \sum_{i = 1}^mE \left\Vert \frac{1}{np_{il}} \left( \nabla_x f_{il}(z^i_{t}) - \nabla_x f_{il}(\tilde{z}^i_{t}) \right) + \nabla_x f_i(\tilde{z}^i_{t}) - \nabla_x f_i(x^\star,y^\star) \right\Vert^2 \nonumber\\
& = s^2 \sum_{i = 1}^m \sum_{j = 1}^n p_{ij} \left\Vert \frac{1}{np_{ij}} \left( \nabla_x f_{ij}(z^i_{t}) - \nabla_x f_{ij}(\tilde{z}^i_{t}) \right) + \nabla_x f_i(\tilde{z}^i_{t}) - \nabla_x f_i(x^\star,y^\star) \right\Vert^2 \nonumber\\
& = s^2 \sum_{i = 1}^m \sum_{j = 1}^n p_{ij} \left\Vert \frac{\nabla_x f_{ij}(z^i_{t}) - \nabla_x f_{ij}(z^\star)}{np_{ij}} + \frac{\nabla_x f_{ij}(z^\star) - \nabla_x f_{ij}(\tilde{z}^i_{t})}{np_{ij}} + \nabla_x f_i(\tilde{z}^i_{t}) - \nabla_x f_i(z^\star) \right\Vert^2 \nonumber\\
& \leq 2s^2 \sum_{i = 1}^m \sum_{j = 1}^n \frac{p_{ij}}{n^2p^2_{ij}} \left\Vert \nabla_x f_{ij}(z^i_{t}) - \nabla_x f_{ij}(z^\star) \right\Vert^2 \nonumber\\ & \ + 2s^2 \sum_{i = 1}^m \sum_{j = 1}^n p_{ij} \left\Vert \frac{\nabla_x f_{ij}(z^\star) - \nabla_x f_{ij}(\tilde{z}^i_{t})}{np_{ij}} + \nabla_x f_i(\tilde{z}^i_{t}) - \nabla_x f_i(z^\star) \right\Vert^2 \nonumber\\
& = \frac{2s^2}{n^2} \sum_{i = 1}^m \sum_{j = 1}^n \frac{1}{p_{ij}} \left\Vert \nabla_x f_{ij}(z^i_{t}) - \nabla_x f_{ij}(z^\star) \right\Vert^2 \nonumber\\ & \ + 2s^2 \sum_{i = 1}^m \sum_{j = 1}^n p_{ij} \left\Vert \frac{  \nabla_x f_{ij}(\tilde{z}^i_{t}) - \nabla_x f_{ij}(z^\star)}{np_{ij}} -( \nabla_x f_i(\tilde{z}^i_{t}) - \nabla_x f_i(z^\star)) \right\Vert^2 \nonumber\\
& \leq \frac{2s^2}{n^2p_{\min}} \sum_{i = 1}^m \sum_{j = 1}^n  \left\Vert \nabla_x f_{ij}(z^i_{t}) - \nabla_x f_{ij}(z^\star) \right\Vert^2 \nonumber\\ & \ + 2s^2 \sum_{i = 1}^m E \left\Vert \frac{  \nabla_x f_{ij}(\tilde{z}^i_{t}) - \nabla_x f_{ij}(z^\star)}{np_{ij}} -( \nabla_x f_i(\tilde{z}^i_{t}) - \nabla_x f_i(z^\star)) \right\Vert^2 , \label{eq:exp_grad_diff_svrg_pmin}
\end{align}
where $p_{\min} = \min_{i,j} \{p_{ij} \}$. Let $u_i = \left\lbrace \frac{ \nabla_x f_{il}(\tilde{z}^i_{t}) - \nabla_x f_{il}(z^\star)}{np_{il}} : l \in \{1,2,\ldots , n \} \right\rbrace $ be a random variable with probability distribution $\mathcal{P}_i = \{p_{il} : l \in \{1,2,\ldots, n \} \}$.
\begin{align}
E\left[ u_i \right] & = E \left[ \frac{ \nabla_x f_{il}(\tilde{z}^i_{t}) - \nabla_x f_{il}(z^\star)}{np_{il}} \right] \nonumber\\
& = \sum_{l = 1}^n \frac{\nabla_x f_{il}(\tilde{z}^i_{t}) - \nabla_x f_{il}(z^\star)}{np_{il}} p_{il} \nonumber\\
& = \frac{1}{n} \sum_{l = 1}^n \nabla_x f_{il}(\tilde{z}^i_{t}) - \frac{1}{n} \sum_{l = 1}^n \nabla_x f_{il}(z^\star) \nonumber\\
& = \nabla_x f_{i}(\tilde{z}^i_{t}) - \nabla_x f_{i}(z^\star) .
\end{align}
We know that $E\left\Vert u_i - Eu_i \right\Vert^2 \leq E\left\Vert u_i \right\Vert^2$. Therefore,
\begin{align}
& E\left\Vert \frac{ \nabla_x f_{ij}(\tilde{z}^i_{t}) - \nabla_x f_{ij}(z^\star)}{np_{ij}} - \left( \nabla_x f_{i}(\tilde{z}^i_{t}) - \nabla_x f_{i}(z^\star) \right) \right\Vert^2 \nonumber\\
& \leq E\left\Vert \frac{ \nabla_x f_{ij}(\tilde{z}^i_{t}) - \nabla_x f_{ij}(z^\star)}{np_{ij}} \right\Vert^2 \nonumber\\
& = \frac{1}{n^2} \sum_{j = 1}^n \left\Vert \frac{ \nabla_x f_{ij}(\tilde{z}^i_{t}) - \nabla_x f_{ij}(z^\star)}{p_{ij}} \right\Vert^2 p_{ij} \nonumber\\
& = \frac{1}{n^2} \sum_{j = 1}^n \frac{1}{p_{ij}} \left\Vert \nabla_x f_{ij}(\tilde{z}^i_{t}) - \nabla_x f_{ij}(z^\star) \right\Vert^2 \nonumber\\
& \leq \frac{1}{n^2p_{\min}} \sum_{j = 1}^n  \left\Vert \nabla_x f_{ij}(\tilde{z}^i_{t}) - \nabla_x f_{ij}(z^\star) \right\Vert^2 .
\end{align}
By substituting the above inequality in \eqref{eq:exp_grad_diff_svrg_pmin}, we obtain
\begin{align}
& s^2 \sum_{i = 1}^mE \left\Vert \frac{1}{np_{il}} \left( \nabla_x f_{il}(z^i_{t}) - \nabla_x f_{il}(\tilde{z}^i_{t}) \right) + \nabla_x f_i(\tilde{z}^i_{t}) - \nabla_x f_i(z^\star) \right\Vert^2 \nonumber\\
& \leq \frac{2s^2}{n^2p_{\min}} \sum_{i = 1}^m \sum_{j = 1}^n  \left\Vert \nabla_x f_{ij}(z^i_{t}) - \nabla_x f_{ij}(z^\star) \right\Vert^2 + \frac{2s^2}{n^2p_{\min}} \sum_{i = 1}^m  \sum_{j = 1}^n  \left\Vert \nabla_x f_{ij}(\tilde{z}^i_{t}) - \nabla_x f_{ij}(z^\star) \right\Vert^2 .
\end{align}
Substituting this inequality in \eqref{eq:exp_xkt_xstar_svrg_F} we  obtain~\eqref{eq:exp_xkt_xstar_svrg_final}.

\paragraph{Finishing the Proof of Lemma~\ref{lem:exp_xkt_ykt_svrg}}\label{sec:finish_proof_xtk_ykt_svrg}
We now compute upper bounds on the last two terms present in~\eqref{eq:exp_xkt_ykt_grad_diff_svrg} using smoothness assumptions. First, observe that
\begin{align}
& \left\Vert \nabla_x f_{ij}(z^i_{t}) - \nabla_x f_{ij}(z^\star) \right\Vert^2 + \left\Vert \nabla_y f_{ij}(z^i_{t}) - \nabla_y f_{ij}(z^\star) \right\Vert^2 \nonumber\\
& = \left\Vert \nabla_x f_{ij}(z^i_{t}) - \nabla_x f_{ij}(x^\star,y^i_{t}) + \nabla_x f_{ij}(x^\star,y^i_{t}) -  \nabla_x f_{ij}(z^\star) \right\Vert^2 \nonumber \\ & \ \ + \left\Vert \nabla_y f_{ij}(z^i_{t}) - \nabla_y f_{ij}(x^i_{t},y^\star) + \nabla_y f_{ij}(x^i_{t},y^\star) - \nabla_y f_{ij}(z^\star) \right\Vert^2 \nonumber\\
& \leq 2 \left\Vert -\nabla_x f_{ij}(z^i_{t}) + \nabla_x f_{ij}(x^\star,y^i_{t}) \right\Vert^2 + 2\left\Vert \nabla_x f_{ij}(x^\star,y^i_{t}) -  \nabla_x f_{ij}(z^\star) \right\Vert^2 \nonumber\\ & \ \ + 2\left\Vert \nabla_y f_{ij}(z^i_{t}) - \nabla_y f_{ij}(x^i_{t},y^\star) \right\Vert^2 + 2\left\Vert \nabla_y f_{ij}(x^i_{t},y^\star) - \nabla_y f_{ij}(z^\star) \right\Vert^2 \nonumber\\
& \leq 4L_{xx} V_{f_{ij},y^i_{t}}(x^\star,x^i_{t}) + 2L^2_{xy} \left\Vert y^i_{t} - y^\star  \right\Vert^2 + 4L_{yy}V_{-f_{ij},x^i_{t}}(y^\star,y^i_{t})  + 2L^2_{yx} \left\Vert x^i_{t} - x^\star  \right\Vert^2 ,
\end{align}
where the last inequality follows from Proposition \ref{prop:smoothness}, Proposition \ref{prop:smoothnessy} and Assumptions \ref{lipschitz_xy_svrg}-\ref{lipschitz_yx_svrg}.  Adding up the above inequality for $j = 1$ to $n$ and using \eqref{bregman_dist_Vy}-\eqref{bregman_dist_Vx}, we obtain
\begin{align}
& \sum_{j = 1}^n \left( \left\Vert \nabla_x f_{ij}(z^i_{t}) - \nabla_x f_{ij}(z^\star) \right\Vert^2 + \left\Vert \nabla_y f_{ij}(z^i_{t}) - \nabla_y f_{ij}(z^\star) \right\Vert^2 \right) \nonumber\\
& \leq 4L_{xx} \sum_{j = 1}^n V_{f_{ij},y^i_{t}}(x^\star,x^i_{t}) + 4L_{yy}  \sum_{j = 1}^n V_{-f_{ij},x^i_{t}}(y^\star,y^i_{t}) + 2nL^2_{xy} \left\Vert y^i_{t} - y^\star  \right\Vert^2 + 2nL^2_{yx} \left\Vert x^i_{t} - x^\star  \right\Vert^2 \nonumber\\
& = 4L_{xx} \sum_{j = 1}^n \left( f_{ij}(x^\star,y^i_{t}) - f_{ij}(x^i_{t},y^i_{t}) - \left\langle \nabla_x f_{ij}(x^i_{t},y^i_{t}), x^\star - x^i_{t} \right\rangle \right) \nonumber\\ & \ + 4L_{yy}  \sum_{j = 1}^n \left( -f_{ij}(x^i_{t},y^\star) + f_{ij}(x^i_{t},y^i_{t}) - \left\langle -\nabla_y f_{ij}(x^i_{t},y^i_{t}), y^\star - y^i_{t} \right\rangle \right) \nonumber \\ & \  + 2nL^2_{xy} \left\Vert y^i_{t} - y^\star  \right\Vert^2 + 2nL^2_{yx} \left\Vert x^i_{t} - x^\star  \right\Vert^2 \nonumber\\
& = 4L_{xx} \left( nf_{i}(x^\star,y^i_{t}) - nf_{i}(x^i_{t},y^i_{t}) - \left\langle n\nabla_x f_{i}(x^i_{t},y^i_{t}), x^\star - x^i_{t} \right\rangle \right) \nonumber\\ & \ + 4L_{yy}  \left( -nf_{i}(x^i_{t},y^\star) + nf_{i}(x^i_{t},y^i_{t}) - \left\langle -n\nabla_y f_{i}(x^i_{t},y^i_{t}), y^\star - y^i_{t} \right\rangle \right) + 2nL^2_{xy} \left\Vert y^i_{t} - y^\star  \right\Vert^2 + 2nL^2_{yx} \left\Vert x^i_{t} - x^\star  \right\Vert^2 \nonumber\\
& = 4nL_{xx} V_{f_{i},y^i_{t}}(x^\star,x^i_{t}) + 4nL_{yy} V_{-f_{i},x^i_{t}}(y^\star,y^i_{t}) + 2nL^2_{xy} \left\Vert y^i_{t} - y^\star  \right\Vert^2 + 2nL^2_{yx} \left\Vert x^i_{t} - x^\star  \right\Vert^2 ,
\end{align}
where the second last step follows from the structure of $f_i(x,y) = \frac{1}{n} \sum_{j = 1}^n f_{ij}(x,y)$. Therefore,
\begin{align}
& \frac{2s^2}{n^2p_{\min}} \sum_{i = 1}^m \sum_{j = 1}^n \left( \left\Vert \nabla_x f_{ij}(z^i_{t}) - \nabla_x f_{ij}(z^\star) \right\Vert^2 + \left\Vert \nabla_y f_{ij}(z^i_{t}) - \nabla_y f_{ij}(z^\star) \right\Vert^2 \right) \nonumber\\
& \leq \frac{8s^2L_{xx}}{np_{\min}} \sum_{i = 1}^m V_{f_{i},y^i_{t}}(x^\star,x^i_{t}) + \frac{8s^2L_{yy}}{np_{\min}} \sum_{i = 1}^m V_{-f_{i},x^i_{t}}(y^\star,y^i_{t}) + \frac{4s^2L^2_{xy}}{np_{\min}} \left\Vert \by_{t} - \mathbf{1}y^\star  \right\Vert^2  + \frac{4s^2L^2_{yx}}{np_{\min}} \left\Vert \bx_{t} - \mathbf{1}x^\star  \right\Vert^2 . \label{eq:grad_diff_zkt_svrg_compact}
\end{align}
Similarly, we bound the last term of \eqref{eq:exp_xkt_ykt_grad_diff_svrg} as 
\begin{align}
& \frac{2s^2}{n^2p_{\min}} \sum_{i = 1}^m  \sum_{j = 1}^n \left( \left\Vert \nabla_x f_{ij}(\tilde{z}^i_{t}) - \nabla_x f_{ij}(z^\star) \right\Vert^2 + \left\Vert \nabla_y f_{ij}(\tilde{z}^i_{t}) - \nabla_y f_{ij}(z^\star) \right\Vert^2 \right) \nonumber\\
& \leq \frac{4s^2(L^2_{xx} + L^2_{yx})}{np_{\min}} \left\Vert \tilde{\bx}_{t} - \mathbf{1}x^\star \right\Vert^2 + \frac{4s^2(L^2_{yy} + L^2_{xy})}{np_{\min}} \left\Vert \tilde{\by}_{t} - \mathbf{1}y^\star \right\Vert^2 . \label{eq:grad_diff_tilde_zkt_svrg_compact}
\end{align}

On substituting \eqref{eq:grad_diff_zkt_svrg_compact} and \eqref{eq:grad_diff_tilde_zkt_svrg_compact} in
\eqref{eq:exp_xkt_ykt_grad_diff_svrg}, we obtain
\begin{align}
& E \left\Vert \bx_{t} - \mathbf{1}x^\star - s\mathcal{G}^\bx_{t} + s\nabla_x F(\mathbf{1}x^\star,\mathbf{1}y^\star)  \right\Vert^2 + E \left\Vert \by_{t} - \mathbf{1}y^\star + s\mathcal{G}^\by_{t} - s\nabla_y F(\mathbf{1}x^\star,\mathbf{1}y^\star)  \right\Vert^2 \nonumber\\
& \leq  \left( 1-\mu_x s + \frac{4s^2L^2_{yx}}{np_{\min}} \right)\left\Vert \bx_{t} - \mathbf{1}x^\star \right\Vert^2 + \left(1-s\mu_y + \frac{4s^2L^2_{xy}}{np_{\min}} \right) \left\Vert \by_{t} - \mathbf{1}y^\star  \right\Vert^2  \nonumber\\ & \ - \left( 2s - \frac{8s^2L_{xx}}{np_{\min}} \right) \sum_{i = 1}^m  V_{f_i,y^i_{t}}(x^{\star},x^i_{t}) - \left( 2s - \frac{8s^2L_{yy}}{np_{\min}} \right) \sum_{i = 1}^m V_{-f_i,x^i_{t}}(y^\star,y^i_{t}) \nonumber\\ & \ + \frac{4s^2(L^2_{xx} + L^2_{yx})}{np_{\min}} \left\Vert \tilde{\bx}_{t} - \mathbf{1}x^\star \right\Vert^2 + \frac{4s^2(L^2_{yy} + L^2_{xy})}{np_{\min}} \left\Vert \tilde{\by}_{t} - \mathbf{1}y^\star \right\Vert^2 , \tag{\ref{eq:exp_xkt_ykt_vfi_svrg}}
\end{align}
proving Lemma~\ref{lem:exp_xkt_ykt_svrg}.
\end{proof}
We now have the following corollary.

\begin{corollary} \label{cor:exp_xkt_ykt_svrg} Let $s = \frac{\mu np_{\min}}{24L^2}$. Then under the settings of Lemma \ref{lem:exp_xkt_ykt_svrg},
\begin{align}
& E \left\Vert \bx_{t} - \mathbf{1}x^\star - s\mathcal{G}^\bx_{t} + s\nabla_x F(\mathbf{1}x^\star,\mathbf{1}y^\star)  \right\Vert^2 + E \left\Vert \by_{t} - \mathbf{1}y^\star + s\mathcal{G}^\by_{t} - s\nabla_y F(\mathbf{1}x^\star,\mathbf{1}y^\star)  \right\Vert^2 \\
& \leq  \left( 1-\mu_x s + \frac{4s^2L^2_{yx}}{np_{\min}} \right)\left\Vert \bx_{t} - \mathbf{1}x^\star \right\Vert^2 + \left(1-s\mu_y + \frac{4s^2L^2_{xy}}{np_{\min}} \right) \left\Vert \by_{t} - \mathbf{1}y^\star  \right\Vert^2  \\ & \ + \frac{4s^2(L^2_{xx} + L^2_{yx})}{np_{\min}} \left\Vert \tilde{\bx}_{t} - \mathbf{1}x^\star \right\Vert^2 + \frac{4s^2(L^2_{yy} + L^2_{xy})}{np_{\min}} \left\Vert \tilde{\by}_{t} - \mathbf{1}y^\star \right\Vert^2 .
\end{align}
\end{corollary}

\begin{proof}
 From the statement of the corollary, we have  $s = \frac{\mu np_{\min}}{24L^2} \leq \frac{np_{\min}}{24L \kappa} < \frac{np_{\min}}{4L} \leq \frac{np_{\min}}{4L_{xx}}$. This implies that 
\begin{align}
\frac{4sL_{xx}}{np_{\min}}  \leq 1 \text{ i.e. }
\frac{8s^2L_{xx}}{np_{\min}}  \leq 2s .
\end{align}
Notice that $V_{f_i,y^i_{t}}(x^\star,x^i_{t}) \geq 0$. Therefore, $\left( 2s - \frac{8s^2L_{xx}}{np_{\min}} \right) \sum_{i = 1}^m  V_{f_i,y^i_{t}}(x^{\star},x^i_{t}) \geq 0$. We also have $s \leq \frac{np_{\min}}{4L_{yy}}$ because $L = \max \{L_{xx}, L_{yy}, L_{xy}, L_{yx} \}$. Therefore, $\frac{8s^2L_{yy}}{np_{\min}}  \leq 2s$. Due to the concavity of $f_i(x,y)$ in $y$, $V_{-f_i,x^i_{t}}(y^\star,y^i_{t})$ is nonnegative. Therefore, $\left( 2s - \frac{8s^2L_{yy}}{np_{\min}} \right) \sum_{i = 1}^m V_{-f_i,x^i_{t}}(y^\star,y^i_{t}) \geq 0$. By substituting these lower bounds in \eqref{eq:exp_xkt_ykt_vfi_svrg}, we get the desired result.
\end{proof}

\subsection{Parameters setting and their Feasibility} \label{svrg_parameters_setup}

\textbf{Parameters Setting:}
Let $p_{\min} = \min_{i,j} \{p_{ij} \}$. We define the following quantities which are instrumental in simplifying the bounds and in Algorithm \ref{alg:IPDHG_with_sgd_svrg_oracle} implementation.
\begin{align}
& \tilde{c}_x  := \frac{8s^2(L^2 + L^2_{yx})}{np_{\min}p}, \ \tilde{c}_y  := \frac{8s^2(L^2 + L^2_{xy})}{np_{\min}p} , \\
& b_x  := s\mu_x - \frac{4s^2L^2_{yx}}{np_{\min}} - \tilde{c}_x p , b_y := s\mu_y - \frac{4s^2L^2_{xy}}{np_{\min}} - \tilde{c}_y p , \label{eq:bx_by_svrg} \\
& \alpha_x := \frac{b_x}{(1+\delta)} , \ \alpha_y := \frac{b_y}{(1+\delta)} , \label{eq:alpha_x_alpha_y_svrg} \\
& \gamma_x := \min \left\lbrace \frac{b_x}{4\sqrt{\delta}(1+\delta)\lambda_{\max}(I-W)} , \frac{1}{4(1+\delta)\lambda_{\max}(I-W)} \right\rbrace , \label{eq:gamma_x_svrg} \\
& \gamma_y := \min \left\lbrace \frac{b_y}{4\sqrt{\delta}(1+\delta)\lambda_{\max}(I-W)} , \frac{1}{4(1+\delta)\lambda_{\max}(I-W)} \right\rbrace ,  \label{eq:gammma_y_svrg}, \\
& M_{x} = 1-\frac{\sqrt{\delta }\alpha_{x}}{1-\frac{\gamma_{x}}{2}\lambda_{\max}(I-W)}, \ M_{y} = 1-\frac{\sqrt{\delta }\alpha_{y}}{1-\frac{\gamma_{y}}{2}\lambda_{\max}(I-W)}, \\
& \hat{\Phi}_{t} := M_{x} \left\Vert \bx_{t} - \mathbf{1}x^\star  \right\Vert^2 + \frac{2s^2}{\gamma_{x}}  \left\Vert  D^\bx_{t} - D^\star_\bx  \right\Vert^2_{(I-W)^\dagger} + \sqrt{\delta}\left\Vert H^\bx_{t} - H^\star_{x} \right\Vert^2 \\ & \qquad + M_{y} \left\Vert \by_{t} - \mathbf{1}y^\star  \right\Vert^2 + \frac{2s^2}{\gamma_{y}} \left\Vert  D^\by_{t} - D^\star_\by  \right\Vert^2_{(I-W)^\dagger} + \sqrt{\delta} \left\Vert H^\by_{t} - H^\star_{y} \right\Vert^2 , \\
& \tilde{\rho}  = \max \left\lbrace \frac{1-b_x}{M_x}, \frac{1-b_y}{M_y} , 1- \frac{\gamma_x}{2}\lambda_{m-1}(I-W) , 1- \frac{\gamma_y}{2}\lambda_{m-1}(I-W) , 1-\alpha_{x} , 1-\alpha_{y} , 1-\frac{p}{2} \right\rbrace,  \\
\rho & = \max \left\lbrace  1 - \frac{3b_x}{7}, 1 - \frac{3b_y}{7} , 1- \frac{\gamma_x}{2}\lambda_{m-1}(I-W) , 1- \frac{\gamma_y}{2}\lambda_{m-1}(I-W) , 1-\alpha_{x} , 1-\alpha_{y} , 1-\frac{p}{2} \right\rbrace , \label{rho_svrg} \\
& \tilde{\Phi}_{t} = \hat{\Phi}_t + \tilde{c}_x \left\Vert \tilde{\bx}_{t} - \mathbf{1}x^\star \right\Vert^2 + \tilde{c}_y \left\Vert \tilde{\by}_{t} - \mathbf{1}y^\star \right\Vert^2 . \label{tilde_Phi_t}
\end{align}
{\red{It is worth mentioning that $\gamma_{x}$ and $\gamma_{y}$} are well defined for $\delta = 0$.}

\begin{lemma} \label{lem:parameters_feas_svrg} \textbf{Parameters Feasibility} The parameters defined in \eqref{eq:bx_by_svrg}, \eqref{eq:alpha_x_alpha_y_svrg}, \eqref{eq:gamma_x_svrg} and \eqref{eq:gammma_y_svrg} satisfy the followings:
\begin{align}
& b_x \in (0,1), \  b_y \in (0,1) , \\
& \alpha_{x} < \min \left\lbrace \frac{b_{x}}{\sqrt{\delta}} , \frac{1}{1+\delta} \right\rbrace , \  \alpha_{y} < \min \left\lbrace \frac{b_{y}}{\sqrt{\delta}} , \frac{1}{1+\delta} \right\rbrace \\
& \gamma_{x} \in \left( 0, \min \left\lbrace \frac{2-2\sqrt{\delta}\alpha_{x}}{\lambda_{\max}(I-W)}, \frac{\alpha_{x} - (1+\delta)\alpha_{x}^2}{\sqrt{\delta}\lambda_{\max}(I-W)} \right\rbrace  \right) , \\
& \gamma_{y} \in \left( 0, \min \left\lbrace \frac{2-2\sqrt{\delta}\alpha_{y}}{\lambda_{\max}(I-W)}, \frac{\alpha_{y} - (1+\delta)\alpha_{y}^2}{\sqrt{\delta}\lambda_{\max}(I-W)} \right\rbrace  \right) ,\\
& \frac{\gamma_{x}}{2}\lambda_{m-1}(I-W) \ \in \ (0,1) , \ \frac{\gamma_{y}}{2}\lambda_{m-1}(I-W) \ \in \  (0,1) , \\
& M_{x}  \in (0,1), \ M_{y} \in (0,1) , \\
& \frac{1-b_{x}}{M_{x}} \in (0,1), \ \ \frac{1-b_{y}}{M_{y}} \in (0,1) .
\end{align}
Moreover,
\begin{align}
M_x & \geq 1 - \frac{8b_x\sqrt{\delta  }}{7(1+\delta)} \geq  1 - \frac{4b_x}{7} , \\
M_y & \geq 1 - \frac{8b_y\sqrt{\delta  }}{7(1+\delta)} \geq  1 - \frac{4b_y}{7} , \\
\frac{1-b_x}{M_x} & < 1 - \frac{3b_x}{7} , \  \frac{1-b_y}{M_y} < 1 - \frac{3b_y}{7} , \\
1 - \frac{\gamma_x}{2} \lambda_{m-1}(I-W) & = \begin{cases} 1 - \frac{b_x}{8\sqrt{\delta}(1+\delta)\kappa_g} \ ; \ \text{if} \ b_x \leq \sqrt{\delta} \\ 1 - \frac{1}{8(1+\delta)\kappa_g} \ ; \ \text{if} \ b_x > \sqrt{\delta} \end{cases} .
\end{align}

\end{lemma}

\begin{proof}

We show that the chosen parameters $\alpha_x, \alpha_y, M_{x}, M_y , \gamma_x$ and $\gamma_y$ satisfy the conditions given in Lemma~\ref{lem:parameters_feas_svrg}.
We first show that $b_x \in (0,1)$ and $b_y \in (0,1)$. From definition, 
\begin{align}
b_x & = s\mu_x - \frac{4s^2L^2_{yx}}{np_{\min}} - \tilde{c}_x p < s\mu_x = \frac{\mu n p_{\min} \mu_x}{24L^2} \leq \frac{n p_{\min} \mu^2_x}{24L_{xx}^2} = \frac{n p_{\min} }{24\kappa_x^2} \leq  \frac{1 }{24} < 1. \label{eq:bx_svrg_ub}
\end{align}
Similarly,
\begin{align}
b_y & = s\mu_y - \frac{4s^2L^2_{xy}}{np_{\min}} - \tilde{c}_y p < s\mu_y = \frac{\mu n p_{\min} \mu_y}{24L^2} \leq \frac{n p_{\min} \mu^2_y}{24L_{yy}^2} = \frac{n p_{\min} }{24\kappa_y^2} \leq  \frac{1 }{24} < 1. \label{eq:by_svrg_ub}
\end{align}
We now focus on the lower bound on $b_x$ and $b_y$.
\begin{align}
b_x & = s\mu_x - \frac{4s^2L^2_{yx}}{np_{\min}} - \tilde{c}_x p \nonumber\\
& \geq s\mu - \frac{12s^2L^2}{np_{\min}} - \frac{8s^2L^2}{np_{\min}} \nonumber\\
& = s\mu - \frac{20s^2L^2}{np_{\min}} \\
& = \frac{\mu^2 n p_{\min}}{24L^2} - \frac{\mu^2 n^2 p^2_{\min}}{576L^4} \frac{20L^2}{np_{\min}} \nonumber\\
& = \frac{ n p_{\min}}{24\kappa_f^2} - \frac{20 n p_{\min}}{576\kappa_f^2} \nonumber\\
& = \frac{np_{\min}}{144\kappa_f^2} \label{eq:bx_lb} \nonumber\\
& > 0 .
\end{align}
In a similar fashion, we get $b_y < 1$ and 
\begin{align}
b_y \geq \frac{np_{\min}}{144\kappa_f^2}>0 . \label{eq:by_lb}
\end{align}

\textbf{Feasibility of $\alpha_{x}$ and $\alpha_{y}$.}

We have, $0 < b_{x} < 1$. Therefore, $\alpha_{x} < \frac{1}{1+\delta}$. Moreover, $\frac{\sqrt{\delta}}{1+\delta} \leq 1/2$. Therefore, $\alpha_{x} \leq \frac{b_{x}}{2\sqrt{\delta}} < b_{x}/\sqrt{\delta}$. Hence, $\alpha_{x} < \min \left\lbrace \frac{b_{x}}{\sqrt{\delta}} , \frac{1}{1+\delta} \right\rbrace$ . Similarly, $\alpha_{y} < \min \left\lbrace \frac{b_{y}}{\sqrt{\delta}} , \frac{1}{1+\delta} \right\rbrace$ because $b_{y} \in (0,1)$.

\textbf{Feasibility of $\gamma_{x}$ and $\gamma_{y}$.} We consider two cases to verify the feasibility of $\gamma_x$ and $\gamma_y$.

\textbf{Case I:} $b_x \leq \sqrt{\delta}$.

This gives $\gamma_x = \frac{b_x}{4\sqrt{\delta}(1+\delta)\lambda_{\max}(I-W)}$. Consider
\begin{align}
\frac{\alpha_{x} - (1+\delta)\alpha_{x}^2}{\sqrt{\delta}\lambda_{\max}(I-W)} & = \frac{b_{x} - b_{x}^2}{\sqrt{\delta}(1+\delta)\lambda_{\max}(I-W)} .
\end{align}
 Using \eqref{eq:bx_svrg_ub}, we have $b_{x}  \leq \frac{1}{24\kappa^2_x } < 0.75$. This allows us to use the inequality $2x-2x^2 \geq x/2$ for all $0 \leq x \leq 0.75$. Therefore,
\begin{align}
\frac{\alpha_{x} - (1+\delta)\alpha_{x}^2}{\sqrt{\delta}\lambda_{\max}(I-W)} & > \frac{b_{x}}{4\sqrt{\delta}(1+\delta) \lambda_{\max}(I-W)} \nonumber\\
& = \gamma_{x,k} .
\end{align}
We also have 
\begin{align}
\frac{2-2\sqrt{\delta}\alpha_{x}}{\lambda_{\max}(I-W)} & = \left( 2 - \frac{2\sqrt{\delta}b_{x}}{1+\delta} \right) \frac{1}{\lambda_{\max}(I-W)} \geq \left( 2 - \frac{2\sqrt{\delta}}{1+\delta} \right) \frac{1}{\lambda_{\max}(I-W)} \nonumber\\
& \geq \frac{1}{\lambda_{\max}(I-W)} > \frac{1}{4(1+\delta) \lambda_{\max}(I-W)} \nonumber\\
& > \frac{b_x}{4\sqrt{\delta}(1+\delta) \lambda_{\max}(I-W)} \nonumber\\
& = \gamma_{x,k} ,
\end{align}
where the second last inequality uses $b_x \leq \sqrt{\delta}$. We know that $b_{y} \in (0,1)$. Therefore, by following similar steps, the chosen $\gamma_{y}$ is also feasible. 

\textbf{Case II:} $ b_x > \sqrt{\delta}$

This give $\gamma_x = \frac{1}{4(1+\delta)\lambda_{\max}(I-W)}$. 

\begin{align}
\frac{\alpha_{x} - (1+\delta)\alpha_{x}^2}{\sqrt{\delta}\lambda_{\max}(I-W)} & = \frac{b_{x} - b_{x}^2}{\sqrt{\delta}(1+\delta)\lambda_{\max}(I-W)} \nonumber\\
& \geq \frac{b_{x}}{4\sqrt{\delta}(1+\delta) \lambda_{\max}(I-W)} \nonumber\\
& > \frac{1}{4(1+\delta) \lambda_{\max}(I-W)} \nonumber\\
& = \gamma_x .
\end{align}
Consider
\begin{align}
\frac{2-2\sqrt{\delta}\alpha_{x}}{\lambda_{\max}(I-W)} & = \left( 2 - \frac{2\sqrt{\delta}b_{x}}{1+\delta} \right) \frac{1}{\lambda_{\max}(I-W)} \nonumber\\
& \geq \left( 2 - \frac{2\sqrt{\delta}}{1+\delta} \right) \frac{1}{\lambda_{\max}(I-W)} \nonumber\\
& \geq \frac{1}{\lambda_{\max}(I-W)} \nonumber\\
& > \frac{1}{4(1+\delta) \lambda_{\max}(I-W)} \nonumber\\
& = \gamma_{x} .
\end{align}
Therefore, $\gamma_x < \min \left\lbrace \frac{\alpha_{x} - (1+\delta)\alpha_{x}^2}{\sqrt{\delta}\lambda_{\max}(I-W)} , \frac{2-2\sqrt{\delta}\alpha_{x}}{\lambda_{\max}(I-W)} \right\rbrace $.

As $\gamma_{x} < \frac{2-2\sqrt{\delta}\alpha_{x}}{\lambda_{\max}(I-W)} < \frac{2}{\lambda_{\max}(I-W)}$. Notice that $\lambda_{m-1}(I-W) < \lambda_{\max}(I-W)$ Therefore,
\begin{align}
\frac{\gamma_{x}}{2}\lambda_{m-1}(I-W) < \frac{\gamma_{x}}{2}\lambda_{\max}(I-W) < 1 .
\end{align}
Similarly, $\frac{\gamma_{y}}{2}\lambda_{m-1}(I-W) < 1$.

\textbf{Feasibility of $M_{x}$ and $M_{y}$.}

Recall $M_{x} = 1-\frac{\sqrt{\delta }\alpha_{x}}{1-\frac{\gamma_{x}}{2}\lambda_{\max}(I-W)}$ and $M_{y} = 1-\frac{\sqrt{\delta }\alpha_{y}}{1-\frac{\gamma_{y}}{2}\lambda_{\max}(I-W)}$.
We have 
\begin{align}
\gamma_{x} < \frac{2-2\sqrt{\delta}\alpha_{x}}{\lambda_{\max}(I-W)} \nonumber\\
\frac{\gamma_{x}\lambda_{\max}(I-W)}{2} < 1-\sqrt{\delta}\alpha_{x} \nonumber\\
1 - \frac{\gamma_{x}\lambda_{\max}(I-W)}{2} > \sqrt{\delta}\alpha_{x} \nonumber\\
\frac{\sqrt{\delta}\alpha_{x}}{1 - \frac{\gamma_{x}\lambda_{\max}(I-W)}{2}} < 1 .
\end{align}
Moreover, $\frac{\sqrt{\delta}\alpha_{x}}{1 - \frac{\gamma_{x}\lambda_{\max}(I-W)}{2}} > 0$. Therefore, $M_{x} \in (0,1) $. Similar steps follow to prove the feasibility of $M_{y}$.

\textbf{Feasibility of $\frac{1-b_{x}}{M_{x}}$ and $\frac{1-b_{y}}{M_{y}}$.}

We derive upper bounds on $\frac{1-b_{x}}{M_{x}}$ and $\frac{1-b_{y}}{M_{y}} $ to verify the feasibility.
We divide the derivation into two cases.

\textbf{Case I:} $b_x \leq \sqrt{\delta}$ 

This implies that
\begin{align}
\gamma_x = \frac{b_x}{4\sqrt{\delta}(1+\delta)\lambda_{\max}(I-W)} \\
\frac{\gamma_x}{2} \lambda_{\max}(I-W) = \frac{b_x}{8\sqrt{\delta}(1+\delta)} .
\end{align}
Recall $M_x$:
\begin{align}
M_x & = 1 - \frac{\sqrt{\delta }\alpha_{x}}{1-\frac{\gamma_{x}}{2}\lambda_{\max}(I-W)} \nonumber\\
& = 1 - \frac{\frac{\sqrt{\delta }b_x}{1+\delta}}{1-\frac{b_x}{8\sqrt{\delta}(1+\delta)}} \nonumber\\
& = 1 - \frac{\sqrt{\delta}b_x \times 8 \sqrt{\delta}(1+\delta)}{(1+\delta)\left(8\sqrt{\delta}(1+\delta) - b_x \right)} \nonumber\\
& = 1 - \frac{8\delta b_x}{\left(8\sqrt{\delta}(1+\delta) - b_x \right)} \nonumber\\
& = 1 - \frac{8\delta }{\frac{8\sqrt{\delta}(1+\delta)}{b_x} - 1} . \label{eq:Mx_svrg_lb_mid}
\end{align}
We know that $\frac{\sqrt{\delta}}{b_x} \geq 1$. Therefore, $\frac{\sqrt{\delta}(1+\delta)}{b_x} > 1$ which in turn implies that 
\begin{align}
\frac{8\sqrt{\delta}(1+\delta)}{b_x} - 1 & > \frac{8\sqrt{\delta}(1+\delta)}{b_x} -  \frac{\sqrt{\delta}(1+\delta)}{b_x} \nonumber\\
& = \frac{7\sqrt{\delta}(1+\delta)}{b_x} \nonumber\\
\frac{1}{\frac{8\sqrt{\delta}(1+\delta)}{b_x} - 1} & < \frac{b_x}{7\sqrt{\delta}(1+\delta)} .
\end{align}
By using above relation in \eqref{eq:Mx_svrg_lb_mid}, we obtain
\begin{align}
M_x & \geq 1 - \frac{8\delta b_x }{7\sqrt{\delta}(1+\delta)} = 1 - \frac{8b_x\sqrt{\delta  }}{7(1+\delta)} \label{eq:Mx_lb_delta_svrg} \\
& \geq 1 - \frac{8b_x}{7} \frac{1}{2}  = 1 - \frac{4b_x}{7} , \label{eq:Mx_lb_svrg} 
\end{align}
where the second last inequality uses $\frac{\sqrt{\delta}}{1+\delta} \leq \frac{1}{2}$.
\begin{align}
\frac{1-b_x}{M_x} & = 1 + \frac{1-b_x}{M_x} - 1  \leq 1 + \frac{1-b_x}{1 - \frac{8b_x\sqrt{\delta  }}{7(1+\delta)}} - 1 \nonumber\\
& =  1 + \frac{1-b_x - 1+ \frac{8b_x\sqrt{\delta  }}{7(1+\delta)} }{1 - \frac{8b_x\sqrt{\delta  }}{7(1+\delta)}}  = 1 - \frac{b_x - \frac{8b_x\sqrt{\delta  }}{7(1+\delta)} }{1 - \frac{8b_x\sqrt{\delta  }}{7(1+\delta)}} \nonumber\\
& = 1 - \frac{7b_x(1+\delta) - 8b_x\sqrt{\delta} }{7(1+\delta) - 8b_x\sqrt{\delta} } = 1 - \frac{7(1+\delta) - 8\sqrt{\delta} }{\frac{7(1+\delta)}{b_x} - 8\sqrt{\delta} } \nonumber\\
& \leq 1 - \frac{7(1+\delta) - \frac{8(1+\delta)}{2} }{\frac{7(1+\delta)}{b_x} - 8\sqrt{\delta} } = 1 - \frac{3(1+\delta)}{\frac{7(1+\delta)}{b_x} - 8\sqrt{\delta} } \nonumber\\
& < 1 - \frac{3(1+\delta)}{\frac{7(1+\delta)}{b_x}} \nonumber\\
& = 1 - \frac{3b_x}{7} \label{eq:one_bx_Mx_ub_svrg} \nonumber\\
& < 1 .
\end{align}
Similarly, we obtain 
\begin{align}
& M_y \geq 1 - \frac{8b_y\sqrt{\delta  }}{7(1+\delta)} \geq  1 - \frac{4b_y}{7} \ \text{and, } \label{eq:My_lb_svrg} \\
& \frac{1-b_y}{M_y} < 1 - \frac{3b_y}{7} . \label{eq:one_by_My_ub_svrg}
\end{align}

\textbf{Case II:}  $b_x > \sqrt{\delta}$ .
\begin{align}
\gamma_x = \frac{1}{4(1+\delta)\lambda_{\max}(I-W)} .
\end{align}
We have
\begin{align}
M_x & = 1 - \frac{\sqrt{\delta }\alpha_{x}}{1-\frac{\gamma_{x}}{2}\lambda_{\max}(I-W)} \nonumber\\
& = 1 - \frac{\sqrt{\delta }\alpha_{x}}{1-\frac{1}{8(1+\delta)}} \nonumber\\
& =  1 - \frac{\frac{\sqrt{\delta }b_{x}}{1+\delta}}{1-\frac{1}{8(1+\delta)}} \nonumber\\
& = 1 - \frac{\sqrt{\delta }b_x \times 8(1+\delta) }{(1+\delta)(8(1+\delta) - 1)} \nonumber\\
& = 1 - \frac{8\sqrt{\delta}b_x}{8(1+\delta) - 1} .
\end{align}
As $8(1+\delta) - 1 > 8(1+\delta) - 1  -\delta = 7(1+\delta)$. Therefore,
\begin{align}
M_x & \geq 1 - \frac{8\sqrt{\delta}b_x}{7(1+\delta)} .
\end{align}
Notice that above lower bound matches with lower bound in \eqref{eq:Mx_lb_delta_svrg}. Therefore, by following steps similar to Case I, we obtain
\begin{align}
\frac{1-b_x}{M_x} & < 1 - \frac{3b_x}{7} , \ \text{and} \\
\frac{1-b_y}{M_y} & < 1 - \frac{3b_y}{7} .
\end{align}

\textbf{Feasibility of $1 - \frac{\gamma_{x}}{2}\lambda_{m-1}(I-W)$ and $1 - \frac{\gamma_{y}}{2}\lambda_{m-1}(I-W)$.}

If $b_x \leq \sqrt{\delta}$, then $\gamma_x = \frac{b_x}{4\sqrt{\delta}(1+\delta)\lambda_{\max}(I-W)}$. 
\begin{align}
1 - \frac{\gamma_x}{2} \lambda_{m-1}(I-W) & = 1 -  \frac{b_x}{8\sqrt{\delta}(1+\delta)\lambda_{\max}(I-W)} \lambda_{m-1}(I-W) \nonumber\\
& = 1 - \frac{b_x}{8\sqrt{\delta}(1+\delta)\kappa_g}.
\end{align}

If $b_x > \sqrt{\delta}$, then  $\gamma_x = \frac{1}{4(1+\delta)\lambda_{\max}(I-W)}$. 
\begin{align}
1 - \frac{\gamma_x}{2} \lambda_{m-1}(I-W) & = 1 -  \frac{1}{8(1+\delta)\lambda_{\max}(I-W)} \lambda_{m-1}(I-W) \\
& = 1 - \frac{1}{8(1+\delta)\kappa_g}.
\end{align}
\end{proof}

\begin{lemma} \label{lem:exp_phi_tilde_rec} Given the initial iterates $\bx_{T_0}, \by_{T_0}, D^\bx_{T_0}, D^\by_{T_0}, H^\bx_{T_0}, H^\by_{T_0}, H^{\ssw,\bx}_{T_0}$ and $ H^{\ssw,\by}_{T_0}$, let $\{\bx_{t} \}_t, \{\by_{t}\}_t$ be the sequences generated by Algorithm \ref{alg:generic_procedure_sgda} using SVRGO. Suppose Assumptions \ref{s_convexity_assumption}-\ref{weight_matrix_assumption} and Assumptions \ref{smoothness_x_svrg}-\ref{lipschitz_yx_svrg} hold. Then for every $t \geq T_0 $, $E[ \tilde{\Phi}_{t+1}] \leq \rho E[\tilde{\Phi}_{t }]$,
where ${\tilde{\Phi}_t}$ denotes the distance of the iterates $\bx_t, \tilde{\bx}_t, \by_t,  \tilde{\by}_t, D^\bx_t, D^\by_t, H^\bx_t, H^\by_t$ from their respective limit points (described in eq. \eqref{tilde_Phi_t} and $ \rho$ is a problem dependent parameter defined in eq. \eqref{rho_svrg}.
\end{lemma}
\begin{proof}
Adding inequalities~\eqref{eq:one_step_progress_y} and~\eqref{eq:one_step_progress_x}  (Lemma \ref{lem:recursion_y} and Lemma \ref{lem:recursion_x}), we have
\begin{align}
& M_x E\left\Vert \bx_{t+1} - \mathbf{1}x^\star  \right\Vert^2 + \frac{2s^2}{\gamma_x} E \left\Vert  D^\bx_{t+1} - D^\star_\bx  \right\Vert^2_{(I-W)^\dagger} + \sqrt{\delta}E\left\Vert H^\bx_{t+1} - H^\star_\bx \right\Vert^2 + \nonumber\\ & \ + M_y E\left\Vert \by_{t+1} - \mathbf{1}y^\star  \right\Vert^2 + \frac{2s^2}{\gamma_y} E \left\Vert  D^\by_{t+1} - D^\star_\by  \right\Vert^2_{(I-W)^\dagger} + \sqrt{\delta}E\left\Vert H^\by_{t+1} - H^\star_\by \right\Vert^2 \nonumber\\
& \leq \left\Vert \bx_{t} - \mathbf{1}x^\star - s\mathcal{G}^\bx_{t} + s\nabla_x F(\mathbf{1}z^\star)  \right\Vert^2 + \frac{2s^2}{\gamma_x}\left( 1- \frac{\gamma_y}{2}\lambda_{m-1}(I-W) \right)\left\Vert D^\bx_{t} - D^\star_\bx \right\Vert^2_{(I-W)^\dagger} \nonumber\\ & \ \ + \sqrt{\delta}(1-\alpha_{x})\left\Vert H^\bx_{t} - H^\star_\bx  \right\Vert^2 \nonumber\\ & + \ \left\Vert \by_{t} - \mathbf{1}y^\star + s\mathcal{G}^\by_{t} - s\nabla_y F(\mathbf{1}z^\star)  \right\Vert^2 + \frac{2s^2}{\gamma_y}\left( 1- \frac{\gamma_y}{2}\lambda_{m-1}(I-W) \right)\left\Vert D^\by_{t} - D^\star_\by \right\Vert^2_{(I-W)^\dagger} \nonumber\\ & \ \ + \sqrt{\delta}(1-\alpha_{y})\left\Vert H^\by_{t} - H^\star_\by  \right\Vert^2  .
\end{align}
By the definition of $\Phi_{t}$, the above inequality can be rewritten as
\begin{align}
& E\left[ \Phi_{t+1} \right] \nonumber\\
& \leq \left\Vert \bx_{t} - \mathbf{1}x^\star - s\mathcal{G}^\bx_{t} + s\nabla_x F(\mathbf{1}z^\star)  \right\Vert^2 + \frac{2s^2}{\gamma_x}\left( 1- \frac{\gamma_y}{2}\lambda_{m-1}(I-W) \right)\left\Vert D^\bx_{t} - D^\star_\bx \right\Vert^2_{(I-W)^\dagger} \nonumber\\ & \ \ + \sqrt{\delta}(1-\alpha_{x})\left\Vert H^\bx_{t} - H^\star_\bx  \right\Vert^2 \nonumber\\ & + \ \left\Vert \by_{t} - \mathbf{1}y^\star + s\mathcal{G}^\by_{t} - s\nabla_y F(\mathbf{1}z^\star)  \right\Vert^2 + \frac{2s^2}{\gamma_y}\left( 1- \frac{\gamma_y}{2}\lambda_{m-1}(I-W) \right)\left\Vert D^\by_{t} - D^\star_\by \right\Vert^2_{(I-W)^\dagger} \nonumber\\ & \ \ + \sqrt{\delta}(1-\alpha_{y})\left\Vert H^\by_{t} - H^\star_\by  \right\Vert^2 .
\end{align}
Taking conditional expectation on stochastic gradient at $t$-th step on both sides of above inequality and applying tower property, we obtain
\begin{align}
& E\left[ \Phi_{t+1} \right] \nonumber\\
& \leq E\left\Vert \bx_{t} - \mathbf{1}x^\star - s\mathcal{G}^\bx_{t} + s\nabla_x F(\mathbf{1}z^\star)  \right\Vert^2 + \frac{2s^2}{\gamma_x}\left( 1- \frac{\gamma_x}{2}\lambda_{m-1}(I-W) \right)\left\Vert D^\bx_{t} - D^\star_\bx \right\Vert^2_{(I-W)^\dagger} \nonumber\\ & \ \ + \sqrt{\delta}(1-\alpha_{x})\left\Vert H^\bx_{t} - H^\star_\bx  \right\Vert^2\nonumber \\ & + \ E\left\Vert \by_{t} - \mathbf{1}y^\star + s\mathcal{G}^\by_{t} - s\nabla_y F(\mathbf{1}z^\star)  \right\Vert^2 + \frac{2s^2}{\gamma_y}\left( 1- \frac{\gamma_y}{2}\lambda_{m-1}(I-W) \right)\left\Vert D^\by_{t} - D^\star_\by \right\Vert^2_{(I-W)^\dagger} \nonumber \\ & \ \ + \sqrt{\delta}(1-\alpha_{y})\left\Vert H^\by_{t} - H^\star_\by  \right\Vert^2 \nonumber\\
& \leq \left( 1-\mu_x s + \frac{4s^2L^2_{yx}}{np_{\min}} \right)\left\Vert \bx_{t} - \mathbf{1}x^\star \right\Vert^2 + \left(1-s\mu_y + \frac{4s^2L^2_{xy}}{np_{\min}} \right) \left\Vert \by_{t} - \mathbf{1}y^\star  \right\Vert^2  \nonumber\\ & \ + \frac{4s^2(L^2 + L^2_{yx})}{np_{\min}} \left\Vert \tilde{\bx}_{t} - \mathbf{1}x^\star \right\Vert^2 + \frac{4s^2(L^2 + L^2_{xy})}{np_{\min}} \left\Vert \tilde{\by}_{t} - \mathbf{1}y^\star \right\Vert^2 \nonumber\\ & + \frac{2s^2}{\gamma_x}\left( 1- \frac{\gamma_x}{2}\lambda_{m-1}(I-W) \right)\left\Vert D^\bx_{t} - D^\star_\bx \right\Vert^2_{(I-W)^\dagger} + \frac{2s^2}{\gamma_y}\left( 1- \frac{\gamma_y}{2}\lambda_{m-1}(I-W) \right)\left\Vert D^\by_{t} - D^\star_\by \right\Vert^2_{(I-W)^\dagger} \nonumber\\ & \ + \sqrt{\delta}(1-\alpha_{x})\left\Vert H^\bx_{t} - H^\star_\bx  \right\Vert^2 + \sqrt{\delta}(1-\alpha_{y})\left\Vert H^\by_{t} - H^\star_\by  \right\Vert^2 . \label{Ephi_t+1}
\end{align}
The last step holds due to Corollary \ref{cor:exp_xkt_ykt_svrg}. From SVRG oracle, we have
\begin{align}
\left\Vert \tilde{x}^i_{t+1} - x^\star \right\Vert^2  & = \begin{cases} \left\Vert x^i_{t} - x^\star \right\Vert^2 \ \text{with probability} \ p \nonumber\\ \left\Vert \tilde{x}^i_{t} - x^\star \right\Vert^2 \ \text{with probability} \ 1-p \end{cases} , \text{ and}\nonumber\\
\left\Vert \tilde{y}^i_{t+1} - y^\star \right\Vert^2  & = \begin{cases} \left\Vert y^i_{t} - y^\star \right\Vert^2 \ \text{with probability} \ p \nonumber\\ \left\Vert \tilde{y}^i_{t} - y^\star \right\Vert^2 \ \text{with probability} \ 1-p \end{cases} .
\end{align}
Therefore,
\begin{align}
& E \left\Vert \tilde{x}_{t+1} - \mathbf{1}x^\star \right\Vert^2 + E\left\Vert \tilde{y}_{t+1} - \mathbf{1}y^\star \right\Vert^2 \nonumber\\
& = \sum_{i = 1}^m E\left\Vert \tilde{x}^i_{t+1} - x^\star \right\Vert^2 + \sum_{i = 1}^m E\left\Vert \tilde{y}^i_{t+1} - y^\star \right\Vert^2 \nonumber\\
& = \sum_{i = 1}^m \left( p\left\Vert x^i_{t} - x^\star \right\Vert^2 + (1-p)\left\Vert \tilde{x}^i_{t} - x^\star \right\Vert^2  \right) + \sum_{i = 1}^m \left( p\left\Vert y^i_{t} - y^\star \right\Vert^2 + (1-p)\left\Vert \tilde{y}^i_{t} - y^\star \right\Vert^2  \right) \nonumber\\
& = p\left\Vert \bx_{t} - \mathbf{1}x^\star \right\Vert^2 + (1-p)\left\Vert \tilde{\bx}_{t} - \mathbf{1}x^\star \right\Vert^2 +  p\left\Vert \by_{t} - \mathbf{1}y^\star \right\Vert^2 + (1-p)\left\Vert \tilde{\by}_{t} - \mathbf{1}y^\star \right\Vert^2 .
\end{align}
Using above equality and \eqref{Ephi_t+1}, we obtain
\begin{align}
& E\left[ \Phi_{t+1} \right] + \tilde{c}_x E \left\Vert \tilde{x}_{t+1} - \mathbf{1}x^\star \right\Vert^2 + \tilde{c}_y E \left\Vert \tilde{y}_{t+1} - \mathbf{1}y^\star \right\Vert^2 \nonumber\\
& \leq \left( 1-\mu_x s + \frac{4s^2L^2_{yx}}{np_{\min}} + \tilde{c}_x p \right)\left\Vert \bx_{t} - \mathbf{1}x^\star \right\Vert^2 + \left(1-s\mu_y + \frac{4s^2L^2_{xy}}{np_{\min}} + \tilde{c}_yp \right) \left\Vert \by_{t} - \mathbf{1}y^\star  \right\Vert^2  \nonumber\\ & \ + \left( \tilde{c}_x(1-p) + \frac{4s^2(L^2 + L^2_{yx})}{np_{\min}} \right) \left\Vert \tilde{\bx}_{t} - \mathbf{1}x^\star \right\Vert^2 + \left( \tilde{c}_y(1-p) + \frac{4s^2(L^2 + L^2_{xy})}{np_{\min}} \right) \left\Vert \tilde{\by}_{t} - \mathbf{1}y^\star \right\Vert^2 \nonumber\\ & + \frac{2s^2}{\gamma_x}\left( 1- \frac{\gamma_x}{2}\lambda_{m-1}(I-W) \right)\left\Vert D^\bx_{t} - D^\star_\by \right\Vert^2_{(I-W)^\dagger} + \frac{2s^2}{\gamma_y}\left( 1- \frac{\gamma_y}{2}\lambda_{m-1}(I-W) \right) \left\Vert D^\by_{t} - D^\star_\by \right\Vert^2_{(I-W)^\dagger} \nonumber\\ & \ + \sqrt{\delta}(1-\alpha_{x})\left\Vert H^\bx_{t} - H^\star_\bx  \right\Vert^2 + \sqrt{\delta}(1-\alpha_{y})\left\Vert H^\by_{t} - H^\star_\by  \right\Vert^2 . \label{eq:exp_phikt_ctilde_svrg}
\end{align}
We have $\tilde{c}_x = \frac{8s^2(L^2 + L^2_{yx})}{np_{\min}p}$. The coefficient of $\left\Vert \tilde{\bx}_{t} - \mathbf{1}x^\star \right\Vert^2$ in~\eqref{eq:exp_phikt_ctilde_svrg} is
\begin{align}
\tilde{c}_x(1-p) + \frac{4s^2(L^2 + L^2_{yx})}{np_{\min}} & = \tilde{c}_x \left( 1-p + \frac{4s^2(L^2 + L^2_{yx})}{np_{\min}\tilde{c}_x}  \right) \nonumber\\
& = \tilde{c}_x \left( 1-p + \frac{4s^2(L^2 + L^2_{yx})}{np_{\min}} \frac{np_{\min}p}{8s^2(L^2 + L^2_{yx})} \right) \nonumber\\
& = \tilde{c}_x \left( 1-p + \frac{p}{2} \right) \nonumber\\
& = \tilde{c}_x \left(1-\frac{p}{2} \right) ,
\end{align}
and the coefficient of $\left\Vert \tilde{\by}_{t} - \mathbf{1}y^\star \right\Vert^2$ in~\eqref{eq:exp_phikt_ctilde_svrg} is
\begin{align}
\tilde{c}_y(1-p) + \frac{4s^2(L^2 + L^2_{xy})}{np_{\min}} & = \tilde{c}_y \left( 1-p + \frac{4s^2(L^2 + L^2_{xy})}{np_{\min}\tilde{c}_y}  \right) \nonumber\\
& = \tilde{c}_y \left( 1-p + \frac{4s^2(L^2 + L^2_{xy})}{np_{\min}} \frac{np_{\min}p}{8s^2(L^2 + L^2_{xy})} \right) \nonumber\\
& = \tilde{c}_y \left( 1-p + \frac{p}{2} \right)\nonumber \\
& = \tilde{c}_y \left(1-\frac{p}{2} \right) .
\end{align}
Substituting the above simplified coefficients into \eqref{eq:exp_phikt_ctilde_svrg}, we see that
\begin{align}
& E\left[ \Phi_{t+1} \right] + \tilde{c}_x E \left\Vert \tilde{x}_{t+1} - \mathbf{1}x^\star \right\Vert^2 + \tilde{c}_y E \left\Vert \tilde{y}_{t+1} - \mathbf{1}y^\star \right\Vert^2 \nonumber\\
& \leq \left( 1-\mu_x s + \frac{4s^2L^2_{yx}}{np_{\min}} + \tilde{c}_x p \right)\left\Vert \bx_{t} - \mathbf{1}x^\star \right\Vert^2 + \left(1-s\mu_y + \frac{4s^2L^2_{xy}}{np_{\min}} + \tilde{c}_yp \right) \left\Vert \by_{t} - \mathbf{1}y^\star  \right\Vert^2  \nonumber\\ & \ + \tilde{c}_x \left(1-\frac{p}{2} \right) \left\Vert \tilde{\bx}_{t} - \mathbf{1}x^\star \right\Vert^2 + \tilde{c}_y \left(1-\frac{p}{2} \right) \left\Vert \tilde{\by}_{t} - \mathbf{1}y^\star \right\Vert^2 \nonumber\\ & + \frac{2s^2}{\gamma_x}\left( 1- \frac{\gamma_x}{2}\lambda_{m-1}(I-W) \right)\left\Vert D^\bx_{t} - D^\star_\bx \right\Vert^2_{(I-W)^\dagger} + \frac{2s^2}{\gamma_y}\left( 1- \frac{\gamma_y}{2}\lambda_{m-1}(I-W) \right) \left\Vert D^\by_{t} - D^\star_\by \right\Vert^2_{(I-W)^\dagger} \nonumber\\ & \ + \sqrt{\delta}(1-\alpha_{x})\left\Vert H^\bx_{t} - H^\star_\bx  \right\Vert^2 + \sqrt{\delta}(1-\alpha_{y})\left\Vert H^\by_{t} - H^\star_\by  \right\Vert^2 .
\end{align}
By taking total expectation on both sides, using tower property and using the definition of $b_x$ and $b_y$, we obtain
\begin{align}
& E\left[ \Phi_{t+1} \right] + \tilde{c}_x E \left\Vert \tilde{x}_{t+1} - \mathbf{1}x^\star \right\Vert^2 + \tilde{c}_y E \left\Vert \tilde{y}_{t+1} - \mathbf{1}y^\star \right\Vert^2 \nonumber\\
& \leq \left( 1-b_x \right)E\left\Vert \bx_{t} - \mathbf{1}x^\star \right\Vert^2 + \left(1- b_y \right) E\left\Vert \by_{t} - \mathbf{1}y^\star  \right\Vert^2 \nonumber\\ & \  + \tilde{c}_x \left(1-\frac{p}{2} \right)E\left\Vert \tilde{\bx}_{t} - \mathbf{1}x^\star \right\Vert^2 + \tilde{c}_y \left(1-\frac{p}{2} \right) E\left\Vert \tilde{\by}_{t} - \mathbf{1}y^\star \right\Vert^2 \nonumber\\ & + \frac{2s^2}{\gamma_x}\left( 1- \frac{\gamma_x}{2}\lambda_{m-1}(I-W) \right)E\left\Vert D^\bx_{t} - D^\star_\bx \right\Vert^2_{(I-W)^\dagger} + \frac{2s^2}{\gamma_y}\left( 1- \frac{\gamma_y}{2}\lambda_{m-1}(I-W) \right) E\left\Vert D^\by_{t} - D^\star_\by \right\Vert^2_{(I-W)^\dagger} \nonumber\\ & \ + \sqrt{\delta}(1-\alpha_{x})E\left\Vert H^\bx_{t} - H^\star_\bx  \right\Vert^2 + \sqrt{\delta}(1-\alpha_{y})E\left\Vert H^\by_{t} - H^\star_\by  \right\Vert^2 \nonumber\\
& = \left( \frac{1-b_x}{M_x} \right)M_xE\left\Vert \bx_{t} - \mathbf{1}x^\star \right\Vert^2 + \left(\frac{1- b_y}{M_y} \right) M_yE\left\Vert \by_{t} - \mathbf{1}y^\star  \right\Vert^2 \nonumber\\ & \  + \tilde{c}_x \left(1-\frac{p}{2} \right)E\left\Vert \tilde{\bx}_{t} - \mathbf{1}x^\star \right\Vert^2 + \tilde{c}_y \left(1-\frac{p}{2} \right) E\left\Vert \tilde{\by}_{t} - \mathbf{1}y^\star \right\Vert^2 \nonumber\\ & + \frac{2s^2}{\gamma_x}\left( 1- \frac{\gamma_x}{2}\lambda_{m-1}(I-W) \right)E\left\Vert D^\bx_{t} - D^\star_\bx \right\Vert^2_{(I-W)^\dagger} + \frac{2s^2}{\gamma_y}\left( 1- \frac{\gamma_y}{2}\lambda_{m-1}(I-W) \right) E\left\Vert D^\by_{t} - D^\star_\by \right\Vert^2_{(I-W)^\dagger} \nonumber\\ & \ + \sqrt{\delta}(1-\alpha_{x})E\left\Vert H^\bx_{t} - H^\star_\bx  \right\Vert^2 + \sqrt{\delta}(1-\alpha_{y})E\left\Vert H^\by_{t} - H^\star_\by  \right\Vert^2 \\
& \leq \max \left\lbrace 1 - \frac{3b_x}{7}, 1 - \frac{3b_y}{7} , 1- \frac{\gamma_x}{2}\lambda_{m-1}(I-W) , 1- \frac{\gamma_y}{2}\lambda_{m-1}(I-W) , 1-\alpha_{x} , 1-\alpha_{y} , 1-\frac{p}{2} \right\rbrace \nonumber \\ & \ \ \ \times \left( E\left[ \Phi_{t} \right] + \tilde{c}_x E \left\Vert \tilde{\bx}_{t} - \mathbf{1}x^\star \right\Vert^2 + \tilde{c}_y E \left\Vert \tilde{\by}_{t} - \mathbf{1}y^\star \right\Vert^2 \right) \nonumber\\
& = \rho \left( E\left[ \Phi_{t} \right] + \tilde{c}_x E \left\Vert \tilde{\bx}_{t} - \mathbf{1}x^\star \right\Vert^2 + \tilde{c}_y E \left\Vert \tilde{\by}_{t} - \mathbf{1}y^\star \right\Vert^2 \right) , \label{eq:exp_phikt_ztilde_rho_svrg}
\end{align}
where
\begin{align}
\rho & = \max \left\lbrace  1 - \frac{3b_x}{7}, 1 - \frac{3b_y}{7} , 1- \frac{\gamma_x}{2}\lambda_{m-1}(I-W) , 1- \frac{\gamma_y}{2}\lambda_{m-1}(I-W) , 1-\alpha_{x} , 1-\alpha_{y} , 1-\frac{p}{2} \right\rbrace . 
\end{align}
By the definition of $\tilde{\Phi}_{t} = \hat{\Phi}_{t} + \tilde{c}_x \left\Vert \tilde{\bx}_{t} - \mathbf{1}x^\star \right\Vert^2 + \tilde{c}_y \left\Vert \tilde{\by}_{t} - \mathbf{1}y^\star \right\Vert^2$, \eqref{eq:exp_phikt_ztilde_rho_svrg} reduces to
\begin{align}
E\left[ \tilde{\Phi}_{t+1} \right] & \leq \rho E\left[ \tilde{\Phi}_{t} \right] .
\end{align}
\end{proof}

\section{Proofs for the convergence analysis of Algorithm \ref{alg:IPDHG_with_sgd_svrg_oracle}} \label{switching_point_proofs}


\subsection{Proof of Lemma \ref{lem:tilde_phi_t_ub_T0_t} } \label{app_sec:lemma2_proof}
At $t=T_0$ in Algorithm \ref{alg:IPDHG_with_sgd_svrg_oracle}, GSGO switches to SVRGO and reference points $\tilde{\bx}_{T_0}, \tilde{\by}_{T_0}$ are initialized to $\bx_{T_0}, \by_{T_0}$. Further the iterates are initialized to $\bx_{T_0}, \by_{T_0}, D^\bx_{T_0}, D^\by_{T_0}, H^\bx_{T_0}, H^\by_{T_0}, H^{\ssw,\bx}_{T_0}, H^{\ssw,\by}_{T_0}$. Lemma \ref{lem:exp_phi_tilde_rec} describes the behavior of IPDHG with SVRG oracle. Therefore, using Lemma \ref{lem:exp_phi_tilde_rec}, we have 
\begin{align}
E[ \tilde{\Phi}_{t+1} ] \leq \rho^{t+1-T_0}\tilde{\Phi}_{T_0 }. \label{eq:phi_tilde_phiT0_ub}
\end{align}
Next step in the proof is to derive an upper bound on $\tilde{\Phi}_{T_0 }$ in terms of $\Phi_{T_0 }$. To obtain this, we first write $\left\Vert H^\bx_{T_0} - H^\star_\bx \right\Vert^2 + \left\Vert H^\by_{T_0} - H^\star_\by \right\Vert^2$ in terms of $\left\Vert H^\bx_{T_0} - H^\star_{\bx,0} \right\Vert^2 + \left\Vert H^\by_{T_0} - H^\star_{\by,0} \right\Vert^2$. In this direction, consider
\begin{align}
& \left\Vert H^\bx_{T_0} - H^\star_\bx \right\Vert^2 + \left\Vert H^\by_{T_0} - H^\star_\by \right\Vert^2 \nonumber \\
& = \left\Vert H^\bx_{T_0} - \mathbf{1}(x^\star - \frac{s_0}{m} \nabla_x f(z^\star)) + \frac{s}{m} \mathbf{1} \nabla_x f(z^\star) - \frac{s_0}{m} \mathbf{1} \nabla_x f(z^\star) \right\Vert^2  + \left\Vert H^\by_{T_0} -  \mathbf{1}(y^\star + \frac{s_0}{m} \nabla_y f(z^\star)) - \frac{s}{m} \mathbf{1} \nabla_y f(z^\star) + \frac{s_0}{m} \mathbf{1} \nabla_y f(z^\star) \right\Vert^2  \nonumber \\
& \leq 2\Vert H^\bx_{T_0} - \mathbf{1}(x^\star - \frac{s_0}{m} \nabla_x f(z^\star)) \Vert^2 + 2(s-s_0)^2\Vert \frac{\mathbf{1}}{m}  \nabla_x f(z^\star) \Vert^2  + 2\Vert H^\by_{T_0} -  \mathbf{1}(y^\star + \frac{s_0}{m} \nabla_y f(z^\star)) \Vert^2 + 2(s-s_0)^2\left\Vert \frac{\mathbf{1}}{m}  \nabla_y f(z^\star) \right\Vert^2 \nonumber \\
& =  2\left\Vert H^\bx_{T_0} - H^\star_{\bx,0} \right\Vert^2 + 2m(s-s_0)^2 \left\Vert \frac{1}{m} \sum_{i = 1}^m  \nabla_x f_i(z^\star) \right\Vert^2  + 2\left\Vert H^\by_{T_0} -  H^\star_{\by,0} \right\Vert^2 + 2m(s-s_0)^2 \left\Vert \frac{1}{m} \sum_{i = 1}^m  \nabla_y f_i(z^\star) \right\Vert^2. \label{eq:Hstar_x_Hxstar_0}
\end{align}
Using \eqref{eq:Hstar_x_Hxstar_0}, we now find an appropriate upper bound on $\tilde{\Phi}_{T_0 }$ in terms of $\Phi_{T_0}$. 
\begin{align}
	\tilde{\Phi}_{T_0 }
	& =  M_x \left\Vert \bx_{T_0} - \mathbf{1}x^\star  \right\Vert^2 + \frac{2s^2}{\gamma_x}  \left\Vert  D^\bx_{T_0} - D^\star_\bx  \right\Vert^2_{(I-W)^\dagger} + \sqrt{\delta}\left\Vert H^\bx_{T_0} - H^\star_\bx \right\Vert^2   + M_y \left\Vert \by_{T_0} - \mathbf{1}y^\star  \right\Vert^2 + \frac{2s^2}{\gamma_y}  \left\Vert  D^\by_{T_0} - D^\star_\by  \right\Vert^2_{(I-W)^\dagger} \nonumber\\ & \ + \sqrt{\delta}\left\Vert H^\by_{T_0} - H^\star_\by \right\Vert^2 \nonumber + \tilde{c}_x \left\Vert \bx_{T_0} - \mathbf{1}x^\star  \right\Vert^2 + \tilde{c}_y \left\Vert \by_{T_0} - \mathbf{1}y^\star  \right\Vert^2  \nonumber \\
	& \leq  M_x \left\Vert \bx_{T_0} - \mathbf{1}x^\star  \right\Vert^2 + \frac{2s^2}{\gamma_x}  \left\Vert  D^\bx_{T_0} - D^\star_\bx  \right\Vert^2_{(I-W)^\dagger} + 2\sqrt{\delta}\left\Vert H^\bx_{T_0} - H^\star_{\bx,0} \right\Vert^2  + M_y \left\Vert \by_{T_0} - \mathbf{1}y^\star  \right\Vert^2 + \frac{2s^2}{\gamma_y}  \left\Vert  D^\by_{T_0} - D^\star_\by  \right\Vert^2_{(I-W)^\dagger} \nonumber\\ & \ + 2\sqrt{\delta}\left\Vert H^\by_{T_0} -  H^\star_{\by,0} \right\Vert^2  + \tilde{c}_x \left\Vert \bx_{T_0} - \mathbf{1}x^\star  \right\Vert^2 + \tilde{c}_y \left\Vert \by_{T_0} - \mathbf{1}y^\star  \right\Vert^2 + 2m\sqrt{\delta}(s-s_0)^2 \Big( \Big \Vert \frac{1}{m} \sum_{i = 1}^m  \nabla_x f_i(z^\star) \Big \Vert^2 + \Big \Vert \frac{1}{m} \sum_{i = 1}^m  \nabla_y f_i(z^\star) \Big \Vert^2 \Big). \nonumber \\
	& \leq \max \left\lbrace \frac{M_x + \tilde{c}_x}{M_{x,0}}, \frac{M_y + \tilde{c}_y}{M_{y,0}}, \frac{s^2\gamma_{x,0}}{s_0^2\gamma_x}, \frac{s^2\gamma_{y,0}}{s_0^2\gamma_y}, 1 \right\rbrace  \times \Big ( M_{x,0} \left\Vert \bx_{T_0} - \mathbf{1}x^\star  \right\Vert^2 + \frac{2s_0^2}{\gamma_{x,0}}  \left\Vert  D^\bx_{T_0} - D^\star_\bx  \right\Vert^2_{(I-W)^\dagger} + \sqrt{\delta}\left\Vert H^\bx_{T_0} - H^\star_{\bx,0} \right\Vert^2  \nonumber\\ & \ \ + M_{y,0} \left\Vert \by_{T_0} - \mathbf{1}y^\star  \right\Vert^2 + \frac{2s_0^2}{\gamma_{y,0}}  \left\Vert  D^\by_{T_0} - D^\star_\by  \right\Vert^2_{(I-W)^\dagger} + \sqrt{\delta}\left\Vert H^\by_{T_0} - H^\star_{\by,0} \right\Vert^2 \Big ) \nonumber \\ & \ \ \ \ +  2m\sqrt{\delta}(s-s_0)^2 \Big( \Big \Vert \frac{1}{m} \sum_{i = 1}^m  \nabla_x f_i(z^\star) \Big \Vert^2 + \Big \Vert \frac{1}{m} \sum_{i = 1}^m  \nabla_y f_i(z^\star) \Big \Vert^2 \Big) \nonumber \\
	& = \max \left\lbrace \frac{M_x + \tilde{c}_x}{M_{x,0}}, \frac{M_y + \tilde{c}_y}{M_{y,0}}, \frac{s^2\gamma_{x,0}}{s_0^2\gamma_x}, \frac{s^2\gamma_{y,0}}{s_0^2\gamma_y}, 2 \right\rbrace  \Phi_{T_0} + 2m\sqrt{\delta}(s-s_0)^2 \left( \left\Vert \frac{1}{m} \sum_{i = 1}^m  \nabla_x f_i(z^\star) \right\Vert^2 + \left\Vert \frac{1}{m} \sum_{i = 1}^m  \nabla_y f_i(z^\star) \right\Vert^2 \right) \nonumber \\
	& = C_{\max} \Phi_{T_0} + C_1 ,\label{eq:phi_tilde_CmaxC1}
\end{align}
where 
\begin{align}
& C_{\max} := \max \left\lbrace \frac{M_x + \tilde{c}_x}{M_{x,0}}, \frac{M_y + \tilde{c}_y}{M_{y,0}}, \frac{s^2\gamma_{x,0}}{s_0^2\gamma_x}, \frac{s^2\gamma_{y,0}}{s_0^2\gamma_y}, 2 \right\rbrace \label{cons_Cmax} \\
& C_1 := 2m\sqrt{\delta}(s-s_0)^2 \left( \left\Vert \frac{1}{m} \sum_{i = 1}^m  \nabla_x f_i(z^\star) \right\Vert^2 + \left\Vert \frac{1}{m} \sum_{i = 1}^m  \nabla_y f_i(z^\star) \right\Vert^2 \right). \label{cons_C1}
\end{align}
By substituting \eqref{eq:phi_tilde_CmaxC1} in \eqref{eq:phi_tilde_phiT0_ub}, we obtain
\begin{align}
& E[ \tilde{\Phi}_{t+1} ] \leq \rho^{t+1-T_0} (C_{\max} \Phi_{T_0} + C_1). \label{eq:exp_phi_tilde_phiT0_ub}
\end{align} 
Note that the expectation in the l.h.s of \eqref{eq:exp_phi_tilde_phiT0_ub} is conditioned on initial iterates at $T_0$-th iteration (when GSGO switches to SVRGO). By taking expectation w.r.t randomness in the initial iterates  on both sides and using tower property, we obtain
\begin{align} 
 E[ \tilde{\Phi}_{t+1} ] & \leq \rho^{t+1-T_0} (C_{\max} E_0\left[ \Phi_{T_0} \right]  + C_1). \label{eq:exp_phitilde_exp_phiT0}
 \end{align}
Note that $E$ in l.h.s of \eqref{eq:exp_phitilde_exp_phiT0} now denotes the total expectation. Next, using Lemma \ref{lem:exp_phit+1_phi_0}, we get
\begin{align}
 E[ \tilde{\Phi}_{t+1} ] & \leq \rho^{t+1-T_0} C_{\max} \left(( \rho_0)^{T_0}  \Phi_{0}  + \frac{2s_0^2(C_x + C_y)}{(1-\rho_0)n^2p_{\min}}\right)  + C_1\rho^{t+1-T_0} \nonumber \\
& = \rho^{t+1-T_0} C_{\max} \left(( \rho_0)^{T_0}  \Phi_{0}  + V_e \right)  + C_1\rho^{t+1-T_0} \nonumber \\
& =  C_{\max} \left( \left(\frac{\rho_0}{\rho} \right)^{T_0} \rho^{t+1} \Phi_{0}  + \frac{V_e \rho^{t+1}}{\rho^{T_0}} \right)  + \frac{C_1\rho^{t+1}}{\rho^{T_0}} \nonumber \\
& \leq C_{\max} \left( \frac{\rho_0^{T_0}}{\rho^{T_0}} \rho^{t+1} \Phi_{0}  + \frac{V_e \rho^{t+1}}{\rho_0^{T_0}} \right)  + \frac{C_1\rho^{t+1}}{\rho_0^{T_0}} \label{eq:phi_tilde_ub_rho_T0},
\end{align} 
where the last inequality uses the fact that $\rho_0 \leq \rho$ ( because $b_{x,0} \geq b_x$, $b_{y,0} \geq b_y$). By substituting $\rho_0^{T_0} = \epsilon_0$ and total number of iterations $t = T-1$, \eqref{eq:phi_tilde_ub_rho_T0} reduces to 
\begin{align}
E[ \tilde{\Phi}_{T} ] &  \leq C_{\max} \left( \frac{\epsilon_0\Phi_{0}}{\rho^{T_0}} \rho^{T}  + \frac{V_e \rho^{T}}{\epsilon_0} \right)  + \frac{C_1\rho^{T}}{\epsilon_0} .\label{eq:phi_tilde_epsilon0}
\end{align}
\subsection{Deciding Switching Point}
Since $T \geq T_0+1$ and $ \rho \in (0,1)$, the upper bound in inequality \eqref{eq:phi_tilde_epsilon0} reduces to
\begin{align}
E[ \tilde{\Phi}_{T} ] & \leq C_{\max} \left( \epsilon_0 \Phi_{0}  + \frac{V_e \rho^{T}}{\epsilon_0} \right)  + \frac{C_1\rho^{T}}{\epsilon_0}. \label{eq:expt_phitilde_epsilon0}
\end{align}
Let $h(\epsilon_0) = C_{\max} \left( \epsilon_0 \Phi_{0}  + \frac{V_e \rho^{T}}{\epsilon_0} \right)  + \frac{C_1\rho^{T}}{\epsilon_0}$. Differentiating $h(\epsilon_0)$ w.r.t $\epsilon_0$, we get 
\begin{align}
h'(\epsilon_0) & = C_{\max} \left( \Phi_{0}  - \frac{V_e \rho^{T}}{\epsilon^2_0} \right)  - \frac{C_1\rho^{T}}{\epsilon^2_0} \nonumber \\
h''(\epsilon_0) & = \frac{2(C_{\max}V_e \rho^{T} + C_1\rho^{T})}{\epsilon^3_0} . \nonumber
\end{align}
Without loss of generality, we can assume that $V_e > 0$. Therefore, $h''(\epsilon_0) > 0$. By solving $h'(\epsilon^\star_0) = 0$, we get 
$\epsilon^\star_0 = \sqrt{\frac{(C_{\max} V_e + C_1)\rho^T}{C_{\max}\Phi_0}}$. The minimum value of $h(\epsilon_0)$ is given by
\begin{align}
h(\epsilon^\star_0) & = C_{\max} \Phi_{0} \sqrt{\frac{(C_{\max V_e + C_1})\rho^T}{C_{\max}\Phi_0}}  +  (C_{\max} V_e + C_1)\rho^T \sqrt{\frac{C_{\max}\Phi_0}{(C_{\max} V_e + C_1)\rho^T}} \nonumber \\
& = 2\sqrt{C_{\max}\Phi_0(C_{\max} V_e + C_1)\rho^T} . \label{eq:min_val_h_epsilon0}
\end{align}
Therefore, at $\epsilon_0 = \epsilon_0^\star$, we get
\begin{align}
E\left[ \tilde{\Phi}_{T} \right] & \leq  2\sqrt{C_{\max}\Phi_0(C_{\max} V_e + C_1)\rho^T} . \label{eq:tilde_phit_app}
\end{align}
Therefore, Algorithm \ref{alg:IPDHG_with_sgd_svrg_oracle} needs $ T(\epsilon) = \frac{2}{-\log \rho}\log \left(\frac{2\sqrt{C_{\max}\Phi_0(C_{\max} V_e + C_1)}}{\epsilon} \right)$ iterations to achieve $\epsilon$-accurate saddle point. By plugging $T = T(\epsilon)$ into $\epsilon^\star_0$, we get $\epsilon^\star_0 = \frac{\epsilon}{2C_{\max}\Phi_0}$. Hence,  switching point $T_0$ is given by $\frac{1}{\log \rho_0}\log(\frac{\epsilon}{2C_{\max}\Phi_0})$.

\subsection{Proof of Theorem \ref{thm:switching_scheme_complexity}}
Using \eqref{eq:tilde_phit_app}, we have $E\left[ \tilde{\Phi}_{T} \right]$  $\leq$  $2\sqrt{C_{\max}\Phi_0(C_{\max} V_e + C_1)\rho^T}$.

By choosing $T$ $=$ $T(\epsilon)$$=$ $\frac{2}{-\log \rho}\log \left(\frac{2\sqrt{C_{\max}\Phi_0(C_{\max} V_e + C_1)}}{\epsilon} \right)$, we get $E\left[ \tilde{\Phi}_{T(\epsilon)} \right] \leq \epsilon$. We now write the iteration complexity in terms of condition numbers $\kappa_f, \kappa_g$ and compression factor $\delta$.
\begin{align}
T(\epsilon) & = \frac{2}{-\log \rho}\log \left(\frac{2\sqrt{C_{\max}\Phi_0(C_{\max} V_e + C_1)}}{\epsilon} \right) \nonumber \\
& \leq 10\max \left\lbrace \frac{336\kappa_f^2}{np_{\min}} , \frac{1152\sqrt{\delta}(1+\delta)\kappa_g \kappa_f^2}{np_{\min}} ,8(1+\delta)\kappa_g, \frac{144(1+\delta)\kappa_f^2}{np_{\min}} , \frac{2}{p}   \right\rbrace log \left(\frac{2\sqrt{C_{\max}\Phi_0(C_{\max} V_e + C_1)}}{\epsilon} \right),
\end{align} 
where the inequality follows from \eqref{eq:log_rho_rec_ub}. It completes the proof of Theorem \ref{thm:switching_scheme_complexity}. The total number of gradients evaluated in Algorithm \ref{alg:IPDHG_with_sgd_svrg_oracle} to achieve target accuracy $\epsilon$ are $BT_0 + (2B+pN_\ell)(T(\epsilon) - T_0)  = (2B+pN_\ell)T(\epsilon) - (B + pN_\ell)
T_0$.

\section{Convergence of IPDHG with SVRGO (C-DPSVRG)}
\label{appendix_svrg_main_result}
In this section, we present the convergence behavior of C-DPSVRG.
\begin{theorem} \label{thm:svrg_complexity} Let $\{\bx_{t} \}_t, \{\by_{t}\}_t$  be the sequences generated by C-DPSVRG. Suppose Assumptions \ref{s_convexity_assumption}-\ref{weight_matrix_assumption} and Assumptions \ref{smoothness_x_svrg}-\ref{lipschitz_yx_svrg} hold. Then  iteration complexity $T(\epsilon)$ of C-DPSVRG for achieving $\epsilon$-accurate saddle point solution in expectation is
	{\footnotesize
		\begin{align} 
			\mathcal{O} (\max \lbrace \frac{\sqrt{\delta}(1+\delta)\kappa_g \kappa_f^2}{np_{\min}} ,(1+\delta)\kappa_g, \frac{(1+\delta)\kappa_f^2}{np_{\min}}   ,\frac{2}{p} \rbrace \log ( \frac{\tilde{\Phi}_0}{\epsilon} ) ) . \nonumber 
		\end{align}
	}
	where ${\tilde{\Phi}_0}$ denotes the distance of the initial values $x_0, y_0, D^x_0, D^y_0, H^x_0, H^y_0$ from their respective limit points (described in equation \eqref{tilde_Phi_t} in \cite{cdctechnicalreport}).
\end{theorem}
\begin{proof}
	This proof is based on several intermediate results proved in Appendices \ref{appendix_A}-\ref{appendix_C}. Hence it would be useful to refer to those results in order to appreciate the proof of  Theorem \ref{thm:svrg_complexity}. 
	
	Observe that 
	\begin{align}
		E \left\Vert \bx_{t+1} - \mathbf{1}x^\star \right\Vert^2 + E\left\Vert \by_{t+1} - \mathbf{1}y^\star \right\Vert^2 
		& \leq  \frac{1}{ \min \{M_{x},M_{y}\}}  \left( M_{x}E\left\Vert \bx_{t+1} - \mathbf{1}x^\star \right\Vert^2 + M_{y} E\left\Vert  \by_{t+1} - \mathbf{1}y^\star \right\Vert^2 \right) \nonumber\\
		& \leq \frac{1}{ \min \{M_{x},M_{y}\}} E\left[ \tilde{\Phi}_{t+1} \right] \\
		& \leq \frac{1}{ M} \rho^{t+1} \tilde{\Phi}_{0} ,
	\end{align}
	where $M := \min \{ M_x, M_y \}$ and last inequality follows from Lemma \ref{lem:exp_phi_tilde_rec} with $T_0 = 0$. Hence, 
	\begin{align}
		E\left\Vert \bx_{T(\epsilon)} - \mathbf{1}x^\star  \right\Vert^2 + E\left\Vert \by_{T(\epsilon)} - \mathbf{1}y^\star  \right\Vert^2 & \leq \epsilon ,
	\end{align}
	for $T(\epsilon) = \frac{1}{-\log \rho} \log \left( \frac{\tilde{\Phi}_0}{M\epsilon} \right)$.
	
	\textbf{Gradient Computation Complexity:}
	
	Recall 
	\begin{align}
		\rho & = \max \left\lbrace 1 - \frac{3b_x}{7}, 1 - \frac{3b_y}{7}, 1- \frac{\gamma_x}{2}\lambda_{m-1}(I-W) , 1- \frac{\gamma_y}{2}\lambda_{m-1}(I-W) , 1-\alpha_{x} , 1-\alpha_{y} , 1-\frac{p}{2} \right\rbrace .
	\end{align}
	Using Lemma \ref{lem:parameters_feas_svrg}, $\rho$ can be upper bounded as
	\begin{align}
		\rho  & \leq  \max \Bigg \{ 1 - \frac{3b_x}{7}, 1 - \frac{3b_y}{7} , 1 - \frac{b_x}{8\sqrt{\delta}(1+\delta)\kappa_g} , {\red{1 - \frac{1}{8(1+\delta)\kappa_g}}} ,1 - \frac{b_y}{8\sqrt{\delta}(1+\delta)\kappa_g} ,   \\ & \qquad \qquad {\red{1 - \frac{1}{8(1+\delta)\kappa_g}}} , 1-\frac{b_x}{1+\delta} , 1-\frac{b_y}{1+\delta},  1-\frac{p}{2} \Bigg \} .
	\end{align}
	Using \eqref{eq:bx_lb} and \eqref{eq:by_lb}, we have
	\begin{align}
		1 - \frac{3b_x}{7} \leq 1- \frac{3}{7} \frac{np_{\min}}{144\kappa_f^2} = 1 - \frac{np_{\min}}{336\kappa_f^2} , \
		1 - \frac{3b_y}{7} \leq 1- \frac{3}{7} \frac{np_{\min}}{144\kappa_f^2} = 1 - \frac{np_{\min}}{336\kappa_f^2} \nonumber\\
		1-\frac{b_x}{1+\delta} \leq 1-\frac{np_{\min}}{144(1+\delta)\kappa_f^2} , \ 1-\frac{b_y}{1+\delta} \leq 1-\frac{np_{\min}}{144(1+\delta)\kappa_f^2} \\
		1 - \frac{b_x}{8\sqrt{\delta}(1+\delta)\kappa_g} \leq 1 - \frac{np_{\min}}{1152\sqrt{\delta}(1+\delta)\kappa_g \kappa_f^2} , \ 1 - \frac{b_y}{8\sqrt{\delta}(1+\delta)\kappa_g} \leq 1 - \frac{np_{\min}}{1152\sqrt{\delta}(1+\delta)\kappa_g \kappa_f^2} .
	\end{align}
	Therefore,
	\begin{align}
		& \rho  \leq \max \left\lbrace 1 - \frac{np_{\min}}{336\kappa_f^2} , 1 - \frac{np_{\min}}{1152\sqrt{\delta}(1+\delta)\kappa_g \kappa_f^2} , 1 - \frac{1}{8(1+\delta)\kappa_g} , 1-\frac{np_{\min}}{144(1+\delta)\kappa_f^2} , 1-\frac{p}{2}   \right\rbrace  \\
		& = 1 - \min \left\lbrace \frac{np_{\min}}{336\kappa_f^2} , \frac{np_{\min}}{1152\sqrt{\delta}(1+\delta)\kappa_g \kappa_f^2} , \frac{1}{8(1+\delta)\kappa_g}, \frac{np_{\min}}{144(1+\delta)\kappa_f^2} , \frac{p}{2}   \right\rbrace \\
		& =: 1 - \tilde{C} .
	\end{align}
	By taking log on both sides, we obtain
	\begin{align}
		\log \rho & \leq \log(1 - \tilde{C}) \nonumber\\
		-\log \rho & \geq -\log(1 - \tilde{C}) \nonumber\\
		\frac{1}{-\log \rho} & \leq \frac{1}{-\log(1 - \tilde{C})} \nonumber\\
		&  \leq \frac{5}{\tilde{C}} \nonumber \\
		& = 5 \left( \min \left\lbrace \frac{np_{\min}}{336\kappa_f^2} , \frac{np_{\min}}{1152\sqrt{\delta}(1+\delta)\kappa_g \kappa_f^2} ,\frac{1}{8(1+\delta)\kappa_g}, \frac{np_{\min}}{144(1+\delta)\kappa_f^2} , \frac{p}{2}   \right\rbrace \right)^{-1}  \nonumber\\
		& = 5\max \left\lbrace \frac{336\kappa_f^2}{np_{\min}} , \frac{1152\sqrt{\delta}(1+\delta)\kappa_g \kappa_f^2}{np_{\min}} ,8(1+\delta)\kappa_g, \frac{144(1+\delta)\kappa_f^2}{np_{\min}} , \frac{2}{p}   \right\rbrace , \label{eq:log_rho_rec_ub}
	\end{align}
	where the fourth inequality uses the fact that $(1/-\log(1-x)) \leq 5/x$ for all $0 < x < 1 $. Using Lemma \ref{lem:parameters_feas_svrg}, we have $M_x \geq 1 - \frac{4b_x}{7}$. Therefore, $M_x > 1 - \frac{4}{7} = \frac{3}{7}$ because $0 < b_x < 1$. Moreover, $M_y > \frac{3}{7}$ as $0 < b_y < 1$. Therefore, $\log \left( \frac{\tilde{\Phi}_0}{M\epsilon} \right) \leq \log \left( \frac{7\tilde{\Phi}_0}{3\epsilon} \right)$ .
	Hence,
	\begin{align} 
		T(\epsilon) & = \mathcal{O} \left(\max \Bigg \lbrace \frac{\kappa_f^2}{np_{\min}} , \frac{\sqrt{\delta}(1+\delta)\kappa_g \kappa_f^2}{np_{\min}} ,(1+\delta)\kappa_g, \frac{(1+\delta)\kappa_f^2}{np_{\min}} , \frac{2}{p} \Bigg \rbrace {\red{\log \left( \frac{\tilde{\Phi}_0}{\epsilon} \right)}} \right).
	\end{align}
\end{proof}

\section{Discussion on the analysis techniques}

In this section, we discuss and compare the analysis techniques of our work with those in  existing works. In \cite{li2021decentralized} a convex composite minimization problem is studied and inexact PDHG method is applied to its saddle point formulation. In this work, we study a different problem \eqref{eq:main_opt_consenus_constraint} where a smooth function depends jointly on primal and dual variables. We prove that it is equivalent to study unconstrained saddle point problem \eqref{minmax_lagrange_problem} to get the solution of \eqref{eq:main_opt_consenus_constraint}. However, \cite{li2021decentralized} uses a well known equivalence between a convex minimization problem and its Lagrangian formulation \cite{lan2020communication}. We define additional quantities $D_y^\star, H_y^\star$ and Bregman distance functions $V_{f_i,y}(x_1,x_2), V_{-f_i,x}(y_1,y_2)$ in Appendix \ref{appendix_A} to obtain appropriate bounds. 
%

\textbf{ C-DPSSG analysis:} Using smoothness, strong convexity strong concavity assumptions, and definitions of $V_{f_i,y}(x_1,x_2)$ and $V_{-f_i,x}(y_1,y_2)$, we upper bound $E \left\Vert \bx_{t} - \mathbf{1}x^\star - s\mathcal{G}^\bx_{t} + s\nabla_x F(\mathbf{1}x^\star,\mathbf{1}y^\star)  \right\Vert^2 + E \left\Vert \by_{t} - \mathbf{1}y^\star + s\mathcal{G}^\by_{t} - s\nabla_y F(\mathbf{1}x^\star,\mathbf{1}y^\star)  \right\Vert^2$ in terms of $\bx_t, \by_t, \tilde{\bx}_t, \tilde{\by}_t, V_{f_i,y}(x_1,x_2)$ and $V_{-f_i,x}(y_1,y_2)$ in Lemma \ref{lem:exp_xkt_ykt_svrg}. Note that the upper bound in Lemma \ref{lem:exp_xkt_ykt_svrg} is complicated and different from that of \cite{li2021decentralized} because we have additional terms contributed by dual variable $y$ with different coefficients and terms containing square norms dependent on the reference points. This intermediate result generates different bounds and sets of parameters in the subsequent analysis. We carefully set the step size and choose algorithm parameters with proven feasibility in Lemma \ref{lem:parameters_feas_svrg}. We rigorously compute lower and upper bounds on chosen parameters in terms of $\kappa_f, \kappa_g$ and $\delta$ in Lemma \ref{lem:parameters_feas_svrg} and Appendix \ref{appendix_svrg_main_result}. In our work, these derivations are more involved in comparison to \cite{li2021decentralized}. Similar observations hold also for analysis of Algorithm \ref{alg:IPDHG_with_sgd_svrg_oracle} with GSGO provided in Appeendix \ref{appendix_C}.

Analysis methods of \cite{xianetal2021fasterdecentnoncvxsp} and \cite{liu2020decentralized} are based on averaging quantities; for example average of iterates and gradients. The analysis methods in \cite{xianetal2021fasterdecentnoncvxsp} and \cite{liu2020decentralized} require separate bounds for consensus error and gradient estimation errors and depend in addition on the smoothness of saddle point problem. In contrast to \cite{xianetal2021fasterdecentnoncvxsp} and \cite{liu2020decentralized}, our analysis does not demand any separate bound on consensus error and gradient estimation error and handles non-smooth functions as well. Unlike our compression based communication scheme, the analysis in \cite{beznosikovetal2020distsaddle} bounds errors using an accelerated gossip scheme and approximate solution obtained by solving an inner saddle point problem at every node. 

\section{Numerical Experiments on Robust Logistic Regression}
\label{appendix_experiments}

We evaluate the effectiveness of proposed algorithms on robust logistic regression problem 
\begin{align}
\min_{ x \in \mathcal{X}} \max_{y \in \mathcal{Y}} \Psi(x,y) = \frac{1}{N} \sum_{i = 1}^N \log\left( 1+ exp\left( -b_ix^\top(a_i + y)\right) \right) + \frac{\lambda}{2} \left\Vert x \right\Vert^2_2 -\frac{\beta}{2} \left\Vert y \right\Vert^2_2  \label{appendix_robust_logistic_regression}
\end{align}
over a binary classification data set $\mathcal{D} = \{(a_i, b_i) \}_{i = 1}^N$.  We consider constraint sets $\mathcal{X}$ and $\mathcal{Y}$ as $\ell_2$ ball of radius $100$ and $1$ respectively. We compute smoothness parameters $L_{xx}, L_{yy}, L_{xy}$ and $L_{yx}$ using Hessian information of the objective function \red{(see Appendix \ref{lipschitz_constant_estimation})} and set strong convexity and strong concavity parameters to $\lambda$ and $\beta$ respectively. Unless stated otherwise, we set $\lambda = \beta = 10$, number of nodes to $m = 20$ and number of batches to $n = 20$ in all our experiments. \red{The initial points $x_0, y_0$ are generated randomly and $D^\bx, D^\by$ are set to $\mathbf{0}$. We set up the step size of proposed methods and baseline methods using the theoretical values provided in the respective papers. } We implement all the experiments in Python programming language on a linux machine with 2.10 GHz Intel\textsuperscript{\textregistered} Xeon\textsuperscript{\textregistered} processor and 32 shared CPUs. 

\subsection{Data Sets}
\vspace{-0.01in}
 We rely on four binary classification datasets namely, a4a, phishing and ijcnn1 from \url{https://www.csie.ntu.edu.tw/~cjlin/libsvmtools/datasets/} and sido data from \url{http://www.causality.inf.ethz.ch/data/SIDO.html}. The characteristics of these datasets are reported in Table \ref{table:data_sets}. 
\begin{table}[h]
\caption{Data Sets used for experiments. $N$ and $d$ denote respectively the number of samples and number of features.} \label{table:data_sets}
\begin{center}
\begin{tabular}{lll}
\toprule
\textbf{Data set}  & $N$ & $d$  \\
\midrule 
a4a    & 4781 & 122\\ 
\midrule 
phishing & 11,055 & 68 \\  
\midrule 
ijcnn1 & 49,990 & 22 \\
\midrule 
sido & 2536 & 4932 \\
\bottomrule
\end{tabular}
\end{center}
\end{table}
\subsection{Baseline methods}
\vspace{-0.01in}
We compare the performance of proposed algorithms C-DPSVRG and C-DPSSG with three non-compression based baseline algorithms: (1) Distributed Min-Max Data similarity \cite{beznosikovetal2020distsaddle} (2) Decentralized Parallel Optimistic Stochastic Gradient (DPOSG) \cite{liu2020decentralized}  and, (3) Decentralized Minimax Hybrid Stochastic Gradient Descent (DM-HSGD) \cite{xianetal2021fasterdecentnoncvxsp}. 

\textbf{Distributed Min-Max data similarity:} This algorithm is based on accelerated gossip scheme employed on model updates and gradient vectors  \cite{beznosikovetal2020distsaddle}. This method requires approximate solution of an  inner saddle point problem at every iterate. We run extragradient method \cite{korpelevich1976extragradient} to solve the inner saddle point problem with a desired accuracy provided in \cite{beznosikovetal2020distsaddle}. We compute the number of iterates in accelerated gossip scheme and the step size using theoretical details provided in \cite{beznosikovetal2020distsaddle}. Throughout this section, we use the shorthand notation for Distributed Min-Max data similarity algorithm as Min-Max similarity. \newline
\textbf{Decentralized Parallel Optimistic Stochastic Gradient (DPOSG):} DPOSG \cite{liu2020decentralized} is a two step algorithm with local model averaging designed for solving unconstrained saddle point problems in a decentralized fashion. We include the projection steps to update both sequences of DPOSG as we are solving constrained problem \eqref{appendix_robust_logistic_regression}. The step size and the number of local model averaging steps are tuned according to Theorem 1 in \cite{liu2020decentralized}. 

\textbf{Decentralized Minimax Hybrid Stochastic Gradient Descent (DM-HSGD):} DM-HSGD \cite{xianetal2021fasterdecentnoncvxsp} is a gradient tracking based algorithm designed for solving saddle point problems with a constraint set on dual variable. To take care of the constraints on primal variable, we adapt DM-HSGD by incorporating projection step to the model update of primal variables as well. We use grid search to find the best step sizes for primal and dual variable updates. Other parameters like initial large batch size and parameters involved in gradient tracking update sequence are chosen according to the experimental settings in \cite{xianetal2021fasterdecentnoncvxsp}.
\vspace{-0.04in}
\subsection{Benchmark Quantities}
\vspace{-0.01in}
We run the centralized and uncompressed version of C-DPSVRG for $50,000$ iterations to find saddle point solution $z^\star = (x^\star, y^\star)$. The performance of all the methods is measured using $\frac{1}{m}\sum_{i = 1}^m \left\Vert z^i_t - z^\star \right\Vert^2$.  

\textbf{Number of gradient computations and communications:} We calculate the total number of gradient computations according to the number of samples used in the gradient computation at a given iterate $t$.  The number of communications per iterate are computed as the number of times a node exchanges information with its neighbors. 

\textbf{Number of bits transmitted:} We set number of bits $b = 4$ in compression operator $Q_\infty(x)$ for C-DPSVRG and C-DPSSG. Similar to \cite{koloskova2019decentralized}, we assume that on an average $5$ bits (1 bit for sign and 4 bits for quantization level) are transmitted at every iterate for C-DPSVRG and C-DPSSG. We assume that on an average 32 bits are transmitted per communication for DPOSG, DM-HSGD and Min-Max similarity algorithm. 
\vspace{-0.04in}
\subsection{Observations} 
\vspace{-0.01in}

\textbf{Compression effect:} Plots in Figure \ref{fig:logistic_regression_distance_from_saddle_torus_appendix} depict that C-DPSVRG and C-DPSSG transmits less number of bits than other baseline methods.  DPOSG and Min-Max similarity are performing poorly against bits transmission because both the schemes involve multiple rounds of communications. 

\textbf{Compression error:}
We plot compression error $\left\Vert Q(\nu^x) - \nu^x \right\Vert^2 + \left\Vert Q(\nu^y) - \nu^y \right\Vert^2$ against number of transmitted bits for C-DPSVRG as shown in Figure \ref{fig:compression_error_num_bits}. We observe that C-DPSVRG with $O(\log d)$ bits achieves compression error $10^{-25}$ in less than 20,000 transmitted bits. It shows a clear advantage of using $O(\log d)$ bits in C-DPSVRG while maintaining low compression error. 

\textbf{Number of bits transmitted:}
As demonstrated in Figure \ref{fig:comparison_bits_dsvrg}, C-DPSVRG transmits less number of bits to achieve highly accurate solution when $b = 1+\log_2 \sqrt{d}$. We can observe that the convergence behavior of C-DPSVRG is affected by setting number of bits less than $1+\log_2 \sqrt{d}$. For example, the convergence of C-DPSVRG becomes slow  for sido data with $b = 2, 4 < 1+\log_2 \sqrt{d} \approx 7 $ as shown in Figure \ref{fig:comparison_bits_dsvrg}. It shows that $Q_\infty(x)$ provides better performance for $b  =\mathcal{O}(\log_2 d)$ especially for high-dimensional data points. 

\textbf{Communication efficiency:} The one-time communication at every iterate in C-DPSVRG speeds up communication and makes C-DPSVRG to be faster than Min-Max similarity and DPOSG methods as shown in Figure \ref{fig:logistic_regression_distance_from_saddle_torus_appendix}.

\textbf{Impact of topology:} Figure \ref{fig:logistic_regression_distance_from_saddle_torus_appendix} and Figure \ref{fig:logistic_regression_distance_from_saddle_ring_appendix} respectively demonstrate the performance of C-DPSVRG and C-DPSSG on 2d torus and ring topology. 
We observe that C-DPSSG is faster than C-DPSVRG in terms of gradient computations but is competitive asymptotically against number of communications and bits transmitted for ring topology as depicted in Figure \ref{fig:logistic_regression_distance_from_saddle_ring_appendix}. Similar behavior is observed  in 2D torus in Figure \ref{fig:logistic_regression_distance_from_saddle_torus_appendix}. Hence for sparse topology, C-DPSSG and C-DPSVRG have similar communication cost but the former has low computation cost and hence is appealing for sparse topology as well.

\textbf{Impact of number of nodes:} 
As the number of nodes increases, C-DPSSG requires fewer gradient computations to achieve similar solution accuracy. However, the number of communications increases for large nodes in C-DPSSG because error due to variance becomes high in this situation. C-DPSVRG requires small number of gradient computations for large number of nodes because it assigns smaller batch size to every node. As number of nodes increases, Figure \ref{fig:dsvrg_behavior_nodes} demonstrates that C-DPSVRG performance does not get affected too much in terms of communications and bits transmitted. C-DPSVRG achieves fast convergence eventually in terms of gradient computations with a smaller number of nodes on ring topology, as demonstrated in Figure \ref{fig:dsvrg_behavior_nodes_ring}. The sparsity level of ring topology is higher than that of 2D torus and increases with number of nodes. In contrast to the performance of C-DPSVRG in terms of  communications in 2D torus (Figure \ref{fig:dsvrg_behavior_nodes}), C-DPSVRG requires more communications for large number of nodes in a ring topology, as shown in Figure \ref{fig:dsvrg_behavior_nodes_ring}. Figure \ref{fig:different_nodes} shows that the asymptotic behavior of C-DPSSG is affected by increasing number of nodes whereas C-DPSVRG is robust to the increase in number of nodes. However by choosing an appropriate $\epsilon$ in C-DPSSG, it would be possible to improve its behavior on topologies with large number of nodes. 


\textbf{Different choices of reference probabilities:}
The full batch gradient computations in C-DPSVRG depends on the reference probability parameter $p$. Inspired from \cite{kovalev2020don}, we run C-DPSVRG with five different reference probabilities as $1/n, 1/(\kappa_f \ n^3)^{1/4}, 1/(\kappa_f \ n)^{1/2}, 1/(\kappa_f^3 \ n)^{1/4}$ and $1/\kappa_f$. From Figure \ref{fig:dsvrg_behavior_ref_prob}, we observe that setting $p = 1/n$ requires the least number of gradient computations as it corresponds to the less frequent computation of full batch gradients. 



\vspace{-0.05in}
\subsection{Estimating Lipschitz parameters}
\label{lipschitz_constant_estimation}
\vspace{-0.01in}
In this section, we estimate Lipschitz parameters $L_{xx}, L_{yy}, L_{xy}, L_{yy}$ of robust logistic regression problem \eqref{appendix_robust_logistic_regression}. Assume that each node $i$ has $N_i$ number of local samples such that $\sum_{i = 1}^m N_i = N$. Recall objective function $\Psi(x,y)$ in equation \eqref{appendix_robust_logistic_regression}:
\begin{align*}
	\Psi(x,y) &= \frac{1}{N} \sum_{i = 1}^N \log\left( 1+ exp\left( -b_ix^\top(a_i + y)\right) \right) + \frac{\lambda}{2} \left\Vert x \right\Vert^2_2 -\frac{\beta}{2} \left\Vert y \right\Vert^2_2 \\
		& = \frac{1}{N} \sum_{i = 1}^m \sum_{l = 1}^{N_i}  \log\left( 1+ exp\left( -b_{il}x^\top(a_{il} + y)\right) \right) + \frac{\lambda}{2} \left\Vert x \right\Vert^2_2 -\frac{\beta}{2} \left\Vert y \right\Vert^2_2 \\
\end{align*}
\begin{align*}
	& = \sum_{i = 1}^m \left( \frac{1}{N}  \sum_{l = 1}^{N_i}  \log\left( 1+ exp\left( -b_{il}x^\top(a_{il} + y)\right) \right) + \frac{\lambda}{2m} \left\Vert x \right\Vert^2_2 -\frac{\beta}{2m} \left\Vert y \right\Vert^2_2 \right) \\
	& = \sum_{i = 1}^m f_i(x,y),
\end{align*}
where $f_i(x,y) = \frac{1}{N}  \sum_{l = 1}^{N_i}  \log\left( 1+ exp\left( -b_{il}x^\top(a_{il} + y)\right) \right) + \frac{\lambda}{2m} \left\Vert x \right\Vert^2_2 -\frac{\beta}{2m} \left\Vert y \right\Vert^2_2 $.
Gradients of $f_i(x,y)$ with respect to $x$ and $y$ are given by 
\begin{align*}
	\nabla_x f_i(x,y) & = \frac{1}{N} \sum_{l = 1}^{N_i} \frac{-b_{il}(a_{il} + y)}{1+ exp\left( b_{il}x^\top(a_{il} + y\right)} + \frac{\lambda}{m} x \\
	\nabla_y f_i(x,y) & = \frac{1}{N} \sum_{l = 1}^{N_i} \frac{-b_{il}x}{1+ exp\left( b_{il}x^\top(a_{il} + y\right)} - \frac{\beta}{m} y .
\end{align*}
We create $n$ batches $\{N_{i1}, \ldots, N_{in} \}$ of local samples $N_i$ and write $f_i(x,y)$ in the form of $\frac{1}{n} f_{ij}(x,y)$. 
\begin{align*}
	f_i(x,y) &= \frac{1}{N}  \sum_{l = 1}^{N_i}  \log\left( 1+ exp\left( -b_{il}x^\top(a_{il} + y)\right) \right) + \frac{\lambda}{2m} \left\Vert x \right\Vert^2_2 -\frac{\beta}{2m} \left\Vert y \right\Vert^2_2 \\
	& = \frac{1}{N} \sum_{j = 1}^n \sum_{l = 1}^{N_{ij}}  \log\left( 1+ exp\left( -b^j_{il}x^\top(a^j_{il} + y)\right) \right) + \frac{\lambda}{2m} \left\Vert x \right\Vert^2_2 -\frac{\beta}{2m} \left\Vert y \right\Vert^2_2  \\
	& =  \sum_{j = 1}^n \left( \frac{1}{N} \sum_{l = 1}^{N_{ij}}  \log\left( 1+ exp\left( -b^j_{il}x^\top(a^j_{il} + y)\right) \right) + \frac{\lambda}{2mn} \left\Vert x \right\Vert^2_2 -\frac{\beta}{2mn} \left\Vert y \right\Vert^2_2 \right) \\
	& = \frac{1}{n} \sum_{j = 1}^n \left( \frac{n}{N} \sum_{l = 1}^{N_{ij}}  \log\left( 1+ exp\left( -b^j_{il}x^\top(a^j_{il} + y)\right) \right) + \frac{\lambda}{2m} \left\Vert x \right\Vert^2_2 -\frac{\beta}{2m} \left\Vert y \right\Vert^2_2 \right) \\
	& = \frac{1}{n} \sum_{j = 1}^n f_{ij}(x,y) ,
\end{align*}
where $f_{ij}(x,y) = \frac{n}{N} \sum_{l = 1}^{N_{ij}}  \log\left( 1+ exp\left( -b^j_{il}x^\top(a^j_{il} + y)\right) \right) + \frac{\lambda}{2m} \left\Vert x \right\Vert^2_2 -\frac{\beta}{2m} \left\Vert y \right\Vert^2_2$. We are now ready to find required Lipschitz parameters. \newline
\textbf{Computing} $L^{ij}_{xx}$:
\begin{align*}
	\nabla_{xx}^2 f_{ij}(x,y) & = \frac{n}{N} \sum_{l = 1}^{N_{ij}} \frac{(a^j_{il} + y)(a^j_{il} + y)^\top exp\left( b^j_{il}x^\top(a^j_{il} + y\right)}{(1+ exp( b^j_{il}x^\top(a^j_{il} + y))^2} + \frac{\lambda}{m}I \\
	\implies \ \left\Vert \nabla_{xx}^2 f_{ij}(x,y) \right\Vert_2 & \leq \frac{n}{4N} \sum_{l = 1}^{N_{ij}} (2 \Vert a^j_{il} \Vert_2^2 + 2R_y^2) + \frac{\lambda}{m} \\
	& = \frac{n}{2N} \sum_{l = 1}^{N_{ij}} \Vert a^j_{il} \Vert_2^2 + \frac{nN_{ij}R^2_y}{2N} + \frac{\lambda}{m} =: L^{ij}_{xx} .
\end{align*}
\textbf{Computing} $L^{ij}_{yy}$:
\begin{align*}
	\nabla_{yy}^2 f_{ij}(x,y) & = \frac{n}{N} \sum_{l = 1}^{N_{ij}} \frac{exp\left( b^j_{il}x^\top(a^j_{il} + y\right)(b^j_{il})^2 xx^\top}{\left(1+exp\left( b^j_{il}x^\top(a^j_{il} + y\right) \right)^2} - \frac{\beta}{m}I \\
	\implies \ \left\Vert \nabla_{xx}^2 f_{ij}(x,y) \right\Vert_2 & \leq \frac{n}{N} \sum_{l = 1}^{N_{ij}} \frac{\left\Vert xx^\top \right\Vert_2}{4} + \frac{\beta}{m} \\
	& \leq \frac{nN_{ij}R_x^2}{4N} + \frac{\beta}{m} =: L^{ij}_{yy} .
\end{align*}
\textbf{Computing} $L^{ij}_{xy}$:
\begin{align*}
	\nabla_y (\nabla_x f_{ij}(x,y)) & = \frac{n}{N} \sum_{l = 1}^{N_{ij}} \left( \frac{-b^j_{il}I}{1 + exp\left( b^j_{il}x^\top(a^j_{il} + y)\right)} + (b^j_{il})^2(a^j_{il} + y)x^\top \frac{exp\left( b^j_{il}x^\top(a^j_{il} + y\right)}{(1+ exp( b^j_{il}x^\top(a^j_{il} + y))^2} \right) \\
	\text{Hence we have }& \\  
	\left\Vert \nabla_{xy}^2 f_{ij}(x,y) \right\Vert_2  & \leq \frac{n}{N} \sum_{l = 1}^{N_{ij}} \left( 1 + \frac{1}{4} \left\Vert (a^j_{il} + y)x^\top \right\Vert_2 \right) \\
	& \leq \frac{n}{N} \sum_{l = 1}^{N_{ij}} \left( 1 + \frac{R_x}{4} \Vert (a^j_{il} + y) \Vert_2 \right) \\
	& \leq \frac{n}{N} \sum_{l = 1}^{N_{ij}} \left( 1 + \frac{R_x}{4}(\Vert a^j_{il} \Vert_2 + R_y) \right) \\
	& = \frac{n}{N} \left( \left(1+ \frac{R_xR_y}{4} \right)N_{ij} + \frac{R_x}{4}  \sum_{l = 1}^{N_{ij}} \Vert a^j_{il} \Vert_2 \right) =: L^{ij}_{xy} . 
\end{align*}
We set $L_{xx} = \max_{i,j} \{ L^{ij}_{xx} \}, \ L_{yy} = \max_{i,j} \{L^{ij}_{yy}\}$ and $L_{xy} = L_{yx} = \max_{i,j} \{L^{ij}_{xy} \}$. The strong convexity and strong concavity parameters are respectively set to $\lambda$ and $\beta$.
\begin{figure*}[!htbp]
	\begin{minipage}{.99\textwidth}
		\centering
		\includegraphics[width=1\linewidth]{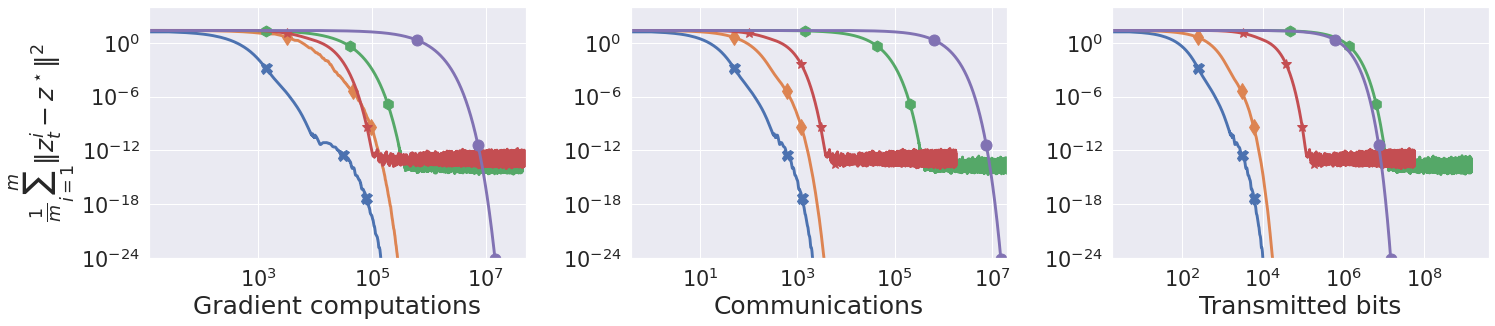}  
	\end{minipage}
	\begin{minipage}{.99\textwidth}
		\centering
		\includegraphics[width=1\linewidth]{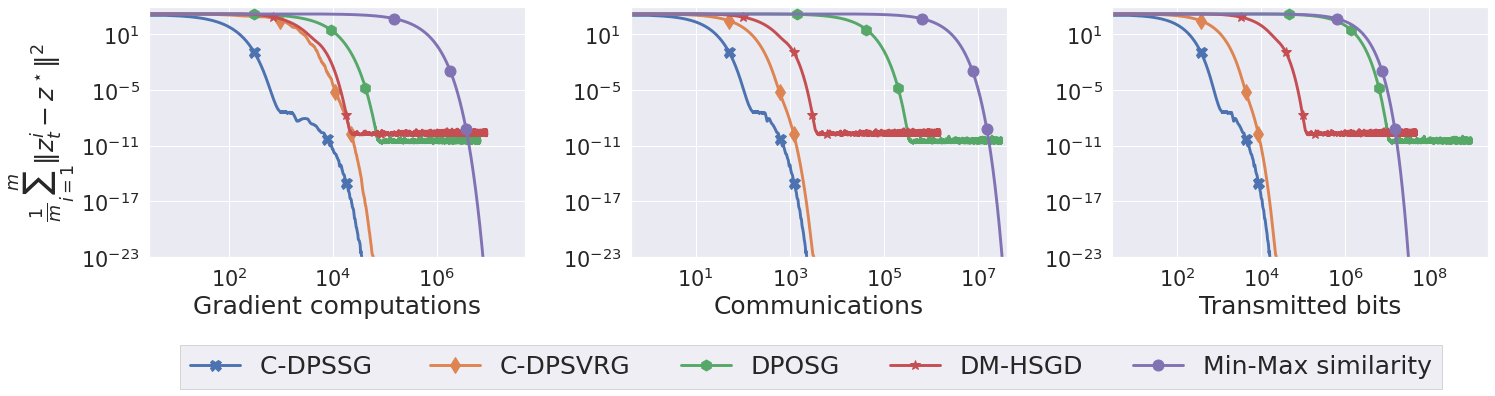}  
	\end{minipage}
	\caption{Convergence behavior of iterates to saddle point vs. Gradient computations (Column 1), Communications (Column 2), Number of bits transmitted (Column 3) for different algorithms in 2d torus topology of 20 nodes. phishing, sido are in Rows 1,2 respectively.} \label{fig:logistic_regression_distance_from_saddle_torus_appendix}
\end{figure*}

\begin{figure*}[!htbp]
	\begin{minipage}{.99\textwidth}
		\centering
		\includegraphics[width=1\linewidth]{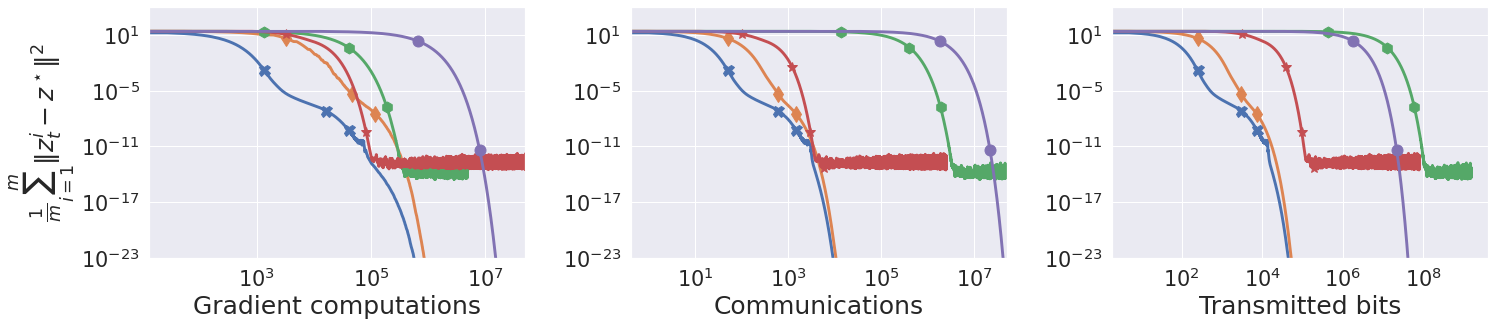}  
	\end{minipage}
	\begin{minipage}{.99\textwidth}
		\centering
		\includegraphics[width=1\linewidth]{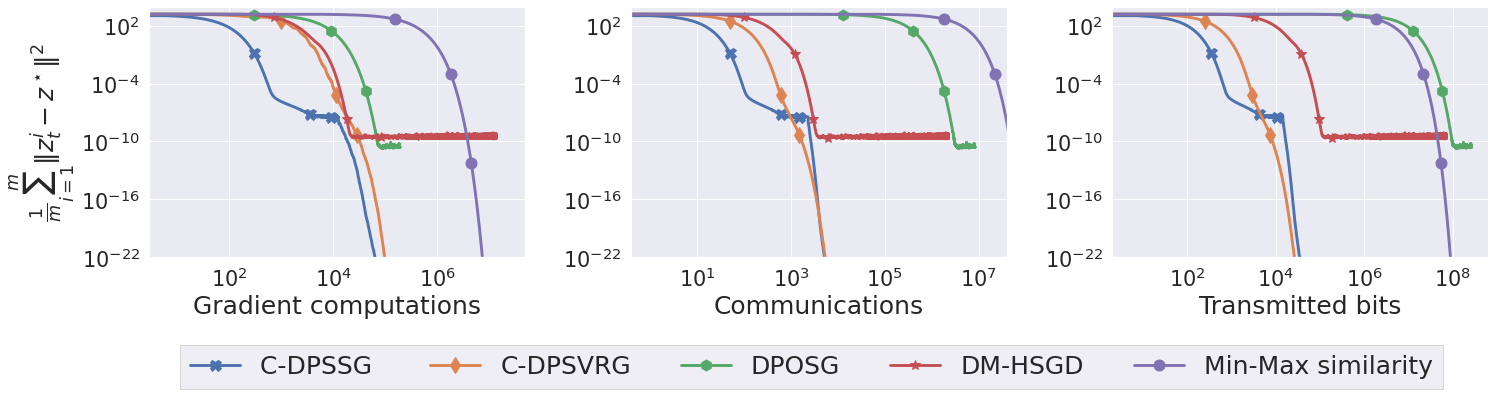}  
	\end{minipage}
	\caption{Convergence behavior of iterates to saddle point vs. Gradient computations (Column 1), Communications (Column 2), Number of bits transmitted (Column 3) for different algorithms in ring topology of 20 nodes. phishing, sido are in Rows 1,2 respectively.} \label{fig:logistic_regression_distance_from_saddle_ring_appendix}
\end{figure*}

\begin{figure}[!h]
\begin{minipage}{.26\textwidth}
  \centering
  \includegraphics[width=1\linewidth]{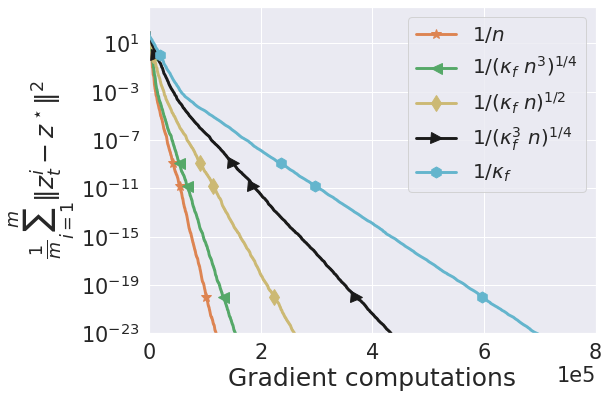}  
\end{minipage}
\begin{minipage}{.24\textwidth}
  \centering
  \includegraphics[width=1\linewidth]{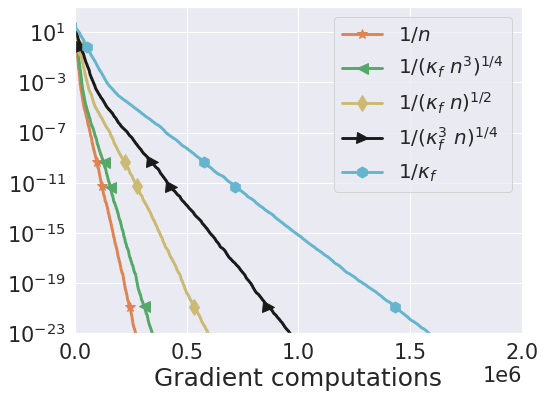}  
\end{minipage}
\begin{minipage}{.24\textwidth}
  \centering
  \includegraphics[width=1\linewidth]{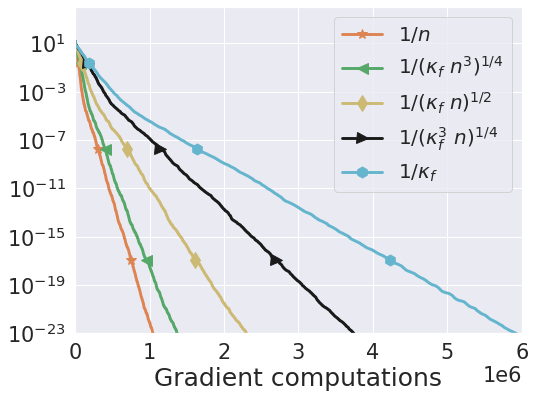}  
\end{minipage}
\begin{minipage}{.24\textwidth}
  \centering
  \includegraphics[width=1\linewidth]{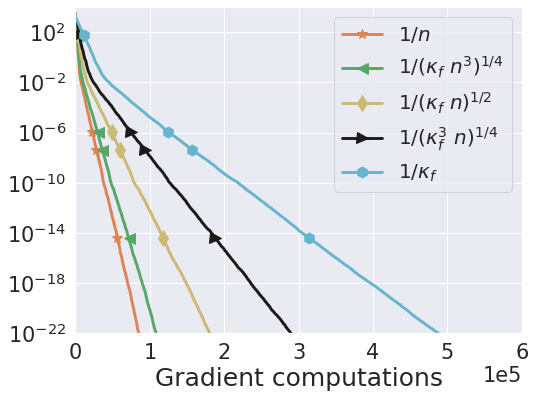}  
\end{minipage}
\caption{Performance of C-DPSVRG with different reference probabilities in 2d torus with 20 nodes. a4a, phishing, ijcnn, sido datasets are in Columns 1,2,3,4 respectively.}
\label{fig:dsvrg_behavior_ref_prob}
\end{figure}
\begin{figure}[!h]
\begin{minipage}{.22\textwidth}
  \centering
  \includegraphics[width=1\linewidth]{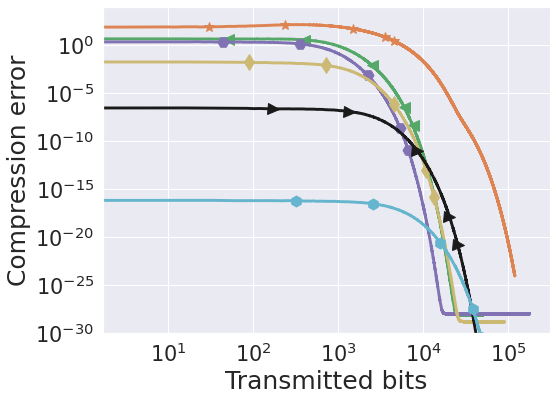}  
\end{minipage}
\begin{minipage}{.21\textwidth}
  \centering
  \includegraphics[width=1\linewidth]{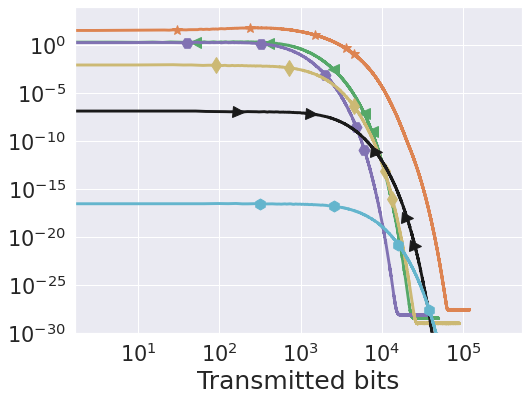}  
\end{minipage}
\begin{minipage}{.21\textwidth}
  \centering
  \includegraphics[width=1\linewidth]{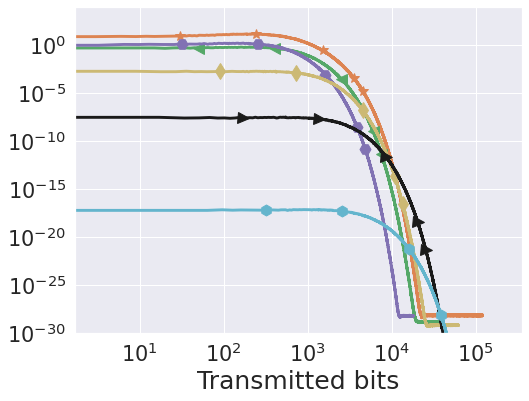}  
\end{minipage}
\begin{minipage}{.32\textwidth}
  \centering
  \includegraphics[width=1\linewidth]{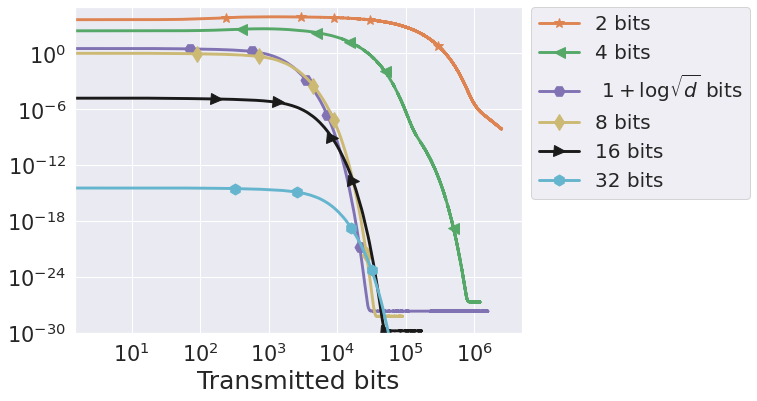}  
\end{minipage}
\caption{Compression error in C-DPSVRG with different number of bits in 2d torus with 20 nodes. a4a, phishing, ijcnn, sido datasets are in Columns 1,2,3,4 respectively.}
\label{fig:compression_error_num_bits}
\end{figure}

\begin{figure}[!hbp]
	\begin{minipage}{.24\textwidth}
		\centering
		\includegraphics[width=1\linewidth]{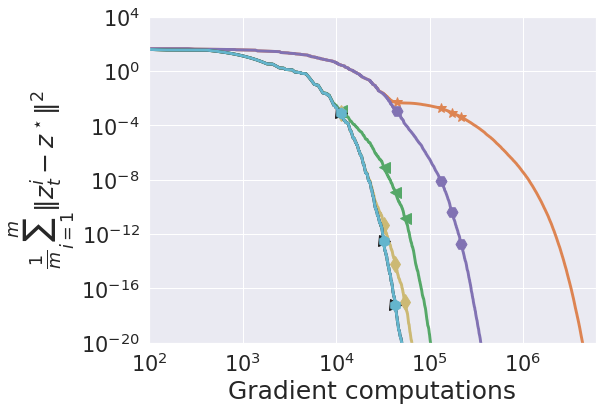}  
	\end{minipage}
	\begin{minipage}{.21\textwidth}
		\centering
		\includegraphics[width=1\linewidth]{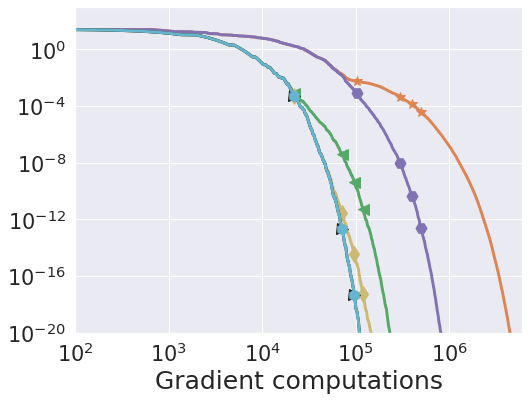}  
	\end{minipage}
	\begin{minipage}{.21\textwidth}
		\centering
		\includegraphics[width=1\linewidth]{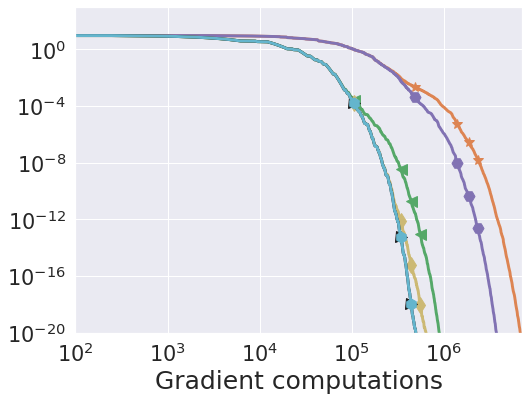}  
	\end{minipage}
	\begin{minipage}{.3\textwidth}
		\centering
		\includegraphics[width=1\linewidth]{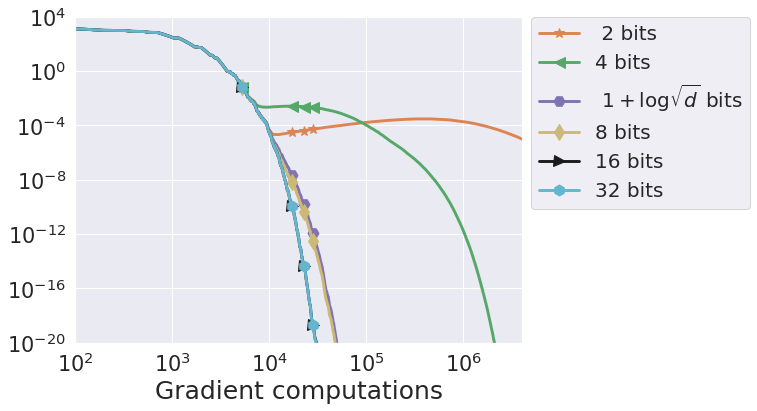}  
	\end{minipage}
	\begin{minipage}{.24\textwidth}
		\centering
		\includegraphics[width=1\linewidth]{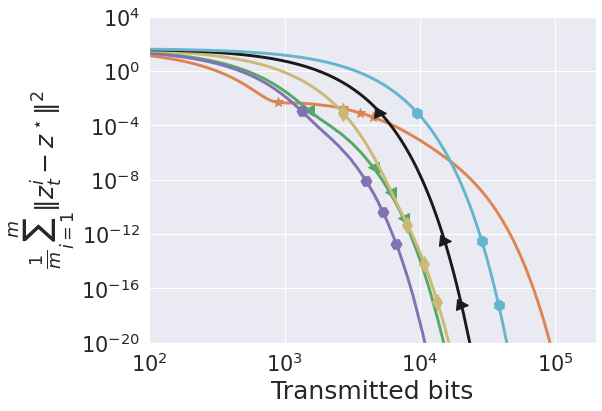}  
	\end{minipage}
	\begin{minipage}{.21\textwidth}
		\centering
		\includegraphics[width=1\linewidth]{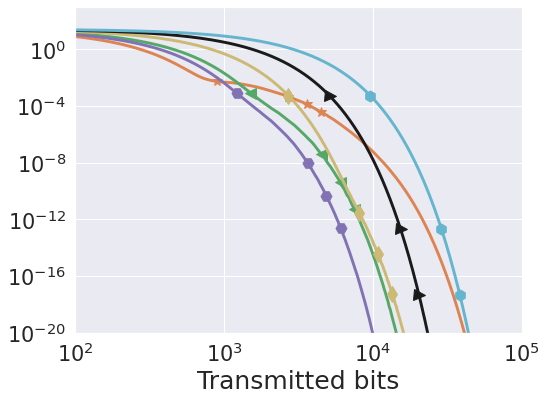}  
	\end{minipage}
	\begin{minipage}{.21\textwidth}
		\centering
		\includegraphics[width=1\linewidth]{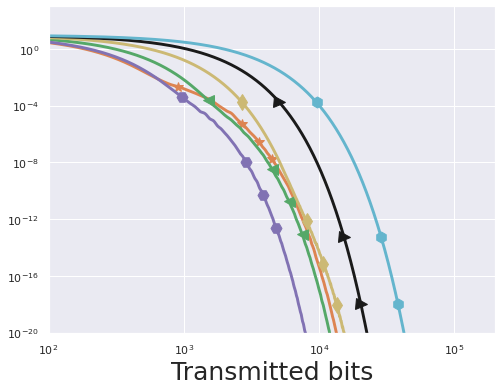}  
	\end{minipage}
	\begin{minipage}{.3\textwidth}
		\centering
		\includegraphics[width=1\linewidth]{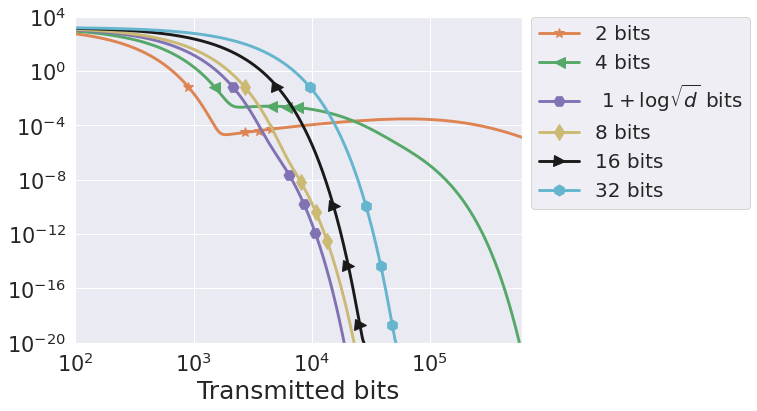}  
	\end{minipage}
	
	\caption{Convergence behavior of iterates to saddle point in C-DPSVRG vs. Gradient computations (Row 1), Number of bits transmitted (Row 2) for C-DPSVRG behavior with \textbf{different number of bits} in 2D torus topology with 20 nodes. a4a, phishing, ijcnn, sido datasets are in Columns 1,2,3,4 respectively. \red{Number of bits $1+\log \sqrt{d}$ for a4a, phishing, ijcnn1 and sido datasets  are $4.465, 4.043, 3.22$ and $7.13$ respectively.}} 
	\label{fig:comparison_bits_dsvrg}
\end{figure}

\begin{figure}[!h]
\begin{minipage}{.99\textwidth}
  \centering
  \includegraphics[width=1\linewidth]{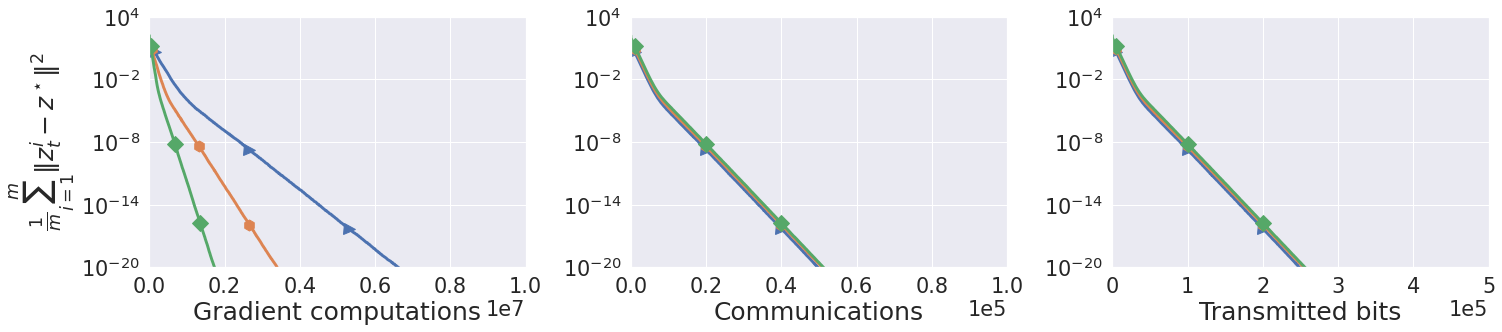}  
\end{minipage}
\begin{minipage}{.99\textwidth}
  \centering
  \includegraphics[width=1\linewidth]{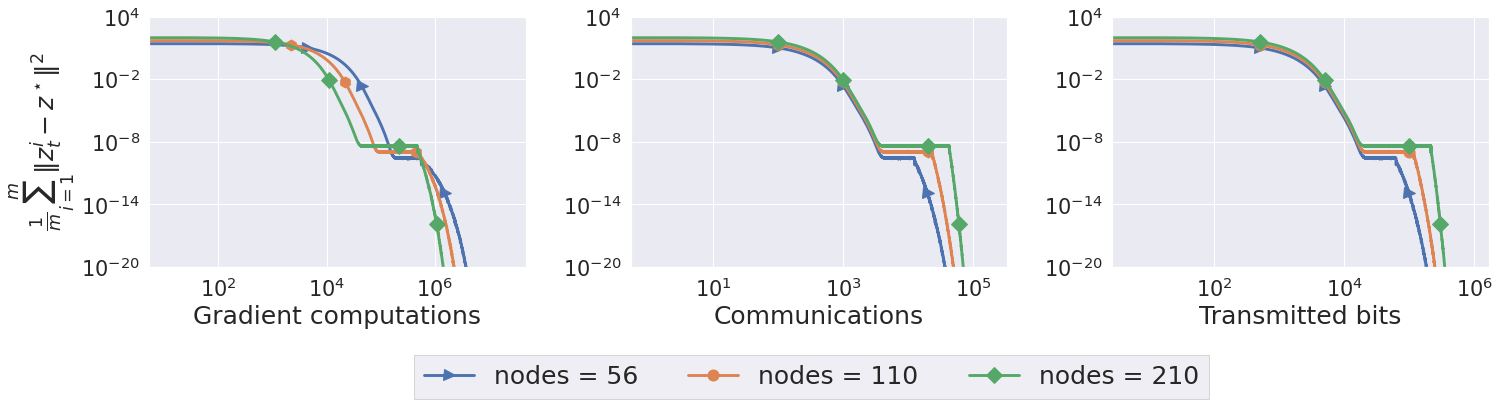}  
\end{minipage}
\caption{Performance of C-DPSVRG and C-DPSSG with different number of nodes on 2d torus topology with ijcnn data. Row 1: C-DPSVRG, Row 2: C-DPSSG}
\label{fig:dsvrg_behavior_nodes}
\end{figure}
%
%

\begin{figure}[!h]
  \centering
  \includegraphics[width=1\linewidth]{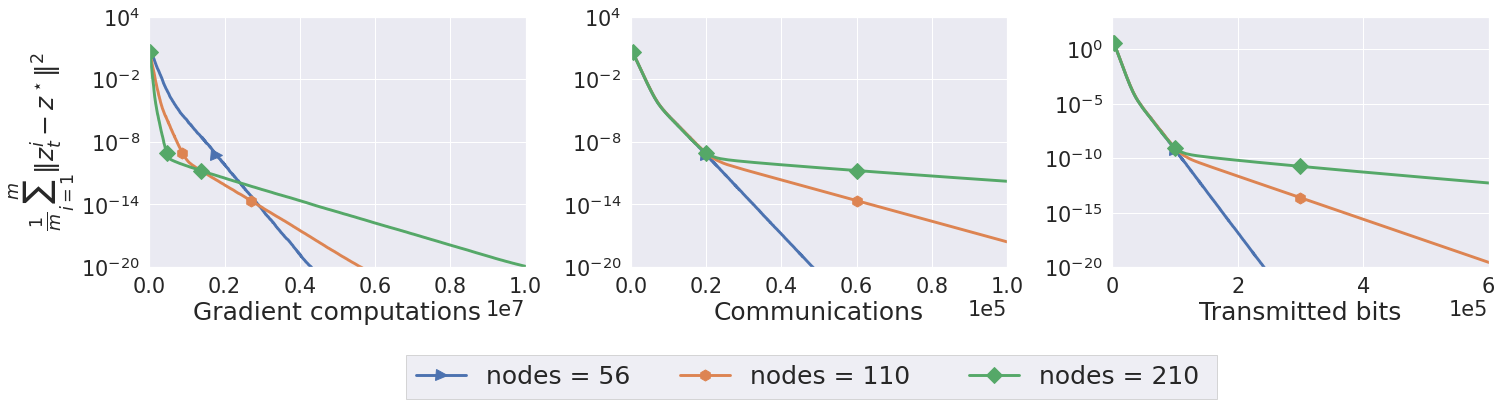}  
\caption{Performance of C-DPSVRG with different number of nodes on ring topology with ijcnn data.}
\label{fig:dsvrg_behavior_nodes_ring}
\end{figure}


\begin{figure}[!h]
	\begin{minipage}{.99\textwidth}
		\centering
		\includegraphics[width=1\linewidth]{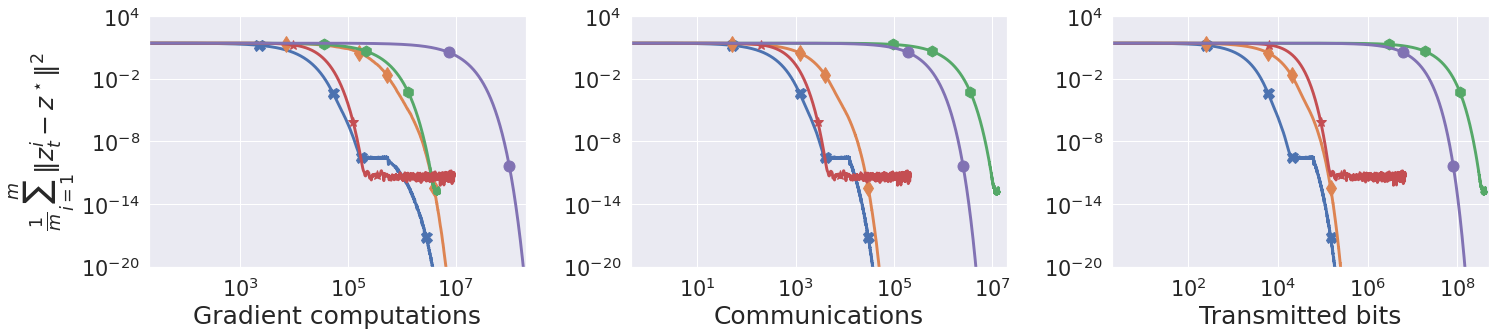}  
	\end{minipage}
	\begin{minipage}{.99\textwidth}
		\centering
		\includegraphics[width=1\linewidth]{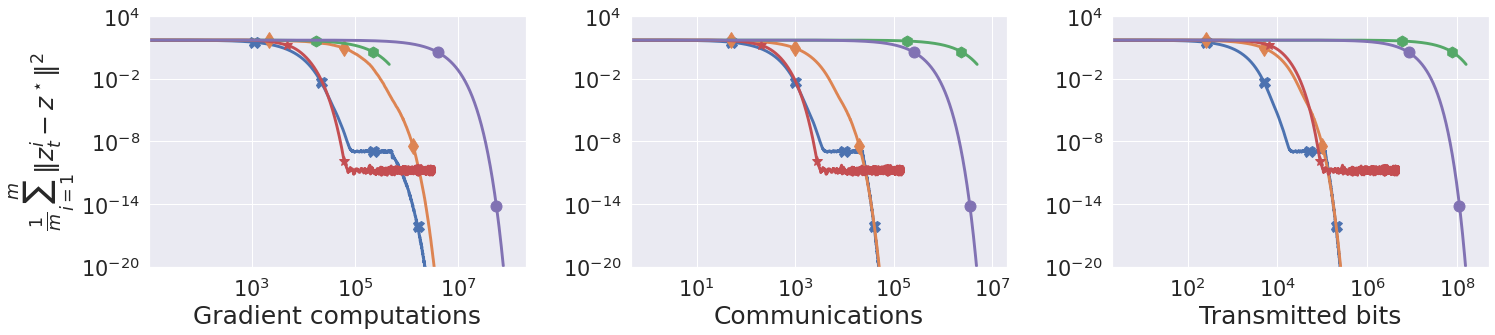}  
	\end{minipage}
	\begin{minipage}{.99\textwidth}
		\centering
		\includegraphics[width=1\linewidth]{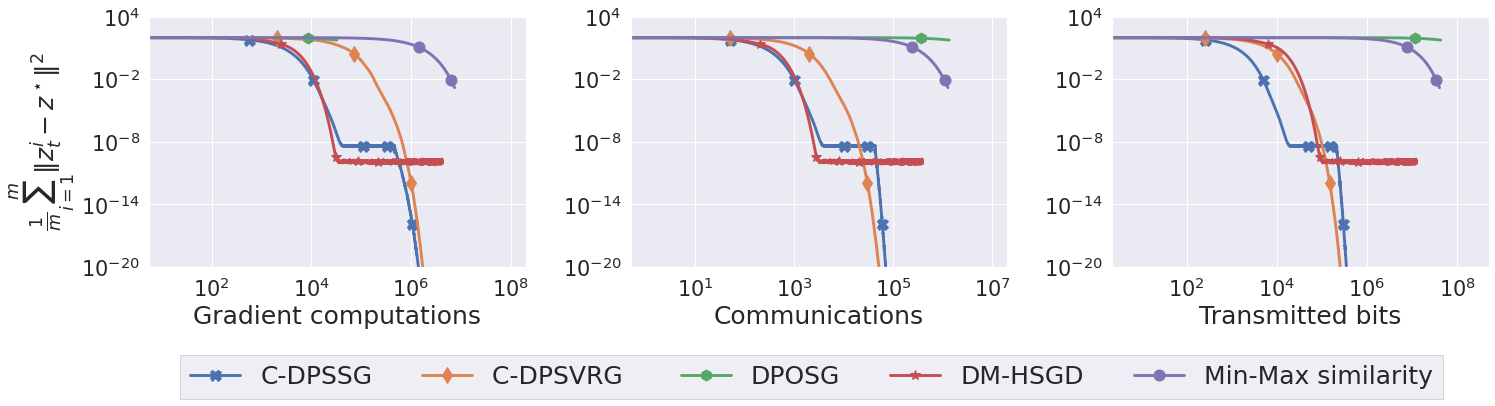}  
	\end{minipage}
\caption{Comparison with baselines with different number of nodes on a 2D torus topology with ijcnn data. Row 1: 56 nodes, Row 2: 110 nodes, Row 3: 210 nodes.}
\label{fig:different_nodes}
\end{figure}

\clearpage
 
\section{Numerical Experiments on AUC maximization} \label{appendix_auc_max}

We evaluate the effectiveness of proposed algorithms on area under receiver operating characteristic curve (AUC) maximization \cite{ying2016stochastic} formulated as:
\begin{align}
	\min_{x,u,v} \max_y \frac{1}{N}\sum_{i = 1}^N F(x,u,v,y;a_i,b_i) + \frac{\lambda}{2} \left\Vert x \right\Vert^2_2 , \label{auc_max_app}
\end{align}
over a binary classification data set $\mathcal{D} = \{(a_i, b_i) \}_{i = 1}^N$ where  $F(x,u,v,y;a_i,b_i) = (1-q)(a_i^\top x - u)^2 \delta_{\left[b_i = 1 \right] } + q(a_i^\top x - v)^2 \delta_{\left[b_i = -1 \right] } - q(1-q)y^2 + 2(1+y)\left( qa_i^\top x \delta_{\left[b_i = -1 \right] } - (1-q)a_i^\top x \delta_{\left[b_i = 1 \right] } \right)$, $q$ denotes the fraction of positive samples, $\delta_{\left[ \cdot \right] }$ is the indicator function. 

\subsection{Parameters Setting} 
We consider a4a and ijcnn1 data sets and set $\lambda = 10^{-5}$ in \eqref{auc_max_app}. We set number of bits $b = 4$ in quantization operator $Q_\infty(x)$. A 2d Torus topology of 20 and 110 nodes is used. For 20 nodes, we consider $\ell_2$ ball of radius 100 and 200 respectively on primal and dual variables. We also incorporate $\ell_2$ ball of large radius $10^8$ and $2 \times 10^8$ respectively for primal and dual variables for large number of nodes (110 nodes). We create 20 and 5 mini-batches respectively for 20 and 110 nodes. Following \cite{xianetal2021fasterdecentnoncvxsp}, step sizes for DM-HSGD computed using theoretical formula are very small in the given parameters settings. We circumvent this issue by finding best step sizes for DM-HSGD using grid search. The step sizes for other methods are set up according to their value proposed in respective papers.

\textbf{Switching Point:} We set threshold value to be $10^{-8}$ and iterations in gossip scheme to be 20 for 2d torus topology of 20 nodes. We observe that C-DPSSG switches to SVRGO after $T_0$ iterations for both a4a and ijcnn1 data sets. For large number of nodes, error term in \eqref{eq:exp_phit_phi_0_ub} is high. Therefore, we use threshold value to be $10^{-6}$ and gossip schme iterations 100 for large number of nodes. We observe that C-DPSSG continues to use GSGO for $T_0$ iterations for a4a data as there is sufficient progress of iterates in the first $T_0^{'}$ iterations. However, for ijcnn1 data, C-DPSSG switches to GSGO after the completion of $T_0^{'}$ iterations.

\subsection{Observations} 
We plot AUC value on training set against number of gradient computations, communications and bits transmitted as depicted in Figure \ref{fig:AUC_value_app}. We observe that there is rapid increase in the AUC value for C-DPSSG which is faster than C-DPSVRG and other existing methods. Figure \ref{fig:auc_distance_from_saddle_app} demonstrates the convergence behavior of iterates $z_t$ towards saddle point solution $z^\star$. In the beginning, iterates of C-DPSSG move faster towards $z^\star$ in comparison to C-DPSVRG and becomes competitive with C-DPSVRG in the long run. These observations suggest that switching scheme is beneficial over purely SVRGO based scheme for obtaining both high AUC value and  better saddle point solution as it saves time and gradient computations in the crucial early stage. We can see that DPOSG is competitive with C-DPSVRG in terms of gradient computations. However, DPOSG needs large number of communications and bits transmission as it involves gossip to reduce the consensus error at every iterate. C-DPSSG and C-DPSVRG converges faster for large number of nodes in comparison to existing methods as demonstrated in last two rows of Figure \ref{fig:AUC_value_app} and Figure \ref{fig:auc_distance_from_saddle_app}. Since DPOSG and Min-Max similarity are based on gossip scheme for every update, the convergence of these algorithms  slows down drastically in terms of communications and bits transmission on a 2d torus topology having 110 nodes as depicted in Figure \ref{fig:auc_distance_from_saddle_app}. 

Table \ref{empirical_epsilon0} reports the true values $\Phi_{0}, \epsilon_0^\star, T_0$ and the approximate values $ \bar{\Phi}^i_0(T'_0), \bar{\epsilon}^i_0, T_0^i$ in AUC maximization. We notice that $\epsilon_0^\star$ and $\bar{\epsilon}^i_0$ are close to each other and hence the difference between $T^i_0$ and $T_0$ is also small. 

\begin{table}[h]
	\begin{center}
		\begin{tabular}{|c|c|c|c|c|c|c|c|}
			\hline
			Data Set & $\Phi_{0}$ & $ \bar{\Phi}^i_0(T'_0)$ & $\epsilon^\star_0 $ & $\bar{\epsilon}^i_0$ & $T_0$ & $T^i_0$ & $T_0^{'}$   \\ \hline
			a4a & $47.4$ & $8.6$ & $5.3 \times 10^{-11}$ & $2.9 \times 10^{-10}$ & $195315$ & $181248$ & $5720$  \\ \hline
			ijcnn1 & $23.8$ & $3.9$ & $ 10^{-10}$ & $6.3 \times 10^{-10}$ & $50947$ & $46960$ & $1536$ \\ \hline
		\end{tabular}
		\caption{Values of $\bar{\Phi}^i_0(T'_0)$, $\bar{\epsilon}^i_0$ $T^i_0$ (observed same for all nodes $i$) and $ \epsilon^\star_0, T_0, T'_0$ obtained from the practical Algorithm \eqref{alg:detecting_switching_point} with 2d torus topology on AUC maximization.} \label{empirical_epsilon0}
	\end{center}
\end{table}

\subsection{Strong Convexity-Strong Concavity and Lipschitz parameters }

We estimate Lipschitz parameters $L_{xx}, L_{yy}, L_{xy}$ and $L_{yy}$, strong convexity concavity parameters $\mu_x$ and $\mu_y$ used in the numerical experiments of AUC maximization problem \eqref{auc_max}. We first write the objective function of \eqref{auc_max_app} in the form of finite sum over number of nodes.
\begin{align}
	\frac{1}{N}\sum_{i = 1}^N F(x,u,v,y;a_i,b_i) + \frac{\lambda}{2} \left\Vert x \right\Vert^2_2 & = \frac{1}{N} \sum_{i = 1}^{m} \sum_{l = 1}^{N_i} F(x,u,v,y;a_{il},b_{il}) + \frac{\lambda}{2} \left\Vert x \right\Vert^2_2 \\
	& = \sum_{i = 1}^{m} \left( \frac{1}{N} \sum_{l = 1}^{N_i} F(x,u,v,y;a_{il},b_{il}) + \frac{\lambda}{2m} \left\Vert x \right\Vert^2_2 \right) \\
	& = \sum_{i = 1}^{m} f_i(x,u,v,y;\{a_{ij},b_{ij}\}_{j= 1}^{N_i}),
\end{align}
where $f_i(x,u,v,y;\{a_{il},b_{il}\}_{l= 1}^{N_i}) = \frac{1}{N} \sum_{l = 1}^{N_i} F(x,u,v,y;a_{il},b_{il}) + \frac{\lambda}{2m} \left\Vert x \right\Vert^2_2$. Next, we create $n$ mini batches of local samples $N_i$ for every node and write $f_i(x,u,v,y;\{a_{il},b_{il}\}_{l = 1}^{N_i}) = \frac{1}{n} \sum_{j = 1}^n f_{ij}(x,u,v,y;\{a^i_{jl},b^i_{jl}\}_{l= 1}^{N_{ij}})$. Now we have
\begin{align*}
	f_i(x,u,v,y;\{a_{il},b_{il}\}_{l= 1}^{N_i}) & = \frac{1}{N} \sum_{l = 1}^{N_i} F(x,u,v,y;a_{il},b_{il}) + \frac{\lambda}{2m} \left\Vert x \right\Vert^2_2 \\
	& = \frac{1}{N} \sum_{j = 1}^{n} \left( \sum_{l = 1}^{N_{ij}} F(x,u,v,y;a^i_{jl},b^i_{jl}) \right)+ \frac{\lambda}{2m} \left\Vert x \right\Vert^2_2 \\
	& = \frac{1}{n} \sum_{j = 1}^{n} \left( \frac{n}{N} \sum_{l = 1}^{N_{ij}} F(x,u,v,y;a^i_{jl},b^i_{jl}) + \frac{\lambda}{2m} \left\Vert x \right\Vert^2_2 \right) \\
	& = \frac{1}{n} \sum_{j = 1}^{n} f_{ij}(x,u,v,y;\{a^i_{jl},b^i_{jl}\}_{l= 1}^{N_{ij}}),
\end{align*}
where $f_{ij}(x,u,v,y;\{a^i_{jl},b^i_{jl}\}_{l= 1}^{N_{ij}}) = \frac{n}{N} \sum_{l = 1}^{N_{ij}} F(x,u,v,y;a^i_{jl},b^i_{jl}) + \frac{\lambda}{2m} \left\Vert x \right\Vert^2_2$. We focus on computing Lipschitz parameters of function $F(x,u,v,y;a^i_{jl},b^i_{jl})$ which are used to set Lipschitz parameters of $f_{ij}$. For simplicity of representation, we denote $F(x,u,v,y;a_i,b_i)$ and its gradient as $F_i$ and $\nabla F_i$ respectively. We now compute gradient and Hessian of function $F_i$.
\begin{align*}
	\nabla_x F_i & = 2(1-q)(a_i^\top x - u) \delta_{\left[b_i = 1 \right] }a_i + 2q(a_i^\top x - v) \delta_{\left[b_i = -1 \right] }a_i  + 2(1+y)\left( qa_i \delta_{\left[b_i = -1 \right] } - (1-q)a_i \delta_{\left[b_i = 1 \right] } \right) + \frac{\lambda}{m}x \\
	& \nabla_u F_i = -2(1-q)(a_i^\top x - u) \delta_{\left[b_i = 1 \right] } \\
	&\nabla_v F_i  = -2q(a_i^\top x - v) \delta_{\left[b_i = -1 \right] } \\
	&\nabla_y F_i  = -2q(1-q)y + 2\left( qa_i^\top x \delta_{\left[b_i = -1 \right] } - (1-q)a_i^\top x \delta_{\left[b_i = 1 \right] } \right) .
\end{align*}
Hessian computations:
\begin{align*}
	\nabla^2_{xx} F_i & = 2(1-q)\delta_{\left[b_i = 1 \right] }a_i a_i^\top + 2q\delta_{\left[b_i = -1 \right] }a_i a_i^\top  + \frac{\lambda}{m} \\
	& \nabla^2_{uu} F_i  = 2(1-q) \delta_{\left[b_i = 1 \right] } \\
	&\nabla^2_{vv} F_i  = 2q \delta_{\left[b_i = -1 \right] }  \\
	& \nabla^2_{yy} F_i  = 2q(1-q) \\
	& \nabla_y(\nabla_x F_i) = 2\left( qa_i \delta_{\left[b_i = -1 \right] } - (1-q)a_i \delta_{\left[b_i = 1 \right] } \right) .
\end{align*}
Let $x^{'} = (x,u,v)$. Then norm of Hessian $\nabla^2_{x^{'}x^{'}}F_i$ is given as:
\begin{align*}
	\left\Vert \nabla^2_{x^{'}x^{'}}F_i \right\Vert & \leq \left\Vert \nabla^2_{xx}F_i \right\Vert + \left\Vert \nabla^2_{uu}F_i \right\Vert + \left\Vert \nabla^2_{vv}F_i \right\Vert \\
	& \leq \left(  2(1-q) \delta_{\left[b_i = 1 \right] } +  2q \delta_{\left[b_i = -1 \right] } \right)\left\Vert a_i a_i^\top \right\Vert + \frac{\lambda}{m} +  2(1-q) \delta_{\left[b_i = 1 \right] } +  2q \delta_{\left[b_i = -1 \right] } =: \tilde{L}^i_{xx}.
\end{align*} 
We have
\begin{align*}
	\left\Vert \nabla_y(\nabla_x F_i)\right\Vert & = \left\Vert 2\left( qa_i \delta_{\left[b_i = -1 \right] } - (1-q)a_i \delta_{\left[b_i = 1 \right] } \right) \right\Vert \\
	& \leq 2\left\vert q \delta_{\left[b_i = -1 \right] } - (1-q) \delta_{\left[b_i = 1 \right] } \right\vert \left\Vert a_i \right\Vert =: \tilde{L}^i_{xy} .
\end{align*}
We also have $\tilde{L}^i_{yy} = 2q(1-q)$. Then Lipschitz parameters of $f_{ij}$ are given by:
\begin{align}
	& L^{ij}_{xx} = \frac{n}{N} \sum_{l = 1}^{N_{ij}} \tilde{L}^l_{xx} + \frac{\lambda}{m} \\
	& L^{ij}_{yy} = \frac{n}{N} \sum_{l = 1}^{N_{ij}} \tilde{L}^l_{yy}, L^{ij}_{xy} = \frac{n}{N} \sum_{l = 1}^{N_{ij}} \tilde{L}^l_{xy} .
\end{align}
Using above parameters, we set $L_{xx} = \max_{i,j} \{ L^{ij}_{xx} \}, \ L_{yy} = \max_{i,j} \{L^{ij}_{yy}\}$ and $L_{xy} = L_{yx} = \max_{i,j} \{L^{ij}_{xy} \}$. Next, we move to estimate strong convexity parameter of $F_i$. For $b_i = 1$, $u^\top\nabla^2_{uu} F_i u \geq 2(1-q)$ and  for $b_i = -1$, $v^\top \nabla^2_{vv} F_i v \geq 2q$. Therefore, $(x^{'})^\top\nabla^2_{x^{'}x^{'}}F_i (x^{'}) \geq  \min\{2q,2(1-q)\}+\frac{\lambda}{m}$ for all $x^{'} \in \mathbb{R}^{d_x+ 2}$. Hence, we set strong convexity and strong concavity parameter of each $f_i$ respectively to $\mu_x = \min\{2q,2(1-q)\}\frac{\min_i(N_i)}{N}+\frac{\lambda}{m}$ and $\mu_y = 2q(1-q)\frac{\min_i(N_i)}{N}$.

\begin{figure*}[!htbp]
	\begin{minipage}{.98\textwidth}
		\centering
		\includegraphics[width=1\linewidth]{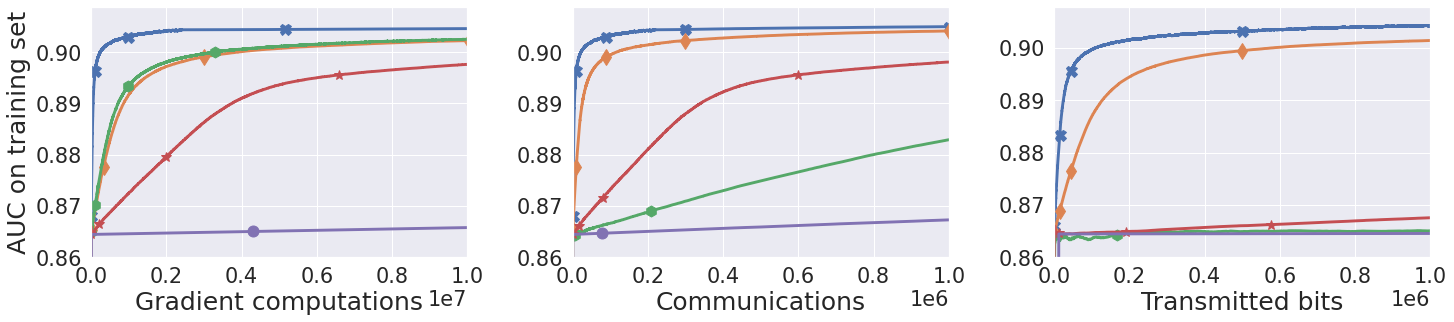}  
	\end{minipage}
	\begin{minipage}{.98\textwidth}
		\centering
		\includegraphics[width=1\linewidth]{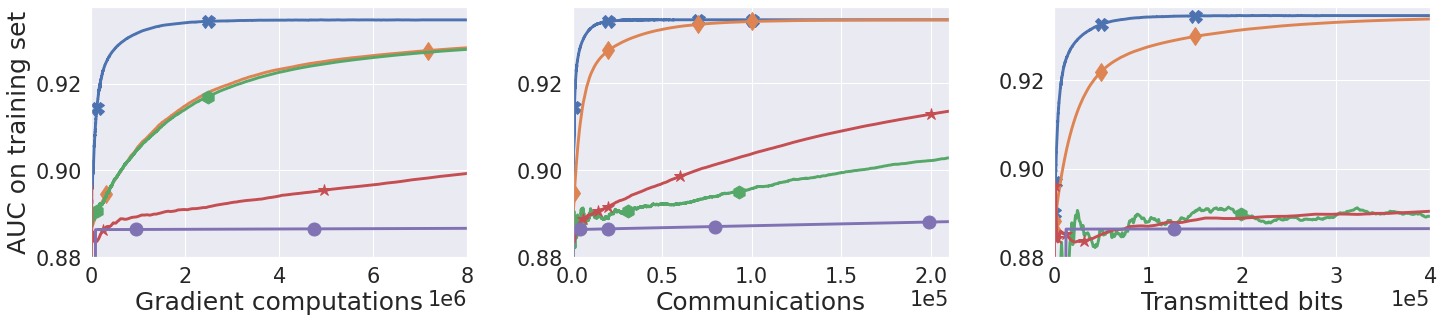}  
	\end{minipage}
	\begin{minipage}{.98\textwidth}
		\centering
		\includegraphics[width=1\linewidth]{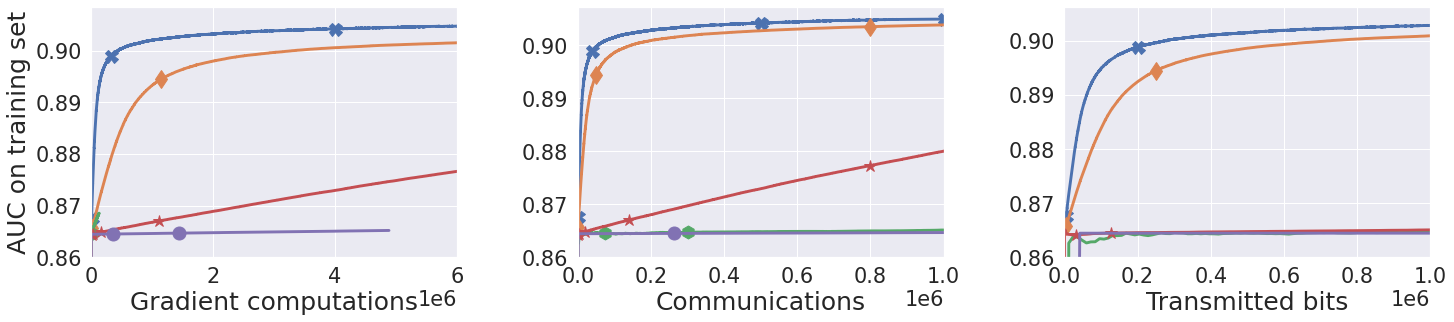}  
	\end{minipage}
	\begin{minipage}{.98\textwidth}
		\centering
		\includegraphics[width=1\linewidth]{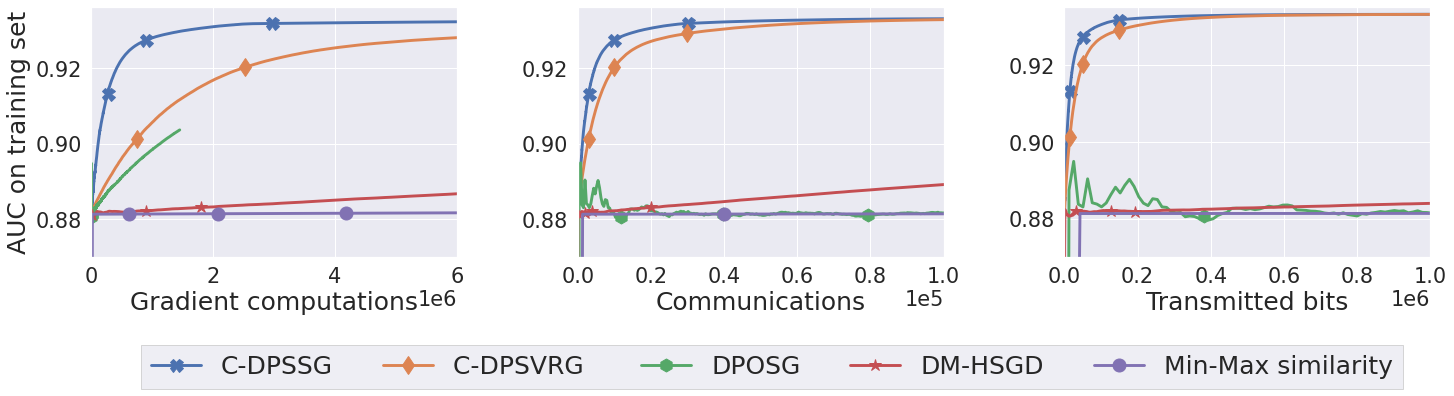}  
	\end{minipage}
	\caption{AUC value on training data vs. Gradient computations (Column 1), Communications
		(Column 2), Number of bits transmitted (Column 3) for different algorithms in 2d torus topology of 20 nodes (Rows 1,2) and 110 nodes(Rows 3,4).
		a4a in Rows 1,3 and ijcnn1 in Rows 2,4 respectively.} \label{fig:AUC_value_app}
\end{figure*}

\begin{figure*}[!htbp]
	\begin{minipage}{.98\textwidth}
		\centering
		\includegraphics[width=1\linewidth]{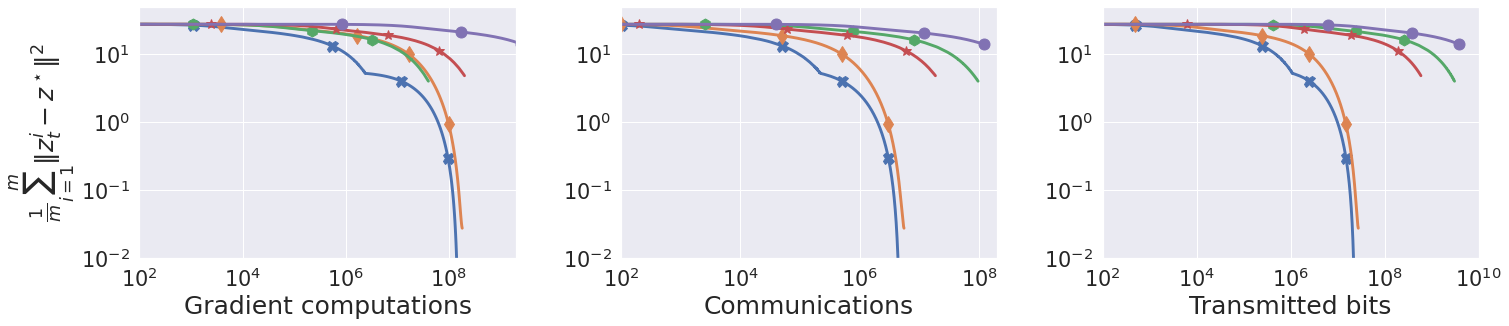}  
	\end{minipage}
	\begin{minipage}{.98\textwidth}
		\centering
		\includegraphics[width=1\linewidth]{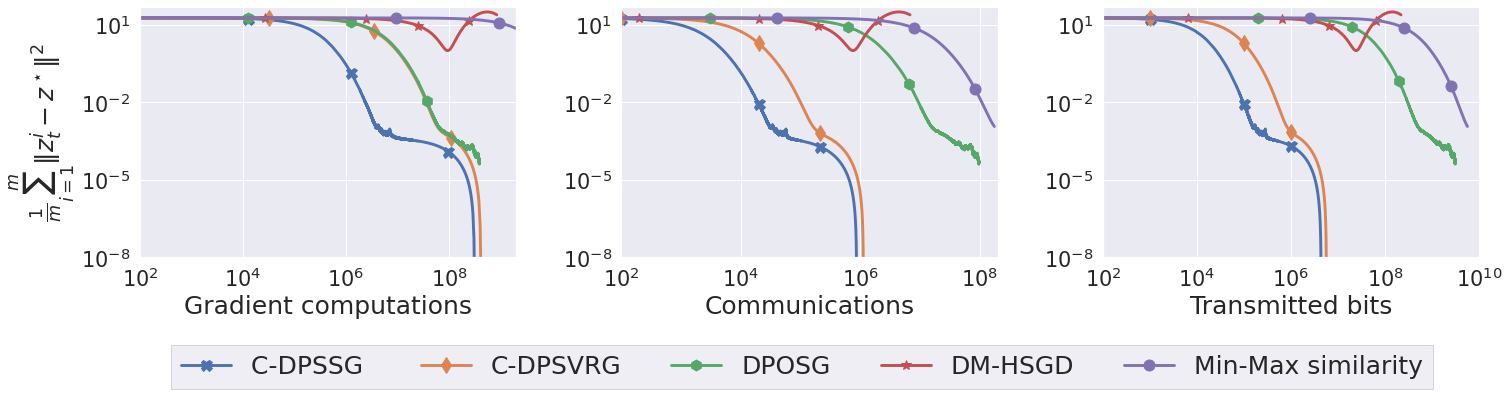}  
	\end{minipage}
	\begin{minipage}{.98\textwidth}
		\centering
		\includegraphics[width=1\linewidth]{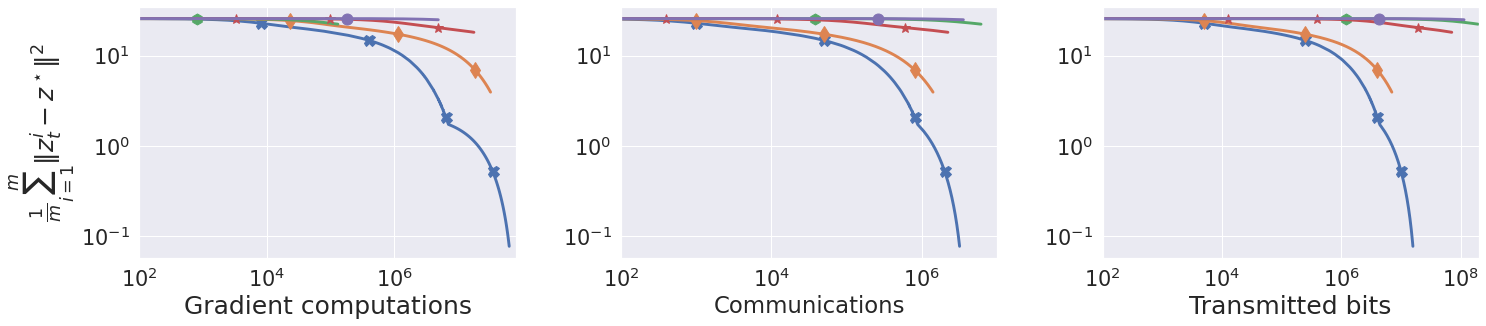}  
	\end{minipage}
	\begin{minipage}{.98\textwidth}
		\centering
		\includegraphics[width=1\linewidth]{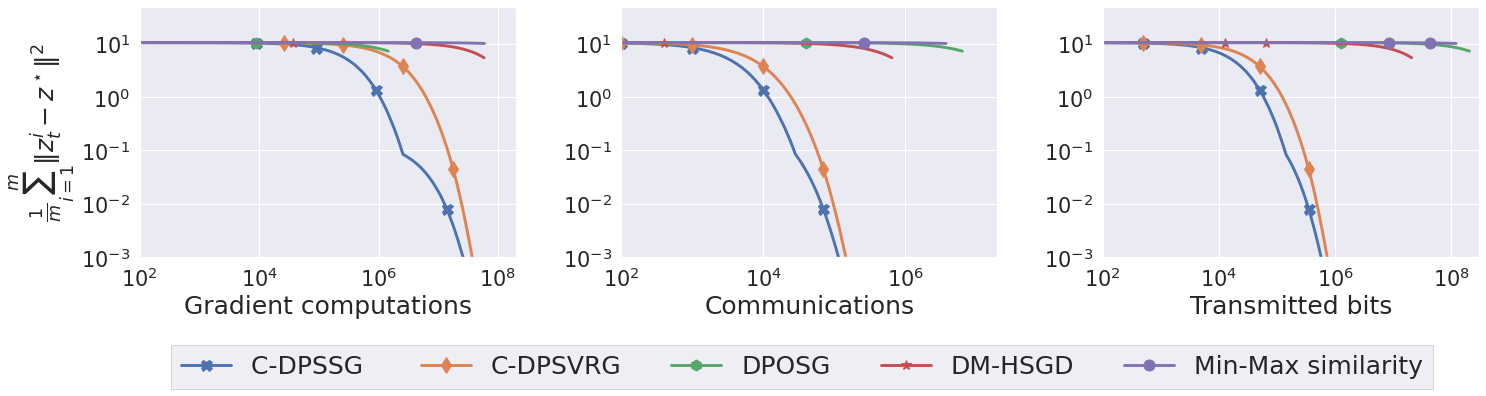}  
	\end{minipage}
	\caption{Convergence behavior of iterates to saddle point vs. Gradient computations (Column 1), Communications (Column 2), Number of bits transmitted (Column 3) for different algorithms in 2d torus topology of 20 nodes (Rows 1,2) and 110 nodes (Rows 3,4). a4a in Rows 1,3 and ijcnn in Rows 2,4.} \label{fig:auc_distance_from_saddle_app}
\end{figure*}
	
\end{document}